\documentclass[11pt]{article}
\usepackage[utf8]{inputenc} % allow utf-8 input
\usepackage[T1]{fontenc}    % use 8-bit T1 fonts
\usepackage{hyperref}       % hyperlinks
\usepackage{url}            % simple URL typesetting
\usepackage{authblk}
\usepackage{amsfonts} 
\usepackage{amsmath} 
\usepackage{xcolor}
\usepackage{geometry}
\usepackage{nicefrac}       % compact
\usepackage{microtype}      % microtypography
\usepackage{lipsum}
\usepackage{fancyhdr}       % header
\usepackage{graphicx}       % graphics
\graphicspath{{media/}}     % organize 
\usepackage{amsmath,amssymb}
\usepackage{natbib}
\usepackage{mathrsfs}
\usepackage{amsthm}
\usepackage[nameinlink]{cleveref}
\usepackage{algorithm}
\usepackage{algpseudocode}
\usepackage{subcaption}
\usepackage[title]{appendix}

\usepackage{enumitem}
\setlist{nolistsep,leftmargin=*}
\allowdisplaybreaks

\newtheorem{theorem}{Theorem}[section]

\newtheorem{lem}[theorem]{Lemma}
\newtheorem{defi}{Definition}
\newtheorem{prop}[theorem]{Proposition}
\newtheorem{rem}[theorem]{Remark}
\newtheorem{ass}{Assumption}
\definecolor{mydarkblue}{rgb}{0,0.08,0.45}

\hypersetup{ %
    pdftitle={},
    pdfauthor={},
    pdfsubject={},
    pdfkeywords={},
    pdfborder=0 0 0,
    pdfpagemode=UseNone,
    colorlinks=true,
    linkcolor=mydarkblue,
    citecolor=mydarkblue,
    filecolor=mydarkblue,
    urlcolor=mydarkblue,
    pdfview=FitH
}

\title{Towards a Unified Analysis of Neural Networks in Nonparametric Instrumental Variable Regression: Optimization and Generalization}

\author[1]{Zonghao Chen}
\author[2,3]{Atsushi Nitanda}
\author[4]{Arthur Gretton}
\author[5,6]{Taiji Suzuki}

\affil[1]{Department of Computer Science, University College London}
\affil[2]{Agency for Science, Technology and Research (A$\ast$STAR)}
\affil[3]{Nanyang Technological University}
\affil[4]{Gatsby Unit, University College London}
\affil[5]{The University of Tokyo}
\affil[6]{RIKEN Center for Advanced Intelligence Project}

\date{}

\newcommand{\E}{\mathbb{E}}
\newcommand{\R}{\mathbb{R}}

\newcommand{\calH}{\mathcal{H}}
\newcommand{\calN}{\mathcal{N}}
\newcommand{\calM}{\mathcal{M}}
\newcommand{\calE}{\mathcal{E}}

\newcommand{\calF}{\mathcal{F}}
\newcommand{\calG}{\mathcal{G}}

\newcommand{\calB}{\mathcal{B}}
\newcommand{\calY}{\mathcal{Y}}
\newcommand{\calO}{\mathcal{O}}
\newcommand{\calX}{\mathcal{X}}
\newcommand{\calP}{\mathcal{P}}
\newcommand{\calA}{\mathcal{A}}
\newcommand{\calW}{\mathcal{W}}
\newcommand{\dd}{\mathrm{d}}

\newcommand{\ba}{\mathbf{a}}
\newcommand{\bw}{\mathbf{w}}
\newcommand{\by}{\mathbf{y}}

\newcommand{\bH}{\mathbf{H}}
\newcommand{\bJ}{\mathbf{J}}
\newcommand{\bnabla}{\boldsymbol{\nabla}}

\newcommand{\scrF}{\mathscr{F}}
\newcommand{\scrX}{\mathscr{X}}
\newcommand{\scrx}{\mathscr{X}}

\newcommand{\scrZ}{\mathscr{Z}}
\newcommand{\scrz}{\mathscr{Z}}

\newcommand{\scrL}{\mathscr{L}}

\newcommand{\Id}{\mathrm{Id}}
\newcommand{\N}{\mathbb{N}}
\newcommand{\kl}{\mathrm{KL}}
\newcommand{\fisher}{\mathrm{FI}}
\newcommand{\tv}{\mathrm{TV}}
\newcommand{\LSI}{\mathrm{LSI}}

\begin{document}

% Generate the title
\maketitle

\begin{abstract}
    % Neural networks have recently been applied to nonparametric instrumental variable regression (NPIV), achieving strong empirical performance. However, a corresponding optimization theory has remained elusive in the literature. 
    % We establish the first global convergence result of two stage least squares (2SLS) with neural networks features in NPIV, by adopting a lifted perspective through mean-field Langevin dynamics (MFLD). 
    We establish the first global convergence result of neural networks for two stage least squares (2SLS) approach in nonparametric instrumental variable regression (NPIV). 
    This is achieved by adopting a lifted perspective through mean-field Langevin dynamics (MFLD), unlike standard MFLD, however, our setting of 2SLS entails a \emph{bilevel} optimization problem in the space of probability measures. To address this challenge, we leverage the penalty gradient approach recently developed for bilevel optimization which formulates bilevel optimization as a Lagrangian problem. This leads to a novel fully first-order algorithm, termed \texttt{F$^2$BMLD}.
    Apart from the convergence bound, we further provide a generalization bound, revealing an inherent trade-off in the choice of the Lagrange multiplier between optimization and statistical guarantees. Finally, we empirically validate the effectiveness of the proposed method on an offline reinforcement learning benchmark.
\end{abstract}
\section{Introduction}
Instrumental variable regression is a method of identifying and estimating the causal effect of the treatment $A$ on the outcome $Y$ based on observational data even in the presence of unobserved confounding~\citep{stock2003retrospectives}. 
This is achieved by leveraging a valid instrumental variable $W$ which only influences the outcome $Y$ via the treatment $A$, known as the \emph{exclusion restriction}; and is independent of the unobserved confounder $U$, known as \emph{unconfoundedness}. For instance, if one would like to identify the causal effect of smoking $A$ on the risk of lung disease $Y$ which may be potentially confounded by one's occupation and early childhood environment, the cigarette cost $W$ would be a valid instrument as it only affects the risk of lung disease $Y$ via smoking $A$~\citep{leigh2004instrumental}. 

\emph{Nonparametric instrumental variable regression} (NPIV) has gained popularity for its flexibility over parametric or semiparametric models when such structures are not warranted~\citep{newey2003instrumental, horowitz2011applied}. 
NPIV can be formulated as the following structural equation: 
\begin{align}\label{eq:npiv}
    Y = h_\circ(A) + U, \quad \E[U \mid W] = 0 ,
\end{align}
where $h_\circ$, referred to as the \emph{structural function}, is the primary object of interest. 
Similar formulation also arises in the context of nonparametric proximal causal learning~\citep{tchetgen2020introduction}, offline policy evaluation in reinforcement learning~\citep{chen2022instrumental} and more general ill-posed inverse problems~\citep{carrasco2007linear}. 
Denote the data generating distribution as $P$ over $(A, W, Y)$ with marginals $P_A, P_W, P_{AW}, P_{WY}$. 
Conditioning both sides of Eq.~\eqref{eq:npiv} on the instrument $W$, NPIV can be cast as the following integral equation:
\begin{align}\label{eq:inverse_problem}
\E[Y\mid W] = (Th_\circ)(W),
\end{align}
where $T : L^2(P_{A}) \to L^2(P_{W})$ is the conditional expectation operator defined by $T f = \E[f(A)\mid W]$.
Estimating $h_\circ$ thus amounts to inverting this operator in Eq.~\eqref{eq:inverse_problem}. In practice, the operator $T$ is unknown and must be estimated from data. Moreover,  even if $T$ were known, its inverse is typically unbounded, rendering recovery of $h_\circ$ an \emph{ill-posed inverse problem}.

A widely used approach for solving Eq.~\eqref{eq:npiv} is two-stage least squares (2SLS) regression. Originally developed for linear models in both stages, 2SLS has since been extended to nonlinear settings with fixed feature maps, such as sieve~\citep{blundell2007semi} and reproducing kernel Hilbert spaces~\citep{singh2019kernel}.
2SLS consists of two successive least-squares regressions, possibly using two distinct set of i.i.d. samples from $P$: in the first stage, one estimates the conditional expectation operator $T$ (or a suitable surrogate, such as the conditional mean embedding), and in the second stage, one estimates the structural function $h_\circ$ using the conditional feature map estimated in the first stage. Under a so-called measure of ill-posedness on $T$, both kernel 2SLS and sieve 2SLS are minimax optimal for Besov targets $h_\circ$~\citep{meunier2024nonparametric,chen2018optimal,chen2012estimation}. 

More recently, motivated by the expressive power of deep learning, \citet{xu2021learning} proposed deep feature instrumental variable (DFIV) regression, which replaces fixed feature maps in 2SLS with adaptive, data-driven neural network representations. 
DFIV has demonstrated superior empirical performances over conventional fixed feature 2SLS on several instrumental variable benchmarks. 
From a statistical perspective, \citet{kim2025optimality} has established that DFIV is minimax optimal for learning Besov $h_\circ$ under measure of ill-posedness condition on $T$, provided that the global minimizers of both stages can indeed be attained. 
However, this leaves the theory only \emph{half complete}. While the statistical guarantees are now well understood, the \emph{optimization} side remains unresolved. In contrast to fixed-feature 2SLS methods, whose solutions admit closed-form expressions in both stages, DFIV does not enjoy such tractability: identifying global minimizers is  an open and challenging problem. 
The challenges are mainly two-fold: 1) the inherent non-convexity with respect to neural network parameters and 2) the intrinsic \emph{bilevel optimization} structure of DFIV. 
Furthermore, the DFIV algorithm proposed by \citet{xu2021learning} requires solving a ridge regression problem with the learned neural network features, which incurs a potentially elevated cubic cost due to matrix inversion and requires a large batch size to have stable training dynamics. 
 
In this paper, we propose a new 2SLS algorithm, termed \emph{fully first-order bilevel mean-field Langevin dynamics} (\texttt{F$^2$BMLD}), for solving Eq.~\eqref{eq:npiv} with adaptive neural network features.
As the name suggests, our algorithm builds on mean-field Langevin dynamics (MFLD), whose finite-particle and time-discretized implementations correspond precisely to noisy gradient descent training of two-layer neural networks~\citep{mei2018mean,suzuki2023convergence,nitanda2022convex}. MFLD offers a \emph{lifted} perspective: noisy gradient descent can be interpreted as optimization in the space of probability measures $\calP$ with a \emph{convex} objective, which enables MFLD to achieve fast non-asymptotic global convergence guarantees~\citep{nitanda2025propagation}, thereby addressing the first challenge.

Under this lifted perspective, however, the second challenge becomes more severe: the bilevel optimization problem now needs to be solved in the space of probability measures. The two most common approaches to bilevel optimization—implicit gradient methods~\citep{ghadimi2018approximation} and explicit gradient methods~\citep{franceschi2017forward,maclaurin2015gradient}—become ineffective in this setting: the former requires second-order derivatives in $\calP$ that are difficult to implement in practice (see \Cref{sec:nested_gradient}), while the latter relies on automatic differentiation, which does not extend naturally to probability measures. 
To overcome this difficulty, we propose \texttt{F$^2$BMLD} which leverages the recent reformulation of bilevel optimization as a constrained optimization problem, and consequently as a Lagrangian problem~\citep{shen2023penalty,kwon2023fully}. The benefit of such a reformulation is that it only requires \emph{first-order} derivatives, which allow a tractable finite-particle implementation in practice. 

\vspace{2mm}
Our contributions in this paper can be summarized as follows: 
\begin{enumerate}[itemsep=5.0pt,topsep=0pt,leftmargin=*]
    \item We prove that the proposed algorithm \texttt{F$^2$BMLD} can indeed find the globally optimal weights when solving Eq.~\eqref{eq:npiv} via 2SLS with adaptive two-layer neural network features. Our proof thus completes the optimization side of the theoretical explanation for why adaptive features yield superior empirical performance over fixed-feature 2SLS. 
    The convergence results of \texttt{F$^2$BMLD} go beyond direct applications of existing MFLD results, due to 
    the \emph{nested} dependence between the solutions of two levels induced by the bilevel structure. 
    Such \emph{nested} dependence makes the upper-level objective smooth yet \emph{non-convex}—a regime that lies outside the standard convex setting where MFLD applies. 
    This non-convexity poses a significant challenge, yet we are  able to establish convergence by imposing an additional mild constraint on the regularization. 
    This completes the theoretical explanation for why neural network adaptive features outperform fixed-feature 2SLS. 
    \item We prove a generalization bound on the minimizer found by \texttt{F$^2$BMLD} when the loss objectives are estimated with finite i.i.d. samples from $P$ in both stages. By contrasting the generalization bound and the convergence bound, we identify a trade-off on the effect of the Lagrange multiplier $\lambda$. For the optimization bound, smaller values of $\lambda$ are preferable, as they yield a `less non-convex' objective and a smaller Lipschitz constant, thereby reducing the time-discretization error. 
    In contrast, for the generalization bound, larger values of $\lambda$ are favorable, since they make the Lagrangian problem more faithful to the original bilevel optimization problem. 
    \item We empirically evaluate \texttt{F$^2$BMLD} on offline policy evaluation, a challenging reinforcement learning task. The results show that \texttt{F$^2$BMLD} matches, and in some cases surpasses, the performance of DFIV, the current state-of-the-art 2SLS regression method for NPIV. Beyond accuracy, \texttt{F$^2$BMLD} offers two additional practical benefits: (i) it exhibits more stable training dynamics, since it relies solely on first-order derivatives, in contrast to DFIV which requires backpropagating through the ridge regression solution; and (ii) it can be trained with a smaller batch size than DFIV, thereby reducing the memory overhead for large-scale models.  
\end{enumerate}

\paragraph{Structure of the paper:} 
This paper is organized as follows. \Cref{sec:2sls} and \Cref{sec:mfld} provide background on two-stage least squares (2SLS) regression and mean-field Langevin dynamics (MFLD). In \Cref{sec:bi-mfld}, we introduce a lifted perspective on 2SLS regression with two-layer neural network features, establish the existence of an optimal solution to such a lifted problem, and present our fully first-order algorithm \texttt{F$^2$BMLD}. \Cref{sec:convergence} proves the \emph{convergence} of \texttt{F$^2$BMLD} to the global optimal solution, while \Cref{sec:generalization} analyzes its \emph{generalization} with access only to finite i.i.d. samples in both stages. \Cref{sec:experiments} reports empirical results on offline policy evaluation. Finally, \Cref{sec:conclusion} concludes the paper. Detailed proofs of the theorems are provided in \Cref{sec:proof}.

\subsection{Related work} 
\textbf{NPIV and 2SLS:} 
Two-stage least squares regression (2SLS) with fixed basis functions has been widely used for NPIV problems in economics, encompassing both classical sieve basis \citep{newey2003instrumental, blundell2007semi} and more recent reproducing kernel Hilbert space (RKHS) estimators with infinite dimensional basis \citep{singh2019kernel,meunier2024nonparametric}. These methods admit tractable closed-form solutions and enjoy well-understood statistical guarantees, including minimax optimal rates of convergence. Building on advances in deep learning, \citet{xu2021deep,xu2021learning} proposed deep feature instrumental variable (DFIV) regression which uses adaptive neural network features in 2SLS, demonstrating superior empirical performance over fixed-basis counterparts. Subsequently, \citet{kim2025optimality} established its minimax optimal statistical properties. However, a corresponding optimization theoretic understanding remains largely absent in the literature, especially given the non-convexity of the loss with respect to the neural network parameters. Neural networks have also been employed in alternative algorithms to solve NPIV \citep{hartford2017deep, dikkala2020minimax, bennett2019deep,wang2022spectral,sun2025spectral}, but these approaches also lack optimization theory. 

\vspace{1em}

\noindent
\textbf{NPIV in offline reinforcement learning:} The Bellman equation in offline reinforcement learning takes the same form as NPIV in Eq.~\eqref{eq:npiv}, a connection first noted by \citet{bradtke1996linear}. We elaborate on this correspondence in \Cref{sec:experiments}. With the recent adoption of neural networks in NPIV, several of these methods have also been adapted to offline reinforcement learning \citep{chen2022instrumental, chen2022well,xu2021deep}. In this setting, adaptive features such as neural networks are generally preferred over fixed basis, due to the more complex relationships among reinforcement learning variables compared to standard causal inference benchmarks. Furthermore, \citet{liao2024instrumental} and \citet{bennett2021off} analyzed offline reinforcement learning under unobserved confounding on the action, relying on additional instrumental variables for identification. Their resulting structural equations remain analogous to Eq.~\eqref{eq:npiv}.

\vspace{1em}
\noindent
\textbf{Mean field Langevin dynamics:}
Through the mean-field perspective on the two-layer neural networks, optimization dynamics can be lifted from the parameter space to the space of probability distributions~\citep{nitanda2017stochastic,chizat2018global,mei2018mean,sirignano2020mean,rotskoff2022trainability,chen2024regularized}. This reformulation enables global convergence guarantees for gradient descent under suitable conditions~\citep{chizat2018global,mei2018mean}. Mean-field Langevin dynamics (MFLD)~\citep{hu2021mean}, a noisy variant of gradient descent, also benefits from this viewpoint; the proximal Gibbs analysis and uniform log-Sobolev inequality (LSI) yield exponential convergence of MFLD under milder assumptions~\citep{nitanda2022convex,chizatmean}. Early analyses of MFLD were restricted to the mean-field limit, leaving the quantitative computational complexity of finite-particle system largely open, although \cite{nitanda2022convex} incorporated a time-discretization error via a one-step interpolation argument~\citep{vempala2019rapid}. \cite{chen2024uniform,suzuki2023convergence} established the fully time-and space-discretized guarantee, which proved the propagation of chaos~\citep{sznitman2006topics} to control the finite-particle approximation error. Fully exploiting the convexity of the objective, \cite{nitanda2024improved} further refined the particle approximation (space-discretization) analysis by eliminating the dependence on the LSI-constant, achieving improved quantitative complexity in combination with uniform-in-$N$ LSI~\citep{kook2024sampling,chewi2024uniform}. More recently, \cite{nitanda2025propagation} provided a direct analysis of MFLD that preserves this improved particle complexity while inheriting the convergence rate in time from the mean-field limit dynamics. 

Another important line of research is the extension of MFLD to min-max optimization problems under double-loop  schemes~\citep{wang2022exponentially,lu2023two,kimsymmetric,lascu2025entropic}. 
Our proposed \texttt{F$^2$BMLD} in \Cref{sec:bi-mfld} is also a double-loop method, but instead addresses a min–min (bilevel) optimization problem, as reviewed next.

%Mean-field Langevin dynamics (MFLD) provides a lifted perspective in the space of probability measures of noisy gradient descent for two-layer mean-field neural networks~\citep{nitanda2017stochastic,chizat2018global,mei2018mean,sirignano2020mean,rotskoff2022trainability}. In the span of a few years, MFLD has quickly progressed from non-quantitative guarantees~\citep{mei2018mean,hu2021mean}, to mean-field convergence rates~\citep{chizat2018global,nitanda2022convex} and fully time-discretized and space-discretized guarantees~\citep{suzuki2023convergence,chen2024uniform,kook2024sampling,nitanda2025propagation}. 
%More recently, MFLD has been extended to min–max optimization problems under adaptation to double-loop  schemes~\citep{lu2023two,kimsymmetric,lascu2025entropic}. 
%Our proposed \texttt{F$^2$BMLD} in \Cref{sec:bi-mfld} is also a double-loop method, but instead addresses a min–min (bilevel) optimization problem, as reviewed next.

\vspace{1em}
\noindent
\textbf{Bilevel optimization:} Bilevel optimization seeks to minimize an upper-level objective that depends implicitly on the solution of a lower-level problem~\citep{dempe2020bilevel}. A key difficulty is that the lower-level solution is defined only through its optimality conditions, so computing gradients with respect to the upper-level variable requires differentiating through these conditions. Two common approaches are: (i) explicit gradient methods, which treat the lower-level solution as the trajectory of a dynamical system and compute gradients via automatic differentiation~\citep{maclaurin2015gradient,franceschi2017forward,bolte2022automatic}; and (ii) implicit gradient methods, which employ the implicit function theorem to derive closed-form expressions for the associated gradients~\citep{pedregosa2016hyperparameter,ghadimi2018approximation,hong2023two,ji2021bilevel,xiao2023generalized,arbel2022amortized,petrulionyte2024functional}. See \citet{liu2021investigating} for a review. 

In this work, adopting a lifted mean-field perspective, we encounter bilevel optimization over the space of probability measures $\calP$. In this setting, the standard approaches above do not apply: automatic differentiation tools are not available over $\calP$, and higher-order derivatives are generally intractable (see \Cref{sec:nested_gradient}). An alternative is offered by penalty-based (or value-function) methods, which reformulate bilevel problems as single-level constrained optimization problems~\citep{ye1997exact,liu2022bome,kwon2023fully}. This reformulation is particularly appealing here as it enables optimization with \emph{only first-order} information.

Two closely related works are \citet{marion2025implicit} and \citet{geuter2025ddeqs}; the former studies optimization through a sampling process and the latter studies deep equilibrium models over distributional inputs. Both are formulated as bilevel optimization problems, yet, in which only the lower-level problems are defined over $\calP$.  
Also relevant are the works of  \citet{wang2024mean,barboni2025ultra}, which consider two layer neural network training as optimization in the space of \emph{signed measures}, which can in turn be recast as bilevel optimization over probability measures. Their lower-level problem admits a closed-form solution, however, which makes the problem substantially simpler.

\subsection{Notations:}
Let $A$ and $W$ be random variables on $\mathcal{A} \subseteq \mathbb{R}^{d_a}$ and $\mathcal{W} \subseteq \mathbb{R}^{d_w}$, respectively.
We use boldface $\ba \in \mathcal{A}$ and $\bw \in \mathcal{W}$ to denote their realizations, which also serve as inputs to the neural network. 
We use plain symbols $x, z$ for neural network parameters.
Ent denotes the negative entropy of a probability measure $\mu$ that admits a density function: $\mathrm{Ent}(\mu)=\int \log \mu(x) \mu(x) \dd x$. 
$\calN(v, \Sigma)$ denotes a Gaussian distribution with mean $v$ and covariance $\Sigma$. 
$\calM(\R^d)$ denotes the set of real-valued signed measures on $\R^{d}$ with finite total variation. $\calP(\R^d)$ denotes the set of probability measures on $\R^{d}$. 
$\calP_2(\R^d)$ denotes the set of probability measures on $\R^{d}$ with finite second moment. 
For any $\mu \in \calP_2(\R^d)$, $L^2(\mu)$ is the Hilbert space of (equivalence class of) functions $f : \R^d \to \R$ such that $\int |f|^2 d\mu < \infty$. 
The symbol $\Id$ denotes the identity, which by context may refer either to the finite dimensional identity matrix or the identity operator. 

The following divergences between two probability measures $\nu$ and $\mu$ will be used extensively in this paper. 1. KL denotes the Kullback-Leibler divergence $\kl(\nu, \mu) = \int \log(\frac{\dd \nu}{\dd \mu}) \dd \nu$ when $\mu$ is absolutely continuous with $\nu$ and $+\infty$ otherwise. 
2. $W_2(\nu, \mu)$ denotes the Wasserstein-2 distance between $\nu$ and $\mu$. 
3. $\fisher(\nu, \mu)$ denotes the Fisher divergence which is the squared $L^2(\nu)$ norm of the difference between the respective score functions. 
4. $\tv(\nu, \mu)$ denotes the total variation distance. 

\section{Two-stage Least Squares Regression (2SLS)}\label{sec:2sls}
In this section, we first review two-stage least squares (2SLS) with fixed features and then with adaptive neural network features. 

\paragraph{Fixed feature 2SLS regression}
Given two sets of fixed feature functions $\psi(\ba), \phi(\bw)$---e.g splines~\citep{blundell2007semi} and reproducing kernel feature maps~\citep{singh2019kernel}---2SLS performs two successive least squares regressions.  Stage I regression targets the \emph{conditional mean embedding} (CME) $\bw \mapsto \E[\psi(A)\mid W = \bw]$, which acts as a surrogate of the conditional expectation operator $T$. Specifically, for any $f$ in the linear span of $\psi$, i.e., $f(\ba) = u^\top \psi(\ba)$, we have $u^\top \E[\psi(A)\mid W=\bw] = \E[f(A)\mid W=\bw]=(Tf)(\bw)$. 
The CME is parameterized as a linear function of another set of features $\phi$, i.e., $\E[\psi(A)\mid W = \bw] = V\phi(\bw)$ with a Hilbert-Schmidt operator $V: \textrm{span}(\phi) \to \textrm{span}(\psi)$, which can be learned via the following vector-valued ridge regression: 
\begin{align}\label{eq:stage_one_intro}
    \hat{V}=\arg \min _V \frac{1}{2m} \sum_{i=1}^m \left[\|\psi(\ba_i) - V \phi(\bw_i)\|^2\right] + \zeta_1\|V\|_{\mathrm{HS}}^2 .
\end{align}
Here, $\zeta_1>0$ is the stage I regularization parameter, $\|\cdot\|_{\mathrm{HS}}$ denotes the Hilbert–Schmidt norm, and $\{\bw_i, \ba_i\}_{i=1}^m$ are $m$ i.i.d. samples from $P_{WA}$.  
Then, stage II regression targets the structural function $h_\circ$ via another ridge regression:
\begin{align}\label{eq:stage_two_intro}
    \hat{u}=\arg \min _u \frac{1}{2n} \sum_{i=1}^n \left[(\by_i - u^{\top} \hat{V} \phi(\bw_i))^2\right] + \zeta_2 \|u\|^2 .
\end{align}
Here, $\zeta_2>0$ is the stage II regularization parameter, $\|u\|$ denotes the $\ell_2$-norm or RKHS norm when appropriate and $\{\bw_i, \by_i\}_{i=1}^n$ are $n$ i.i.d. samples from $P_{WY}$.  
In contrast to standard non-parametric regression, 2SLS replaces the feature $\psi(A)$ with $\hat{V} \phi(W)$---an estimate of the conditional mean embedding from Stage I. This substitution arises from conditioning on the instrument $W$ as a means of adjusting for the unobserved confounder $U$ as in Eq.~\eqref{eq:inverse_problem}. 
The final estimator for the structural function $h_\circ$ is given by $\hat{h}(\ba) = \hat{u}^\top\psi(\ba)$.

\paragraph{Deep feature instrumental variable regression}
Rather than using fixed feature functions $\phi,\psi$, \emph{deep feature instrumental variable} (DFIV) regression proposes to use data-adaptive features $\phi,\psi$ learned by deep neural networks. 
Compared against 2SLS with fixed features, DFIV has achieved better empirical performance~\citep{xu2021learning}. 
The original DFIV parameterizes the features $\phi,\psi$ with deep neural networks and proposes to solve the linear coefficients $V$ and $u$ (in Eq.~\eqref{eq:stage_one_intro} and Eq.~\eqref{eq:stage_two_intro}) via closed-form ridge regression, such procedure incurs a cost cubic in the feature dimensions. 
Instead, following \citet{kim2025optimality}, we propose an alternative equivalent formulation which approximates $h_\circ$ directly with a single neural network $h_{\theta_a}(\cdot):\calA\to\R$ by minimizing the projected error $\E_{YW}[(Y-\E[h_{\theta_a}(A)\mid W])^2]$, where the conditional expectation $\E[h_{\theta_a}(A)\mid W]$ is learned in stage I via another regression and parameterized by another neural network $h_{\theta_w}(\cdot):\calW\to\R$. 

Specifically, 
\begin{align}\label{eq:weight_space_optimization_deep}
\begin{aligned}
    \text{Stage I:} \qquad\qquad \theta_w^\ast(\theta_a) &= \underset{\theta_w}{\arg \min } \quad \frac{1}{2m} \sum_{i=1}^m \left[ \left( h_{\theta_w}(\bw_i) - h_{\theta_a}(\ba_i) \right)^2 \right] , \\
    \text{Stage II:} \qquad\qquad 
    \theta_a^\ast &=  \underset{\theta_a}{\arg \min } \quad \frac{1}{2n} \sum_{i=1}^n \left[  \left(h_{\theta_w^\ast(\theta_a)}(\bw_i) - \by_i \right)^2 \right] .
\end{aligned}
\end{align}
The structural function $h_\circ$ is estimated by the neural network $h_{\theta_a^\ast}$. 
It is shown in \cite{kim2025optimality} that if both Stage I and Stage II optimization algorithms reach their respective \emph{global optima}, then the generalization error $\|h_{\theta_a^\ast} - h_\circ\|_{L^2(P_A)}$ achieves the minimax optimal rate, provided that the structural function $h_\circ$ lies in a Besov space and the size of the neural networks increases as the number of samples increases. 
Unfortunately, it remains a challenging and open problem whether this \emph{global optimum} can be actually achieved. 

To better illustrate the challenge of finding the global optimum of Eq.~\eqref{eq:weight_space_optimization_deep}, we compare it against the fixed–feature 2SLS approach described in Eq.~\eqref{eq:stage_one_intro}–Eq.~\eqref{eq:stage_two_intro}. 
The two stages of fixed–feature 2SLS are sequential but \emph{decoupled}. Stage I estimates the conditional mean embedding operator $V$ without reference to the Stage II parameter $u$. 
Once $\hat{V}$ is obtained, Stage II simply solves a standard ridge regression problem for $u$. 
% This separation makes the optimization tractable: each stage reduces to a convex problem with a unique solution. 
By contrast, the DFIV formulation in Eq.~\eqref{eq:weight_space_optimization_deep} intertwines the two stages. The Stage I problem depends on the Stage II parameter $\theta_a$, and consequently Stage II optimization would require differentiation through the mapping $\theta_a \mapsto \theta_w^\ast(\theta_a)$. 
As a result, optimization no longer decomposes into two convex subproblems but instead takes the form of a more challenging \emph{bilevel optimization}. 
In addition to the bilevel structure, another challenge arises from the non-convexity in terms of the neural network parameters. 

In this paper, to tackle the first challenge, we adopt the penalty gradient methods reformulating the bilevel optimization as a constrained optimization problem and then as a Lagrangian problem~\citep{shen2023penalty,kwon2023fully}; to tackle the second challenge, we follow the line of work on mean-field Langevin dynamics, which establishes global convergence of the training dynamics of two-layer neural networks~\citep{chizat2018global, hu2021mean, suzuki2023convergence, nitanda2025propagation}, which we review below.

\section{Mean Field Langevin Dynamics (MFLD)}\label{sec:mfld}
In this section, we briefly review the existing convergence results on gradient-based optimization of a two-layer neural network through the lens of \emph{mean field Langevin dynamics} (MFLD).  

Consider neural networks with a single hidden layer of size $N$: 
$h(\ba, \scrx) = \frac{1}{N} \sum_{i=1}^N \Psi(\ba, x^{(i)})$ where $\scrx=[x^{(1)},\ldots,x^{(N)}]\in(\R^{d_a})^N$ are the network parameters and $\ba$ is the network input. 
Here, $\Psi(\ba, x^{(i)})$ denotes a neural network with a single-neuron, such as $\Psi(\ba, x) = \mathfrak{w}_2 a(\mathfrak{w}_1^\top \ba + b)$ with $x = (\mathfrak{w}_1, \mathfrak{w}_2, b)$ and $a$ being an activation function. 
This representation offers a lifted perspective in which a two-layer neural network with fixed input $\ba$ is interpreted as a linear functional (i.e. expectation) on probability measures: $\mu \mapsto \E_{X\sim\mu}[\Psi(\ba,X)]$ where $\mu$ is the empirical distribution $\frac{1}{N} \sum_{i=1}^N \delta_{x^{(i)}}$. 
To emphasize the dependence on the network parameters $x$ rather than the input $\ba$, we adopt the notation $\Psi_\ba(x) := \Psi(\ba, x)$. 

Through lifting, the gradient-based optimization dynamics of the neural network parameters $\scrx=[x^{(1)},\ldots,x^{(N)}]$ has been translated to the optimization dynamics of the probability measure $\mu$ when quantized with $N$ particles. 
This connection has been pointed out
by \citet{nitanda2017stochastic,rotskoff2022trainability,mei2018mean,chizat2018global,sirignano2020mean}. 
When the size of the hidden layer tends to infinity and the empirical distribution weakly converges to a probability measure, $\frac{1}{N} \sum_{i=1}^N \delta_{x^{(i)}} \to \mu$ as $N\to\infty$, the resulting model is referred to as the \emph{mean-field} limit of the neural network.

A key advantage of this lifted view is that the risk objective with $\ell_2$-norm regularization,
\begin{align}\label{eq:F_objective}
    F(\mu) := \frac{1}{2} \E_{(\ba,\by)\sim\rho}\left[ (\E_{X\sim\mu}[\Psi(\ba,X)] - \by)^2\right] + \frac{\zeta}{2}\E_{X\sim\mu}[\|X\|^2], 
\end{align}
where $\rho$ denotes a joint distribution over observations $(\ba,\by)$, either empirical or population, becomes \emph{linear convex}\footnote{This is distinct from \emph{geodesic} convexity. Linear convexity as in Eq.~\eqref{eq:convex} means convexity along mixture curves: $\mu_\vartheta = \vartheta \mu+(1-\vartheta) \nu$. 
In contrast, geodesic convexity refers to convexity along Wasserstein geodesics, where the interpolation $\mu_\vartheta$ is obtained by optimal transport displacement.} in $\mu$. 
The $\ell_2$ regularization is crucial here to ensure the optimization dynamics would converge to a distribution that satisfies a Log-Sobolev inequality~\citep{bakry2013analysis}. 
Since the convergence analysis applies to any data distribution $\rho$, we do not distinguish between the empirical data and population data distributions here. 
Therefore, for any probability measures $\mu,\nu$, 
\begin{align}\label{eq:convex}
    F(\vartheta \mu+(1-\vartheta) \nu) \leq \vartheta F(\mu)+(1-\vartheta) F(\nu) ,\quad \forall \vartheta\in(0,1) . 
\end{align}
As a result of such convexity, the corresponding gradient flow of $F$ in the metric space $(\calP_2, W_2)$: the space of probability measures on $\R^d$ (with finite second moment) endowed with the Wasserstein-2 distance, has been proved to converge to its unique global minima~\citep{chizat2018global,rotskoff2022trainability,sirignano2020mean}. Such global convergence indicates that gradient based training of two-layer neural networks can indeed find its global optimum in the mean field limit ($N\to\infty$).

Recent advances have strengthened this picture by establishing fast non-asymptotic convergence rates even with finite $N$, albeit under additional Gaussian noise~\citep{hu2021mean,suzuki2023convergence,chizatmean,nitanda2025propagation}. 
Such dynamics are known as the \emph{mean field Langevin dynamics}: 
for $\sigma > 0$ and an initial distribution $\mu_0$, 
\begin{align}\label{eq:mfld}
    \dd x_t = -\bnabla F(\mu)(x_t) \dd t+\sqrt{2 \sigma} \mathrm{~d} W_t, \quad \mu_t=\operatorname{Law}(x_t) \tag{MFLD}.
\end{align}
Here, $W_t$ denotes the Brownian motion on $\R^d$ and $\bnabla F(\mu): \R^d\to\R^d$ denotes the \emph{Wasserstein gradient} of $F$ at $\mu$, which is an element in the tangent space of $\mu$ with respect to the Riemmanian geometry $(\calP_2, W_2)$ in the sense of the Otto's calculus~\citep{villani2008optimal}. 
Fortunately, for the set of functionals of the form in Eq.~\eqref{eq:F_objective} that we primarily focus on in this paper, its Wasserstein gradient equals the Euclidean gradient of the first variation of $F$ (defined in \Cref{defi:first_var}), i.e. $\bnabla F(\mu)= \nabla [\delta F(\mu)]$~\citep[Lemma 10.4.1]{ambrosio2008gradient}:
\begin{align*}
    \bnabla F(\mu): \R^d \to \R^d, \quad \bnabla F(\mu)(x) = \E_{(\ba,\by)\sim\rho} \big[(\E_{X\sim\mu}[\Psi_\ba(X)] - \by) \nabla \Psi_\ba(x)\big] + \zeta x.  
\end{align*}
\begin{defi}[First variation]\label{defi:first_var}
    The first variation $\delta G$ of a functional $G: \calP_2 \to \R$ at $\mu \in \calP_2$ is defined as a continuous functional $\calP_2 \times \R^d \to \R$ that satisfies $\lim_{\epsilon \to 0} \epsilon^{-1} G(\epsilon \nu+(1-\epsilon) \mu)=\int \delta G(\mu)(x) \dd (\nu-\mu)$ for any $\nu \in \calP_2(\R^d)$. 
\end{defi}
The above \eqref{eq:mfld} can also be interpreted as the gradient flow of the functional with \emph{entropy regularization}: $\scrF(\mu)=F(\mu)+\sigma \mathrm{Ent}(\mu)$ in the Wasserstein geometry, since the gradient of the entropy functional $\mathrm{Ent}(\mu)$ corresponds to a diffusion term in the Fokker–Planck equation, which yields an additive Brownian noise in the corresponding stochastic process~\citep[Theorem 5.4]{sarkka2019applied}. 
Since $\mathrm{Ent}(\mu)$ is strictly linear convex~\citep[Theorem 2.7.3]{cover1999elements}, one can immediately see that $\scrF$ is also a \emph{strictly linear convex} objective over $\calP_2(\R^{d_a})$. 

The ideal dynamics in \eqref{eq:mfld}, however, cannot be simulated in practice due to the continuous time dynamics and infinite number of samples (i.e., mean field limit of two-layer neural network). 
Therefore, one may consider the following implementable version of the MFLD with space- and time-discretization. 
For initial particles $\scrx_0 = [x_0^{(1)},\ldots,x_0^{(N)}]$ and $s\in\{0, \ldots, S\}$ for any $S\in\N^+$, 
\begin{align}\label{eq:practical_mfld}
    x_{s+1}^{(i)}= x_s^{(i)} - \gamma \bnabla F(\mu_{\scrx, s})(x_s^{(i)}) + \sqrt{2 \sigma \gamma} \; \xi_s^{(i)}, \quad \text{ where } \mu_{\scrx, s}=\frac{1}{N}\sum_{i=1}^N \delta_{x_s^{(i)}},
\end{align}
for $i=1, \ldots, N$. 
Here, $\{\xi_s^{(i)}\}_{i=1}^N$ are $N$ i.i.d. standard Gaussian random variables on $\R^d$ and $\gamma >0$ is the step size. Substituting the explicit form of the Wasserstein gradient $\bnabla F$ into Eq.~\eqref{eq:practical_mfld}, one sees that the resulting dynamics \emph{coincide with} the training dynamics of two-layer neural networks $h(\ba, \scrx) = \frac{1}{N}\sum_{i=1}^N \Psi(\ba, x^{(i)})$ under Euclidean gradient descent with additional Gaussian noise. 
Denote the unique global minimizer 
$\mu^\ast=\arg\min_{\mu\in\calP_2(\R^d)} \scrF(\mu)$. 
Therefore, a natural question that arises in the field of MFLD would be: \emph{What is the convergence rate of $\mu_{\scrx, s}$ to $\mu^\ast$ in terms of particle number $N$ and iteration number $S$?}

Over the years, \citet{hu2021mean,suzuki2023convergence,nitanda2024improved,nitanda2025propagation} have presented an increasingly well-refined theoretical analysis of the above question, under mild regularity conditions that $\Psi$ is smooth and bounded, which is satisfied by smooth activations like $\mathrm{tanh}$, sigmoid plus a smooth clipping on the neural network output~\citep{suzuki2023convergence,hu2021mean}). 
Among them, the state-of-the-art convergence results have been recently proved by \citet{nitanda2025propagation} which enjoys the mildest dependence on the number of particles $N$, the number of iterations $S$ and the dimension of the input $d$. 
Here, we briefly review this result with the introduction of the following definitions. 

\begin{defi}[Logarithmic Sobolev inequality~{\citep[Definition 5.1.1]{bakry2013analysis}}]
For $\mu\in \calP(\R^d)$, we say $\mu$ satisfies the \emph{logarithmic Sobolev inequality} (LSI) with constant $C>0$ if for any locally Lipschitz function $g: \R^d \to \R$ with $\E_\mu[g^2]<\infty$, we have
$\E_\mu[g^2 \log (g^2)]-\E_\mu[g^2] \log (\E_\mu[g^2]) \leq 2 C^{-1} \E_\mu[\|\nabla g\|_2^2]$. 
\end{defi}
It is proved in Lemma 5 of \citet{suzuki2023convergence} and \citet{chewi2024uniform} that the optimum $\mu^\ast$ satisfies a LSI inequality with a LSI constant $C_\LSI=\Theta(\sigma^{-1} \exp(-\zeta^{-1} \sigma^{-1} \sqrt{d}))$. This constant deteriorates exponentially as the dimension $d$ increases, as $\sigma \to 0$ and as $\zeta\to0$. 
\begin{rem}
    Two direct consequences of a probability distribution $\mu$ satisfying LSI are that: for any probability distribution $\nu$,  $\kl(\nu, \mu) \leq (2 C_\LSI)^{-1} \E_{x\sim\nu}[\|\nabla \log (\frac{\dd \nu }{\dd \mu}(x))\|^2] = (2 C_\LSI)^{-1} \fisher(\nu,\mu)$ and $W_2^2(\nu,\mu)\leq 2 C_\LSI^{-1} \kl(\nu,\mu)$.  
\end{rem}
\begin{defi}[Bregman divergence]
For $\mu, \mu'\in \calP(\R^d)$, the Bregman divergence of a functional $F:\calP(\R^d)\to\R$ is defined as
$B_F(\mu, \mu^{\prime}) = F(\mu)-F(\mu^{\prime}) - \smallint \delta F(\mu^{\prime}) \; \dd (\mu-\mu^{\prime})$. 
\end{defi}
\begin{rem}
    If $F_0$ is a linear in terms of $\mu$, e.g. $F_0(\mu)=\E_{\mu}[f]$, its Bregman divergence $B_{F_0}\equiv 0$. Thus, the Bregman divergence quantifies the deviation of $F$ from its linear (first-order) approximation at $\mu'$. 
    For a linear convex functional $F$, it is immediate that $B_F(\mu, \mu^{\prime})\geq 0$ for all $\mu, \mu'$. 
\end{rem}
\noindent
To present the convergence result of $\mu_{\scrx, s}$ to $\mu^\ast$, we need to 
define the following auxiliary objective $\scrF^{(N)}: \calP_2((\R^d)^N)\to\R$ and its corresponding global minimum $\mu_\ast^{(N)}$: 
\begin{align}
    \scrF^{(N)}(\mu^{(N)}) = N \E_{\scrx \sim \mu^{(N)}}[F(\mu_{\scrx})] + \sigma \mathrm{Ent}(\mu^{(N)}) , \quad \mu_\ast^{(N)} = \arg\min_{\mu^{(N)}\in\calP_2((\R^d)^N)} \scrF^{(N)}(\mu^{(N)}) . 
\end{align}
One can easily verify that if $\mu^{(N)}=\mu^{\otimes N}$ is a $N$-fold product measure of $\mu$, then $\mathscr{F}^{(N)}(\mu^{(N)}) \geq N \mathscr{F}(\mu)$ by the linear convexity of $F$. 
It is proved in Lemma 1 of \citet{nitanda2025propagation} that for more general $\mu^{(N)}\in\calP_2((\R^d)^N)$, 
\begin{align}
    N^{-1} \scrF^{(N)}(\mu^{(N)})-\scrF(\mu_*) &= N^{-1} \sigma \kl(\mu^{(N)}, \mu_*^{\otimes N}) + \E_{\scrx \sim \mu^{(N)}} [B_{F}(\mu_{\scrx}, \mu_*)] \label{eq:lem_1_nitanda_1} \\
    N^{-1} \scrF^{(N)}(\mu^{(N)})-\scrF(\mu_*) &\leq \calO(N^{-1}) + (2 C_\LSI N)^{-1}\sigma \fisher(\mu^{(N)}, \mu_\ast^{(N)}) \label{eq:lem_1_nitanda_2} .
\end{align}
The first equality Eq.~\eqref{eq:lem_1_nitanda_1} indicates that $N^{-1} \scrF^{(N)}(\mu^{(N)})-\scrF(\mu_*)$ is a viable upper bound on $N^{-1} \sigma \kl(\mu^{(N)}, \mu_*^{\otimes N})$, thanks to the non-negativity of the Bregman divergence. 
Denote as $\mu_s^{(N)}$ the joint distribution of the $N$ particles $\scrx_s = [x_s^{(1)},\ldots,x_s^{(N)}]$ of Eq.~\eqref{eq:practical_mfld} at iteration $s\in\N^+$. To analyze the convergence of the empirical law $\mu_{\scrx, s}$ to $\mu^\ast$, it therefore suffices to study the decay of $N^{-1} \scrF^{(N)}(\mu^{(N)})-\scrF(\mu_*)$ to $0$. 
The second inequality Eq.~\eqref{eq:lem_1_nitanda_2} is referred to as a \emph{defective uniform logarithmic Sobolev inequality} in the MFLD literature. 
Conceptually, Eq.~\eqref{eq:lem_1_nitanda_2} plays the role of a Polyak–Łojasiewicz (PL) inequality: the Fisher divergence $\fisher(\mu^{(N)}, \mu_\ast^{(N)})$ measures the $L^2(\mu^{(N)})$ norm of the update direction and upper bounds the distance between the current iterate and the global minimum. 
The terminology “LSI” arises because, chaining Eq.~\eqref{eq:lem_1_nitanda_1} and Eq.~\eqref{eq:lem_1_nitanda_2}, the Fisher divergence upper bounds KL divergence up to $\calO(N^{-1})$ and Bregman divergence. 

\begin{rem}
    The above two inequalities  Eq.~\eqref{eq:lem_1_nitanda_1} and Eq.~\eqref{eq:lem_1_nitanda_2} are the key ingredients for establishing convergence of MFLD, which will later be proved in the context of our bilevel optimization algorithm \texttt{F$^2$BMLD}: \Cref{prop:KL_upper_bound_functional} corresponds to Eq.~\eqref{eq:lem_1_nitanda_1} and \Cref{prop:uniform_LSI} corresponds to Eq.~\eqref{eq:lem_1_nitanda_2}.
\end{rem}

With the above two key inequalities  Eq.~\eqref{eq:lem_1_nitanda_1} and Eq.~\eqref{eq:lem_1_nitanda_2}, we are now ready to present the convergence~\citep[Theorem 1]{nitanda2025propagation}. 
For any number of iterations $S\in\N^+$,  
\begin{align}\label{eq:mfld_convergence}
    N^{-1} \E\left[\scrF^{(N)}(\mu_S^{(N)})\right] -\scrF(\mu_*) \leq \calO\left(\frac{1}{N}\right) + \calO\left(\frac{\gamma^2 + \gamma \sigma d}{C_\LSI \sigma} \right) + \exp(-\gamma C_\LSI \sigma S) \Delta_0^{(N)} . 
\end{align}
Here, the expectation is taken with respect to the randonmess in the initial particles $[x_0^{(1)},\ldots, x_0^{(N)}]\sim \mu_0^{(N)}$ and the Gaussian noise at each iteration. 
$\calO(N^{-1})$ represents the particle approximation error and $\calO(\frac{\gamma^2 + \gamma \sigma d}{C_\LSI \sigma})$ represents the time-discretization error. 
The term $\Delta_0^{(N)} := N^{-1}\E[\scrF^{(N)}(\mu_0^{(N)})] - \scrF(\mu_\ast)$ denotes the initial error, which decays exponentially fast in terms of $S$ as a consequence of the PL (log-Sobolev) inequality. 

The above upper bound on $N^{-1} \E[\scrF^{(N)}(\mu_S^{(N)})] -\scrF(\mu_*)$ can be translated to upper bound on KL divergence through Eq.~\eqref{eq:lem_1_nitanda_1}: for any number of iterations $S\in\N^+$,  
\begin{align}\label{eq:kl_w2_F}
    N^{-1} \sigma \E[\kl(\mu_S^{(N)}, \mu_\ast^{\otimes N})] \leq N^{-1} \E[\scrF^{(N)}(\mu_S^{(N)})] - \scrF(\mu_*) . 
\end{align}
The above equation suggests a phenomenon known as the \emph{propagation of chaos} that the particles become asymptotically independent as both $N, S$ tend to infinity~\citep{sznitman2006topics}. The above upper bound on KL divergence implies convergence of neural network output~\citep[Proposition 1]{nitanda2025propagation}. Define $\hat{h}_S(\ba) = \frac{1}{N} \sum_{i=1}^{N} \Psi_{\ba}(x_S^{(i)})$ the output of a trained neural network where the particles $\{x_S^{(i)}\}_{i=1}^{N}$ follow a joint distribution $\mu_S^{(N)}$, and define $h_\ast(\ba)=\int \Psi_{\ba}(x) \dd \mu_\ast(x)$ the output of the optimal mean-field neural network. For any $\ba\in\calA$, 
\begin{align*}
    \mathbb{E} \left[\left( \hat{h}_S(\ba) - h_\ast(\ba) \right)^2\right] \leq \sqrt{\frac{1}{N} \kl(\mu_S^{(N)}, \mu_\ast^{\otimes N})} + \calO(N^{-1}) . 
\end{align*}
% In the sequel, we are going to apply the convergence results of MFLD to analyze the dynamics of neural networks training in Eq.~\eqref{eq:weight_space_optimization_deep}. 

% For any $\mu$, consider its associated proximal Gibbs distribution $\underline{\mu}\propto \exp (-\sigma^{-1} \delta F(\mu)(x))$ which is the stationary distribution of \eqref{eq:mfld} when the Wasserstein gradient $\bnabla F$ is fixed at $\mu$. The optimum $\mu^\ast$ satisfies a self-consistency condition $\mu^\ast=\underline{\mu^\ast}$. 
% As a direct consequence, 

% It is proved by \citet{suzuki2023convergence} using the Miclo’s trick~\citep{bardet2018functional} that the LSI constant is lower bounded by $C_\LSI $. 

\section{Bilevel Mean Field Langevin Dynamics}\label{sec:bi-mfld}
In this section, motivated by mean field Langevin dynamics, we first present in \Cref{sec:lifted_dfiv} a lifted perspective of the bilevel optimization problem in Eq.~\eqref{eq:weight_space_optimization_deep}. Through this lifted perspective, our aim is to establish convergence to the \emph{global optimum}, thereby completing the optimization theory of neural networks in 2SLS for NPIV regression, as emphasized in the introduction. 

Solving this lifted problem is challenging, however, as it amounts to a bilevel optimization in the space of probability measures.
The two standard approaches in bilevel optimization—implicit gradient methods~\citep{ghadimi2018approximation} and explicit gradient methods~\citep{franceschi2017forward,maclaurin2015gradient}—are ineffective in this setting: the former requires second-order derivatives on $\calP$, which are difficult to compute in practice (see \Cref{sec:nested_gradient}), while the latter relies on automatic differentiation, which does not extend naturally to probability measures.
To overcome this difficulty, we take inspiration from recent advances in bilevel optimization~\citep{shen2023penalty,kwon2023fully}, which reformulate the stage-I problem as a Lagrangian penalty embedded in stage II.
This reformulation requires only first-order gradients that can naturally extend to Wasserstein gradients in the space of probability measures. 
Building on this idea, we introduce \texttt{F$^2$BMLD} in \Cref{sec:penalty_gradient}, a \emph{fully first-order} algorithm for solving Eq.~\eqref{eq:weight_space_optimization_deep}, thereby avoiding higher-order derivatives.
Its convergence will be established in \Cref{sec:convergence}.

\subsection{Mean field formulation of DFIV}\label{sec:lifted_dfiv}
Consider two-layer neural networks with a single hidden layer: $h(\ba, \scrx) = \frac{1}{N_x} \sum_{i=1}^{N_x} \Psi_\ba(x^{(i)})$ (resp. $h(\bw, \scrz) = \frac{1}{N_z} \sum_{i=1}^{N_z} \Psi_\bw(z^{(i)})$) where $\scrx=[x^{(1)},\ldots,x^{(N_x)}]\in(\R^{d_x})^{N_x}$ (resp. $\scrz=[z^{(1)},\ldots,z^{(N_z)}]\in(\R^{d_z})^{N_z}$) are the network parameters and $\ba$ (resp. $\bw$) is the network input. 
% Here, $\Psi(\ba, x^{(i)})$ denotes a neural network with a single-neuron, such as $\Psi(\ba, x) = a(w_1^\top \ba + b)$ or $\Psi(\ba, x) = a(w_2) a(w_1^\top \ba + b)$ with $x = (w_1, w_2, b)$ and $a$ being an activation function. 
Therefore, the bilevel optimization problem in Eq.~\eqref{eq:weight_space_optimization_deep} can be re-written as the following: 
\begin{align}\label{eq:weight_space_optimization}
\begin{aligned}
    \text{Stage I:} \qquad\qquad \scrz^\ast(\scrx) &= \underset{\scrz \in (\R^{d_z})^{N_x} }{\arg \min } \quad \frac{1}{2} \E_{\rho} \left[ \left( h(\bw, \scrz) - h(\ba, \scrx) \right)^2 \right] , \\
    \text{Stage II:} \qquad\qquad\qquad 
    \scrx^\ast &=  \underset{\scrx \in (\R^{d_x})^{N_z} }{\arg \min } \quad \frac{1}{2} \E_{\rho} \left[ \left(h(\bw, \scrz^\ast(\scrx)) - \by \right)^2 \right] .
\end{aligned}
\end{align}
We use $\mathbb{E}_\rho$ to denote expectation with respect to a generic joint data distribution $\rho$ over $(\ba,\by,\bw)\in\calA\times\calY\times\calW$. 
In this section, we do not distinguish between the population and the empirical distribution, since this distinction is irrelevant for analyzing convergence of the optimization dynamics. 
In contrast, when studying generalization of the learned network in \Cref{sec:generalization}, we will explicitly take $\rho$ to be the empirical distribution consisting of finite i.i.d. samples from the data generating distribution $P$. 

Inspired by MFLD, we adopt a \emph{lifted perspective} of both neural networks $\int \Psi_\ba(x) \dd \mu_x(x)$ and $\int \Psi_\bw(z) \dd \mu_z(z)$ where $\mu_x, \mu_z$ are the mean-field limit of the hidden layer. Under $\ell_2$ and entropic regularizations, we obtain the following bilevel optimization problem over $\calP_2(\R^{d_x})$ and $\calP_2(\R^{d_z})$. 
\begin{align}
    &\text{Stage I:} \quad \mu_{z}^\ast(\mu_x) = \underset{\mu_z \in \calP_2 (\R^{d_z})}{\arg \min } \frac{1}{2} \E_{\rho} [ ( \smallint\!_{\R^{d_z}} \Psi_\bw \dd \mu_z - \smallint\!_{\R^{d_x}} \Psi_\ba \dd \mu_x )^2] + \frac{\zeta_1}{2} \E_{\mu_z}[\|z\|^2] + \sigma_1 \mathrm{Ent}(\mu_z) , \nonumber \\
    &\text{Stage II:} \quad 
    \mu_{x}^\ast = \underset{\mu_x \in \calP_2(\R^{d_x})}{\arg \min } \frac{1}{2} \E_{\rho}[(\smallint\!_{\R^{d_z}} \Psi_\bw \dd \mu_z^\ast(\mu_x) - \by)^2] + \frac{\zeta_2}{2}  \E_{\mu_x}[\|x\|^2] + \sigma_2 \mathrm{Ent}(\mu_x) \label{eq:distribution_space_optimization} \tag{Bi-MFLD} .
\end{align}
Here, $\sigma_1, \zeta_1,\sigma_2,\zeta_2>0$ are levels of $\ell_2$ and entropic regularization in Stage I and Stage II, respectively. 
The stage I solution approximates the conditional expectation $\int \Psi_\bw(z) \dd \mu_z^\ast(\mu_x) \approx \E[\int \Psi_A(x) \dd \mu_x \mid W = \bw]$; and the stage II solution approximates the structural function $\int \Psi_\ba(x) \dd \mu_x^\ast \approx h_\circ(\ba)$. 

To help with the analysis, we denote the following objectives 
\begin{align*}
    U_1(\mu_x, \mu_z) &= \frac{1}{2} \E_{\rho} [ ( \smallint \Psi_\bw(z) \; \dd \mu_z(z) - \smallint \Psi_\ba(x) \; \dd \mu_x(x) )^2], \\
    \quad 
    U_2(\mu_z) &= \frac{1}{2} \E_{\rho}[(\smallint \Psi_\bw(z) \; \dd \mu_z(z) - \by)^2], 
\end{align*}
which are mean squared error of both stages without any regularization. 
We also denote 
\begin{align*}
    F_1(\mu_x, \mu_z) = U_1(\mu_x, \mu_z) + \frac{\zeta_1}{2} \E_{\mu_z}[\|z\|^2] , \quad \scrF_1(\mu_x, \mu_z) = F_1(\mu_x, \mu_z) + \sigma_1 \mathrm{Ent}(\mu_z) \\
    F_2(\mu_x,\mu_z) = U_2(\mu_z) + \frac{\zeta_2}{2} \E_{\mu_x}[\|x\|^2], \quad \scrF_2(\mu_x,\mu_z) = F_2(\mu_x,\mu_z) + \sigma_2 \mathrm{Ent}(\mu_x),
\end{align*}
which are objectives of both stages with $\ell_2$ and entropic regularization. 
Following the terminology of bilevel optimization, we sometimes refer to stage I as `inner-loop' optimization and stage II as `outer-loop' optimization.
For the entropy to be finite, both stage I and stage II solutions must be absolutely continuous with respect to the Lebesgue measure. Hence, we sometimes abuse $\mu(x)$ to denote both the probability measure and its density. 

Throughout the following sections, we make the following assumptions.   

\begin{ass}[Bounded target]\label{ass:npiv}
    There exists a universal constant $M$ such that the target random variable $|Y| \leq M$ and $|h_\circ(X)| \leq M$ almost surely. 
\end{ass}
\Cref{ass:npiv} can be relaxed. When  $\rho$ is the true data generating distribution $P$, it suffices to only assume $|h_\circ(\ba)| \leq M$ for any $\ba\in\calA$. When $\rho$ is the empirical data distribution consisting of $n$ i.i.d. samples from $P$, it suffices to assume bounded $h_\circ$ and sub-Gaussian residual $Y - (Th_\circ)(Z)$ so that $\max_{i\in \{1, \ldots, n\}} |\by_i|$ is $\calO(\log n)$ with high probability~\citep[Exercise 2.5.10]{vershynin2018high}. 
This would result in an extra logarithmic factor in the final bound. 

\begin{ass}[Bounded and smooth neural networks]\label{ass:network}
    There exists a universal positive constant $R$ such that $\sup_{x\in \mathbb{R}^{d_x}, \ba \in \calA}|\Psi_\ba(x)| \leq R$ and $\sup_{z \in \mathbb{R}^{d_z}, \bw \in \calW}|\Psi_\bw(z)| \leq R$. Also, $\sup_{x\in \mathbb{R}^{d_x}, \ba \in \calA}|\nabla_x \Psi_\ba(x)| \leq R$ and $\sup_{z \in \mathbb{R}^{d_z}, \bw \in \calW}|\nabla_z \Psi_\bw(z)| \leq R$. 
\end{ass}

\Cref{ass:network} is standard in the literature of mean field Langevin dynamics (e.g. \citet[Assumption 3.2]{hu2021mean}, \citet[Assumption 2]{suzuki2023convergence}, \citet{nitanda2025propagation}). It is satisfied for instance by neural networks of the form $\Psi_\ba(x) = \mathfrak{w}_2 a(\mathfrak{w}_1^{\top} \ba)$ for $x=(\mathfrak{w}_1, \mathfrak{w}_2)$ with a smooth clipping and with smooth activation function $a$ such as $\tanh$, sigmoid.

We begin by establishing that the solution to \eqref{eq:distribution_space_optimization} exists and is well-defined. To this end, it is necessary to verify several key properties of the problem: specifically, the partial convexity of $F_1$ and $F_2$, as well as the continuity of the mapping $\mu_x \mapsto \mu_z^\ast(\mu_x)$. 

\begin{prop}[Partial convexity of $F_1$ and $F_2$]\label{prop:partial_convex_simple}
    For any fixed $\mu_x\in\calP_2(\R^{d_x})$, the mappings $\mu_z\mapsto F_1(\mu_x, \mu_z)$, $\mu_z\mapsto F_2(\mu_x, \mu_z)$ are \emph{linear} convex. For any fixed $\mu_z\in\calP_2(\R^{d_z})$, the mappings $\mu_x\mapsto F_1(\mu_x, \mu_z)$ the mapping $\mu_x\mapsto F_2(\mu_x, \mu_z)$ are also \emph{linear} convex. 
\end{prop}
\begin{proof}
    The proof is trivial since both $U_1,U_2$ are composition of a linear mapping $\mu \mapsto \int \Psi \dd \mu$ and a quadratic cost function, and since $\ell_2$ regularizations $\E_{\mu_x}[\|x\|^2], \E_{\mu_z}[\|z\|^2]$ are linear functionals. 
\end{proof}

\begin{prop}[Continuity of mapping $\mu_x \mapsto \mu_z^\ast(\mu_x)$]
Suppose \Cref{ass:network} holds. 
Let $\mu_x, \mu_x^\prime\in \calP_2(\R^{d_x})$. Let $\mu_z^\ast(\mu_x)$ and $\mu_z^\ast(\mu_x^\prime)$ be the solution to the Stage I optimization problem in \eqref{eq:distribution_space_optimization}. Then, we have $\kl(\mu_z^\ast(\mu_x), \; \mu_z^\ast(\mu_x^\prime)) + \kl(\mu_z^\ast(\mu_x^\prime), \; \mu_z^\ast(\mu_x)) \leq \frac{R^2}{8\sigma_1} \kl(\mu_x, \mu_x^\prime)$.
\end{prop}
\begin{proof}
    This proposition follows as a special case of the more general result in \Cref{prop:continuity}. 
\end{proof}
\begin{prop}[Existence of solutions in \eqref{eq:distribution_space_optimization}] \label{prop:existence}
    Suppose \Cref{ass:npiv} and \ref{ass:network} hold. For any $\mu_x\in\calP_2(\R^{d_x})$, the solution $\mu_z^\ast(\mu_x)$ to the Stage I optimization problem exists, is unique, is absolutely continuous with respect to the Lebesgue measure, and belongs to $\calP_2(\R^{d_z})$. The solution $\mu_x^\ast$ to the Stage II optimization problem exists, is not necessarily unique, is absolutely continuous with respect to the Lebesgue measure, and belongs to $\calP_2(\R^{d_x})$. 
\end{prop}
The proof can be found in \Cref{sec:proof_prop_43}. 
\Cref{prop:existence} ensures that the solutions to the two optimization problems in \eqref{eq:distribution_space_optimization} exist and are well-defined. The proof of the first half of \Cref{prop:existence} is standard and follows exactly that of Proposition 2.5 in \citet{hu2021mean}; whereas the proof of the second half is novel and relies on the continuity of the mapping $\mu_x\mapsto F_2(\mu_x, \mu_{z}^{\ast}(\mu_{x}))$ in terms of the weak topology. Unfortunately, the solution to the outer loop might not be unique due to the lack of convexity of the nested mapping $\mu_x\mapsto \scrF_2(\mu_x, \mu_{z}^{\ast}(\mu_{x}))$. 
% This lack of uniqueness will be d

Note that in the original formulation of DFIV in \citet{xu2021learning} and \citet{kim2025optimality}, the entropic regularizations $\sigma_1 \mathrm{Ent}(\mu_x)$ and $\sigma_2 \mathrm{Ent}(\mu_z)$ are not present in the respective objectives. 
As reviewed in \Cref{sec:mfld}, entropic regularizations are crucial for establishing fast, non-asymptotic finite-particle convergence of MFLD. 
The following proposition establishes that the solutions of the entropically regularized \eqref{eq:distribution_space_optimization} remain consistent with those of the original formulation of DFIV in the limit $\sigma_1, \sigma_2 \to 0$. 

% \textcolor{red}{I feel like the reader of this paper should probably know what is $\Gamma$-convergence...? I really do not want to add another definition here as this paper is already very long. I hope that explanation following this proposition gives a good explanation of $\Gamma$-convergence in text. Otherwise if that is still not clear, I am happy to avoid the terminology of $\Gamma$-convergence completely in the main text.} 

\begin{prop}[$\Gamma$-convergence as $(\sigma_1,\sigma_2) \to (0,0)$] \label{prop:gamma_convergence}
Suppose \Cref{ass:npiv} and \ref{ass:network} hold. 
Let $\mu_{z,\sigma_1}^{\ast}(\mu_{x})$ be the solution to the inner-loop optimization problem in \eqref{eq:distribution_space_optimization} with entropic regularization scale $\sigma_1$. 
We write the stage II objective as $\scrF_{2, (\sigma_1,\sigma_2)}(\mu_x, \mu_{z,\sigma_1}^\ast(\mu_x)) = U_2(\mu_{z,\sigma_1}^{\ast}(\mu_{x})) + \frac{\zeta_2}{2}\E_{\mu_x}[\|x\|^2]+\sigma_2 \mathrm{Ent}(\mu_x)$ with an explicit emphasis on its dependence on the entropic regularization scales $\sigma_1,\sigma_2$. 
Then, as $(\sigma_1,\sigma_2) \to (0,0)$, the family of functionals $\scrF_{2, (\sigma_1,\sigma_2)}$ would $\Gamma$-converge to $\scrF_{2, (0,0)}$ with respect to the weak topology on $\calP_2(\R^{d_x})$.
\end{prop}

The proof can be found in \Cref{sec:proof_prop_44}. 
The $\Gamma$-convergence result above guarantees stability of minimizers under vanishing entropic regularization.
Let $(\mu^\ast_{x,\sigma_1,\sigma_2})$ denote a sequence of global minimizers of $\mathscr{F}_{2, (\sigma_1,\sigma_2)}$ introduced in \Cref{prop:gamma_convergence}, i.e.,
$\mu^\ast_{x,\sigma_1,\sigma_2}=\arg\min_{\mu_x\in\calP_2(\R^{d_x})} \mathscr{F}_{2, (\sigma_1,\sigma_2)}(\mu_x, \mu_{z,\sigma_1}^\ast(\mu_x))$. 
If this sequence converges (in the sense of weak topology) to some ${\mu}^\diamond_x$ as $(\sigma_1,\sigma_2) \to (0,0)$, then ${\mu}^\diamond_x$ is a global minimizer of $\mathscr{F}_{2, (0,0)}$. 

\subsection{Penalty gradient method}\label{sec:penalty_gradient}

Having established the existence of solutions (\Cref{prop:existence}) and their consistency with the original DFIV problem as $(\sigma_1,\sigma_2) \to (0,0)$ (\Cref{prop:gamma_convergence}), we now introduce a \emph{fully first-order} bilevel mean field Langevin dynamics for solving \eqref{eq:distribution_space_optimization}, termed by \texttt{F$^2$BMLD}. 
Our algorithm builds on recent advances in bilevel optimization, where the inner-loop problem is reformulated as a constraint embedded in the outer-loop problem~\citep{shen2023penalty,kwon2023fully}. 
% In \Cref{sec:nested_gradient}, we study the convergence of a straight-forward algorithm which solves the solution $\mu_z^\ast(\mu_x)$ to the inner-loop first and then solve the outer-loop by computing gradient through the nested mapping $\mu_x\mapsto \scrF_2(\mu_x,\mu_z^\ast(\mu_x))$, a method referred to as \emph{implicit gradient method} or \emph{implicit gradient method}.
% We identify that its gradient involves second order derivatives which are computationally expensive and whose convergence is difficult to analyze. 
% Instead, we 

In \eqref{eq:distribution_space_optimization}, the inner-loop optimization corresponds to a standard mean-field Langevin dynamics, which enjoys fast convergence rates given the convexity of the mapping $\mu_z \mapsto F_1(\mu_x, \mu_z)$ proved in \Cref{prop:partial_convex_simple} for any fixed $\mu_x$. 
The primary challenge, however, lies in the outer-loop optimization, as the \emph{nested} mapping $\mu_x \mapsto F_2(\mu_x, \mu_z^\ast(\mu_x))$ is no longer convex, and the Wasserstein gradient of this mapping requires higher-order gradients which are computationally expensive (see \Cref{prop:xi_frechet_subdiff}). 
To address this challenge, we adopt the reformulation which casts the bilevel optimization problem as the following constrained optimization problem:
\begin{align}\label{eq:penalty_formulation}
    \min_{\mu_x,\mu_z} \scrF_2(\mu_x, \mu_z), \quad \scrF_1(\mu_x, \mu_z) - \scrF_1(\mu_x, \mu_z^\ast(\mu_x)) \leq \varepsilon \tag{$\varepsilon$-constrained}.
\end{align}
Here $\mu_z^\ast(\mu_x) = \arg\min_{\mu_z} \scrF_1(\mu_x,\mu_z)$ is the solution to the inner-loop optimization. 
It is immediate that the constrained problem \eqref{eq:penalty_formulation} recovers the original \eqref{eq:distribution_space_optimization} when $\varepsilon=0$. 

The above constrained optimization is still challenging to solve due to the imposed hard constraints. Following \citet{shen2023penalty} and \citet{kwon2023fully}, we formulate the above constrained optimization problem as the following Lagrangian problem. 
\begin{align}\label{eq:lagrangian}
    (\mu_{x,\lambda}^\ast, \mu_{z,\lambda}^\ast) = \arg\min_{\mu_x,\mu_z} \scrL_\lambda(\mu_x, \mu_z) := \scrF_2(\mu_x, \mu_z) + \lambda \left(\scrF_1(\mu_x, \mu_z) - \scrF_1(\mu_x, \mu_z^\ast(\mu_x)) \right) . \tag{$\lambda$-penalty}
\end{align}
Here, $\lambda>0$ is the Lagrange multiplier. 
It is again immediate that the Lagrangian formulation \eqref{eq:lagrangian} recovers the original \eqref{eq:distribution_space_optimization} when $\lambda=\infty$. 
The estimator of the structural function $h_\circ:\calA\to\R$ would be $\ba \mapsto \int \Psi_\ba(x) \dd \mu_{x,\lambda}^\ast$. 

Next, we establish a more quantitative connection between \eqref{eq:lagrangian}, \eqref{eq:penalty_formulation} and the original \eqref{eq:distribution_space_optimization} in terms of the Lagrange multiplier $\lambda$. 
In particular, we prove that one can recover the approximate global solution of \eqref{eq:distribution_space_optimization}
with a global solution of \eqref{eq:lagrangian}. 
Before we introduce the result, we give the following definition of an $\epsilon$-global-minimum.

\begin{defi}[$\epsilon$-global-minimum] 
Given a functional $\ell: \calP_2(\R^d) \mapsto \mathbb{R}$, for the optimization problem defined as $\min_{\mu\in\calP_2(\R^d)} \ell(\mu)$, we say $\mu_0 \in \calP_2(\R^d)$ is an $\epsilon$-global-minimum of this problem if $\ell(\mu_0) \leq \ell(\mu)+\epsilon$ for any $\mu \in \calP_2(\R^d)$.
\end{defi}
Recall that $\mu_z^\ast(\cdot)$ and $\mu_x^\ast$ are the global optimal solution of the inner-loop and the outer-loop in \eqref{eq:distribution_space_optimization}. 
Define $(\mu_{x,\lambda}^{(\epsilon)}, \mu_{z,\lambda}^{(\epsilon)})$ as the $\epsilon$-global-minimum of \eqref{eq:lagrangian}. 
When $\epsilon=0$, it becomes the true global minimum $(\mu_{x,\lambda}^{(0)}, \mu_{z,\lambda}^{(0)}) = (\mu_{x,\lambda}^\ast,\mu_{z,\lambda}^\ast) = \arg\min_{\mu_x, \mu_z} \scrL_\lambda(\mu_x, \mu_z)$. 

\begin{theorem}[Relations of solutions for \eqref{eq:distribution_space_optimization}, \eqref{eq:penalty_formulation} and \eqref{eq:lagrangian}]\label{thm:relation}
Suppose \Cref{ass:npiv} and \ref{ass:network} hold. Then, we have the following relations:
\begin{enumerate}[itemsep=5.0pt,topsep=5pt,leftmargin=*]
\item The global solution of \eqref{eq:distribution_space_optimization} is $\frac{R^2(R+M)^2}{8\lambda\sigma_1}$-global-minimum of \eqref{eq:lagrangian}.
\item Given $\epsilon_1, \epsilon_2> 0$ and $\lambda_0 = \epsilon_1^{-1} \frac{R^2(R+M)^2}{8\sigma_1}$, let $(\mu_{x,\lambda}^{(\epsilon_2)}, \mu_{z,\lambda}^{(\epsilon_2)})$ be $\epsilon_2$-global-minimum of \eqref{eq:lagrangian} with $\lambda>\lambda_0$. Then, $(\mu_{x,\lambda}^{(\epsilon_2)}, \mu_{z,\lambda}^{(\epsilon_2)})$ is also $\epsilon_2$-global-minimum of \eqref{eq:penalty_formulation} with $\varepsilon \leq (\epsilon_1+\epsilon_2) /(\lambda-\lambda_0)$.
\item Let $(\mu_{x,\varepsilon}^{(\epsilon_3)}, \mu_{z,\varepsilon}^{(\epsilon_3)})$ be $\epsilon_3$-global-minimum of \eqref{eq:penalty_formulation}. Then, $\scrF_2(\mu_{x,\varepsilon}^{(\epsilon_3)}, \mu_z^\ast(\mu_{x,\varepsilon}^{(\epsilon_3)})) - R(R+M) \sqrt{(2\sigma_1)^{-1} \varepsilon} \leq \scrF_2(\mu_{x,\varepsilon}^{(\epsilon_3)}, \mu_{z,\varepsilon}^{(\epsilon_3)}) \leq \scrF_2(\mu_x^\ast,\mu_z^\ast(\mu_x^\ast)) + \epsilon_3$. 
\end{enumerate}
\end{theorem}
The proof can be found in \Cref{sec:proof_relation}.  
\begin{rem}
From the second and the third bullet points of the above proposition, we can see that the global solution to the Lagrangian problem $ \arg\min_{\mu_x,\mu_z} \scrL_\lambda(\mu_x,\mu_z)$ can be a good approximation of the solution to the original bilevel optimization problem in \eqref{eq:distribution_space_optimization} for $\lambda > \lambda_0$. This relationship will be crucial in the generalization analysis in \Cref{sec:generalization}. 
\end{rem}

Next, we propose a concrete algorithm to solve \eqref{eq:lagrangian}. 
A particular advantage of the Lagrangian formulation is that the Wasserstein gradient of the Lagrangian objective $\scrL_\lambda$ only involves first-order derivatives, as shown in the following proposition. 
\begin{prop}[Wasserstein gradient of $\scrL_\lambda$]\label{prop:derivative_F_lambda}
Let $\mu_z^\ast(\mu_x) = \arg\min_{\mu_z} \scrF_1(\mu_x,\mu_z)$ be the solution to the inner-loop optimization. Then, for $\scrL_\lambda$  defined in \eqref{eq:lagrangian}, 
\begin{align*}
    \bnabla_1 \scrL_\lambda(\mu_x,\mu_z) &= \bnabla_1 \scrF_2(\mu_x, \mu_z) + \lambda \bnabla_1 \scrF_1(\mu_x, \mu_z) - \lambda \bnabla_1 \scrF_1(\mu_x, \mu_z^\ast(\mu_x)), \\
    \bnabla_2 \scrL_\lambda(\mu_x,\mu_z) &= \bnabla_2 \scrF_2(\mu_x, \mu_z) + \lambda \bnabla_2 \scrF_1(\mu_x, \mu_z).
\end{align*}
$\bnabla_1$ (resp. $\bnabla_2$) denotes the Wasserstein gradient with respect to the first (resp. second) argument.  
\end{prop}
\begin{proof}
The gradient of $\scrL_\lambda$ with respect to $\mu_x$ requires taking the gradient of the nested mapping $\mu_x\mapsto\scrF_1(\mu_x, \mu_z^\ast(\mu_x))$. Fortunately, by the envelope theorem, the optimality of $\mu_z^\ast(\mu_x)$ ensures that the Wasserstein gradient depends solely on the first argument of $\scrF_1$, and no additional terms arise from the dependence of $\mu_z^\ast(\mu_x)$ on $\mu_x$. 
The derivative of $\scrL_\lambda$ with respect to $\mu_z$ is standard as it does not involve nested mapping. 
\end{proof}
% The proof of \Cref{prop:derivative_F_lambda} is thus concluded. 

To aid the following analysis in the spirit of MFLD, we define another objective $L_\lambda$ which is $\scrL_\lambda$ yet excluding the entropic regularization on $\mu_x$. 
\begin{align}\label{eq:defi_L_lambda}
    L_\lambda(\mu_x,\mu_z) = F_2(\mu_x, \mu_z) + \lambda \left(\scrF_1(\mu_x, \mu_z) - \scrF_1(\mu_x, \mu_z^\ast(\mu_x)) \right) . 
\end{align}

Instead of performing mean-field Langevin dynamics (MFLD) directly on $L_\lambda$, or equivalently, running Wasserstein gradient flow on $\scrL_\lambda$ jointly with respect to $(\mu_x,\mu_z)$, we adopt a sequential optimization scheme. Specifically, we employ an \emph{alternating strategy}: for a fixed $\mu_x$, we first optimize $\mu_z$ to convergence, obtaining $\tilde{\mu}_z^\ast(\mu_x) = \arg\min_{\mu_z} \scrL_\lambda(\mu_x,\mu_z)$, 
and subsequently we update $\mu_x$ by performing another MFLD on the reduced objective $\mu_x \mapsto L_\lambda(\mu_x, \tilde{\mu}_z^\ast(\mu_x))$. This approach is motivated by the structural properties of $L_\lambda$ (and, analogously, $\scrL_\lambda$), which exhibits \emph{partial convexity}---convexity in $\mu_z$ for fixed $\mu_x$---but not \emph{joint convexity} in $(\mu_x,\mu_z)$. 
The lack of joint convexity is easy to verify: even a simple mapping such as $(\mu_x,\mu_z)\mapsto (\int \Psi(x) \dd \mu_x(x) - \int \Psi(z) \dd \mu_z(z))^2$ is not convex due to subtraction. 
The partial convexity is formalized in the following proposition.

\begin{prop}[Convexity of $\mu_z\mapsto L_\lambda(\mu_x,\mu_z)$]\label{prop:partial_convex_L_1}
The mapping $\mu_z\mapsto L_\lambda(\mu_x, \mu_z)$ is \emph{linear} convex, for any fixed $\mu_x\in\calP_2(\R^{d_x})$. 
% The mapping $\mu_x\mapsto L_\lambda(\mu_x, \mu_z)$ is \emph{linear} convex, for any fixed $\mu_z\in\calP_2(\R^{d_z})$. 
\end{prop}
\begin{proof}
    The proof is straightforward from \Cref{prop:partial_convex_simple}. 
\end{proof}

% \begin{rem}
%     The partial convexity of $\mu_x\mapsto L_\lambda(\mu_x, \mu_z^\ast(\mu_x))$ are studied in \Cref{sec:additional_theory}. 
% \end{rem}

\begin{algorithm}[t]
\caption{\textsc{InnerLoop}($\mu_x$, $T$, $\alpha$, $\beta$, $\lambda$, $\sigma_1$)}\label{alg:inner_loop}
\textbf{Input:} Inner-loop iteration count $T$, step sizes $\alpha,\beta$, penalty parameter $\lambda$, diffusion noise level $\sigma_1$.  
\begin{algorithmic}
    \State Initialize $\mu_{\scrz,0} = \frac{1}{N_z} \sum_{j=1}^{N_z} \delta_{z_{0}^{(j)}}$. 
\For{$t=0, \ldots, T$}
    \For{$i=1, \ldots, N_z$}
    \State $z_{t+1}^{(i)} = z_t^{(i)} - \alpha \bnabla_2 F_1(\mu_x, \mu_{\scrz,t})(z_t^{(i)}) + \sqrt{2 \alpha \sigma_1} \xi_{z,t}^{(i)}, \quad \xi_{z,t}^{(i)}\sim \calN(0, \Id_{d_z})$. 
    \EndFor
    \State Update $\mu_{\scrz,t} = \frac{1}{N_z} \sum_{j=1}^{N_z} \delta_{z_{t}^{(j)}}$. 
\EndFor
\State Initialize $\tilde{\mu}_{\scrz,0} = \frac{1}{N_z} \sum_{j=1}^{N_z} \delta_{\tilde{z}_{0}^{(j)}}$. 
\For{$t=0, \ldots, T$}
    \For{$i=1, \ldots, N_z$}
    \State \hspace{-1.5cm} $\begin{aligned}
    \tilde{z}_{t+1}^{(i)} = \tilde{z}_t^{(i)} - \beta \bnabla U_2(\tilde{\mu}_{\scrz, t}) (\tilde{z}_t^{(i)}) - \beta \lambda \bnabla_2 F_1(\mu_x, \tilde{\mu}_{\scrz, t}) (\tilde{z}_t^{(i)}) + \sqrt{2 \beta \lambda \sigma_1} \tilde{\xi}_{z,t}^{(i)}, \text{ } \tilde{\xi}_{z,t}^{(i)}\sim \calN(0, \Id_{d_z}) \end{aligned}$. 
    \EndFor
    \State Update $\tilde{\mu}_{\scrz,t} = \frac{1}{N_z} \sum_{j=1}^{N_z} \delta_{\tilde{z}_{t}^{(j)}}$. 
\EndFor
\end{algorithmic}
\textbf{Return: $\mu_{\scrz,T}(\mu_x), \tilde{\mu}_{\scrz,T}(\mu_x)$.}  
\end{algorithm}

\begin{rem}
In the inner loop, we solve the following two optimization problems:
\begin{align}
    &\mu_z^\ast(\mu_x) = \arg\min_{\mu_z} \scrF_1(\mu_x,\mu_z) = \arg\min_{\mu_z} F_1(\mu_x,\mu_z) + \sigma_1 \mathrm{Ent}(\mu_z), \label{eq:inner_loop_problem_1} \\
    & \tilde{\mu}_z^\ast(\mu_x) = \arg\min_{\mu_z} \scrL_\lambda(\mu_x, \mu_z) = \arg\min_{\mu_z} F_2(\mu_x, \mu_z) + \lambda F_1(\mu_x, \mu_z) + \lambda \sigma_1 \mathrm{Ent}(\mu_z) \label{eq:inner_loop_problem_2} . 
\end{align}
$\mu_z^\ast(\mu_x)$ and $\tilde{\mu}_z^\ast(\mu_x)$ are two distinct quantities. 
$\mu_z^\ast(\mu_x)$ is the solution to the stage I optimization problem in \eqref{eq:distribution_space_optimization}, and it approximates the conditional expectation operator $T$ in the sense that $\int \Psi_\bw(z) \dd \mu_z^\ast(\mu_x) \approx T(\int \Psi_A(x) \dd \mu_x)(\bw)$. 
In contrast, $\tilde{\mu}_z^\ast(\mu_x)$ is the partial solution to the Lagrangian problem in \eqref{eq:lagrangian}\footnote{
Although $\tilde{\mu}_z^\ast(\mu_x)$ in Eq.~\eqref{eq:inner_loop_problem_2} depends on the Lagrange multiplier $\lambda$, we do not make this dependence explicit in the notation for two reasons: (i) to keep the notation lightweight, and (ii) to retain notation symmetric with ${\mu}_z^\ast(\mu_x)$ in Eq.~\eqref{eq:inner_loop_problem_1} which is also computed in the inner loop.}. Unlike $\mu_z^\ast(\mu_x)$, $\tilde{\mu}_z^\ast(\mu_x)$ is only an intermediate quantity and does not admit a direct interpretation in the NPIV problem of Eq.~\eqref{eq:npiv}. 
The motivation for computing the partial solution of \eqref{eq:lagrangian} comes from the partial convexity of the mapping $\mu_z \mapsto \scrL_\lambda(\mu_x, \mu_z)$, as established in \Cref{prop:partial_convex_L_1}.
Moreover, since both optimization problems above are taken with respect to $\mu_z$ while keeping $\mu_x$ fixed, this naturally suggests solving them within the same inner loop. 
\end{rem}

The precise inner-loop algorithm is presented in \Cref{alg:inner_loop}, where both dynamics are simulated for $T$ iterations and $N_z$ particles using step sizes $\alpha$ and $\beta$. 
Owing to the convexity of both objectives $\mu_x \mapsto \scrF_1(\mu_x, \mu_z)$ and $\mu_x \mapsto \scrL_\lambda(\mu_x, \mu_z)$, the convergence of the corresponding mean-field Langevin dynamics is expected to be fast, following the same reasoning as in \Cref{sec:mfld}. 

After obtaining both $\mu_z^\ast(\mu_x)$ and $\tilde{\mu}_z^\ast(\mu_x)$, to complete the solution to the Lagrangian problem \eqref{eq:lagrangian}, what remains is to solve $\mu_{x,\lambda}^\ast = \arg\min_{\mu_x} \scrL_\lambda(\mu_x, \tilde{\mu}_z^\ast(\mu_x))$. 
Recall that $\mu_{x,\lambda}^\ast$ is the main quantity of interest as $\ba \mapsto \int \Psi_\ba(x) \dd \mu_{x,\lambda}^\ast$ is our final estimator of the structural function $h_\circ:\calA\to\R$.
Therefore, for the outer loop, the target is to find the following: 
\begin{align*}
    \mu_{x,\lambda}^\ast &= \arg\min_{\mu_x} \scrL_\lambda(\mu_x, \tilde{\mu}_z^\ast(\mu_x)) \nonumber = \arg\min_{\mu_x} L_\lambda(\mu_x, \tilde{\mu}_z^\ast(\mu_x)) + \sigma_2 \mathrm{Ent}(\mu_x) \\
    &= \arg\min_{\mu_x} F_2(\mu_x, \tilde{\mu}_z^\ast(\mu_x)) + \lambda (\scrF_1(\mu_x, \tilde{\mu}_z^\ast(\mu_x)) - \scrF_1(\mu_x, \mu_z^\ast(\mu_x)) ) + \sigma_2 \mathrm{Ent}(\mu_x) \nonumber .
\end{align*}
From the optimality of $\tilde{\mu}_z^\ast(\mu_x)$ and $\mu_z^\ast(\mu_x)$, along with the envelope theorem, the Wasserstein gradient of the mapping $\mu_x\mapsto L_\lambda(\mu_x, \tilde{\mu}_z^\ast(\mu_x))$ can be written as
\begin{align}\label{eq:gradient_L_lambda}
    \bnabla L_\lambda(\mu_x, \tilde{\mu}_z^\ast(\mu_x))(x) &= \bnabla_1 L_\lambda(\mu_x, \tilde{\mu}_z^\ast(\mu_x))(x) \nonumber \\
    &\hspace{-20pt} = \bnabla_1 F_2(\mu_x, \tilde{\mu}_z^\ast(\mu_x))(x) + \lambda (\bnabla_1 \scrF_1(\mu_x, \tilde{\mu}_z^\ast(\mu_x))(x) - \bnabla_1 \scrF_1(\mu_x, \mu_z^\ast(\mu_x))(x) ) \nonumber \\
    &\hspace{-20pt} = \zeta_2 x + \lambda (\bnabla_1 U_1(\mu_x, \tilde{\mu}_z^\ast(\mu_x))(x) - \bnabla_1 U_1(\mu_x, \mu_z^\ast(\mu_x))(x) ) .
\end{align}
Fortunately, the Wasserstein gradient above admits a closed-form expression and an efficient finite-particle implementation, because it does not involve the nested mappings $\mu_z \mapsto \mu_z^\ast(\mu_x)$ or $\mu_z \mapsto \tilde{\mu}_z^\ast(\mu_x)$, and consequently, it coincides with the Euclidean gradient of its first variation. 
The exact outer-loop algorithm is outlined in  \Cref{alg:outer_loop}, where $\tilde{\mu}_z^\ast(\mu_x)$ and $\mu_z^\ast(\mu_x)$ are replaced with the outputs of the inner-loop. 
The output of \Cref{alg:outer_loop} $\mu_{\scrx, S} = \frac{1}{N_x} \sum_{j=1}^{N_x} \delta_{x_{S}^{(j)}}$ corresponds to the result of a time-discretized, finite-particle implementation of the mean field Langevin dynamics, simulated for $S$ iterations and $N_x$ particles using a positive step size $\gamma$.

\begin{algorithm}[t]
\caption{$\texttt{F$^2$BMLD}: \textsc{Outerloop}(S,T, \alpha, \beta, \gamma, \lambda, \sigma_1,\sigma_2)$}
\label{alg:outer_loop}
\textbf{Input:} Inner-loop iteration count $T$, outer-loop iteration count $S$, step sizes $\alpha,\beta,\gamma$, penalty parameter $\lambda$, diffusion noise level $\sigma_1, \sigma_2$. 
\begin{algorithmic}[1]
\State Initialize $\mu_{\scrx,0} = \frac{1}{N_x} \sum_{j=1}^{N_x} \delta_{x_{0}^{(j)}}$. 
\For{$s=0,\ldots,S$}
\State $\tilde{\mu}_{\scrz, s}, \mu_{\scrz, s}\gets \textsc{InnerLoop}(\mu_{\scrx,s}, T, \alpha, \beta, \lambda, \sigma_1)$. 
\For{$i=1, \ldots, N_x$}
\State $\begin{aligned}
x_{s+1}^{(i)} &= x_s^{(i)} - \gamma \left( \zeta_2 x_s^{(i)} + \lambda (\bnabla_1 U_1(\mu_{\scrx,s}, \tilde{\mu}_{\scrz, s})(x_s^{(i)}) - \bnabla_1 U_1(\mu_{\scrx,s}, \mu_{\scrz, s})(x_s^{(i)})) \right) \\
&\hspace{17em} + \sqrt{2 \gamma \sigma_2} \xi_{x, s}^{(i)} , \quad \xi_{x,s}^{(i)}\sim \calN(0, \Id_{d_x}). 
\end{aligned}
$ 
\EndFor
\State Update $\mu_{\scrx,s} = \frac{1}{N_x} \sum_{j=1}^{N_x} \delta_{x_{s}^{(j)}}$. 
\EndFor
\end{algorithmic}
\textbf{Return: $\mu_{\scrx,S}$.}  
\end{algorithm}

\begin{rem}[A fully first-order algorithm]
    A key advantage of the proposed algorithm, \emph{\texttt{F$^2$BMLD}}, derived from the Lagrangian reformulation of the original bilevel optimization problem, is that it relies solely on first-order Wasserstein gradients, which after finite-particle implementation, would correspond to standard first-order Euclidean gradients, making the method computationally efficient. 
    In contrast, existing approaches that leverage deep neural networks to solve 2SLS require taking derivatives through the nested mappings, which are computationally demanding~\citep{xu2021learning}.
    The notable exception is the method proposed by \citet{petrulionyte2024functional}, which only requires \emph{functional} second-order derivatives—quantities that often admit closed-form expressions for common objectives such as the mean squared loss. Their analysis remains a weight-space perspective, however, and therefore establishes convergence only to a stationary point. 
\end{rem}

\begin{rem}[Implicit gradient method]
    In \Cref{sec:nested_gradient}, we show that directly solving \eqref{eq:distribution_space_optimization} without resorting to its Lagrangian formulation would lead to an optimization problem over the space of probability measures that is intractable with a finite number of particles. The underlying intuition is that, unlike variational integrals in Equation 10.4.1 of \citet{ambrosio2008gradient}, the mapping of the outer-level objective $\mu_x \mapsto \scrF_2(\mu_x, \mu_z^\ast(\mu_x))$ does not preserve a structure where the Wasserstein gradient can be expressed simply as the gradient of its first variation. Consequently, we must revert to the original definition of the Fréchet subdifferential, as introduced in Section 10 of \citet{ambrosio2008gradient}, which yields Wasserstein gradients that involve terms difficult to approximate with finitely many particles (see \Cref{prop:xi_frechet_subdiff} for details), unlike our fully first order method \emph{\texttt{F$^2$BMLD}}. 
\end{rem}

\section{Convergence of \texttt{F$^2$BMLD}}\label{sec:convergence}
In this section, we establish the non-asymptotic convergence of the output $\mu_{\scrx,s}$ of the proposed \texttt{F$^2$BMLD} towards the global optimum $\mu_{x, \lambda}^*$ of the Lagrangian problem \eqref{eq:lagrangian}, for a fixed $\lambda>0$. We focus explicitly on $\mu_x$, rather than $\mu_z$, since $\mu_x$ corresponds to the estimation of the structural function $h_\circ$, which constitutes the primary objective in the 2SLS setting.

\subsection{Convergence of the inner loop}\label{sec:convergence_inner_loop}
First, we study the convergence of the outputs of the inner loop algorithm in \Cref{alg:inner_loop}, with respect to the number of iterations $T$ and the number of particles $N_z$, towards the optima $\mu_z^\ast(\mu_x), \tilde{\mu}_z^\ast(\mu_x)$ for a fixed $\mu_x$. 
As emphasized in the previous section, the inner-loop optimization corresponds to standard mean-field Langevin dynamics. Consequently, we can directly leverage the existing results on MFLD reviewed in \Cref{sec:mfld}. To this end, we first establish in the following lemma that the global optima $\mu_z^\ast(\mu_x), \tilde{\mu}_z^\ast(\mu_x)$ satisfy a log-Sobolev inequality. 

\begin{lem}[Log-Sobolev inequality of the inner-loop]\label{lem:LSI}
Suppose \Cref{ass:network} holds. 
For any fixed $\mu_x\in\calP_2(\R^{d_x})$, 
both $\tilde{\mu}_z^\ast(\mu_x)$ and ${\mu}_z^\ast(\mu_x)$ satisfy a Log-Sobolev inequality with constant $C_{\LSI,z} = \frac{\zeta_1}{2 \sigma_1} \exp(-\frac{16 R^2}{\zeta_1 \sigma_1} \sqrt{2 d_z / \pi})$. 
\end{lem}

\begin{proof}
The proof is a direct application of Lemma 5 of \citet{suzuki2023convergence} which itself is based on the contraction principle~\citep[Proposition
5.4.3]{bakry2013analysis}. 
\end{proof}
\begin{rem}
    The Log-Sobolev constant is independent of the Lagrange multiplier $\lambda$. This shall be contrasted with the Log-Sobolev constant $C_{\LSI,x}$ of the outer-loop optimization, which gets exponentially small as $\lambda$ increases, as detailed in \Cref{lem:LSI_x}.
\end{rem}
% Following the same argument as in the proof of Lemma 1 in \citet{nitanda2025propagation}, we can show that the defective log-Sobolev inequality of Eq.~\eqref{eq:lem_1_nitanda_2} is satisfied. Also, 
% First, we need to check whether the required assumptions are all satisfied. 
% The boundedness and the Lipschitzness of the Wasserstein gradients are satisfied due to our assumption on the boundedness and bounded gradient of the network output in \Cref{ass:network}. 
In addition, the linear convexity of the objectives $\mu_z\mapsto F_1(\mu_x, \mu_z)$ and $\mu_z\mapsto L_\lambda(\mu_x, \mu_z)$ for any fixed $\mu_x$ have already been proved in \Cref{prop:partial_convex_simple} and \Cref{prop:partial_convex_L_1}. 
Therefore, following Theorem 1 of \citet{nitanda2025propagation}, combining Eq.~\eqref{eq:mfld_convergence} and Eq.~\eqref{eq:kl_w2_F}, we obtain the following convergence results. 
% The corresponding proximal Gibbs distributions both satisfy a Log-Sobolev inequality with constants $C_{\LSI,z}=C_{\LSI,z}$ due to the $l_2$ norm regularization.  
% Therefore, Assumption 1-4 of \citet{nitanda2025propagation} are satisfied. 
% Here, we primarily focus on \citet{nitanda2025propagation} as it is the state of the art convergence result with the mildest dependence on the Log-Sobolev constant. 

\begin{prop}[Inner-loop convergence towards  $\mu_z^\ast(\mu_x)$]\label{prop:inner}
Suppose \Cref{ass:npiv} and \ref{ass:network} hold. 
Given a fixed $\mu_x\in\calP_2(\R^{d_x})$. 
Let $\scrZ=\{z^{(i)}\}_{i=1}^{N_z}$ be the first output of the inner-loop algorithm \textsc{InnerLoop}($\mu_x$, $T$, $\alpha$, $\beta$, $\lambda$, $\sigma_1$), detailed in \Cref{alg:inner_loop}, where the initial $N_z$ particles are sampled i.i.d from some distribution $\mu_{z,0}\in\calP_2(\R^{d_z})$. 
Denote $\mu_{z}^{(N_z)}$ as the joint distribution of these $N_z$ particles $\scrZ$. 
Suppose the step size $\alpha$ satisfies $\alpha\leq\frac{1}{\zeta_1}$. 
For any $T>0$, 
\begin{align*}
\frac{\sigma_1}{N_z} \kl\left(\mu_{z}^{(N_z)}, (\mu_{z}^\ast(\mu_x))^{\otimes N_z} \right) \leq \frac{R^2}{N_z} + \frac{\delta_\alpha}{C_{\LSI,z} \sigma_1} + \exp (-C_{\LSI,z} \sigma_1 \alpha T) \Delta_0^{(N_z)} .
\end{align*}
Here, $\Delta_0^{(N_z)}= \E_{\scrz\sim\mu_{z,0}^{\otimes N_z}}[F_1(\mu_x, \mu_\scrz)] + \sigma_1 \mathrm{Ent}(\mu_{z,0}) - \scrF_1(\mu_x,\mu_z^\ast(\mu_x))$ represent the approximation error at initialization, and $\delta_\alpha = 8 \alpha (C_2^2 + \zeta_1^2)(\alpha C_1^2 + \sigma_1 d_z )+32 \alpha^2 \zeta_1^2 (C_2^2 + \zeta_1^2) (\E_{z\sim\mu_{z,0}}[\|z\|_2^2]+\frac{1}{\zeta_1}(\frac{C_1^2}{4 \zeta_1}+\sigma_1 d_z))$ represents the time discretization error. The expectation above is taken with respect to the randomness of the initial particles and injected Gaussian noise at each iteration. The constants $C_1=2R^2$ and $C_2=R+R^2$. 
\end{prop}

\begin{prop}[Inner-loop convergence towards $\tilde{\mu}_z^\ast(\mu_x)$]\label{prop:tilde_inner}
Suppose \Cref{ass:npiv} and \ref{ass:network} hold. 
Given a fixed $\mu_x\in\calP_2(\R^{d_x})$ and a fixed $\lambda > 0$. 
Let $\tilde{\scrZ}=\{\tilde{z}^{(i)}\}_{i=1}^{N_z}$ be the second output of the inner-loop algorithm \textsc{InnerLoop}($\mu_x$, $T$, $\alpha$, $\beta$, $\lambda$, $\sigma_1$), detailed in \Cref{alg:inner_loop}, where the initial $N_z$ particles are sampled i.i.d from some distribution $\tilde{\mu}_{z,0}\in\calP_2(\R^{d_z})$. 
Denote $\tilde{\mu}_{z}^{(N_z)}$ as the joint distribution of these $N_z$ particles $\tilde{\scrZ}$. 
Suppose the step size $\beta$ satisfies $\beta\leq\frac{1}{\lambda \zeta_1}$. 
For any $T>0$, 
\begin{align*}
\frac{\lambda \sigma_1}{N_z} \kl\left(\tilde{\mu}_{z}^{(N_z)}, (\tilde{\mu}_{z}^\ast(\mu_x))^{\otimes N_z} \right) \leq \frac{\lambda R^2}{N_z} + \frac{\delta_\beta}{C_{\LSI,z} \lambda \sigma_1} + \exp (-C_{\LSI,z} \lambda \sigma_1 \beta T) \tilde{\Delta}_0^{(N_z)} .
\end{align*}
Here, $\tilde{\Delta}_0^{(N_z)}= \E_{\scrz\sim\tilde{\mu}_{z,0}^{\otimes N_z}}[U_2(\mu_\scrz) + \lambda F_1(\mu_x, \mu_\scrz)] + \lambda \sigma_2 \mathrm{Ent}(\tilde{\mu}_{z,0}) - (U_2(\tilde{\mu}_z^\ast(\mu_x)) + \lambda \scrF_1(\mu_x,\tilde{\mu}_z^\ast(\mu_x)))$ represent the approximation error at initialization, and $\delta_\beta = 8 \beta \lambda^3(C_2^2 + \zeta_1^2)( \beta C_1^2 + \sigma_1 d_z )+32 \beta^2 \lambda^4 \zeta_1^2 (C_2^2 + \zeta_1^2) (\E_{z\sim\mu_{z,0}}[\|z\|_2^2]+\frac{1}{\zeta_1}(\frac{C_1^2}{4 \zeta_1}+\sigma_1 d_z))$ represents the time discretization error. The expectation above is taken with respect to the randomness of the initial particles and injected Gaussian noise at each iteration. The constants $C_1=2 R^2$ and $C_2=R+R^2$. 
\end{prop}

The proofs of these two propositions are omitted as they are straight forward applications of Theorem 1 of \citet{nitanda2025propagation} which has been reviewed in \Cref{sec:mfld}. Note that in \Cref{prop:tilde_inner}, the dependence on the Lagrange multiplier $\lambda$ arises because the both $\ell_2$ and entropic regularizations have been rescaled by $\lambda$, see Eq.~\eqref{eq:inner_loop_problem_2}. 
Since the initial $N_z$ particles are sampled i.i.d from some distributions $\mu_{z,0}$ and $\tilde{\mu}_{z,0}$, the initial approximation error terms $\Delta_0^{(N_z)}$ and $\tilde{\Delta}_0^{(N_z)} \lambda^{-1}$ with fixed $\lambda > 0$ are uniformly bounded for any $\mu_x$ as long as $\mu_{z,0}$ and $\tilde{\mu}_{z,0}$ have finite second moment and finite entropy.

\begin{rem}[Iteration and particle complexity of the inner-loop] \label{rem:complexity_inner}
By \Cref{prop:inner}, in order to achieve $N_z^{-1} \kl(\mu_{z}^{(N_z)}, (\mu_{z}^\ast(\mu_x))^{\otimes N_z}) \leq \delta < 1$ with $\delta$ sufficiently small, it suffices to choose the step size $\alpha \lesssim \delta \sigma_1 C_{\LSI,z} \zeta_1^{-2} d_z^{-1}$, which yields the following iteration and sample complexity: 
\begin{align}\label{eq:inner_loop_complexity}
    T \geq \frac{\log (\delta^{-1})}{\delta} \frac{\zeta_1^2 d_z}{(C_{\LSI,z}\sigma_1)^2}, \quad N_z \geq \frac{1}{\delta} \frac{1}{\sigma_1} .
\end{align}
Similarly, to achieve $N_z^{-1} \kl(\tilde{\mu}_{z}^{(N_z)}, (\tilde{\mu}_{z}^\ast(\mu_x))^{\otimes N_z}) \leq \delta < 1$, it suffices to choose the step size $\beta \leq \lambda^{-1} \delta \sigma_1 C_{\LSI,z} \zeta_1^{-2} d_z^{-1}$ which results in the same iteration and particle complexity as Eq.~\eqref{eq:inner_loop_complexity}. The equality of complexities for the two inner-loop problems in Eq.~\eqref{eq:inner_loop_problem_1} and Eq.~\eqref{eq:inner_loop_problem_2} follows immediately from the fact that the latter’s objective is roughly a rescaled version of the former’s objective by a factor of $\lambda$, and both share the same Log–Sobolev constant (see \Cref{lem:LSI}).
We explicitly show the dependence on the log-Sobolev constant $C_{\LSI,z}$ to emphasize the dependence through it on the dimension $d_z$.  
\end{rem}

\subsection{Convergence of the outer loop}\label{sec:convergence_outer_loop}
In this section, we analyze the convergence of the outer loop in \Cref{alg:outer_loop}, namely the proposed \texttt{F$^2$BMLD} algorithm, towards the global optimum $\mu_{x,\lambda}^\ast$ of \eqref{eq:lagrangian} for a fixed $\lambda > 0$. 
Recall the definition of the two functionals $L_\lambda$ and $\scrL_\lambda$ in Eq.~\eqref{eq:defi_L_lambda} and \eqref{eq:lagrangian}:
\begin{align*}
    L_\lambda(\mu_x,\mu_z) = F_2(\mu_x, \mu_z) + \lambda \left(\scrF_1(\mu_x, \mu_z) - \scrF_1(\mu_x, \mu_z^\ast(\mu_x)) \right), \text{ } \scrL_\lambda(\mu_x,\mu_z) = L_\lambda(\mu_x,\mu_z) + \sigma_2 \mathrm{Ent}(\mu_x) . 
\end{align*}
In practice, the stage I solution $\mu_z^\ast(\mu_x)$ is learned in the inner-loop (\Cref{alg:inner_loop}) via  mean field Langevin dynamics. To make explicit this dependence on $\mu_z^\ast(\mu_x)$, we extend both functionals $L_\lambda, \scrL_\lambda$ to mappings from the product space $\calP_2(\R^{d_x}) \times\calP_2(\R^{d_z})\times\calP_2(\R^{d_z})$ to $\R$:
\begin{align}\label{eq:defi_L_lambda_new}
    L_\lambda(\mu_x, \tilde{\mu}_z, \mu_z) &:= F_2(\mu_x, \tilde{\mu}_z) + \lambda \cdot \scrF_1(\mu_x, \tilde{\mu}_z) - \lambda \cdot \scrF_1(\mu_x, \mu_z) \\
    \scrL_\lambda(\mu_x, \tilde{\mu}_z, \mu_z) &:= L_\lambda(\mu_x, \tilde{\mu}_z, \mu_z) + \sigma_2 \mathrm{Ent}(\mu_x) . 
\end{align}
The outer-loop \Cref{alg:outer_loop} is a space- and
time-discretized implementation of mean field Langevin dynamics of the functional $\mu_x\mapsto L_\lambda(\mu_x, \tilde{\mu}_z^\ast(\mu_x), \mu_z^\ast(\mu_x))$. As a first step in analyzing its convergence, and following the framework reviewed in \Cref{sec:mfld}, we establish that the global optimum $\mu_{x,\lambda}^\ast = \arg\min_{\mu_x\in\calP_2(\R^{d_x})} \scrL_\lambda(\mu_x, \tilde{\mu}_z^\ast(\mu_x), \mu_z^\ast(\mu_x))$ satisfies a log-Sobolev inequality, as stated in the following lemma. 

\begin{lem}[Log-Sobolev constant of the outer-loop]\label{lem:LSI_x}
Suppose \Cref{ass:network} holds. 
For any fixed $\lambda>0$, $\mu_{x,\lambda}^\ast$ satisfies a log-Sobolev inequality with constant $C_{\LSI,x} = \frac{\zeta_2}{2 \sigma_2} \exp(-16 \frac{\lambda^2 R^2}{\zeta_2 \sigma_2} \sqrt{2 d_x / \pi})$. 
\end{lem}
\begin{proof}
The proof is a direct application of Lemma 5 of \citet{suzuki2023convergence} which itself is based on the contraction principle~\citep[Proposition
5.4.3]{bakry2013analysis}. 
\end{proof}
\begin{rem}\label{rem:compare_LSI}
    Comparing the LSI constant $C_{\LSI,z}$ for the inner loop (proved in \Cref{lem:LSI}) with the LSI constant $C_{\LSI,x}$ for the outer loop (proved in \Cref{lem:LSI_x}), we observe that $C_{\LSI,z}$ is independent of the Lagrange multiplier $\lambda$, whereas $C_{\LSI,x}$ deteriorates exponentially as $\lambda$ increases. This discrepancy arises because the vector field of the outer loop is scaled by $\lambda$ (see Eq.~\eqref{eq:gradient_L_lambda}) , while the entropic and $\ell_2$ regularizations do not. 
\end{rem}

Define $\mu_{z,s}^{(N_z)}\in\calP_2((\R^{d_z})^{N_z})$ (resp. $\tilde{\mu}_{z,s}^{(N_z)} \in\calP_2((\R^{d_z})^{N_z})$) as the joint distribution of the $N_z$ particles $\scrz_s =[z_s^{(1)}, \ldots, z_s^{(N_z)}]$ (resp. $\tilde{\scrz}_s=[\tilde{z}_s^{(1)}, \ldots, \tilde{z}_s^{(N_z)}]$) which are the output of the inner loop algorithm at time $s$. The corresponding empirical distributions are $\tilde{\mu}_{\scrz, s}=\frac{1}{N_z} \sum_{j=1}^{N_z} \delta_{\tilde{z}_s^{(j)}}$ and $\mu_{\scrz, s}=\frac{1}{N_z} \sum_{j=1}^{N_z} \delta_{z_s^{(j)}}$. 
Define $\mu_{x,s}^{(N_x)} \in\calP_2((\R^{d_x})^{N_x})$ as the joint distribution of the $N_x$ particles $\scrx_s =[x_s^{(1)}, \ldots, x_s^{(N_x)}]$ of the outer loop algorithm at time $s$. 
The corresponding empirical distribution is $\mu_{\scrx, s}=\frac{1}{N_x} \sum_{j=1}^{N_x} \delta_{x_s^{(j)}}$. 
Following the procedures on the non-asymptotic convergence bound of MFLD reviewed in \Cref{sec:mfld}, we introduce the following auxiliary functionals $\calP_2((\R^{d_x})^{N_x}) \to \R$: 
\begin{align}\label{eq:defi_L_Nx}
    L_\lambda^{(N_x)}(\mu_x^{(N_x)}) &:= N_x \E_{\scrx \sim\mu_x^{(N_x)}} [F_2(\mu_{\scrx}, \tilde{\mu}_z^\ast(\mu_{\scrx}))] +  \lambda \cdot N_x\E_{\scrx \sim\mu_x^{(N_x)}} [\scrF_1(\mu_{\scrx}, \tilde{\mu}_z^\ast(\mu_{\scrx}))] \nonumber \\
    &\qquad - \lambda \cdot N_x\E_{\scrx \sim\mu_x^{(N_x)}} [\scrF_1(\mu_{\scrx}, \mu_z^\ast(\mu_{\scrx}))] \\
    \scrL_\lambda^{(N_x)}(\mu_x^{(N_x)}) &:= L_\lambda^{(N_x)}(\mu_x^{(N_x)}) + \sigma_1 \mathrm{Ent}(\mu_x^{(N_x)}) . 
\end{align} 
% \hudson{ToDO: Explain the reason behind these functionals}
Next, we are about to inspect whether the key inequalities in MFLD, namely Eq.~\eqref{eq:lem_1_nitanda_1} and Eq.~\eqref{eq:lem_1_nitanda_2} hold in the current context. Unfortunately, the mapping $\mu_x\mapsto L_\lambda(\mu_x, \tilde{\mu}_z^\ast(\mu_x), \mu_z^\ast(\mu_x))$ is \emph{no longer} linear convex due to the nested mapping. 

A direct consequence of the lack of convexity is that the Bregman divergence $B_{L_\lambda}$ associated with the mapping $\mu_x\mapsto L_\lambda(\mu_x, \tilde{\mu}_z^\ast(\mu_x), \mu_z^\ast(\mu_x))$ is \emph{no longer} positive. For any $\mu_x,\mu_x'\in\calP_2(\R^{d_x})$, 
\begin{align}\label{eq:bregman_L}
    B_{L_\lambda}(\mu_x,\mu_x') := L_\lambda(\mu_x, \tilde{\mu}_z^\ast(\mu_x), \mu_z^\ast(\mu_x)) &- L_\lambda(\mu_x', \tilde{\mu}_z^\ast(\mu_x'), \mu_z^\ast(\mu_x')) \nonumber \\
    &\quad - \smallint \delta_{\mu_x} L_\lambda(\mu_x', \tilde{\mu}_z^\ast(\mu_x'), \mu_z^\ast(\mu_x')) \; \dd (\mu_x-\mu_x^{\prime}) 
\end{align}
Here, $\delta_{\mu_x}$ denotes taking the first variation of the mapping $\mu_x\mapsto L_\lambda(\mu_x, \tilde{\mu}_z^\ast(\mu_x), \mu_z^\ast(\mu_x))$. 
Fortunately, however, we can prove in the following lemma that the Bregman divergence of $\mu_x,\mu_x'$ is lower bounded by the negative squared total variation distance of $\mu_x,\mu_x'$. 

\begin{lem}[Lower-bound on the Bregman divergence]\label{lem:bregman}
Suppose \Cref{ass:network} holds. 
Then, we have $B_{L_\lambda}(\mu_x, \mu_x^{\prime}) \geq -\frac{R \lambda}{4\sigma_1} \E_{\rho} \left[ \left( \smallint \Psi_\ba \; \dd (\mu_x - \mu_x^\prime) \right)^2 \right] \geq -\frac{R^3 \lambda}{4\sigma_1} \tv^2(\mu_x,  \mu_x^\prime)$. 
\end{lem}
The proof can be found in \Cref{sec:proof_bregman}. 
The above lemma implies that although the functional $\mu_x\mapsto L_\lambda(\mu_x, \tilde{\mu}_z^\ast(\mu_x), \mu_z^\ast(\mu_x))$ is not convex, it is actually \emph{weakly convex} with respect to the total variation norm. 
This is directly analogous to the Euclidean setting, where the Bregman divergence of a differentiable function measures the deviation from its linear approximation, and the existence of a quadratic lower bound is equivalent to weak convexity~\citep{boyd2004convex}. The weak convexity parameter deteriorates as $\lambda$ increases, but improves as $\sigma_1$ increases.

We are now ready to establish the two key inequalities of MFLD in our setting, namely Eq.~\eqref{eq:lem_1_nitanda_1} and Eq.~\eqref{eq:lem_1_nitanda_2}.

\begin{prop}[Defective Bregman divergence gap]\label{prop:KL_upper_bound_functional}
Suppose \Cref{ass:network} holds and let $\mathfrak{c} > 0$. 
Assume that $\sigma_1\sigma_2 \mathfrak{c} \geq 4R^3 \lambda $. 
Then, for any $\mu_x^{(N_x)} \in \calP_2((\R^{d_x})^{N_x})$, we have 
\begin{align*}
    N_x^{-1} \scrL_\lambda^{(N_x)}(\mu_x^{(N_x)}) - \scrL_\lambda(\mu_{x,\lambda}^\ast, \tilde{\mu}_z^*(\mu_{x,\lambda}^\ast) , \mu_z^*(\mu_{x,\lambda}^\ast)) \geq N_x^{-1} \frac{\sigma_2}{2} \kl(\mu_x^{(N_x)}, (\mu_{x,\lambda}^\ast)^{\otimes N_x}) - \frac{2R^3\lambda\mathfrak{c}}{\sigma_1}  - \frac{2 R^3\lambda}{N_x\sigma_1} . 
\end{align*}
\end{prop}
The proof can be found in \Cref{sec:proof_kl_gap}, which is based on \Cref{lem:bregman} and Proposition 1 in \citet{nitanda2025propagation}. 
% For a fixed $\lambda$, from \Cref{lem:LSI_x}, the condition $\sigma_1 \sigma_2 C_{\LSI,x} \geq 2 R^3 \lambda$ can be satisfied by simultaneously increasing $\zeta_2, \sigma_1, \sigma_2$ such that $\frac{\zeta_2}{2} \exp(-16 \frac{\lambda^2 R^2}{\zeta_2 \sigma_2} \sqrt{2 d_x / \pi}) \geq 2 \sigma_1^{-1} R^3 \lambda$ holds. 

\begin{prop}[Defective uniform log-Sobolev inequality]\label{prop:uniform_LSI}
Suppose \Cref{ass:network} holds and let $\mathfrak{c} > 0$. 
Let $\mu_\ast^{(N_x)} = \arg \min _{\mu^{(N_x)} \in \calP_2((\R^{d_x})^{N_x})} \scrL_\lambda^{(N_x)}(\mu^{(N_x)})$. 
Assume that $\sigma_1\sigma_2 \mathfrak{c} \geq 4R^3 \lambda $. 
Then, for any $\mu_x^{(N_x)} \in \calP_2((\R^{d_x})^{N_x})$, we have 
\begin{align*}
    \frac{\scrL_\lambda^{(N_x)} (\mu_x^{(N_x)})}{N_x} - \scrL_\lambda(\mu_{x,\lambda}^\ast, \tilde{\mu}_z^*(\mu_{x,\lambda}^\ast), \mu_z^*(\mu_{x,\lambda}^\ast)) \leq \frac{ \sigma_2 \fisher\left(\mu_x^{(N_x)}, \mu_\ast^{(N_x)} \right)}{C_{\LSI,x} N_x} + \frac{2\lambda R^2 (\frac{R}{\sigma_1} + 1)}{ N_x} + \frac{\mathfrak{c}^2 C_{\LSI,x}}{8}. 
\end{align*}
\end{prop}
The proof can be found in \Cref{sec:proof_uniform_LSI}. 
\Cref{prop:KL_upper_bound_functional} and \Cref{prop:uniform_LSI} serve as counterparts of the key inequalities Eq.~\eqref{eq:lem_1_nitanda_1} and Eq.~\eqref{eq:lem_1_nitanda_2}, which underpin the convergence analysis of mean-field Langevin dynamics (see \Cref{sec:mfld}). 

\begin{rem}
The main difficulty in our setting, as repeatedly emphasized, is the non-convexity of the functional $\mu_x \mapsto L_\lambda(\mu_x, \tilde{\mu}_z^\ast(\mu_x), \mu_z^\ast(\mu_x))$. Nevertheless, \Cref{lem:bregman} shows that this functional is weakly convex with respect to the total variation norm. This allows us to overcome the lack of convexity by imposing the condition $\sigma_1\sigma_2 \geq 2 R^3 \lambda \mathfrak{c}$, where a larger entropic regularizations $\sigma_1, \sigma_2$ improves convexity of the problem. 
Here $\mathfrak{c} > 0$ is a free slack parameter introduced in the analysis: smaller values yield sharper bounds but require stronger entropic regularization, while larger values loosen the bound but relax the condition (see \Cref{thm:convergence_outer_loop}). Overall, our results demonstrate that the analysis of MFLD extends to more general weakly convex functionals, provided an additional condition is imposed on the scales of the $\ell_2$ and entropic regularization.
\end{rem}

With the two propositions in place, we are now prepared to establish the convergence of the proposed algorithm \texttt{F$^2$BFLD} in the theorem below. 
\begin{theorem}[Convergence bound]\label{thm:convergence_outer_loop}
Suppose \Cref{ass:npiv} and \ref{ass:network} hold. Let $\mathfrak{c} > 0$ and assume that $\sigma_1\sigma_2 \mathfrak{c} \geq 4R^3 \lambda $. 
Suppose the step size $\gamma \leq \zeta_2^{-1}$. 
Denote $\calH(s) := N_x^{-1} \E[\scrL_\lambda^{(N_x)} (\mu_{x,s}^{(N_x)})] - \scrL_\lambda(\mu_{x,\lambda}^\ast, \tilde{\mu}_z^*(\mu_{x,\lambda}^\ast), \mu_z^*(\mu_{x,\lambda}^\ast))$ for any $s\in\N^+$ where the expectation is taken over the randomness of the initial i.i.d $N_x$ particles samples from $\mu_{x,0}$, the initial $N_z$ particles in each inner-loop, and the injected Gaussian noise at each iteration. 

For any number of iterations $S\in\N^+$, we have 
\begin{align}
    &\quad \calH(S) \lesssim \exp \left(-\frac{\sigma_2 C_{\LSI,x} S \gamma}{4} \right)\calH(0) + \frac{\lambda R^2 (\frac{R}{\sigma_1} + 1) }{N_x} + \frac{\lambda^2 R^4\left(\sqrt{\frac{\mathfrak{KL}}{N_z}} + \sqrt{\frac{\tilde{\mathfrak{KL}_\lambda}}{N_z}} + \frac{1}{N_z} \right)}{\sigma_2 C_{\LSI,x}} \nonumber \\
    &+ \frac{\lambda^2 R^4 + \zeta_2^2 + \frac{\lambda^2R^6}{\sigma_1}}{\sigma_2 C_{\LSI,x}} \left(\gamma^2 (\zeta_2^2 \E_{\mu_{x,0}}[\|x\|^2] + \lambda^2 R^2) + \gamma \sigma_2 d_x \right) + \mathfrak{c}^2 C_{\LSI,x} \label{eq:outer_loop_bound}.
\end{align}
$\mathfrak{KL}$ and $\tilde{\mathfrak{KL}}_\lambda$ in the last term of the first line above denote the KL upper bound of the convergence results in the inner-loop presented in \Cref{prop:inner} and \Cref{prop:tilde_inner}, respectively.  
\end{theorem}
The proof can be found in \Cref{sec:proof_convergence_outer_loop}. 
The convergence bound in Eq.~\eqref{eq:outer_loop_bound} consists of five terms. The first term, $\exp(-\sigma_2 C_{\LSI,x} S \gamma)\calH(0)$, decays exponentially fast with the number of iterations $S$, as a consequence of the uniform log-Sobolev inequality. 
Since the initial $N_x$ particles are sampled i.i.d from some distribution $\mu_{x,0}$, the initial approximation error term $\calH(0)$ with a fixed $\lambda$ is finite as long as $\mu_{x,0}$ has finite second moment and finite entropy. 
The second term, $\calO(N_x^{-1})$, accounts for the particle approximation error, while the fourth term, $\calO(\gamma^2 + \gamma \sigma_2)$, corresponds to the time discretization error. The third term, involving $\mathfrak{KL}$ and $\tilde{\mathfrak{KL}}_\lambda$, reflects the error from the inner loop, which is a unique term in our double-loop algorithm \texttt{F$^2$BMLD}. The inner-loop error arises here because the inner loop optima $\mu_z^\ast(\mu_x), \tilde{\mu}_z^\ast(\mu_x)$, that show up in computing the Wasserstein gradient of the outer loop, are approximated by the outputs of the inner-loop algorithm. 
Finally, the fifth term $\mathfrak{c}^2 C_{\LSI,x}$ is an artifact of the proof that arises due to the weak convexity. 
The parameter $\mathfrak{c} > 0$ acts as a slack variable: choosing a smaller $\mathfrak{c}$ yields sharper convergence bounds but requires stronger entropic regularization $(\sigma_1,\sigma_2)$ to satisfy the condition $\sigma_1\sigma_2 \mathfrak{c} \geq 4R^3 \lambda$, while larger $\mathfrak{c}$ relaxes this condition at the expense of a looser bound.

\begin{rem}[Uniform convergence of the neural network]
Our convergence bound in \Cref{thm:convergence_outer_loop} on $\calH(S)$ can be translated into an convergence bound on the neural network output via \Cref{prop:KL_upper_bound_functional} and Proposition 1 of \citet{nitanda2025propagation}. 
Define $\hat{h}_S(\ba) = \frac{1}{N_x}\sum_{i=1}^{N_x} \Psi_\ba(x_S^{(i)})$ where $\{x_S^{(i)}\}_{i=1}^{N_x}$ are $N_x$ particles which are the output of \emph{\texttt{F$^2$BMLD}}. Define $h_{\ast,\lambda}(\ba) = \smallint \Psi_\ba(x) \; \dd \mu_{x,\lambda}^\ast(x)$ where $\mu_{x,\lambda}^\ast$ is the global optimum of \eqref{eq:lagrangian} for a fixed $\lambda>0$. Then, for any $\ba\in\calX$, 
\begin{align}\label{eq:optimization_gap_nn}
    \E \left[\left( \hat{h}_S(\ba) - h_{\ast,\lambda}(\ba) \right)^2\right] \leq 8R^2 \sqrt{ \sigma_2^{-1} \calH(S) + \frac{R^3\lambda \mathfrak{c}}{\sigma_1\sigma_2} + N_x^{-1} \frac{R^3 \lambda}{\sigma_1\sigma_2} } + \frac{4 R^2}{N_x} .
\end{align}
Here, the expectation is taken over the randomness in the proposed algorithm \emph{\texttt{F$^2$BMLD}}. 
\end{rem}

\begin{rem}[Iteration and particle complexity of \texttt{F$^2$BMLD}] \label{rem:complexity_outer}
For simplicity, we consider $\lambda \geq 1$ and $C_{\LSI,x} \leq 1$ which are often met in practice. To reach $\E[(\hat{h}_S(\ba) - h_{\ast,\lambda}(\ba))^2] \leq \delta < 1$ with $\delta$ sufficiently small, it suffices to  reach $\calH(S) \leq \delta^2$ and hence suffices to take the slack parameter $\mathfrak{c} = \delta$. Then, it suffices to 
take the step size $\gamma \lesssim \delta^2 (\lambda^2 + \zeta_2^2 + \sigma_1^{-1} \lambda^2)^{-1} \sigma_2 C_{\LSI,x} d_x^{-1}$, which yields the following iteration and sample complexity: 
\begin{align*}
    S \geq \frac{\log(\delta^{-1})}{\delta^2} \frac{(\lambda^2 + \zeta_2^2 + \sigma_1^{-1} \lambda^2) d_x}{C_{\LSI,x}^2 \sigma_2^2}, \quad N_x \geq \frac{1}{\delta^2} \frac{\lambda(\sigma_1^{-1} + 1)}{\sigma_1 \sigma_2} ,
\end{align*}
and the following iteration and particle complexity of the inner loop as per \Cref{rem:complexity_inner}:
\begin{align*}
    T \geq \frac{\log (\delta^{-1})}{\delta^4} \frac{\lambda^4}{\sigma_2^4 C_{\LSI,x}^2} \frac{\zeta_1^2 d_z}{\sigma_1^2 C_{\LSI,z}^2}, \quad N_z \geq \frac{1}{\delta^4} \frac{\lambda^4}{\sigma_2^4 C_{\LSI,x}^2} \frac{1}{\sigma_1} .
\end{align*}
Note that the iteration complexities in both stages improve as the entropic  regularizations $\sigma_1, \sigma_2$ increase, and deteriorate as the log Sobolev constants $C_{\LSI,x}, C_{\LSI,z}$ decrease. The sample complexity of $N_x$ is independent of log Sobolev constants $C_{\LSI,x}, C_{\LSI,z}$ as a result of using the state-of-the-art propagation of chaos bound of MFLD from \citet{nitanda2025propagation}. 
\end{rem}

\section{Generalization of \texttt{F$^2$BMLD}} \label{sec:generalization}
In this section, we study the statistical properties of the optimal solution obtained via Lagrangian formulation \eqref{eq:lagrangian} when the objectives are computed with finite i.i.d samples from the joint data generating distribution $P$ over $(\ba,\by,\bw)$. 
Our analysis focuses on the \emph{generalization} error of the mean field network $h_{\lambda}^\ast :\ba \mapsto \int \Psi(\ba, x)\dd \mu_{x,\lambda}^\ast$ induced by the optimal solution $\mu_{x,\lambda}^\ast$ to \eqref{eq:lagrangian}. 
Together with the optimization error studied in \Cref{sec:convergence}, our analysis provides a \emph{complete} characterization of the performance of the proposed algorithm \texttt{F$^2$BMLD}. 

Denote the corresponding marginal distributions as $P_A, P_W, P_{WA}, P_{WY}$ and denote the conditional distribution as $P_{X\mid Z}$. 
In particular, given $m$ i.i.d samples $\{\bw_i, \ba_i\}_{i=1}^m\sim P_{WA}$ in stage I and $n$ i.i.d samples $\{\bw_i, \by_i\}_{i=1}^n\sim P_{WY}$ in stage II, the objectives $\scrF_1, \scrF_2$ in both stages now become 
\begin{align}
    &\scrF_1(\mu_x,\mu_z) = \sum_{i=1}^{m} \frac{1}{2m} \left( \smallint \Psi(\bw_i, z) \dd \mu_z - \smallint \Psi(\ba_i, x) \dd \mu_x \right)^2 + \frac{\zeta_1}{2} \E_{\mu_z}[\|z\|^2] + \sigma_1 \mathrm{Ent}(\mu_z) , \nonumber \\
    &\scrF_2(\mu_x,\mu_z) = \sum_{i=1}^{n} \frac{1}{2n} \left(\smallint \Psi(\bw_i, z) \dd \mu_z^\ast(\mu_x) - \by_i\right)^2 + \frac{\zeta_2}{2}  \E_{\mu_x}[\|x\|^2] + \sigma_2 \mathrm{Ent}(\mu_x) \label{eq:distribution_space_optimization_sample} .
\end{align}
Recall from Eq.~\eqref{eq:npiv} that the  conditional expectation operator 
$T : L^2(P_A) \to L^2(P_W)$ defined as $T:f\mapsto \E[f(A)\mid W]$. In the remainder of this section, we use the full notation $\Psi(\ba, x)$ and $\Psi(\bw, z)$ to emphasize the dependence on the network inputs $\ba, \bw$. 

To start with, we make a few assumptions on the regression targets in both stages. 

\begin{ass}[Stage II well-specifiedness] \label{ass:s2_well_posed}
The structural function $h_\circ$ belongs to a KL restricted Barron space $\calB_{M_x}:= \{ \smallint \Psi(\cdot, x) \dd \mu_x(x) \mid \kl(\mu_x, \nu_x)\leq M_x\}$, where $\nu_x=\calN(0, \zeta_2\sigma_2^{-1} \Id_{d_x})$. That is, there exists a measure $\mu_x^\circ \in \mathcal{B}_{M_x}$ such that $h_\circ(\cdot) = \smallint \Psi(\cdot, x) \dd \mu_x^\circ$.
\end{ass}

\begin{ass}[Stage I well-specifiedness]\label{ass:s1_well_posed}
For any $\mu_x$ with $\kl(\mu_x,\nu_x)\leq \kl(\mu_x^\circ,\nu_x) + 2\sigma_2^{-1}R^2$, the conditional expectation $T[\smallint \Psi(\cdot, x) \dd \mu_x(x)](\bw) =  \smallint \E[\Psi(A, x)\mid Z=\bw] \; \dd \mu_x(x)$ belongs to a KL restricted Barron space $\calB_{M_z}:= \{ \smallint \Psi(\cdot, z) \dd \mu_z(z) \mid \kl(\mu_z, \nu_z)\leq M_z\}$, where $\nu_z=\calN(0, \zeta_1\sigma_1^{-1} \Id_{d_z})$. 
That is, there exists a measure $\mu_z^\circ(\mu_x) \in \mathcal{B}_{M_z}$ such that $T[\smallint \Psi(\cdot, x) \dd \mu_x(x)](\bw) = \smallint \Psi(\cdot, z) \dd \mu_z^\circ(\mu_x)$. 
\end{ass}

\Cref{ass:s2_well_posed} is standard in studying the generalization error of \emph{two-layer mean field neural networks}~\citep{chen2020generalized, takakura2024mean}. 
The KL divergence upper bound $M_x$ quantifies the difficulty for a target function to be learned by a mean field neural network.  
\Cref{ass:s1_well_posed} states that for any suitably regular distribution $\mu_x$, its associated mean-field neural network, once smoothed by the compact operator $T$, can again be expressed as a mean-field neural network. 
To give a concrete example, suppose $d_x=d_z$ hence $\calP_2(\R^{d_x}) = \calP_2(\R^{d_z})$: if the conditional density of $P_{X\mid Z}$ is translation-invariant, i.e., $p(\ba\mid \bw)=p(\bw-\ba)$, then the condition in \Cref{ass:s1_well_posed} is satisfied with $\Psi(\bw, z) = \int \Psi(\ba, z) p(\bw-\ba) \dd \ba$, $\mu_z^\circ(\mu_x)=\mu_x$ and $M_z = M_x + 2\sigma_2^{-1}R^2$. 
The KL constraint on $\mu_x$ in \Cref{ass:s1_well_posed} arises from \Cref{lem:kl_bound_assumption}, which proves that $\kl(\mu_{x,\lambda}^\ast, \nu_x) \leq \kl(\mu_{x}^\circ, \nu_x) + 2 \sigma_2^{-1} R^2$ for any $\lambda > 0$. 
This KL constraint is necessary as it rules out irregular $\mu_x$, such as dirac delta distributions. 

% \begin{ass}[Bounded target $Y$] \label{ass:subgaussian}
% The residual $Y - (Th_\circ)(Z) \in [-\varsigma, \varsigma]$. 
% \end{ass}
% Since $h_\circ$ is uniformly bounded as in \Cref{ass:npiv}, together with \Cref{ass:subgaussian}, this implies that the squared loss $\hat{y} \mapsto (\hat{y} - y)^2$ is Lipschitz for any $y \in \mathcal{Y}$. This Lipschitz property simplifies the analysis by enabling the use of the Talagrand's contraction
% lemma~\citep[Lemma 5.7]{mohri2018foundations}. More generally, this assumption can be relaxed to requiring that the residual $Y - (T h_\circ)(Z)$ is subgaussian, in which case $\hat{y} \mapsto (\hat{y} - y_i)^2$ remains Lipschitz by showing that $\max_{i\in \{1, \ldots, n\}} |y_i|$ is $\calO(\log n)$ with high probability~\citep[Exercise 2.5.10]{vershynin2018high}. 
% This would result in an extra $\log n$ factor in the stage II generalization error bound. 

\begin{theorem}[Generalization bound]\label{thm:stage_2_generalization}
Suppose \Cref{ass:npiv}, \ref{ass:network}, \ref{ass:s2_well_posed} and \ref{ass:s1_well_posed} hold. 
For $\lambda > 0$, let $\mu_{x,\lambda}^\ast$ be the optimal solution to the Lagrangian problem \eqref{eq:lagrangian} and $h_{\ast,\lambda}(\ba) = \smallint \Psi(\ba, x) \; \dd \mu_{x,\lambda}^\ast(x)$ be its associated mean field neural network. 
Then, with $P^{\otimes (m+n)}$ probability at least $1-8\delta$, 
\begin{align*}
    &\E_{P_W}\left[ \Big( (Th_{\ast,\lambda})(W) - (Th_\circ)(W) \Big)^2 \right] \lesssim \sigma_2 M_x + \sigma_1 M_z + \frac{R^2(R+M)^2}{\sigma_1 \lambda} +R^2 \sqrt{\frac{\log(\delta^{-1})}{m}} \\
    &\quad + R^2 \sqrt{\frac{M_z + \frac{R^2}{\sigma_1}}{m}}+ (R + M)^2 \sqrt{\frac{\log(\delta^{-1})}{n}} + R(R+M) \sqrt{\frac{M_x+\frac{R^2}{\sigma_2}}{n}} . 
\end{align*}
\end{theorem}

The proof can be found in \Cref{sec:proof_generalization}. The proof is a non-trivial adaptation of existing generalization bounds of mean field neural networks into our setting of bilevel optimization and its Lagrangian formulation, which requires careful control of the interaction between two stages. The
final generalization bound can be dissected into four components: $\sigma_2 M_x + \sigma_1 M_z$ captures the increased complexity of the KL-restricted Barron spaces in \Cref{ass:s2_well_posed} and \ref{ass:s1_well_posed}; $\calO(\frac{1}{\sigma_1 \lambda})$ reflects the approximation error introduced by the Lagrangian formulation shown in \Cref{thm:relation}; and $1/\sqrt{n}, 1/\sqrt{m}$ correspond to the sample complexities of stage I and stage II, respectively. 

Some of the existing generalization results for 2SLS, either with fixed features~\citep{chen2018optimal,meunier2024nonparametric} or with adaptive features~\citep{kim2025optimality}, are expressed in terms of the unprojected norm $\|\cdot\|_{L^2(P_A)}$. In contrast, our bound is stated with respect to the projected norm $\|T(\cdot)\|_{L^2(P_W)}$ which is a weaker metric because $T$ is bounded. 
Such a bound in weaker metric is to be expected, since we do not impose these strong structural assumptions on $T$—such as measures of ill-posedness~\citep{chen2018optimal} or link conditions~\citep{chen2011rate}—which are generally difficult to verify in practice.

\begin{rem}[Trade-off on $\lambda$ between optimization and generalization]
A closer inspection of the role of the Lagrange multiplier $\lambda$ in the optimization bound of \Cref{thm:convergence_outer_loop} and the generalization bound of \Cref{thm:stage_2_generalization} reveals a clear \emph{trade-off}. For the optimization bound, smaller values of $\lambda$ are preferable, as they yield a weaker convexity parameter (\Cref{lem:bregman}) and smaller Lipschitz constants, thereby reducing both the time-discretization error and the contribution from the inner-loop error. In contrast, for the generalization bound, larger values of $\lambda$ are favorable, since they make the Lagrangian relaxation more faithful to the original bilevel optimization problem (\Cref{thm:relation}). Taken together, the optimization bound of \Cref{thm:convergence_outer_loop} and the generalization bound of \Cref{thm:stage_2_generalization} lead to a combined error bound. Due to the inherent trade-off between these two effects, we do not expect any choice of $(\sigma_1,\sigma_2,\lambda)$ to eliminate the total error. 
\end{rem}
%%%%%%%%%%%%%%%%%%
%%%%%%%%%%%%%%%%%%%%

\section{Experiments}\label{sec:experiments}

In this section, we empirically evaluate our proposed method, \texttt{F$^2$BMLD}, on the offline policy evaluation (OPE) problem, a fundamental challenge in reinforcement learning~\citep{sutton1998reinforcement,levine2020offline}. 
As early as in \citet{bradtke1996linear}, it was observed that two-stage least squares (2SLS)—originally developed for instrumental variable regression—can also be applied to estimate the value function in offline reinforcement learning. More recently, OPE has become a standard benchmark for evaluating 2SLS algorithms, either with fixed or adaptive features, as it presents a greater challenge than synthetic causal inference datasets~\citep{xu2021deep,chen2022instrumental}.  

Formally, consider a reinforcement learning environment $\langle \mathbb{S}, \mathbb{B}, P, R, \nu_0, \eta \rangle$, where $\mathbb{S}$ is the state space, $\mathbb{B}$ is the action space, $P: \mathbb{S} \times \mathbb{B} \times \mathbb{S} \to [0,1]$ is the transition kernel, $R: \mathbb{S} \times \mathbb{B} \times \mathbb{S} \times \mathbb{R} \to \mathbb{R}$ is the reward distribution, $\nu_0: \mathbb{S} \to [0,1]$ is the initial state distribution, and $\eta \in (0,1]$ is the discount factor. A policy $\pi$ is defined such that $\pi(b \mid s)$ is the probability of selecting action $b$ in state $s \in \mathbb{S}$. Given a policy $\pi$, the \emph{$Q$-function} is defined as
\begin{align*}
    Q^\pi(s, b) = \E\left[\sum_{t=0}^{\infty} \eta^t r_t \mid s_0=s, b_0=b\right]
\end{align*}
with $b_t \sim \pi\left(\cdot \mid s_t\right), s_{t+1} \sim P\left(\cdot \mid s_t, b_t\right), r_t \sim R\left(\cdot \mid s_t, b_t, s_{t+1} \right)$. 
The goal of offline policy evaluation to estimate the expected $Q$-value of a given target policy $\pi$ under the initial state distribution, also known as the policy value: 
\begin{align*}
    V(\pi) = \E_{s \sim \nu_0, b \mid s \sim \pi}\left[Q^\pi(s, b)\right] . 
\end{align*}
The challenge of OPE, as suggested by its name, is that direct interaction with the environment is not permitted. Instead, one must rely on an existing pre-collected dataset of trajectories tuples $\left(s, b, r, s^{\prime}\right)$ to estimate the policy value, and potentially to deduce an optimal policy. Such \emph{offline} datasets are typically generated by one or more  unknown behavior policies $\pi_b$. One popular family of OPE approaches is to estimate the value function based on the Bellman equation~\citep{sutton1998reinforcement},
\begin{align}\label{eq:bellman}
    \E[r \mid s, b] = Q(s, b) - \eta \E[Q(s^{\prime}, b^{\prime}) \mid s, b], 
\end{align}
where the first expectation is taken with respect to the reward distribution, while the second expectation is taken with respect to the policy $\pi$ and the transition kernel $P$. Notably, Eq.\eqref{eq:bellman} has the same structure as the conditional moment equations in NPIV (Eq.\eqref{eq:npiv}) and hence can be solved via 2SLS: 
the conditional expectation operator $T$, induced by $\pi(b'\mid s') \times P(s'\mid s, b)$ would be learned by samples generated by the behavior policy $\pi_b$. 

\begin{figure}[t]
    \centering
    \begin{subfigure}[t]{0.55\textwidth}
        \centering
        \includegraphics[width=\textwidth]{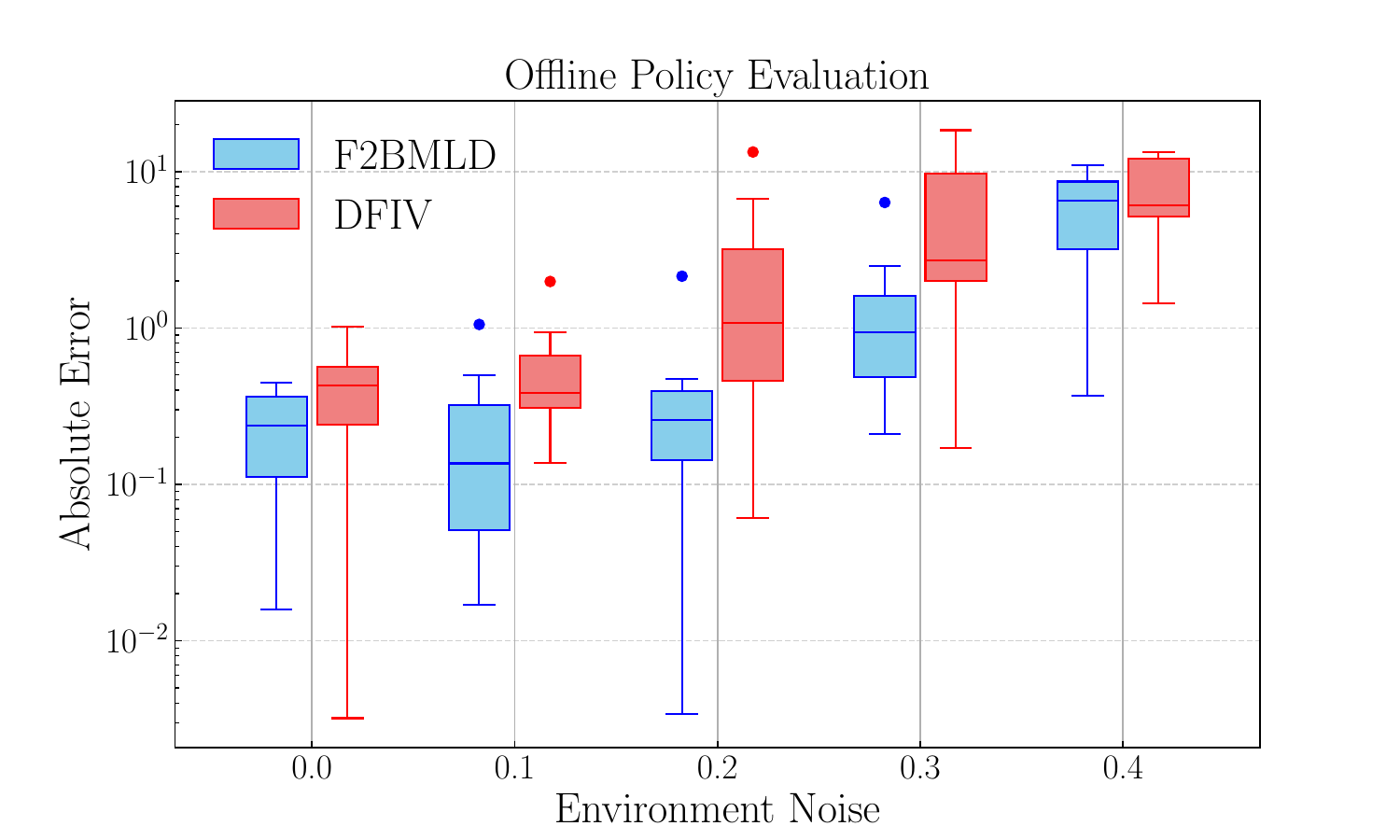}
        \label{fig:box}
    \end{subfigure}
    \hfill
    \begin{subfigure}[t]{0.44\textwidth}
        \centering
        \includegraphics[width=\textwidth]{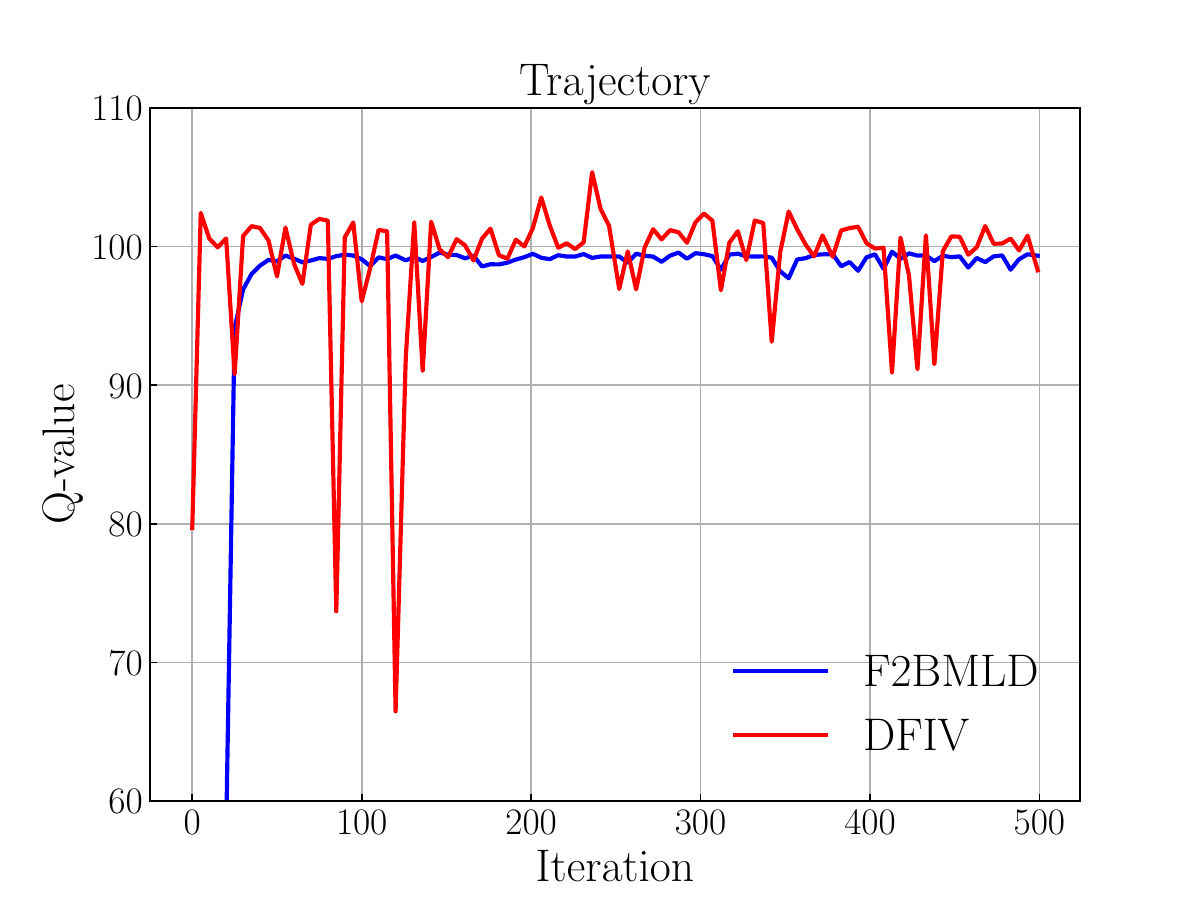}
        \label{fig:trajectory}
    \end{subfigure}
    \label{fig:experiment}
    \caption{\textbf{Left:} Comparison of DFIV and F2BMLD in terms of target policy value. \textbf{Right:} Comparison of DFIV and F2BMLD training trajectories.}
\end{figure}

We evaluate our proposed \texttt{F$^2$BMLD} on Cartpole where an agent can move a cart left/right on a plane to keep a balanced
pole upright~\citep{osbandbehaviour}. 
The original system dynamics are deterministic. To create a stochastic environment, we randomly replace the agent action by a uniformly sampled action with probability $p \in \{0, 0.1, 0.2, 0.3, 0.4\}$. The noise level $p$ controls the level of confounding effect. The target policy is trained with DQN~\citep{mnih2015human}, and an offline dataset for OPE is generated by executing the policy in the same environment with an additional random action probability of $0.1$ (applied on top of the environment’s randomization $p$). We primarily compare \texttt{F$^2$BMLD} against DFIV, which represents the state of the art in 2SLS with adaptive neural network features. 

For \texttt{F$^2$BMLD}, we use a learning rate of $10^{-4}$ for both inner and outer loops, set the Lagrange multiplier to $\lambda=0.3$ (following \citet{shen2023penalty}), use a batch size of $32$, and apply $\ell_2$ regularization $\zeta_1=\zeta_2=10^{-5}$ together with noise regularization $\sigma_1=\sigma_2=10^{-2}$. The inner loop is trained for $10$ steps per outer iteration, and the outer loop is trained for $50{,}000$ iterations, at which point convergence is observed. 
For DFIV, we adopt the same hyperparameter settings as \citet{chen2022instrumental}\footnote{The original code provided in \citet{chen2022instrumental} relies on old versions of tensorflow and acme which are not compatible with the latest versions. We implemented DFIV from scratch with OpenAI Gymnasium in our repository.}. 
To retain consistency, we use a two-layer neural network with hidden layer of width $50$ for both methods.  
The code to reproduce all the results can be found in \url{https://github.com/hudsonchen/F2BMLD}.

The empirical results are summarized in \Cref{fig:experiment}. From \Cref{fig:box}, we observe that \texttt{F$^2$BMLD} achieves comparable, and in some cases smaller, absolute error than DFIV when estimating the target policy value. \Cref{fig:trajectory} further shows that \texttt{F$^2$BMLD} exhibits a more stable training trajectory than DFIV. The instability of DFIV is likely caused by propagating the gradient through the ridge regression solution for the final layer. 
An additional advantage of \texttt{F$^2$BMLD} is that it allows a smaller batch size during training than DFIV~\citep{xu2021learning}. 
In our experiments, \texttt{F$^2$BMLD} used a batch size of $32$, whereas DFIV required a batch size of $1024$~\citep{xu2021learning}. Although our experiments are conducted with two-layer networks, the \texttt{F$^2$BMLD} algorithm can, in principle, be extended to deeper neural networks, albeit without theoretical guarantees from mean-field Langevin dynamics. 
Given the theoretical focus of this paper, we leave a more extensive empirical evaluation on larger and more challenging reinforcement learning benchmarks to future work.

\section{Conclusion}\label{sec:conclusion}
We introduced \texttt{F$^2$BMLD}, a fully first-order bilevel mean-field Langevin dynamics algorithm derived from a Lagrangian reformulation of bilevel optimization. By adopting a lifted perspective over the space of probability measures, we established global convergence guarantees under an extra condition on the noise regularizations $\sigma_1, \sigma_2$ to overcome the non-convexity of the outer level objective. We also provided a generalization bound of the global optimum under access to i.i.d samples, which reveals a trade-off on the Lagrange multiplier $\lambda$.  

Following our work, several interesting open problems remain.
(i) Our analysis can, in principle, be extended to study the convergence of standard mean-field Langevin dynamics with weakly convex functionals, thereby allowing a broader class of loss functions.
(ii) A more comprehensive empirical evaluation of \texttt{F$^2$BMLD} on challenging benchmarks is warranted. While our theory focuses on two-layer neural networks, the algorithm itself readily extends to deeper architectures.
(iii) The lifted perspective provided by mean-field Langevin dynamics could also be applied to study the optimization and generalization of neural networks in NPIV beyond 2SLS, including min–max estimation~\citep{dikkala2020minimax}.

\section*{Acknowledgement}
ZC was supported by the Engineering and Physical Sciences Research Council (ESPRC) through grants
[EP/S021566/1]. 
AG was supported by the Gatsby Charitable Foundation. 
TS was partially supported by JSPS KAKENHI (24K02905, 25H01107) and JST CREST (JPMJCR2015). 
AN is supported by the National Research Foundation, Singapore, Infocomm Media Development Authority under its Trust Tech Funding Initiative, and the Ministry of Digital Development and Information under the AI Visiting Professorship Programme (award number AIVP-2024-004). Any opinions, findings and conclusions or recommendations expressed in this material are those of the author(s) and do not reflect the views of National Research Foundation, Singapore, Infocomm Media Development Authority, and the Ministry of Digital Development and Information.

\section{Proofs}\label{sec:proof}
\subsection{Proofs in \Cref{sec:bi-mfld}}\label{sec:proof_first_part}
\subsubsection{Proof of \Cref{prop:existence}}\label{sec:proof_prop_43}
\begin{proof}[Proof of \Cref{prop:existence}]
    Given the partial convexity of $\mu_z\mapsto U_1(\mu_x, \mu_z)$, the properties of the stage I solution $\mu_z^\ast(\mu_x)$ have been proved in Proposition 2.5 of \citet{hu2021mean}, for any $\mu_x$. We focus on the properties of the stage II solution $\mu_x^\ast$. Note that the mapping $\mu_x\mapsto U_2(\mu_z^\ast(\mu_x))$ is continuous under the weak convergence topology. To see why, notice that
\begin{align}
    |U_2(\mu_{z}^{\ast}(\mu_{x}')) - U_2(\mu_{z}^{\ast}(\mu_{x}))| &\stackrel{(i)}{\leq} (R+M) \E_\rho \left[\left| \smallint \Psi_\bw \; \dd \mu_{z}^{\ast}(\mu_{x}') - \smallint \Psi_\bw \; \dd \mu_{z}^{\ast}(\mu_{x}) \right| \right] \nonumber \\
    &\stackrel{(ii)}{\leq} R(R+M) \tv(\mu_{z}^{\ast}(\mu_{x}'), \mu_{z}^{\ast}(\mu_{x})) \nonumber \\
    &\stackrel{(iii)}{\leq} R(R+M) \sqrt{\kl(\mu_{z}^{\ast}(\mu_{x}'), \mu_{z}^{\ast}(\mu_{x}))} \nonumber\\
    &\stackrel{(iv)}{\leq} R(R+M) \sqrt{(4\sigma_1)^{-1}} \sqrt{\E_{\rho} \left[ \left( \smallint \Psi_\ba \; \dd (\mu_x - \mu_x^\prime) \right)^2 \right]} \nonumber\\
    &\stackrel{(v)}{\leq} R^2(R+M) \sqrt{(4\sigma_1)^{-1}} W_2(\mu_x,\mu_x^\prime) \label{eq:continuity_U_2}.
\end{align}
In the above derivations, $(i)$ holds by \Cref{lem:lip_wass_grad}; $(ii)$ holds by \Cref{ass:network} that $\Psi_\bw$ is bounded by $R$; $(iii)$ holds by Pinsker’s inequality; $(iv)$ holds by \Cref{prop:continuity} and $(v)$ holds by \Cref{ass:network} that the gradient $\Psi_\bw$ is bounded by $R$.

Since the negative entropy is lower-semicontinuous under the weak convergence topology~\citep[Lemma 1.4.3]{dupuis2011weak}, the mapping $\mu_x \mapsto \scrF_2(\mu_x, \mu_{z}^{\ast}(\mu_{x})) = U_2(\mu_{z}^{\ast}(\mu_{x})) + \frac{\zeta}{2} \E_{\mu_x}[\|x\|^2] + \sigma_1 \mathrm{Ent}(\mu_x)$ is also lower-semicontinuous. 
Clearly there exists $\bar{\mu}_x\in\calP(\R^{d_x})$ such that $\scrF_2(\bar{\mu}_x, \mu_{z}^{\ast}(\bar{\mu}_{x})) = \mathfrak{M} <\infty$. Consider the following subset of $\calP_2(\R^{d_x})$ 
\begin{align*}
    \mathcal{S}:=\left\{\mu_x: \sigma_1 \mathrm{Ent}(\mu_x) \leq \mathfrak{M} -\inf_{\mu_x^{\prime} \in \mathcal{P}(\mathbb{R}^{d_x})} F_2(\mu_x^{\prime}, \mu_z^\ast(\mu_x^\prime)) \right\} .
\end{align*} 
As a sublevel set of the negative entropy, $\mathcal{S}$ is weakly compact, see e.g. \citet[Lemma 1.4.3]{dupuis2011weak}. Together with the lower semi-continuity of $\mu_x \mapsto \scrF_2(\mu_x, \mu_{z}^{\ast}(\mu_{x}))$, the minimum on $\mathcal{S}$ is attained. Notice that for all $\mu_x \notin \mathcal{S}$, we have $\scrF_2(\mu_x, \mu_{z}^{\ast}(\mu_{x})) \geq \scrF_2(\bar{\mu}_x, \mu_{z}^{\ast}(\bar{\mu}_{x}))$, so the minimum on $\mathcal{S}$ coincides with the global minimum. However, unlike the stage I solution, the stage II solution $\mu_x^\ast$ may be non-unique due to the lack of convexity of the nested mapping $\mu_x\mapsto U_2(\mu_{z}^{\ast}(\mu_{x}))$. 
Finally, to ensure $\mathrm{Ent}(\mu_x) < \infty$ and $\E_{\mu_x}[\|x\|^2]<\infty$, the solution $\mu_x^\ast$ is absolutely continuous with respect to Lebesgue measure, and belongs to $\calP_2(\R^{d_x})$. 
\end{proof}

\subsubsection{Proof of \Cref{prop:gamma_convergence}}\label{sec:proof_prop_44}
\begin{proof}[Proof of \Cref{prop:gamma_convergence}]
From Proposition 2.3 of \citet{hu2021mean}, for any $\mu_x\in\calP_2(\R^{d_x})$, we know that $\lim_{\sigma_1\to0} W_2(\mu_{z,\sigma_1}^{\ast}(\mu_{x}), \mu_{z,0}^{\ast}(\mu_{x})) = 0$. From  \Cref{lem:lip_wass_grad}, we know that 
$| U_2(\mu_{z,\sigma_1}^{\ast}(\mu_{x})) - U_2(\mu_{z,0}^{\ast}(\mu_{x}))| \leq R(R+M) W_2(\mu_{z,\sigma_1}^{\ast}(\mu_{x}), \mu_{z,0}^{\ast}(\mu_{x}))$. 
So we have
\begin{align}\label{eq:gamma_convergence_1}
    \lim_{\sigma_1\to0} 
    \left| U_2(\mu_{z,\sigma_1}^{\ast}(\mu_{x})) - U_2(\mu_{z,0}^{\ast}(\mu_{x})) \right| = 0. 
\end{align}
In the meanwhile, from the continuity of $\mu_x\mapsto U_2(\mu_z^\ast(\mu_x))$ in terms of the weak convergence topology proved in Eq.~\eqref{eq:continuity_U_2}, for a fixed $\sigma_1$ and a sequence $(\mu_{x,n})_{n\in\N^+}$ converging weakly to $\mu_x$, we have 
\begin{align}\label{eq:gamma_convergence_2}
    \lim_{n\to+\infty} U_2(\mu_{z,\sigma_1}^{\ast}(\mu_{x,n})) = U_2(\mu_{z,\sigma_1}^{\ast}(\mu_{x})).
\end{align} 
Now we combine Eq.~\eqref{eq:gamma_convergence_1} and Eq.~\eqref{eq:gamma_convergence_2}: for a positive sequence $(\sigma_{1,n})_{n\in\N^+}$ decreasing to $0$ and any sequence of distributions $(\mu_{x,n})_{n\in\N^+}$ converging weakly to $\mu_x$, we have
\begin{align*}
    \lim_{n\to+\infty} U_2(\mu_{z,\sigma_{1,n}}^{\ast}(\mu_{x,n})) = U_2(\mu_{z,0}^{\ast}(\mu_{x})). 
\end{align*}
Therefore, for two positive sequences that converge to $0$: $(\sigma_{1,n})_{n\in\N^+}\to 0$ and $(\sigma_{2,n})_{n\in\N^+}\to 0$, we have 
$\lim\inf_{n \to\infty} \scrF_{2, (\sigma_{1,n}, \sigma_{2,n})}(\mu_{x,n}) \geq \lim_{n \to\infty} \scrF_{2, (\sigma_{1,n},0)}(\mu_{x,n}) = \scrF_{2, (0,0)}(\mu_x)$. On the other hand, we construct a sequence of distributions $(\mu_x \ast r_n)_{n\in\N^+}$ where $r_n(x)=\sigma_{2,n}^{-d_x/2} r(x / \sqrt{\sigma_{2,n}})$ and $r$ is the heat kernel. Following the same derivations as in Proposition 2.3 of \citet{hu2021mean}, we have $\lim\sup_{n \rightarrow+\infty} \scrF_{2, (\sigma_{1,n}, \sigma_{2,n})} (\mu_x * r_n)\leq \scrF_{2, (0,0)}(\mu_x)$. So we have concluded the proof of $\Gamma$-convergence. 
\end{proof}

\subsubsection{Proof of \Cref{thm:relation}}\label{sec:proof_relation}
\begin{proof}[Proof of \Cref{thm:relation}]
Recall that $\scrF_2(\mu_x,\mu_z)=U_2(\mu_z)+\frac{\zeta}{2}\E_{\mu_x}[\|x\|^2] + \sigma_2 \mathrm{Ent}(\mu_x)$. From \Cref{lem:lip_wass_grad}, we have, for any fixed $\mu_x\in\calP_2(\R^{d_x})$, 
\begin{align*}
    \scrF_2(\mu_x, \mu_z) - \scrF_2(\mu_x, \mu_z^\ast(\mu_x)) &\geq -(R+M) \E_\rho \left[\left| \smallint \Psi_\bw \; \dd \mu_{z} - \smallint \Psi_\bw \; \dd \mu_z^\ast(\mu_x) \right| \right] \\
    &\geq -(R+M)R \cdot \tv(\mu_z,\mu_z^\ast(\mu_x)) .
\end{align*}
The last inequality holds because $\Psi_\bw$ is bounded from \Cref{ass:network}. 
From the convexity of $U_2$ proved in \Cref{prop:partial_convex_simple} and the entropy sandwich theorem proved in Lemma 3.4 of \citet{chizatmean}, we have, for any fixed $\mu_x\in\calP_2(\R^{d_x})$ and any $\mu_z\in\calP_2(\R^{d_z})$, 
\begin{align*}
    \scrF_1(\mu_x, \mu_z) - \scrF_1(\mu_x, \mu_z^\ast(\mu_x)) \geq \sigma_1 \kl(\mu_z, \mu_z^\ast(\mu_x)).
\end{align*}
Combine the above two inequalities, for any $\mu_x,\mu_z$, we obtain
\begin{align}\label{eq:key_inequality_of_bilevel_lag}
    &\quad \scrF_2(\mu_x, \mu_z) - \scrF_2(\mu_x, \mu_z^\ast(\mu_x)) + \lambda \left( \scrF_1(\mu_x, \mu_z) - \scrF_1(\mu_x, \mu_z^\ast(\mu_x)) \right) \\
    &\geq -R(R+M)\tv(\mu_z, \mu_z^\ast(\mu_x)) + \lambda \sigma_1 \kl(\mu_z, \mu_z^\ast(\mu_x)) \nonumber \\
    &\geq -R(R+M)\tv(\mu_z, \mu_z^\ast(\mu_x)) + 2\lambda\sigma_1 \tv^2(\mu_z, \mu_z^\ast(\mu_x)) \nonumber \\
    &\geq \min_{t \geq0} - R(R+M) t + 2\lambda \sigma_1 t^2 = -\frac{R^2(R+M)^2}{8\lambda\sigma_1} \nonumber .
\end{align}
The second last inequality holds from the Pinsker's inequality. Therefore, for any $\mu_x,\mu_z$, we have
\begin{align*}
    &\quad \scrL_\lambda(\mu_x^\ast, \mu_z^\ast(\mu_x^\ast)) - \scrL_\lambda(\mu_x,\mu_z) \\
    &= \scrF_2(\mu_x^\ast, \mu_z^\ast(\mu_x^\ast)) - \scrF_2({\mu_x}, {\mu_z}) - \lambda \left( \scrF_1(\mu_x, \mu_z)  - \scrF_1(\mu_x, \mu_z^\ast(\mu_x)) \right) \\
    &\leq \scrF_2(\mu_x, \mu_z^\ast(\mu_x)) - \scrF_2({\mu_x}, {\mu_z}) - \lambda \left( \scrF_1(\mu_x, \mu_z)  - \scrF_1(\mu_x, \mu_z^\ast(\mu_x)) \right) \\
    &\leq \frac{R^2(R+M)^2}{8\lambda\sigma_1} .
\end{align*}
So we have proved the first claim. 

Now we are going to prove the second claim.  
By Eq.~\eqref{eq:key_inequality_of_bilevel_lag} and that $\frac{R^2(R+M)^2}{8\lambda_0\sigma_1} =\epsilon_1$, there is
\begin{align}
    &\quad \scrF_2(\mu_x, \mu_z) - \scrF_2(\mu_x^\ast, \mu_z^\ast(\mu_x^\ast)) + \lambda_0 \left( \scrF_1(\mu_x, \mu_z) - \scrF_1(\mu_x, \mu_z^\ast(\mu_x)) \right) \nonumber \\
    &\geq \scrF_2(\mu_x, \mu_z) - \scrF_2(\mu_x, \mu_z^\ast(\mu_x)) + \lambda_0 \left( \scrF_1(\mu_x, \mu_z) - \scrF_1(\mu_x, \mu_z^\ast(\mu_x)) \right) \geq -\epsilon_1 \label{eq:interm_connection}.
\end{align}
From the $\epsilon_2$-global-optimality of $({\mu_x}_{\lambda}^{(\epsilon_2)}, {\mu_z}_{\lambda}^{(\epsilon_2)})$, there is
\begin{align*}
&\quad \scrF_2({\mu_x}_{\lambda}^{(\epsilon_2)}, {\mu_z}_{\lambda}^{(\epsilon_2)}) + \lambda \left( \scrF_1({\mu_x}_{\lambda}^{(\epsilon_2)}, {\mu_z}_{\lambda}^{(\epsilon_2)}) - \scrF_1({\mu_x}_{\lambda}^{(\epsilon_2)}, \mu_z^\ast({\mu_x}_{\lambda}^{(\epsilon_2)})) \right) - \epsilon_1 \\
& \leq \scrF_2(\mu_x^\ast, \mu_z^\ast(\mu_x^\ast)) - \epsilon_1 + \epsilon_2 \quad \big(\text { since } \scrF_1(\mu_x^\ast, \mu_z^\ast(\mu_x^\ast)) - \scrF_1(\mu_x^\ast, \mu_z^\ast(\mu_x^\ast))=0 \big) \\
& \leq \scrF_2({\mu_x}_{\lambda}^{(\epsilon_2)}, {\mu_z}_{\lambda}^{(\epsilon_2)}) + \lambda_0 \left( \scrF_1({\mu_x}_{\lambda}^{(\epsilon_2)}, {\mu_z}_{\lambda}^{(\epsilon_2)}) - \scrF_1({\mu_x}_{\lambda}^{(\epsilon_2)}, \mu_z^\ast({\mu_x}_{\lambda}^{(\epsilon_2)})) \right) +\epsilon_2 .
\end{align*}
where the last inequality holds by Eq.~\eqref{eq:interm_connection}. 
Therefore, by rearranging the terms in the above inequality, we have
\begin{align*}
    &\quad (\lambda-\lambda_0) \left( \scrF_1({\mu_x}_{\lambda}^{(\epsilon_2)}, {\mu_z}_{\lambda}^{(\epsilon_2)}) - \scrF_1({\mu_x}_{\lambda}^{(\epsilon_2)}, \mu_z^\ast({\mu_x}_{\lambda}^{(\epsilon_2)})) \right) \leq \epsilon_1 + \epsilon_2 \\
    &\Rightarrow  \quad \varepsilon := \scrF_1({\mu_x}_{\lambda}^{(\epsilon_2)}, {\mu_z}_{\lambda}^{(\epsilon_2)}) - \scrF_1({\mu_x}_{\lambda}^{(\epsilon_2)}, \mu_z^\ast({\mu_x}_{\lambda}^{(\epsilon_2)})) \leq \frac{\epsilon_1 + \epsilon_2}{\lambda-\lambda_0} . 
\end{align*}
Now we have proved that $({\mu_x}_{\lambda}^{(\epsilon_2)}, {\mu_z}_{\lambda}^{(\epsilon_2)})$ satisfy the constraint in \eqref{eq:penalty_formulation}. 
In the meantime, consider $(\mu_x,\mu_z)$ that are the global solution to the penalty formulation \eqref{eq:penalty_formulation}, we have
\begin{align*}
    &\quad \scrF_2({\mu_x}_{\lambda}^{(\epsilon_2)}, {\mu_z}_{\lambda}^{(\epsilon_2)}) - \scrF_2(\mu_x, \mu_z) \\
    &\leq -\lambda\left( \scrF_1({\mu_x}_{\lambda}^{(\epsilon_2)}, {\mu_z}_{\lambda}^{(\epsilon_2)}) - \scrF_1({\mu_x}_{\lambda}^{(\epsilon_2)}, \mu_z^\ast({\mu_x}_{\lambda}^{(\epsilon_2)}) \right) + \lambda \left( \scrF_1(\mu_x, \mu_z) - \scrF_1(\mu_x, \mu_z^\ast(\mu_x) \right) + \epsilon_2 \\
    &\leq \lambda(\scrF_1(\mu_x, \mu_z) - \scrF_1(\mu_x, \mu_z^\ast(\mu_x) - \varepsilon) + \epsilon_2 \leq \epsilon_2 . 
\end{align*}
The first inequality holds from the $\epsilon_2$-global-optimality of $({\mu_x}_{\lambda}^{(\epsilon_2)}, {\mu_z}_{\lambda}^{(\epsilon_2)})$, and the last inequality holds from the fact that $(\mu_x,\mu_z)$ are the global solution to the penalty formulation \eqref{eq:penalty_formulation}. 
So we have proved that $({\mu_x}_{\lambda}^{(\epsilon_2)}, {\mu_z}_{\lambda}^{(\epsilon_2)})$ is the $\epsilon_2$-global optimum of \eqref{eq:penalty_formulation}, which concludes the proof of the second claim.

Now we are going to prove the third claim. 
Let $(\mu_x^{(\epsilon_3)}, \mu_z^{(\epsilon_3)})$ be the $\epsilon_3$-global-solution to the constrained problem \eqref{eq:penalty_formulation}. 
Then, we have
\begin{align*}
    \left| \scrF_2(\mu_x^{(\epsilon_3)}, \mu_z^\ast(\mu_x^{(\epsilon_3)})) - \scrF_2(\mu_x^{(\epsilon_3)}, \mu_z^{(\epsilon_3)}) \right| &= \left| U_2(\mu_z^\ast(\mu_x^{(\epsilon_3)})) - U_2(\mu_z^{(\epsilon_3)}) \right| \\
    &\leq R(R+M) \cdot \tv(\mu_z^{(\epsilon_3)}, \mu_z^\ast(\mu_x^{(\epsilon_3)})). 
\end{align*}
Next, since $\mu_z^{(\epsilon_3)}$ satisfies the constraint in \eqref{eq:penalty_formulation}, we have
\begin{align*}
    \varepsilon \geq \scrF_1(\mu_x^{(\epsilon_3)}, \mu_z^{(\epsilon_3)}) - \scrF_1(\mu_x^{(\epsilon_3)}, \mu_z^\ast(\mu_x^{(\epsilon_3)})) \geq \sigma_1 \kl(\mu_z^{(\epsilon_3)}, \mu_z^\ast(\mu_x^{(\epsilon_3)})) \geq 2\sigma_1 \tv^2(\mu_z^{(\epsilon_3)}, \mu_z^\ast(\mu_x^{(\epsilon_3)})) . 
\end{align*}
Combine the above two inequalities, we achieve
\begin{align*}
    \left| \scrF_2(\mu_x^{(\epsilon_3)}, \mu_z^\ast(\mu_x^{(\epsilon_3)})) - \scrF_2(\mu_x^{(\epsilon_3)}, \mu_z^{(\epsilon_3)}) \right| \leq R(R+M) \sqrt{(2\sigma_1)^{-1} \varepsilon} . 
\end{align*}
Note that $(\mu_x^\ast,\mu_z^\ast(\mu_x^\ast))$ the global optimum of \eqref{eq:distribution_space_optimization} satisfies the constraint in \eqref{eq:penalty_formulation}. By definition of $\epsilon_3$-global-minimum, so we have
\begin{align*}
    \scrF_2(\mu_x^{(\epsilon_3)}, \mu_z^\ast(\mu_x^{(\epsilon_3)})) - R(R+M) \sqrt{(2\sigma_1)^{-1} \varepsilon} \leq \scrF_2(\mu_x^{(\epsilon_3)}, \mu_z^{(\epsilon_3)}) \leq \scrF_2(\mu_x^\ast,\mu_z^\ast(\mu_x^\ast)) + \epsilon_3 . 
\end{align*}
The proof is thus concluded. 
\end{proof}

\subsection{Proofs in \Cref{sec:convergence}}
\subsubsection{Proof of \Cref{thm:convergence_outer_loop}}\label{sec:proof_convergence_outer_loop}
For fixed $s\in\N^+$, the update scheme from time $s$ to $s+1$ of the outer loop in \Cref{alg:outer_loop} is the following: let $\{\xi_{x,s}^{(i)}\}_{i=1}^{N_x}$ be $N_x$ i.i.d samples from $\calN(0, \Id_{d_x})$, 
\begin{align}\label{eq:outer_particle}
    x_{s+1}^{(i)} = x_s^{(i)} - \gamma \bnabla_1 L_\lambda(\mu_{\scrx, s}, \tilde{\mu}_{\scrz, s} , \mu_{\scrz, s})(\hat{x}_s^{(i)}) + \sqrt{2 \gamma \sigma_2} \xi_{x, s}^{(i)} , \quad i=1, \ldots, N_x,
\end{align}
Here, $\bnabla_1 L_\lambda$ denotes taking the Wasserstein gradient with respect to the first input of $L_\lambda$ and its formula is given in Eq.~\eqref{eq:gradient_L_lambda}. 

Following the one-step interpolation technique from \citet{vempala2019rapid,suzuki2023convergence}, we define another system of particles with the initialization: $\hat{x}_{0}^{(i)} = x_{s}^{(i)}$ (for all $i=1, \ldots, N_x$) and the update scheme: for $0 \leq \tau \leq \gamma$ and $i=1, \ldots, N_x$,  
\begin{align}\label{eq:one_step_interpolation}
    \dd \hat{x}_\tau^{(i)} &= -\bnabla_1 L_\lambda(\hat{\mu}_{\scrx, 0}, \tilde{\mu}_{\scrz, s} , \mu_{\scrz, s})(\hat{x}_0^{(i)}) \; \dd \tau + \sqrt{2 \sigma_2} \dd W_{x, \tau}^{(i)} , \quad \hat{\mu}_{\scrx, 0} = \frac{1}{N_x}\sum_{i=1}^{N_x} \delta_{\hat{x}_{0}^{(i)}} .
\end{align}
Here, $W_{x, \tau}^{(i)}$ is the $d_x$ dimensional standard Brownian motion. 
Then, $\{\hat{x}_{\gamma}^{(i)}\}_{i=1}^{N_x}$ follow the same distribution as $\{x_{s+1}^{(i)}\}_{i=1}^{N_x}$. 
Define $\hat{\mu}_{x,\tau}^{(N_x)}\in\calP_2((\R^{d_x})^{N_x})$ as the joint distribution of $N_x$ particles $\{\hat{x}_{\tau}^{(i)} \}_{i=1}^{N_x}$ for the intermediate time $\tau\in[0,\gamma]$. From \citet{sarkka2019applied}, the corresponding Fokker plank equation of Eq.~\eqref{eq:one_step_interpolation} is
\begin{align}
    &\quad \frac{\dd}{\dd \tau} \hat{\mu}_{x, \tau}^{(N_x)} \left(\scrX \mid \hat{\scrX}_0 \right) \nonumber \\
    & = \sum_{i=1}^{N_x} \nabla_i \cdot\left(\hat{\mu}_{x, \tau}^{(N_x)} \left(\scrX \mid \hat{\scrX}_0 \right) \bnabla_1 L_\lambda(\hat{\mu}_{\scrx,0}, \tilde{\mu}_{\scrz, s} , \mu_{\scrz, s})(x^{(i)}) \right) + \sigma_2 \sum_{i=1}^{N_x} \Delta_i \hat{\mu}_{x, \tau}^{(N_x)} \left(\scrX \mid \hat{\scrX}_0 \right). \label{eq:time_derivative_1}
\end{align}
Take the expectation over $\hat{\scrx}_0 \sim \hat{\mu}_{x,0}^{(N_x)}$ and we obtain
\begin{align*}
    \frac{\dd }{\dd \tau} \hat{\mu}_{x, \tau}^{(N_x)}\left(\scrX\right) &= \sum_{i=1}^{N_x} \nabla_i \cdot\left(\hat{\mu}_{x, \tau}^{(N_x)}(\scrX) \E_{\hat{\scrX}_0 \mid \hat{\scrX}_\tau}[\bnabla_1  L_\lambda(\hat{\mu}_{\scrx,0}, \tilde{\mu}_{\scrz, s} , \mu_{\scrz, s})(\hat{x}_0^{(i)}) \mid \hat{\scrX}_\tau = \scrX] \right) \\
    &\qquad\qquad + \sigma_2 \sum_{i=1}^{N_x} \Delta_i \hat{\mu}_{x, \tau}^{(N_x)} \left(\scrx\right) . 
\end{align*}
Recall that $\mu_\ast^{(N_x)} = \arg\min_{\mu^{(N_x)}\in\calP_2((\R^{d_x})^{N_x})} \scrL_\lambda^{(N_x)}(\mu^{(N_x)})$ for $\scrL_\lambda^{(N_x)}$ defined in Eq.~\eqref{eq:defi_L_Nx}. It satisfies 
\begin{align}\label{eq:defi_proximal_gibbs_mu_x}
    \mu_\ast^{(N_x)}(\scrX) &\propto \exp\left( - \frac{N_x}{\sigma_2} \Big( F_2(\mu_{\scrx}, \tilde{\mu}_{z}^\ast(\mu_{\scrx})) + \lambda\cdot \scrF_1(\mu_{\scrx}, \tilde{\mu}_{z}^\ast(\mu_{\scrx})) - \lambda\cdot \scrF_1(\mu_{\scrx}, \mu_{z}^\ast(\mu_{\scrx})) \right). 
\end{align}
For any $1 \leq i \leq N_x$, denote $\nabla_i$ as taking the derivative with respect to $x^{(i)}$ for a mapping from $(\R^{d_x})^N$ to $\R$. We have 
\begin{align}\label{eq:nabla_proximal_gibbs}
    &\quad -\sigma_2 \cdot \nabla_i \log (\mu_\ast^{(N_x)}(\scrX)) 
    \nonumber \\
    &= N_x \nabla_i F_2(\mu_\scrx, \tilde{\mu}_{z}^\ast(\mu_{\scrx})) + \lambda \cdot N_x \nabla_i F_1(\mu_\scrx, \tilde{\mu}_{z}^\ast(\mu_{\scrx})) + \lambda\sigma_1\cdot N_x \nabla_i \log (\tilde{\mu}_{z}^\ast(\mu_{\scrx}))\nonumber \\
    &\qquad\qquad - \lambda \cdot N_x \nabla_i F_1(\mu_\scrx, \mu_{z}^\ast(\mu_{\scrx})) - \lambda \sigma_1\cdot N_x \nabla_i \log (\mu_{z}^\ast(\mu_{\scrx})) \nonumber \\
    &\stackrel{(\ast)}{=} \bnabla_1 F_2(\mu_\scrx, \tilde{\mu}_{z}^\ast(\mu_{\scrx}))(x^{(i)}) + \lambda \cdot \bnabla_1 F_1(\mu_\scrx, \tilde{\mu}_{z}^\ast(\mu_{\scrx}))(x^{(i)}) - \lambda \cdot \bnabla_1 F_1(\mu_\scrx, \mu_{z}^\ast(\mu_{\scrx}))(x^{(i)}) \nonumber \\
    &= \bnabla_1 L_\lambda(\mu_{\scrX}, \tilde{\mu}_{z}^\ast(\mu_{\scrx}), \mu_{z}^\ast(\mu_{\scrx}))(x^{(i)}). 
\end{align}
In $(\ast)$ above, the gradient with respect to the nested mapping $\mu_{z}^\ast(\mu_{\scrx}), \tilde{\mu}_{z}^\ast(\mu_{\scrx})$ vanish due to the optimality of $\mu_{z}^\ast(\mu_{\scrx})$ and $\tilde{\mu}_{z}^\ast(\mu_{\scrx})$, which gives 
\begin{align*}
    \bnabla_2 F_2(\mu_\scrx, \tilde{\mu}_{z}^\ast(\mu_{\scrx}))(\cdot)  +\lambda \cdot \bnabla_2 F_1(\mu_\scrx, \tilde{\mu}_{z}^\ast(\mu_{\scrx}))(\cdot)  +\lambda\sigma_1\cdot \bnabla \mathrm{Ent} (\tilde{\mu}_{z}^\ast(\mu_{\scrx}))(\cdot) &= 0 \\
    \bnabla_2 F_1(\mu_\scrx, \mu_{z}^\ast(\mu_{\scrx}))(\cdot) + \sigma_1\cdot \bnabla \mathrm{Ent} (\mu_{z}^\ast(\mu_{\scrx}))(\cdot) &= 0 .
\end{align*}
So we have, continuing from Eq.~\eqref{eq:time_derivative_1}, 
\begin{align}
    \frac{\dd }{\dd \tau} \hat{\mu}_{x, \tau}^{(N_x)}\left(\scrX\right) &= \sum_{i=1}^{N_x} \sigma_2 \nabla_i \cdot\left(\hat{\mu}_{x, \tau}^{(N_x)} (\scrX) \left(\nabla_i \log \left(\frac{\hat{\mu}_{x, \tau}^{(N_x)}}{\mu_\ast^{(N_x)}}(\scrX) \right)\right) \right) \nonumber \\ & + \sum_{i=1}^{N_x} \nabla_i \cdot\left\{\hat{\mu}_{x, \tau}^{(N_x)}(\scrX)\left( \E_{\hat{\scrX}_0 \mid \hat{\scrX}_\tau}[\bnabla_1 L_\lambda(\hat{\mu}_{\scrX,0}, \tilde{\mu}_{\scrz, s} , \mu_{\scrz, s})(\hat{x}_0^{(i)}) \mid \hat{\scrX}_\tau = \scrX] \right.\right. \nonumber \\
    &\qquad\qquad\qquad\qquad\qquad\qquad - \left.\left.\bnabla_1 L_\lambda(\mu_{\scrX}, \tilde{\mu}_{z}^\ast(\mu_{\scrx}), \mu_{z}^\ast(\mu_{\scrx}))(x^{(i)}) \right)\right\} \label{eq:time_derivative_hat_mu} .
\end{align}
Recall that $\mu_{z,s}^{(N_z)}\in \calP_2((\R^{d_z})^{N_z})$ and $ \tilde{\mu}_{z,s}^{(N_z)} \in \calP_2((\R^{d_z})^{N_z})$ are the joint distributions of the $N_z$ particles which are the output of the inner-loop algorithm \Cref{alg:inner_loop} at time $s$. The corresponding empirical distributions are $\tilde{\mu}_{\scrz, s}, \mu_{\scrz, s}$. 
These distributions are fixed since $s$ is fixed so they are independent of $\tau$. 

Consider the time derivative with respect to $\tau$ of $\E_{\scrz \sim \mu_{z,s}^{(N_z)}, \; \tilde{\scrZ} \sim \tilde{\mu}_{z,s}^{(N_z)} }[\scrL_\lambda^{(N_x)}(\hat{\mu}_{x,\tau}^{(N_x)})]$, where the expectation is taken with respect to the implicit dependence of $\hat{\mu}_{x,\tau}^{(N_x)}$ on the particles $\scrz_s = \{z_s^{(i)}\}_{i=1}^{N_z}$ (resp. $\tilde{\scrz}_s = \{\tilde{z}_s^{(i)} \}_{i=1}^{N_z}$) of the empirical distributions $\tilde{\mu}_{\scrz, s}$ (resp. $\mu_{\scrz, s}$) through the update scheme in Eq.~\eqref{eq:one_step_interpolation}. 
We have
\begin{align}
    &\quad \frac{\dd}{\dd \tau} \left\{ \E_{\scrz \sim \mu_{z,s}^{(N_z)}, \; \tilde{\scrZ} \sim \tilde{\mu}_{z,s}^{(N_z)} }\left[ \scrL_\lambda^{(N_x)}(\hat{\mu}_{x,\tau}^{(N_x)}) \right] \right\} \nonumber \\
    &= \E_{\scrz, \tilde{\scrZ}} \left[ \int \frac{\dd \hat{\mu}_{x,\tau}^{(N_x)}}{\dd \tau}(\scrX) \left( N_x F_2(\mu_{\scrx}, \tilde{\mu}_z^\ast(\mu_{\scrx})) + \lambda \cdot N_x F_1(\mu_{\scrx}, \tilde{\mu}_z^\ast(\mu_{\scrx})) + \lambda \sigma_1\cdot N_x \mathrm{Ent}(\tilde{\mu}_z^\ast(\mu_{\scrx})) \right. \right. \nonumber \\
    &\qquad - \lambda \cdot N_x F_1(\mu_{\scrx}, \mu_z^\ast(\mu_{\scrx})) - \lambda \sigma_1\cdot N_x \mathrm{Ent}(\mu_z^\ast(\mu_{\scrx})) + \left. \left. \sigma_2 \log( \hat{\mu}_{x,\tau}^{(N_x)})(\scrX) \right) \; \dd \scrX \right] . \label{eq:long_complicated_time_derivative} 
\end{align}
Next, we plug in the definition of $\frac{\dd}{\dd \tau} \hat{\mu}_{x,\tau}^{(N_x)}$ from Eq.~\eqref{eq:time_derivative_hat_mu} into the above equation. Also, we apply Eq.~\eqref{eq:nabla_proximal_gibbs} and the integration by parts. Then, we obtain
\begin{align}
    &\quad \frac{\dd}{\dd \tau} \left\{\E_{\scrz \sim \mu_{z,s}^{(N_z)}, \; \tilde{\scrZ} \sim \tilde{\mu}_{z,s}^{(N_z)} }\left[ \scrL_\lambda^{(N_x)}(\hat{\mu}_{x,\tau}^{(N_x)}) \right] \right\} \nonumber \\
    &= -\sigma_2^2\E_{\scrz, \tilde{\scrZ}} \left[ \sum_{i=1}^{N_x} \left\| \nabla_i \log \left(\frac{\hat{\mu}_{x,\tau}^{(N_x)}}{\mu_\ast^{(N_x)}}(\scrX) \right) \right\|_{L^2 \left( \hat{\mu}_{x,\tau}^{(N_x)} \right)}^2 \right] - \sigma_2 \E_{\scrz, \tilde{\scrZ}} \left[ \sum_{i=1}^{N_x} \int \hat{\mu}_{x,\tau}^{(N_x)}(\scrX) \right. \nonumber \\
    &\quad \left( \E_{\hat{\scrX}_0 \mid \hat{\scrX}_\tau}[\bnabla_1 L_\lambda(\hat{\mu}_{\scrX,0}, \tilde{\mu}_{\scrz, s} , \mu_{\scrz, s})(\hat{x}_0^{(i)}) \mid \hat{\scrX}_\tau = \scrX] - \bnabla_1 L_\lambda(\mu_{\scrX}, \tilde{\mu}_z^\ast(\mu_{\scrx}), \mu_z^\ast(\mu_{\scrx}))(x^{(i)}) \right)^\top \nonumber \\
    &\qquad\qquad \nabla_i \log \left(\frac{\hat{\mu}_{x,\tau}^{(N_x)}}{\mu_\ast^{(N_x)}}(\scrx) \right) \; \dd \scrx \Big]  \nonumber \\
    &\leq -\frac{\sigma_2^2}{2} \E_{\scrz, \tilde{\scrZ}} \left[ \sum_{i=1}^{N_x} \left\| \nabla_i \log \left(\frac{\hat{\mu}_{x,\tau}^{(N_x)}}{\mu_\ast^{(N_x)}}(\scrX) \right) \right\|_{L^2\left( \hat{\mu}_{x,\tau}^{(N_x)} \right)}^2 \right] \label{eq:descent_term_one} \\
    &+ \frac{1}{2} \sum_{i=1}^{N_x} \E_{\substack{
    \hat{\scrx}_0 \sim \hat{\mu}_{x,0}^{(N_x)}, \; \hat{\scrX}_\tau \sim \hat{\mu}_{x,\tau}^{(N_x)} \\
    \scrz \sim \mu_{z,s}^{(N_z)}, \; \tilde{\scrZ} \sim \tilde{\mu}_{z,s}^{(N_z)}
    }} \left[ \left\| \bnabla_1 L_\lambda(\hat{\mu}_{\scrX,0}, \tilde{\mu}_{\scrz, s} , \mu_{\scrz, s})(\hat{x}_0^{(i)})\right. \right. \nonumber \\
    &\hspace{15em} \left. \left. - \bnabla_1 L_\lambda\left(\hat{\mu}_{\scrX,\tau}, \tilde{\mu}_z^\ast(\mu_{\scrX,\tau}), \mu_z^\ast(\mu_{\scrX,\tau}) \right)(\hat{x}_\tau^{(i)}) \right\|^2 \right] \label{eq:descent_term_two} .
\end{align}
The last inequality holds by applying Cauchy-Schwartz inequality. 
The first term in Eq.~\eqref{eq:descent_term_one} can be upper bound by a defective uniform Log-Sobolev inequality proved in \Cref{prop:uniform_LSI}: 
\begin{align}\label{eq:first_term}
    &\quad -\frac{\sigma_2^2}{2 N_x} \sum_{i=1}^{N_x}  \E_{\scrz \sim \mu_{z,s}^{(N_z)}, \; \tilde{\scrZ} \sim \tilde{\mu}_{z,s}^{(N_z)} } \left[ \left\| \nabla_i \log \left(\frac{\hat{\mu}_{x,\tau}^{(N_x)}}{\mu_\ast^{(N_x)}}(\scrX) \right) \right\|_{L^2\left( \hat{\mu}_{x,\tau}^{(N_x)} \right)}^2 \right] \nonumber \\
    & \leq -\frac{C_{\LSI,x} \sigma_2}{4 N_x} \left( \E_{\scrz \sim \mu_{z,s}^{(N_z)}, \; \tilde{\scrZ} \sim \tilde{\mu}_{z,s}^{(N_z)} } \left[ \scrL_\lambda^{(N_x)} \left(\hat{\mu}_{x,\tau}^{(N_x)}\right) \right] - N_x \scrL_\lambda(\mu_{x,\lambda}^\ast, \tilde{\mu}_z^\ast(\mu_{x,\lambda}^\ast), \mu_z^\ast(\mu_{x,\lambda}^\ast))\right) \nonumber \\
    &\qquad + \frac{\lambda C_{\LSI,x} R^2 \sigma_2 \left(\frac{R}{2\sigma_1} + 2 \right) }{2 N_x} + \frac{\mathfrak{c}^2 C_{\LSI,x}^2 \sigma_2}{8}. 
\end{align} 

And the second term in Eq.~\eqref{eq:descent_term_one} can be upper bounded by the following three upper bounds proved in \Cref{lem:bnabla_G_lip_3}, \Cref{lem:bnabla_G_lip_1} and \Cref{lem:bnabla_G_lip_2}, respectively. 
From \Cref{lem:bnabla_G_lip_3}, we have, for $i=1,\ldots,N_x$, 
\begin{align}\label{eq:second_term_one}
    &\quad \E \left[ \left\| \bnabla_1 L_\lambda(\mu_{\scrX,\tau}, \tilde{\mu}_z^\ast(\mu_{\scrX, \tau}), \mu_z^\ast(\mu_{\scrX, \tau}))(\hat{x}_\tau^{(i)}) - \bnabla_1 L_\lambda(\mu_{\scrX,\tau}, \tilde{\mu}_z^\ast(\mu_{\scrX,0}), \mu_z^\ast(\mu_{\scrX,0}))(\hat{x}_\tau^{(i)}) \right\|^2 \right] \nonumber \\
    &\leq \frac{\lambda^2 R^6}{\sigma_1} \left( \gamma^2 \left(\zeta_2^2 \frac{1}{N_x} \sum_{i=1}^{N_x} \E[\|\hat{x}_0^{(i)}\|^2] + \lambda^2 R^2 \right) + \gamma \sigma_2 d_x \right)  . 
\end{align}
Here, the expectation on the left hand side is taken over the joint distribution of the particles $\{\hat{x}_\tau^{(i)}\}_{i=1}^{N_x} \sim\hat{\mu}_{x,\tau}^{(N_x)}$ and $\{\hat{x}_0^{(i)}\}_{i=1}^{N_x} \sim\hat{\mu}_{x,0}^{(N_x)}$. 
From \Cref{lem:bnabla_G_lip_1}, we have, for $i=1,\ldots,N_x$, 
\begin{align}\label{eq:second_term_two}
    &\quad \E \left[ \left\| \bnabla_1 L_\lambda(\mu_{\scrX,0}, \tilde{\mu}_{z}^\ast(\hat{\mu}_{\scrx,0}), \mu_{z}^\ast(\hat{\mu}_{\scrx,0}))(\hat{x}_0^{(i)}) - \bnabla_1 L_\lambda(\mu_{\scrX,\tau}, \tilde{\mu}_{z}^\ast(\hat{\mu}_{\scrx,0}), \mu_{z}^\ast(\hat{\mu}_{\scrx,0}))(\hat{x}_\tau^{(i)}) \right\|^2 \right] \nonumber \\
    &\leq (2\lambda^2 R^4 + \zeta_2^2) \left( \gamma^2 \left(\zeta_2^2 \frac{1}{N_x} \sum_{i=1}^{N_x} \E[\|\hat{x}_0^{(i)}\|^2] + \lambda^2 R^2 \right) + \gamma \sigma_2 d_x \right). 
\end{align}
Here, the expectation is taken over the joint distribution of the particles $\{\hat{x}_\tau^{(i)}\}_{i=1}^{N_x} \sim\hat{\mu}_{x,\tau}^{(N_x)}$ and $\{\hat{x}_0^{(i)}\}_{i=1}^{N_x} \sim\hat{\mu}_{x,0}^{(N_x)}$. From \Cref{lem:bnabla_G_lip_2}, we have, for $i=1,\ldots,N_x$, 
\begin{align}\label{eq:second_term_three}
    &\quad \E \left[ \left\| \bnabla_1 L_\lambda(\mu_{\scrX,0}, \tilde{\mu}_{z}^\ast(\hat{\mu}_{\scrx,0}), \mu_{z}^\ast(\hat{\mu}_{\scrx,0}))(\hat{x}_0^{(i)}) - \bnabla_1 L_\lambda(\mu_{\scrX,0}, \tilde{\mu}_{\scrz, s} , \mu_{\scrz, s})(\hat{x}_0^{(i)}) \right\|^2 \right] \\
    &\lesssim \frac{\lambda^2 R^4}{N_z} \left(\kl\left(\mu_{z,s}^{(N_z)}, (\mu_{z}^\ast(\hat{\mu}_{\scrx,0}))^{\otimes N_z} \right) + \kl\left( \tilde{\mu}_{z,s}^{(N_z)}, (\tilde{\mu}_{z}^\ast(\hat{\mu}_{\scrx,0}))^{\otimes N_z} \right) + 1 \right)  \nonumber . 
\end{align}
Here, the expectation is taken over the joint distribution of the particles  $\{\tilde{z}_s^{(i)}\}_{i=1}^{N_z}\sim\tilde{\mu}_{z,s}^{(N_z)}$ and $\{z_s^{(i)}\}_{i=1}^{N_z}\sim\mu_{z,s}^{(N_z)}$. 

Now, we combine Eq.~\eqref{eq:first_term}, Eq.~\eqref{eq:second_term_one}, Eq.~\eqref{eq:second_term_two} and Eq.~\eqref{eq:second_term_three}, and plug these bounds back to Eq.~\eqref{eq:descent_term_two}. 
To simplify the notation, define 
\begin{align}\label{eq:defi_calE}
    \calE(\tau) := N_x^{-1} \E_{\scrz \sim \mu_{z,s}^{(N_z)}, \; \tilde{\scrZ} \sim \tilde{\mu}_{z,s}^{(N_z)} }\left[ \scrL_\lambda^{(N_x)}(\hat{\mu}_{x,\tau}^{(N_x)})) \right] - \scrL_\lambda({\mu_x}_\lambda^\ast, 
    \tilde{\mu}_{z}^\ast({\mu_x}_\lambda^\ast),
    \mu_{z}^\ast({\mu_x}_\lambda^\ast)) .
\end{align}
So we obtain, 
\begin{align*}
    \frac{\dd}{\dd \tau} \calE(\tau) &\leq -\frac{C_{\LSI,x} \sigma_2}{4} \calE(\tau) + \frac{\lambda C_{\LSI,x} R^2 \sigma_2 (\frac{R}{\sigma_1}+2)}{N_x} + \frac{\mathfrak{c}^2 C_{\LSI,x}^2 \sigma_2}{8} \\
    &\quad + (2\lambda^2 R^4 + \zeta_2^2 + \frac{\lambda^2R^6}{\sigma_1}) \left(\gamma^2 \left(\zeta_2^2 \frac{1}{N_x}\sum_{i=1}^{N_x} \E[\|\hat{x}_0^{(i)}\|^2] + \lambda^2 R^2\right) + 2\gamma \sigma_2 d_x \right)\\
    &\qquad + \lambda^2 R^4 \left(\sqrt{ \frac{\kl\left(\mu_{z}^{(N_z)}, {(\mu_{z}^\ast(\hat{\mu}_{\scrx,0}))}^{\otimes N_z} \right)}{N_z} } + \sqrt{ \frac{\kl\left(\tilde{\mu}_{z}^{(N_z)}, {(\tilde{\mu}_{z}^\ast(\hat{\mu}_{\scrx,0}))}^{\otimes N_z} \right)}{N_z} } + 1 \right) . 
\end{align*}
Since $\{\hat{x}_0^{(i)}\}_{i=1}^{N_x} = \{x_s^{(i)}\}_{i=1}^{N_x}$, we have $\E[\|\hat{x}_0^{(i)}\|^2]=\E[\|x_s^{(i)}\|^2]$. Also from the uniform second moment bound proved in \Cref{lem:second_mom_bound} provided that $\gamma \leq \frac{1}{\zeta_2}$, we have $\zeta_2^2 \E[\|x_s^{(i)}\|^2]\leq \zeta_2^2 \E[\|x_0^{(i)}\|^2] + \lambda^2 R^2 + 2\zeta_2 \sigma_2 d_x$. Notice that $\frac{1}{N_x}\sum_{i=1}^{N_x}\E[\|x_0^{(i)}\|^2] = \E_{\mu_{x,0}}[\|x\|^2]$ which is the average second moment of all particles at initialization. So we have
\begin{align*}
    \gamma^2 \left(\zeta_2^2 \frac{1}{N_x}\sum_{i=1}^{N_x} \E[\|\hat{x}_0^{(i)}\|^2] + \lambda^2 R^2 \right) + 2\gamma \sigma_2 d_x &\leq \gamma^2 (\zeta_2^2\E_{\mu_{x,0}}[\|x\|^2] + 2\lambda^2 R^2 + \zeta_2\sigma_2 d_x) + 2\gamma \sigma_2 d_x \\
    &\leq \gamma^2 (\zeta_2^2\E_{\mu_{x,0}}[\|x\|^2] + 2\lambda^2 R^2) + 3\gamma \sigma_2 d_x . 
\end{align*}
Recall that for any $s\in\N^+$, $\mu_{z,s}^{(N_z)}$ and $\tilde{\mu}_{z,s}^{(N_z)}$ denote the joint 
distribution of the $N_z$ particles which are the output of the inner-loop algorithm \textsc{InnerLoop}($\mu_{\scrx,s}$, $T$, $\alpha$, $\beta$, $\lambda$, $\sigma_1$), detailed in \Cref{alg:inner_loop} and $\mu_{\scrx,s}=\hat{\mu}_{\scrx,0}$. 
As a result, let 
$\mathfrak{KL} = \kl(\mu_{z,s}^{(N_z)}, (\mu_{z}^\ast(\mu_{\scrx,s})^{\otimes N_z})$ and $\tilde{\mathfrak{KL}} = \kl(\tilde{\mu}_{z,s}^{(N_z)}, (\tilde{\mu}_{z}^\ast(\mu_{\scrx,s})^{\otimes N_z})$ as defined in the statement of \Cref{thm:convergence_outer_loop} that represent the error from the inner-loop algorithm. Both terms are upper bounded in \Cref{prop:inner} and \Cref{prop:tilde_inner} respectively. 
Next, we have  
\begin{small}
\begin{align*}
    \frac{\dd}{\dd \tau} (\calE(\tau) - \ldots) &\leq -\frac{C_{\LSI,x} \sigma_2}{4} \Bigg( \calE(\tau) -  \frac{\lambda R^2(\frac{R}{\sigma_1} + 2)}{N_x} - \mathfrak{c}^2 C_{\LSI,x} - \left(\frac{\sigma_2 C_{\LSI,x}}{4} \right)^{-1} \lambda^2 R^4\left(\sqrt{\frac{\mathfrak{KL}}{N_z}} + \sqrt{\frac{\tilde{\mathfrak{KL}}}{N_z}} + \frac{1}{N_z} \right) \\
    &\hspace{-1cm} - \left(\frac{\sigma_2 C_{\LSI,x}}{4} \right)^{-1} (2\lambda^2 R^4 + \zeta_2^2 + \frac{\lambda^2R^6}{\sigma_1}) \left(\gamma^2 (\zeta_2^2\E_{\mu_{x,0}}[\|x\|^2] + 2\lambda^2 R^2) + 3\gamma \sigma_2 d_x \right) \Bigg). 
\end{align*}
\end{small}
Here, $\ldots$ represents the same quantity in the large bracket of the right hand side. 
Since $\{\hat{x}_0^{(i)}\}_{i=1}^{N_x} = \{x_s^{(i)}\}_{i=1}^{N_x}$ and  $\{\hat{x}_\gamma^{(i)}\}_{i=1}^{N_x} = \{x_{s+1}^{(i)}\}_{i=1}^{N_x}$ by construction of the one-step interpolation particle system defined in Eq.~\eqref{eq:one_step_interpolation}.
So we have $\hat{\mu}_{x,0}^{(N_x)} = \mu_{x,s}^{(N_x)}$ and $\hat{\mu}_{x,\gamma}^{(N_x)} = \mu_{x,s+1}^{(N_x)}$ and consequently $\calE(0)=\calH(s)$ and $\calE(\gamma)=\calH(s+1)$ for $\calH$ defined in the statement of the proposition. 
Hence, by integrating the above equation from $0$ to $\gamma$, we obtain
\begin{align*}
    \calH(s+1) - \ldots &\leq \exp \left(-\frac{\sigma_2 C_{\LSI,x} \gamma}{4} \right) \Bigg( \calH(s) - \frac{\lambda R^2(\frac{R}{\sigma_1} + 1) }{N_x} - \mathfrak{c}^2 C_{\LSI,x}\\
    &\hspace{-2cm} - \left(\frac{\sigma_2 C_{\LSI,x}}{4} \right)^{-1} \lambda^2 R^4\left(\sqrt{\frac{\mathfrak{KL}}{N_z}} + \sqrt{\frac{\tilde{\mathfrak{KL}}}{N_z}} + \frac{1}{N_z} \right) \\
    &\hspace{-1cm} - \left(\frac{\sigma_2 C_{\LSI,x}}{4} \right)^{-1} (2\lambda^2 R^4 + \zeta_2^2 + \frac{\lambda^2R^6}{\sigma_1}) \left(\gamma^2 (\zeta_2^2\E_{\mu_{x,0}}[\|x\|^2] + 2\lambda^2 R^2) + 3\gamma \sigma_2 d_x \right) \Bigg). 
\end{align*}
Since the above equation holds for any $s\in\N^+$, we obtain
\begin{align*}
    \calH(S) &\leq \exp \left(-\frac{\sigma_2 C_{\LSI,x} S \gamma}{4} \right)\calH(0) + \frac{\lambda R^2(\frac{R}{\sigma_1} + 1) }{N_x} + \mathfrak{c}^2 C_{\LSI,x} + \frac{\lambda^2 R^4\left(\sqrt{\frac{\mathfrak{KL}}{N_z}} + \sqrt{\frac{\tilde{\mathfrak{KL}}}{N_z}} + \frac{1}{N_z} \right)}{\sigma_2 C_{\LSI,x}} \\
    &\quad + \frac{\lambda^2 R^4 + \zeta_2^2 + \frac{\lambda^2R^6}{\sigma_1}}{\sigma_2 C_{\LSI,x}} \left(\gamma^2 (\zeta_2^2\E_{\mu_{x,0}}[\|x\|^2] + \lambda^2 R^2) + \gamma \sigma_2 d_x \right) .
\end{align*}
In the last step, we omit all the positive scalar coefficients like $2$, $3$ to simplify the presentation. 

\subsubsection{Proof of \Cref{lem:bregman}}\label{sec:proof_bregman}
\begin{proof}[Proof of \Cref{lem:bregman}]
From the partial linear convexity of $\mu_z\mapsto L_\lambda(\mu_x, \tilde{\mu}_z, \mu_z)$ for fixed $\tilde{\mu}_z, \mu_z$ proved in \Cref{prop:partial_convex_L_1}, we have
\begin{align}\label{eq:delta_mu_x}
    \hspace{-10pt} \int \delta_{\mu_x} L_\lambda(\mu_x^{\prime}, \tilde{\mu}_z^\ast(\mu_x^\prime), \mu_z^\ast(\mu_x^\prime))(x) \; \dd (\mu_x^{\prime} -\mu_x) 
    \geq L_\lambda(\mu_x^{\prime}, \tilde{\mu}_z^\ast(\mu_x^\prime), \mu_z^\ast(\mu_x^\prime)) - L_\lambda(\mu_x, \tilde{\mu}_z^\ast(\mu_x^\prime), \mu_z^\ast(\mu_x^\prime)) . 
\end{align}
The optimality of $\mu_z^\ast(\mu_x^\prime)$, implies that $ L_\lambda(\mu_x, \tilde{\mu}_z^\ast(\mu_x^\prime), \mu_z^\ast(\mu_x^\prime)) \leq L_\lambda(\mu_x, \tilde{\mu}_z^\ast(\mu_x^\prime), \mu_z^\ast(\mu_x))$ and hence, 
\begin{align*}
    \eqref{eq:delta_mu_x} \geq L_\lambda(\mu_x^{\prime}, \tilde{\mu}_z^\ast(\mu_x^\prime), \mu_z^\ast(\mu_x^\prime)) - L_\lambda(\mu_x, \tilde{\mu}_z^\ast(\mu_x^\prime), \mu_z^\ast(\mu_x)) . 
\end{align*}
Finally, by \Cref{prop:continuity_U_lambda} where $U_\lambda$ defined in Eq.~\eqref{eq:defi_U_lambda} is $L_\lambda$ excluding the $\ell_2$ regularization term $\frac{\zeta_2}{2}\E_{\mu_x}[\|x\|^2]$, we have
\begin{align*}
    \eqref{eq:delta_mu_x} &\geq L_\lambda(\mu_x^{\prime}, \tilde{\mu}_z^\ast(\mu_x^\prime), \mu_z^\ast(\mu_x^\prime)) - L_\lambda(\mu_x, \tilde{\mu}_z^\ast(\mu_x^\prime), \mu_z^\ast(\mu_x)) \\
    &= U_\lambda(\mu_x^{\prime}, \tilde{\mu}_z^\ast(\mu_x^\prime), \mu_z^\ast(\mu_x^\prime)) - U_\lambda(\mu_x, \tilde{\mu}_z^\ast(\mu_x^\prime), \mu_z^\ast(\mu_x)) + \frac{\zeta_2}{2}\E_{\mu_x}[\|x\|^2] - \frac{\zeta_2}{2}\E_{\mu_x'}[\|x\|^2] \\
    &\geq U_\lambda(\mu_x^{\prime}, \tilde{\mu}_z^\ast(\mu_x^\prime), \mu_z^\ast(\mu_x^\prime)) - U_\lambda(\mu_x, \tilde{\mu}_z^\ast(\mu_x), \mu_z^\ast(\mu_x)) - \frac{R \lambda}{4\sigma_1} \E_{\rho} \left[ \left( \smallint \Psi_\ba \; \dd (\mu_x - \mu_x^\prime) \right)^2 \right] \\
    &\qquad\qquad + \frac{\zeta_2}{2}\E_{\mu_x}[\|x\|^2] - \frac{\zeta_2}{2}\E_{\mu_x'}[\|x\|^2] \\
    &= L_\lambda(\mu_x^{\prime}, \tilde{\mu}_z^\ast(\mu_x^\prime), \mu_z^\ast(\mu_x^\prime)) - L_\lambda(\mu_x, \tilde{\mu}_z^\ast(\mu_x), \mu_z^\ast(\mu_x)) - \frac{R \lambda}{4\sigma_1} \E_{\rho} \left[ \left( \smallint \Psi_\ba \; \dd (\mu_x - \mu_x^\prime) \right)^2 \right] . 
\end{align*}
The proof is thus concluded. 
\end{proof}

\subsubsection{Proof of \Cref{prop:KL_upper_bound_functional}}\label{sec:proof_kl_gap}
\begin{proof}[Proof of \Cref{prop:KL_upper_bound_functional}]
From Eq.~\eqref{eq:lem_1_nitanda_1}, for any $\mu_x^{(N_x)}$, we have
\begin{align*}
    &\quad N_x^{-1} \scrL_\lambda^{(N_x)}(\mu_x^{(N_x)}) - \scrL_\lambda(\mu_{x,\lambda}^\ast, \tilde{\mu}_z^*(\mu_{x,\lambda}^\ast) , \mu_z^*(\mu_{x,\lambda}^\ast)) \\
    &= N_x^{-1} \sigma_2 \kl(\mu_x^{(N_x)}, (\mu_{x,\lambda}^\ast)^{\otimes N_x}) + \E_{\scrx \sim \mu_x^{(N_x)}} \left[B_{L_\lambda} (\mu_{\scrx}, \mu_{x,\lambda}^\ast)\right] . 
\end{align*}
From \Cref{lem:bregman} and Proposition 1 of \citet{nitanda2025propagation}, we have
\begin{align*}
    \E_{\scrx \sim \mu_x^{(N_x)}} \left[B_{L_\lambda} (\mu_{\scrx}, \mu_{x,\lambda}^\ast) \right] &\geq -\frac{R \lambda}{4\sigma_1} \E_{\scrx \sim \mu_x^{(N_x)}} \left[ \E_{\rho} \left[ \left( \smallint \Psi_\ba \; \dd (\mu_\scrx - \mu_{x,\lambda}^\ast) \right)^2 \right] \right] \\
    &\geq -\frac{R \lambda}{4\sigma_1} \left( 8R^2 \sqrt{N_x^{-1} \kl(\mu_x^{(N_x)}, (\mu_{x,\lambda}^\ast)^{\otimes N_x})} + \frac{4R^2}{N_x} \right) \\
    &\geq -\frac{2R^3\lambda}{\sigma_1} \left( \mathfrak{c}^{-1} N_x^{-1} \kl(\mu_x^{(N_x)}, (\mu_{x,\lambda}^\ast)^{\otimes N_x}) + \mathfrak{c} \right) - \frac{R^3\lambda}{\sigma_1 N_x} . 
\end{align*}
The last line holds by the fact that
\begin{align*}
    \sqrt{N_x^{-1}\kl(\mu_x^{(N_x)}, (\mu_{x,\lambda}^\ast)^{\otimes N_x})} \leq \mathfrak{c}^{-1} N_x^{-1}\kl(\mu_x^{(N_x)}, (\mu_{x,\lambda}^\ast)^{\otimes N_x}) + \mathfrak{c} . 
\end{align*} 
Therefore, we have
\begin{align*}
    &\quad N_x^{-1} \scrL_\lambda^{(N_x)}(\mu_x^{(N_x)}) - \scrL_\lambda(\mu_{x,\lambda}^\ast, \tilde{\mu}_z^*(\mu_{x,\lambda}^\ast) , \mu_z^*(\mu_{x,\lambda}^\ast)) \\
    &\geq N_x^{-1} \sigma_2 \kl(\mu_x^{(N_x)}, (\mu_{x,\lambda}^\ast)^{\otimes N_x}) - N_x^{-1} \frac{2R^3\lambda }{\sigma_1\mathfrak{c}} \kl(\mu_x^{(N_x)}, (\mu_{x,\lambda}^\ast)^{\otimes N_x}) - \frac{2R^3\lambda}{\sigma_1} \mathfrak{c} - \frac{R^3\lambda}{N_x\sigma_1} \\
    &\geq N_x^{-1} \left( \sigma_2 - \frac{2R^3\lambda}{\sigma_1 \mathfrak{c}} \right) \kl(\mu_x^{(N_x)}, (\mu_{x,\lambda}^\ast)^{\otimes N_x}) - \frac{2R^3\lambda}{\sigma_1} \mathfrak{c} - \frac{R^3\lambda}{N_x\sigma_1} \\
    &\geq N_x^{-1} \frac{\sigma_2}{2} \kl(\mu_x^{(N_x)}, (\mu_{x,\lambda}^\ast)^{\otimes N_x}) - \frac{2R^3\lambda}{\sigma_1} \mathfrak{c} - \frac{R^3\lambda}{N_x\sigma_1} .
\end{align*}
% The second last inequality holds by the fact that $(\mu_{x,\lambda}^\ast)^{\otimes N}$ satisfies a LSI inequality with the same LSI constant as $\mu_{x,\lambda}^\ast$~\citep[Proposition 5.2.7]{bakry2013analysis}.
The last inequality holds by the condition that $\sigma_1\sigma_2 \mathfrak{c}> 4R^3 \lambda$.  
\end{proof}

\subsubsection{Proof of \Cref{prop:uniform_LSI}}\label{sec:proof_uniform_LSI}
\begin{proof}[Proof of \Cref{prop:uniform_LSI}]
We follow the leave-one-out argument from \citet{nitanda2025propagation} which is a refinement over previous leave-one-out arguments in \citet{chen2024uniform} and \citet{suzuki2023convergence}. 
Recall from Eq.~\eqref{eq:nabla_proximal_gibbs} that $-\sigma_2 \cdot \nabla_i \log (\mu_\ast^{(N_x)}(\scrX))= \bnabla_1 L_\lambda(\mu_{\scrX}, \tilde{\mu}_{z}^\ast(\mu_{\scrx}), \mu_{z}^\ast(\mu_{\scrx}))(x^{(i)})$ due to the optimality of $\mu_{z}^\ast(\mu_{\scrx})$ and $\tilde{\mu}_{z}^\ast(\mu_{\scrx})$. We start with the fisher information, 
\begin{align}
&\quad - \E_{\scrX \sim \mu_x^{(N_x)}}\left[\left\| \nabla \log \left(\frac{\mu_x^{(N_x)}}{\mu_\ast^{(N_x)}}(\scrX)\right)\right\|^2 \right] \nonumber \\
&= -\sum_{i=1}^{N_x} \E_{\scrX \sim \mu_x^{(N_x)}}\left[\left\|\nabla_i \log \left(\frac{\mu_x^{(N_x)}}{\mu_\ast^{(N_x)}}(\scrX)\right)\right\|^2 \right] \nonumber \\
& =-\sum_{i=1}^{N_x} \E_{\scrX \sim \mu_x^{(N_x)}}\left[\left\|\nabla_i \log \left(\mu_x^{(N_x)}(\scrX)\right) + \frac{1}{\sigma_2} \bnabla_1 L_\lambda(\mu_{\scrX}, \tilde{\mu}_{z}^\ast(\mu_{\scrx}), \mu_{z}^\ast(\mu_{\scrx}))(x^{(i)}) \right\|^2\right] \label{eq:nabla_i_bnabla_G}
\end{align}
Given $N_x$ particles $\scrX$ with a joint distribution $\mu_x^{(N_x)}$, denote by $P_{x^{(i)} \mid \scrX_{-i}}$ the conditional law of $x^{(i)}$ conditioned by $\scrX_{-i}$ and denote by $P_{\scrX_{-i}}$ the marginal law of $\scrX_{-i}$. 
Then, it holds that
\begin{align*}
&\quad \nabla_i \log \left(\mu_x^{(N_x)}(\scrX)\right) = \frac{\nabla_i\left(P_{x^{(i)} \mid \scrX_{-i}}(x^{(i)})P_{\scrX_{-i}}\left(\scrX_{-i}\right)\right)}{P_{x^{(i)} \mid \scrX_{-i}}(x^{(i)})P_{\scrX_{-i}}\left(\scrX_{-i}\right)} \\
&=\frac{\nabla_i P_{x^{(i)} \mid \scrX_{-i}}\left(x^{(i)}\right)}{P_{x^{(i)} \mid \scrX_{-i}}\left(x^{(i)}\right)} = \nabla_i \log \left(P_{x^{(i)} \mid \scrX_{-i}}\left(x^{(i)}\right)\right) .
\end{align*}
Denote by $\mu_{x\cup \scrx_{-i}}$ as the empirical distribution $\mu_{x\cup \scrx_{-i}}=\frac{1}{N}\sum_{j\neq i}\delta_{x^{(j)}}+\frac{1}{N}\delta_{x}$ which augments the leave-one-out distribution with a new different sample $x$. 
Define another proximal Gibbs distribution 
\begin{align}\label{eq:defi_proximal_gibbs_leave_one_out}
    \underline{\mu}_{\scrX, i\mid -i}^{(N_x)}(x \mid \scrx_{-i}) &\propto \exp\left( - \frac{N_x}{\sigma_2} \Big( F_2(\mu_{x\cup \scrx_{-i}}, \tilde{\mu}_{z}^\ast(\mu_{x\cup \scrx_{-i}}) + \lambda\cdot F_1(\mu_{x\cup \scrx_{-i}}, \tilde{\mu}_{z}^\ast(\mu_{x\cup \scrx_{-i}})) \right. \\
    & + \lambda\sigma_1\cdot \mathrm{Ent}(\tilde{\mu}_{z}^\ast(\mu_{x\cup \scrx_{-i}})) - \lambda\cdot F_1(\mu_{x\cup \scrx_{-i}}, \mu_{z}^\ast(\mu_{x\cup \scrx_{-i}})) - \lambda\sigma_1\cdot \mathrm{Ent}(\mu_{z}^\ast(\mu_{x\cup \scrx_{-i}})) \Big) \Bigg) \nonumber . 
\end{align}
Let $\nabla_{x,1}F_2(\mu_{x\cup \scrx_{-i}}, \mu_{z}^\ast(\mu_{x\cup \scrx_{-i}})$ mean taking the derivative of $F_2(\mu_{x\cup \scrx_{-i}}, \mu_{z}^\ast(\mu_{x\cup \scrx_{-i}}))$ with respect to the augmented new sample $x$ in the first argument. From the optimality of $\mu_{z}^\ast, \tilde{\mu}_{z}^\ast$ and the same derivations as in Eq.~\eqref{eq:nabla_proximal_gibbs}, we have
\begin{align*}
    &\quad \sigma_2 \cdot \nabla \log \underline{\mu}_{\scrX, i\mid -i}^{(N_x)}(x \mid \scrx_{-i}) = -N_x \nabla_{x,1} F_2(\mu_{x\cup \scrx_{-i}}, \tilde{\mu}_{z}^\ast(\mu_{x\cup \scrx_{-i}}) - \lambda \cdot N_x \nabla_{x,1} F_1(\mu_{x\cup \scrx_{-i}}, \tilde{\mu}_{z}^\ast(\mu_{x\cup \scrx_{-i}}) \\
    &\qquad - \lambda \cdot N_x\nabla_{x,1} F_1(\mu_{x\cup \scrx_{-i}}, \mu_{z}^\ast( \mu_{x\cup \scrx_{-i}}) \\
    % &= - \bnabla_{1} F_2(\mu_{x\cup \scrx_{-i}},\tilde{\mu}_{z}^\ast(\mu_{x\cup \scrx_{-i}}))(x) - \lambda \bnabla_{1} F_1(\mu_{x\cup \scrx_{-i}}, \tilde{\mu}_{z}^\ast(\mu_{x\cup \scrx_{-i}}))(x) - \lambda \bnabla_{1} F_1(\mu_{x\cup \scrx_{-i}}, \mu_{z}^\ast(\mu_{x\cup \scrx_{-i}}))(x) \\
    &= \bnabla_1 L_\lambda(\mu_{x\cup \scrx_{-i}}, \tilde{\mu}_{z}^\ast(\mu_{x\cup \scrx_{-i}}), \mu_{z}^\ast(\mu_{x\cup \scrx_{-i}}))(x) .
\end{align*}
Here, $L_\lambda$ is defined in Eq.~\eqref{eq:defi_L_lambda_new}. 
An immediate consequence of the above derivations is that, taking $x=x^{(i)}$, then $\sigma_2 \nabla \log \underline{\mu}_{\scrX, i\mid -i}^{(N_x)}(x^{(i)} \mid \scrx_{-i}) = \bnabla_1 L_\lambda(\mu_{\scrx}, \tilde{\mu}_{z}^\ast(\mu_{\scrx}), \mu_{z}^\ast(\mu_{\scrx}))(x^{(i)})$. 
Also, notice that $\bnabla_{1} F_2(\mu_{\scrx}, \mu_{z}^\ast(\mu_{\scrx}))(x) = \zeta_2 x$, so from the same argument as in \Cref{lem:LSI}, the proximal Gibbs distribution $\underline{\mu}_{\scrX, i\mid -i}^{(N_x)}$ satisfies a Log-Sobolev inequality with constant $C_{\LSI,x}$. 
So we can proceed from Eq.~\eqref{eq:nabla_i_bnabla_G} above to obtain
\begin{align}
    \eqref{eq:nabla_i_bnabla_G} &= -\sum_{i=1}^{N_x} \E_{\scrX_{-i} \sim P_{\scrX_{-i}}} \left[ \E_{x^{(i)}\sim P_{x^{(i)} \mid \scrX_{-i}}} \left[\left\| \nabla_i \log \left(P_{x^{(i)} \mid \scrX_{-i}}(x^{(i)})\right) + \nabla \log \underline{\mu}_{\scrX, i\mid -i}^{(N_x)}(x^{(i)} \mid \scrx_{-i}) \right\|^2\right] \right] \nonumber \\
    &\leq - 2 C_{\LSI,x} \sum_{i=1}^{N_x} \E_{\scrX_{-i} \sim P_{\scrX_{-i}}} \left[ \kl \left(P_{x^{(i)} \mid \scrX_{-i}}, \; \underline{\mu}_{\scrX, i\mid -i}^{(N_x)} \right)\right] \label{eq:KL_term}.
\end{align}
Recall the definition of $L_\lambda$ in Eq.~\eqref{eq:defi_L_lambda_new},  
\begin{align}
    \quad L_\lambda(\mu_x, \tilde{\mu}_z^\ast(\mu_{x}), \mu_z^\ast(\mu_{x})) &= F_2(\mu_x, \tilde{\mu}_z^\ast(\mu_{x})) + \lambda \cdot \scrF_1(\mu_x, \tilde{\mu}_z^\ast(\mu_{x})) - \lambda \cdot \scrF_1(\mu_x, \mu_z^\ast(\mu_{x}))  \nonumber \\
    &\hspace{-20pt} = U_2(\tilde{\mu}_z^\ast(\mu_{x})) + \lambda \cdot \scrF_1(\mu_x, \tilde{\mu}_z^\ast(\mu_{x}))  - \lambda \cdot \scrF_1(\mu_x, \mu_z^\ast(\mu_{x})) + \frac{\zeta_2}{2} \E_{\mu_x}[\|x\|^2] \nonumber \\
    &\hspace{-20pt} =: U_\lambda(\mu_x, \tilde{\mu}_z^\ast(\mu_{x}), \mu_z^\ast(\mu_{x})) + \frac{\zeta_2}{2} \E_{\mu_x}[\|x\|^2] \label{eq:defi_U_lambda}. 
\end{align}
Consider the first variation of $U_\lambda$ with respect to $\mu_x$, denoted as $\delta_{\mu_{x}} U_\lambda$, which is a mapping from $\R^{d_x}\to\R$. 
From the optimality of $\tilde{\mu}_z^\ast(\mu_{x})$ and $\mu_z^\ast(\mu_{x})$, we have
\begin{align*}
    \delta_{\mu_x} \Big( U_2(\tilde{\mu}_z^\ast(\mu_{x})) + \lambda \cdot F_1(\mu_x, \tilde{\mu}_z^\ast(\mu_{x})) + \lambda \sigma_1 \cdot \mathrm{Ent}\left(\tilde{\mu}_z^\ast(\mu_{x})\right) \Big)(x) &= \lambda \cdot \delta_1 F_1(\mu_x, \tilde{\mu}_z^\ast(\mu_{x}))(x) + \text{const} \\
    &= \lambda \cdot \delta_1 U_1(\mu_x, \tilde{\mu}_z^\ast(\mu_{x}))(x) + \text{const}, \\ 
    \delta_{\mu_x} \Big( F_1(\mu_x, \mu_z^\ast(\mu_{x})) + \sigma_1 \mathrm{Ent}\left(\mu_z^\ast(\mu_{x})\right) \Big)(x) &= \delta_1 F_1(\mu_x, \mu_z^\ast(\mu_{x}))(x) + \text{const} \\
    &= \delta_1 U_1(\mu_x, \mu_z^\ast(\mu_{x}))(x) + \text{const}. 
\end{align*}
Here, the constants are independent of $\mu_x, x$. Hence, we can conclude that the first variation of $U_\lambda$ with respect to $\mu_x$ is:
\begin{align}\label{eq:first_var_U_lambda}
    \delta_{\mu_{x}} U_\lambda(\mu_x, \tilde{\mu}_z^\ast(\mu_{x}), \mu_z^\ast(\mu_{x}))(\cdot) = \lambda \cdot \delta_1 U_1(\mu_x, \tilde{\mu}_z^\ast(\mu_{x}))(\cdot) - \lambda \cdot \delta_1 U_1(\mu_x, \mu_z^\ast(\mu_{x}))(\cdot) + \text{const} .
\end{align}
Based on the partial convexity of $U_1$ and $U_2$ proved in \Cref{prop:partial_convex}, we can see that the mapping from $\mu_x \mapsto U_\lambda(\mu_x, \tilde{\mu}_z, \mu_z)$ is convex if the latter $\tilde{\mu}_z, \mu_z$ are independent of $\mu_x$. Unfortunately however, this is not the case here as both $\tilde{\mu}_z^\ast(\mu_x), \mu_z^\ast(\mu_x)$ explicitly depends on $\mu_x$. As a result, the mapping $\mu_x \mapsto U_\lambda(\mu_x, \tilde{\mu}_z^\ast(\mu_x), \mu_z^\ast(\mu_x))$ is \emph{not convex}, and we must account carefully for this lack of convexity in our analysis.

% Define the following auxiliary proximal Gibbs distribution:
% \begin{align}
%     &\underline{\mu}_{\scrX_{-i}}(x^{(i)}) \propto \exp\left( - \frac{1}{\sigma_2} \delta L_\lambda(\mu_{\scrX_{-i}}, \tilde{\mu}_z^\ast(\mu_{\scrx,-i}), \mu_z^\ast(\mu_{\scrx,-i})))(x^{(i)}) \right) \nonumber \\
%     &= \underset{\mu_x \in \calP_2(\R^{d_x})}{\operatorname{argmin}} \left\{\int \delta U_\lambda\left(\mu_{\scrX_{-i}}, \; \tilde{\mu}_z^\ast(\mu_{\scrx,-i}), \mu_z^\ast(\mu_{\scrx,-i}))\right) (x^{(i)})\; \dd \mu_x + \frac{\zeta_2}{2} \E_{\mu_x} \left[ \|x^{(i)}\|^2 \right] + \sigma \operatorname{Ent}(\mu_x) \right\} \nonumber\\
%     &= \underset{\mu_x \in \calP_2(\R^{d_x})}{\operatorname{argmin}} \left\{\int \delta U_\lambda\left(\mu_{\scrX_{-i}}, \; \tilde{\mu}_z^\ast(\mu_{\scrx,-i}), \mu_z^\ast(\mu_{\scrx,-i}))\right)  (x^{(i)})\; \dd \mu_x + \sigma \kl(\mu_x, \nu) \right\} \label{eq:optimal_under_mu_neg_i},
% \end{align}
% which satisfies the Log-Sobolev inequality with $C_{\LSI,z}$. 
By the optimality of the proximal Gibbs distribution $\underline{\mu}_{\scrX, i\mid -i}^{(N_x)}$ defined in Eq.~\eqref{eq:defi_proximal_gibbs_leave_one_out}, we have 
\begin{align}
    &\underline{\mu}_{\scrX, i\mid -i}^{(N_x)}(x \mid \scrx_{-i}) \propto \exp\left( - \frac{1}{\sigma_2} \delta_{\mu_x} L_\lambda \Big(\mu_{x\cup \scrx_{-i}}, \tilde{\mu}_z^\ast(\mu_{x\cup \scrx_{-i}}), \mu_z^\ast(\mu_{x\cup \scrx_{-i}}) \Big)(x) \right) \nonumber \\
    &= \underset{\mu_x \in \calP_2(\R^{d_x})}{\operatorname{argmin}} \left\{\int \delta_{\mu_{x}} U_\lambda\left(\mu_{x\cup \scrx_{-i}}, \; \tilde{\mu}_z^\ast(\mu_{x\cup \scrx_{-i}}), \mu_z^\ast(\mu_{x\cup \scrx_{-i}}))\right)(x) \; \dd \mu_x(x) + \frac{\zeta_2}{2} \E_{\mu_x} \left[ \|x\|^2 \right] + \sigma \operatorname{Ent}(\mu_x) \right\} \nonumber\\
    &= \underset{\mu_x \in \calP_2(\R^{d_x})}{\operatorname{argmin}} \left\{\int \delta_{\mu_{x}} U_\lambda\left(\mu_{x\cup \scrx_{-i}}, \; \tilde{\mu}_z^\ast(\mu_{x\cup \scrx_{-i}}), \mu_z^\ast(\mu_{x\cup \scrx_{-i}}))\right)(x) \; \dd \mu_x(x) + \sigma_2 \kl(\mu_x, \nu) \right\} \label{eq:optimal_under_mu_neg_i},
\end{align}
Here, $\nu(x)\propto\exp( -\frac{\zeta_2}{2 \sigma_2} \|x\|^2)$ is a Gaussian distribution.  
To proceed, we obtain the following lower bound of the KL divergence in Eq.~\eqref{eq:KL_term},
\begin{align}
    & \sigma_2 \kl \left(P_{x^{(i)} \mid \scrX_{-i}}, \underline{\mu}_{\scrX, i\mid -i}^{(N_x)}\right) \nonumber \\
    = & \int \delta U_\lambda \left(\mu_{x\cup \scrx_{-i}},  \tilde{\mu}_z^\ast(\mu_{x\cup \scrx_{-i}}), \mu_z^\ast(\mu_{x\cup \scrx_{-i}})) \right)(x) \left(P_{x^{(i)} \mid \scrX_{-i}} -\underline{\mu}_{\scrX, i\mid -i}^{(N_x)} \right) (x)\;\dd x \nonumber \\
    &\qquad\qquad + \sigma_2 \kl \left(P_{x^{(i)} \mid \scrX_{-i}}, \nu \right) - \sigma_2 \kl \left(\underline{\mu}_{\scrX, i\mid -i}^{(N_x)}, \nu \right) \nonumber \\
    \geq & \int \delta U_\lambda \left(\mu_{x\cup \scrx_{-i}},  \tilde{\mu}_z^\ast(\mu_{x\cup \scrx_{-i}}), \mu_z^\ast(\mu_{x\cup \scrx_{-i}}) \right)(x) \left( P_{x^{(i)} \mid \scrX_{-i}} - \mu_{x,\lambda}^\ast \right)(x) \; \dd x \nonumber \\
    &\qquad\qquad + \sigma_2 \kl\left(P_{x^{(i)} \mid \scrX_{-i}}, \nu \right) - \sigma_2 \kl \left(\mu_{x,\lambda}^\ast, \nu\right) \label{eq:KL_U_two_terms}. 
\end{align}
The last inequality holds by the optimality of the proximal Gibbs distribution $\underline{\mu}_{\scrX, i\mid -i}^{(N_x)}$. 
Take the expectation $\E_{\scrX_{-i} \sim P_{\scrX_{-i}}}$ with respect to the first term of the right hand side in Eq.~\eqref{eq:KL_U_two_terms} above,
we then proceed to obtain a further lower bound by the following:
\begin{small}
\begin{align}
& \sum_{i=1}^{N_x} \E_{\scrX_{-i} \sim P_{\scrX_{-i}}} \left[\int \delta_{\mu_{x}} U_\lambda \left(\mu_{x\cup \scrx_{-i}},  \tilde{\mu}_z^\ast(\mu_{x\cup \scrx_{-i}}), \mu_z^\ast(\mu_{x\cup \scrx_{-i}}) \right)(x) \left(P_{x^{(i)} \mid \scrX_{-i}}-\mu_{x,\lambda}^\ast\right) (x) \; \dd x \right] \nonumber \\
= & \sum_{i=1}^{N_x} \E_{\scrX \sim \mu_x^{(N_x)}} \left[ \delta_{\mu_{x}} U_\lambda \left(\mu_{\scrX},  \tilde{\mu}_z^\ast(\mu_{\scrx}), \mu_z^\ast(\mu_{\scrx}) \right)(x^{(i)}) - \int \delta_{\mu_{x}} U_\lambda\left(\mu_{x\cup \scrx_{-i}}, \tilde{\mu}_z^\ast(\mu_{x\cup \scrx_{-i}}), \mu_z^\ast(\mu_{x\cup \scrx_{-i}})\right)(x) \; \dd \mu_{x,\lambda}^\ast(x) \right] \nonumber \\
\geq & \sum_{i=1}^{N_x}  \E_{\scrX \sim \mu_x^{(N_x)}}\left[\int \delta_{\mu_{x}} U_\lambda(\mu_{\scrX},  \tilde{\mu}_z^\ast(\mu_{\scrx}), \mu_z^\ast(\mu_{\scrx}))(x) \; \dd \mu_{\scrX}(x) - \int \delta_{\mu_{x}} U_\lambda(\mu_{\scrX},  \tilde{\mu}_z^\ast(\mu_{\scrx}), \mu_z^\ast(\mu_{\scrx}))(x) \; \dd \mu_{x,\lambda}^\ast(x) \right] \nonumber \\
&\qquad + 2\lambda \left( R + \sqrt{\frac{2}{\sigma_1}}  R^2\right) \label{eq:diff_U_lambda_temp_1} \\
= & N_x \E_{\scrX \sim \mu_x^{(N_x)}}\left[ B_{U_\lambda}(\mu_{\scrX}, \mu_{x,\lambda}^\ast) \right] + N_x \E_{\scrX \sim \mu_x^{(N_x)}} \left[U_\lambda(\mu_{ \scrX}, \tilde{\mu}_z^\ast(\mu_{\scrx}), \mu_z^\ast(\mu_{\scrx})) \right] \nonumber\\ &\qquad - N_x \cdot U_\lambda(\mu_{x,\lambda}^\ast, \tilde{\mu}_z^\ast(\mu_{x,\lambda}^\ast), \mu_z^\ast(\mu_{x,\lambda}^\ast)) + 2\lambda \left( R + \sqrt{\frac{2}{\sigma_1}}  R^2\right) .
\label{eq:diff_U_lambda_temp} 
\end{align}
\end{small}
The second last inequality Eq.~\eqref{eq:diff_U_lambda_temp_1} holds because of the following upper bound on the difference: for each $i\in\{1,\ldots,N_x\}$, 
\begin{align*}
    &\quad \left| \delta_{\mu_{x}} U_\lambda\left(\mu_{x\cup\scrX_{-i}},  \tilde{\mu}_z^\ast(\mu_{x\cup\scrX_{-i}}), \mu_z^\ast(\mu_{x\cup\scrX_{-i}})\right)(x^{(i)}) - \delta_{\mu_{x}} U_\lambda\left(\mu_{\scrX} \; , \tilde{\mu}_z^\ast(\mu_{\scrx}), \mu_z^\ast(\mu_{\scrx})\right)(x^{(i)}) \right| \nonumber \\
    &\leq \lambda \left| \delta_1 U_1(\mu_{x\cup\scrX_{-i}}, \tilde{\mu}_z^\ast(\mu_{x\cup\scrX_{-i}}))(x^{(i)}) - \delta_1 U_1(\mu_{\scrX}, \tilde{\mu}_z^\ast(\mu_{\scrX}))(x^{(i)}) \right| \nonumber \\
    &\qquad \qquad + \lambda \left| \delta_1 U_1(\mu_{x\cup\scrX_{-i}}, \mu_z^\ast(\mu_{x\cup\scrX_{-i}}))(x^{(i)}) - \delta_1 U_1(\mu_{\scrX}, \mu_z^\ast(\mu_{\scrX}))(x^{(i)}) \right| \\
    &\leq \frac{2\lambda}{N_x} \left( R + \sqrt{\frac{2}{\sigma_1}}  R^2\right),
\end{align*}
where the first inequality holds from the expression the first variation of $U_\lambda$ derived in Eq.~\eqref{eq:first_var_U_lambda}, and the second inequality holds by using 
\Cref{prop:leave_one_out_lipschtiz_delta_Ux}. 
To proceed from Eq.~\eqref{eq:diff_U_lambda_temp} above, we can keep obtaining a lower bound using the lower bound on the Bregman divergence $B_{U_\lambda}$ proved in \Cref{lem:bregman}.  
\begin{align*}
    \eqref{eq:diff_U_lambda_temp} &\geq N_x \E_{\scrX \sim \mu_x^{(N_x)}} \left[U_\lambda(\mu_{ \scrX}, \tilde{\mu}_z^\ast(\mu_{\scrx}), \mu_z^\ast(\mu_{\scrx})) \right]  - N_x \cdot U_\lambda(\mu_{x,\lambda}^\ast, \tilde{\mu}_z^\ast(\mu_{x,\lambda}^\ast), \mu_z^\ast(\mu_{x,\lambda}^\ast)) \\
    &\qquad - \frac{N_x R \lambda}{4\sigma_1} \cdot \E_{\scrX \sim \mu_x^{(N_x)}} \left[ \E_\rho\left[\left(\smallint \Psi_{\ba} \; \dd\left(\mu_\scrx -\mu_{x,\lambda}^\ast \right)\right)^2\right] \right] + 2\lambda \left( R + \sqrt{\frac{2}{\sigma_1}}  R^2\right) \\
    &\geq N_x \E_{\scrX \sim \mu_x^{(N_x)}} \left[U_\lambda(\mu_{ \scrX}, \tilde{\mu}_z^\ast(\mu_{\scrx}), \mu_z^\ast(\mu_{\scrx})) \right]  - N_x \cdot U_\lambda(\mu_{x,\lambda}^\ast, \tilde{\mu}_z^\ast(\mu_{x,\lambda}^\ast), \mu_z^\ast(\mu_{x,\lambda}^\ast)) \\
    &\qquad - \frac{N_x R^3 \lambda}{4\sigma_1}\cdot \left( 8\sqrt{ N_x^{-1}\kl(\mu_x^{(N_x)}, (\mu_{x,\lambda}^\ast)^{\otimes N_x})} + 4 N_x^{-1}\right) + 2 \lambda \left( R + \sqrt{\frac{2}{\sigma_1}}  R^2\right) \\
    &\geq N_x \E_{\scrX \sim \mu_x^{(N_x)}} \left[U_\lambda(\mu_{ \scrX}, \tilde{\mu}_z^\ast(\mu_{\scrx}), \mu_z^\ast(\mu_{\scrx})) \right]  - N_x \cdot U_\lambda(\mu_{x,\lambda}^\ast, \tilde{\mu}_z^\ast(\mu_{x,\lambda}^\ast), \mu_z^\ast(\mu_{x,\lambda}^\ast)) \\
    &\qquad - \frac{N_x R^3 \lambda}{4\sigma_1}\cdot \left( 8\mathfrak{c}^{-1} N_x^{-1}\kl(\mu_x^{(N_x)}, (\mu_{x,\lambda}^\ast)^{\otimes N_x}) + 2\mathfrak{c} + 4 N_x^{-1}\right) + 2 \lambda \left( R + \sqrt{\frac{2}{\sigma_1}}  R^2\right). 
\end{align*}
The second last inequality holds by using Proposition 1 of \citet{nitanda2025propagation}, and the last inequality holds by 
\begin{align*}
    \sqrt{N_x^{-1}\kl(\mu_x^{(N_x)}, (\mu_{x,\lambda}^\ast)^{\otimes N_x})} \leq \mathfrak{c}^{-1} N_x^{-1}\kl(\mu_x^{(N_x)}, (\mu_{x,\lambda}^\ast)^{\otimes N_x}) + \frac{1}{4} \mathfrak{c} . 
\end{align*} 
Now we have successfully obtained a lower bound of the first term in Eq.~\eqref{eq:KL_U_two_terms}, next we are about to lower bound the second term in Eq.~\eqref{eq:KL_U_two_terms}. 

Take the expectation of the second term in Eq.~\eqref{eq:KL_U_two_terms}, we have
\begin{align*}
    \sum_{i=1}^{N_x} \E_{\scrx\sim\mu_x^{(N_x)}} \left[ \kl\left(P_{x^{(i)} \mid \scrX_{-i}}, \nu \right) \right] & =\sum_{i=1}^{N_x} \E_{\mu_x^{(N_x)}} \left[ \operatorname{Ent} \left( P_{x^{(i)} \mid \scrx_{-i} } \right) \right]-\E_{\mu_x^{(N_x)}} \left[ \sum_{i=1}^{N_x} \log \left( \nu(x^{(i)}) \right)\right] \\
    & \geq \operatorname{Ent} \left(\mu_x^{(N_x)} \right)-\E_{\mu_x^{(N_x)}} \left[\sum_{i=1}^{N_x} \log \left(\nu (x^{(i)})\right)\right] \\
    &= \kl \left(\mu_x^{(N_x)}, \nu^{\otimes N_x}\right),
\end{align*}
where the inequality hold by using Lemma 3.6 of \citet{chen2024uniform}. 
Combing all the above together results in a lower bound of Eq.~\eqref{eq:KL_U_two_terms}.  If we plug it back to Eq.~\eqref{eq:KL_term} and then back to Eq.~\eqref{eq:nabla_i_bnabla_G}, then we obtain
\begin{align}
    &\quad -\frac{1}{N_x} \sum_{i=1}^{N_x} \E_{\mu_x^{(N_x)}} \left[ \left\| \nabla_i \log \left(\frac{\mu_x^{(N_x)}}{\mu_\ast^{(N_x)}}(\scrX) \right) \right\|^2 \right] \nonumber \\
    &\leq -\frac{2C_{\LSI,x}}{\sigma_2 N_x} \left( N_x \E_{\scrX \sim \mu_x^{(N_x)}} \left[U_\lambda(\mu_{ \scrX}, \tilde{\mu}_z^\ast(\mu_{\scrx}), \mu_z^\ast(\mu_{\scrx})) - U_\lambda(\mu_{x,\lambda}^\ast, \tilde{\mu}_z^\ast(\mu_{x,\lambda}^\ast), \mu_z^\ast(\mu_{x,\lambda}^\ast)) \right] \right. \nonumber \\
    &\quad + \left. \sigma_2 \kl \left(\mu_x^{(N_x)}, \nu^{\otimes N_x}\right) - N_x \sigma_2 \kl \left(\mu_{x,\lambda}^\ast, \nu\right) \right)  + \frac{2 C_{\LSI,x} \lambda}{\sigma_2 N_x} \left(R^2 + \sqrt{\frac{2}{\sigma_1}}R^3\right) \nonumber \\
    &\qquad + \frac{C_{\LSI,x}}{\sigma_2 N_x }\frac{R^3 \lambda}{\sigma_1} \cdot \left( 2\mathfrak{c}^{-1} \kl(\mu_x^{(N_x)}, (\mu_{x,\lambda}^\ast)^{\otimes N_x} ) + \frac{\mathfrak{c}}{2} N_x + 1 \right)  \nonumber \\
    &= -\frac{2C_{\LSI,x}}{\sigma_2 N_x} \left( \scrL_\lambda^{(N_x)}\left(\mu_x^{(N_x)}\right) - N_x \scrL_\lambda(\mu_{x,\lambda}^\ast, \tilde{\mu}_z^*(\mu_{x,\lambda}^\ast), \mu_z^*(\mu_{x,\lambda}^\ast))\right) + \frac{2 \lambda C_{\LSI,x}\left(R^2 + \sqrt{\frac{2}{\sigma_1}}R^3\right) }{\sigma_2 N_x} \nonumber \\
    &\qquad+ \frac{R^3 \lambda C_{\LSI,x}}{\sigma_1\sigma_2 N_x} \cdot \left( 2\mathfrak{c}^{-1} \kl(\mu_x^{(N_x)}, (\mu_{x,\lambda}^\ast)^{\otimes N_x} ) + \frac{\mathfrak{c}}{2} N_x + 1 \right) 
    \label{eq:descent_intermediate}.
\end{align}
\begin{rem}[Comparison with the defective LSI from standard mean field Langevin dynamics]
\label{rem:compare_with_standard_LSI}
    We compare Eq.~\eqref{eq:descent_intermediate} with the defective LSI commonly used in the analysis of standard mean field Langevin dynamics with a convex objective functional (e.g., Lemma 1 of \citet{nitanda2025propagation}). The first two terms on the right-hand side of Eq.~\eqref{eq:descent_intermediate} represent the descent term and the finite particle approximation error term, which align with the corresponding terms in Lemma 1 of \citet{nitanda2025propagation}. However, our analysis includes an additional positive third term, which arises due to the \emph{non-convexity} of the objective functional $U_\lambda$ and $L_\lambda$.
\end{rem}
To proceed from Eq.~\eqref{eq:descent_intermediate}, we next upper bound $N_x^{-1}\kl(\mu_x^{(N_x)}, (\mu_{x,\lambda}^\ast)^{\otimes N_x} )$ with \Cref{prop:KL_upper_bound_functional}.
\begin{small}
\begin{align*}
    N_x^{-1} \kl(\mu_x^{(N_x)}, (\mu_{x,\lambda}^\ast)^{\otimes N_x} ) \leq \frac{2}{\sigma_2} \left( N_x^{-1} \scrL_\lambda^{(N_x)}(\mu_x^{(N_x)}) - \scrL_\lambda(\mu_{x,\lambda}^\ast, \tilde{\mu}_z^\ast(\mu_{x,\lambda}^\ast), \mu_z^\ast(\mu_{x,\lambda}^\ast)) \right) + \frac{4R^3\lambda}{N_x \sigma_1\sigma_2} + \frac{4R^3\lambda\mathfrak{c}}{\sigma_1\sigma_2} .
\end{align*}
\end{small}
Therefore, plugging the above back to Eq.~\eqref{eq:descent_intermediate}, we have
\begin{align*}
    &\quad -\frac{1}{N_x} \sum_{i=1}^{N_x} \E_{\mu_x^{(N_x)}} \left[ \left\| \nabla_i \log \left(\frac{\mu_x^{(N_x)}}{\mu_\ast^{(N_x)}}(\scrX) \right) \right\|^2 \right] \\
    &\leq-\frac{2C_{\LSI,x}}{\sigma_2 N_x} \left( \scrL_\lambda^{(N_x)}\left(\mu_x^{(N_x)}\right) - N_x \scrL_\lambda(\mu_{x,\lambda}^\ast, \tilde{\mu}_z^*(\mu_{x,\lambda}^\ast), \mu_z^*(\mu_{x,\lambda}^\ast))\right) + \frac{\lambda C_{\LSI,x} R^2 \left(\frac{R}{\sigma_1} + 2 + \sqrt{\frac{2}{\sigma_1}}R \right) }{\sigma_2 N_x}  \nonumber \\
    &\qquad + \frac{C_{\LSI,x}}{\sigma_2 N_x} \left( \scrL_\lambda^{(N_x)}\left(\mu_x^{(N_x)}\right) - N_x \scrL_\lambda(\mu_{x,\lambda}^\ast, \tilde{\mu}_z^*(\mu_{x,\lambda}^\ast), \mu_z^*(\mu_{x,\lambda}^\ast))\right) + \frac{\mathfrak{c}^2 C_{\LSI,x}}{8} \\
    &\leq -\frac{C_{\LSI,x}}{\sigma_2 N_x} \left( \scrL_\lambda^{(N_x)} \left(\mu_x^{(N_x)}\right) - N_x \scrL_\lambda(\mu_{x,\lambda}^\ast, \tilde{\mu}_z^*(\mu_{x,\lambda}^\ast), \mu_z^*(\mu_{x,\lambda}^\ast))\right) + \frac{\lambda C_{\LSI,x} R^2 \left(\frac{ R}{\sigma_1} + 2 + \sqrt{\frac{2}{\sigma_1}}R \right) }{\sigma_2 N_x} \\
    &\qquad+ \frac{\mathfrak{c}^2 C_{\LSI,x}}{8} . 
\end{align*}
The proof is concluded by taking expectation with respect to $\E_{\scrz \sim \mu_{z,s}^{(N_z)}, \; \tilde{\scrZ} \sim \tilde{\mu}_{z,s}^{(N_z)} }$ on both sides. 
\end{proof}

\begin{lem}\label{lem:bnabla_G_lip_3}
Suppose \Cref{ass:network} holds. Let $L_\lambda$ be as defined in \eqref{eq:defi_L_lambda_new} where $L_\lambda(\mu_x, \tilde{\mu}_z, \mu_z) = F_2(\mu_x, \tilde{\mu}_z) + \lambda \cdot F_1(\mu_x, \tilde{\mu}_z) + \lambda \sigma_1 \cdot \mathrm{Ent}(\tilde{\mu}_z) - \lambda \cdot F_1(\mu_x, \mu_z) - \lambda \sigma_1 \cdot \mathrm{Ent}(\mu_z) $. 
Consider the following particle system: 
for $i=1,\ldots,N_x$ and $0 < \tau \leq \gamma$,  
\begin{align}\label{eq:update_scheme_in_lemma_1}
    \hat{x}_\tau^{(i)} = \hat{x}_0^{(i)} - \tau \cdot \bnabla_1 L_\lambda(\hat{\mu}_{\scrx, 0}, \tilde{\mu}_{\scrz, s} , \mu_{\scrz, s})(\hat{x}_0^{(i)}) + \sqrt{2\sigma_2 \tau} \xi_{x}^{(i)},
\end{align}
where $\{\xi_{x}^{(i)}\}_{i=1}^{N_x}$ are $N_x$ i.i.d unit normal random variables. Then, for any $x\in\R^{d_x}$, we have
\begin{align*}
    &\left\| \bnabla_1 L_\lambda(\hat{\mu}_{\scrx,\tau}, \tilde{\mu}_z^\ast(\hat{\mu}_{\scrx,\tau}), \mu_z^\ast(\hat{\mu}_{\scrx,\tau}))(x) - \bnabla_1 L_\lambda(\hat{\mu}_{\scrx,\tau}, \tilde{\mu}_z^\ast(\hat{\mu}_{\scrx,0}), \mu_z^\ast(\hat{\mu}_{\scrx,0}))(x) \right\|^2 \\
    &\leq \frac{\lambda^2 R^6}{\sigma_1} \left( \tau^2 \left(\zeta_2^2 \frac{1}{N_x} \sum_{i=1}^{N_x} \E[\|\hat{x}_0^{(i)}\|^2] + \lambda^2 R^2 \right) + \tau \sigma_2 d_x \right) .
\end{align*}
\end{lem}
\begin{proof}
From the definition of $L_\lambda$, we have the following relation on its Wasserstein gradient $\bnabla_1 L_\lambda(\mu_{\scrx, 0}, \tilde{\mu}_{\scrz, s} , \mu_{\scrz, s})$: 
\begin{align}\label{eq:bnabla_G_lip_3_bound_first}
    & \left\| \bnabla_1 L_\lambda(\mu_{\scrx,\tau}, \tilde{\mu}_z^\ast(\mu_{\scrx,\tau}), \mu_z^\ast(\mu_{\scrx,\tau}))(x) - \bnabla_1 L_\lambda(\mu_{\scrx,\tau}, \tilde{\mu}_z^\ast(\mu_{\scrx,0}), \mu_z^\ast(\mu_{\scrx,0}))(x) \right\|^2 \nonumber \\
    &= \Big\| \zeta_2 x + \lambda \cdot \bnabla_1 U_1(\mu_{\scrx,\tau}, \tilde{\mu}_z^\ast(\mu_{\scrx,\tau}))(x) - \lambda \cdot \bnabla_1 U_1(\mu_{\scrx,\tau}, \mu_z^\ast(\mu_{\scrx,\tau}))(x) \nonumber\\
    &\qquad - \zeta_2 x + \lambda \cdot \bnabla_1 U_1(\mu_{\scrx,\tau}, \tilde{\mu}_z^\ast(\mu_{\scrx,0}))(x) + \lambda \cdot \bnabla_1 U_1(\mu_{\scrx,\tau}, \mu_z^\ast(\mu_{\scrx,0}))(x) \Big\|^2 \nonumber\\
    &\leq \lambda \left\| \bnabla_1 U_1(\mu_{\scrx,\tau}, \tilde{\mu}_z^\ast(\mu_{\scrx,\tau}))(x) - \bnabla_1 U_1(\mu_{\scrx,\tau}, \tilde{\mu}_z^\ast(\mu_{\scrx,0}))(x) \right\|^2 \nonumber \\
    &\qquad + \lambda \left\| \bnabla_1 U_1(\mu_{\scrx,\tau}, \mu_z^\ast(\mu_{\scrx,\tau}))(x) - \bnabla_1 U_1(\mu_{\scrx,\tau}, \mu_z^\ast(\mu_{\scrx,0}))(x) \right\|^2 \nonumber \\
    &\leq \lambda^2 R^4 \cdot \Big( \tv^2(\tilde{\mu}_z^\ast(\mu_{\scrx,\tau}), \tilde{\mu}_z^\ast(\mu_{\scrx,0})) + \tv^2(\mu_z^\ast(\mu_{\scrx,\tau}), \mu_z^\ast(\mu_{\scrx,0})) \Big). 
\end{align}
The last inequality holds from the Lipschitzness of $\bnabla_1 U_1$ proved in \Cref{lem:lip_wass_grad}. 
Next, from the Pinsker's inequality and \Cref{prop:continuity}, we have
\begin{align*}
    \tv^2(\tilde{\mu}_z^\ast(\mu_{\scrx,0}), \tilde{\mu}_z^\ast(\mu_{\scrx,\tau})) \leq 2 \kl(\tilde{\mu}_z^\ast(\mu_{\scrx,0}), \tilde{\mu}_z^\ast(\mu_{\scrx,\tau})) &\leq \frac{R^2}{2\sigma_1} \left[ W_2^2(\mu_{\scrx,0}, \mu_{\scrx,\tau})\right] .
\end{align*}
From the update scheme in Eq.~\eqref{eq:update_scheme_in_lemma_1}, we have 
\begin{align}
    &\quad \E_{\hat{\scrx}_0 \sim \hat{\mu}_{\scrx, 0}, \; \hat{\scrX}_\tau \sim \hat{\mu}_{\scrx, \tau}} \left[ W_2^2(\mu_{\scrx,0}, \mu_{\scrx,\tau})\right] \nonumber \\
    &\leq \frac{1}{N_x} \sum_{i=1}^{N_x} \E_{\hat{\scrx}_0 \sim \hat{\mu}_{\scrx, 0}, \; \hat{\scrX}_\tau \sim \hat{\mu}_{\scrx, \tau}} \left[ \|\hat{x}_0^{(i)} - \hat{x}_\tau^{(i)}\|^2 \right] \nonumber \\
    &= \frac{1}{N_x} \sum_{i=1}^{N_x} \E \left[ \left\| \tau \bnabla_1 L_\lambda(\mu_{\scrx, 0}, \tilde{\mu}_{\scrz, s} , \mu_{\scrz, s})(\hat{x}_0^{(i)}) + \sqrt{2\sigma_2 \tau}  \xi_{x}^{(i)} \right\|^2 \right] \nonumber \\
    &= \frac{1}{N_x} \sum_{i=1}^{N_x} \tau^2 \E \left[ \left\| \bnabla_1 L_\lambda(\mu_{\scrx, 0}, \tilde{\mu}_{\scrz, s} , \mu_{\scrz, s})(\hat{x}_0^{(i)})\right\|^2 \right] + 2\tau \sigma_2 d_x 
    \label{eq:wass_distance_iterates}. 
\end{align}
For $i=1,\ldots,N_x$, 
\begin{align*}
    &\quad \E \left[ \left\| \bnabla_1 L_\lambda(\mu_{\scrx, 0}, \tilde{\mu}_{\scrz, s} , \mu_{\scrz, s})(\hat{x}_0^{(i)}) \right\|^2 \right] \\
    &= \E \left[ \left\| \zeta_2 \hat{x}_0^{(i)} + \lambda \cdot \bnabla_1 U_1(\mu_{\scrx,0}, \mu_{\scrz, s}) (\hat{x}_0^{(i)}) - \lambda \cdot \bnabla_1 U_1(\mu_{\scrx,0}, \tilde{\mu}_{\scrz, s})(\hat{x}_0^{(i)}) \right\|^2 \right]\\
    &\lesssim \zeta_2^2 \E[\|\hat{x}_0^{(i)}\|^2] + \lambda^2 R^2 . 
\end{align*}
The last inequality holds from the fact that $\bnabla_1 U_1$ is bounded as per \Cref{ass:network}. 
Combined, we obtain
\begin{align}\label{eq:W_2_bound_tentative}
    \tv^2(\tilde{\mu}_z^\ast(\mu_{\scrx,0}), \tilde{\mu}_z^\ast(\mu_{\scrx,\tau})) &\leq \frac{R^2}{2\sigma_1} W_2^2(\mu_{\scrx,0}, \mu_{\scrx,\tau}) \nonumber \\
    &\lesssim \frac{R^2}{\sigma_1} \left( \tau^2 \left(\zeta_2^2 \frac{1}{N_x} \sum_{i=1}^{N_x} \E[\|\hat{x}_0^{(i)}\|^2] + \lambda^2 R^2 \right) + \tau \sigma_2 d_x \right) .
\end{align}
Similarly, we also have
\begin{align}\label{eq:W_2_bound_tentative_tilde}
    \tv^2(\mu_z^\ast(\mu_{\scrx,0}), \mu_z^\ast(\mu_{\scrx,\tau})) \lesssim \frac{R^2}{\sigma_1} \left( \tau^2 \left(\zeta_2^2 \frac{1}{N_x} \sum_{i=1}^{N_x} \E[\|\hat{x}_0^{(i)}\|^2] + \lambda^2 R^2 \right) + \tau \sigma_2 d_x \right).
\end{align}
We plug the upper bound in Eq.~\eqref{eq:W_2_bound_tentative} and Eq.~\eqref{eq:W_2_bound_tentative_tilde} back to Eq.~\eqref{eq:bnabla_G_lip_3_bound_first}, and we obtain
\begin{align*}
    &\left\| \bnabla_1 L_\lambda(\mu_{\scrx,\tau}, \tilde{\mu}_z^\ast(\mu_{\scrx,\tau}), \mu_z^\ast(\mu_{\scrx,\tau}))(x) - \bnabla_1 L_\lambda(\mu_{\scrx,\tau}, \tilde{\mu}_z^\ast(\mu_{\scrx,0}), \mu_z^\ast(\mu_{\scrx,0}))(x) \right\|^2 \\
    &\lesssim \frac{\lambda^2 R^6}{\sigma_1} \left( \tau^2 \left(\zeta_2^2 \frac{1}{N_x} \sum_{i=1}^{N_x} \E[\|\hat{x}_0^{(i)}\|^2] + \lambda^2 R^2 \right) + \tau \sigma_2 d_x \right) .
\end{align*}
\end{proof}

\begin{lem}\label{lem:bnabla_G_lip_1}
Suppose \Cref{ass:network} holds. Let $L_\lambda$ be as defined in Eq.~\eqref{eq:defi_L_lambda_new} where $L_\lambda(\mu_x, \tilde{\mu}_z, \mu_z) = F_2(\mu_x,\tilde{\mu}_z) + \lambda \cdot F_1(\mu_x, \tilde{\mu}_z) + \lambda \sigma_1 \cdot \mathrm{Ent}(\tilde{\mu}_z) - \lambda \cdot F_1(\mu_x, \mu_z) - \lambda \sigma_1 \cdot \mathrm{Ent}(\mu_z) $. 
Consider the following particle system: 
for $i=1,\ldots,N_x$ and $0 < \tau \leq \gamma$,  
\begin{align}\label{eq:update_scheme_in_lemma}
    \hat{x}_\tau^{(i)} = \hat{x}_0^{(i)} - \tau \cdot \bnabla_1 L_\lambda(\hat{\mu}_{\scrx, 0}, \tilde{\mu}_{\scrz, s} , \mu_{\scrz, s})(\hat{x}_0^{(i)}) + \sqrt{2\sigma_2 \tau} \xi_{x}^{(i)},
\end{align}
where $\{\xi_{x}^{(i)}\}_{i=1}^{N_x}$ are $N_x$ i.i.d unit normal random variables. Then, for any fixed $\tilde{\mu}_z, \mu_z\in\calP_2(\R^{d_z})$, 
\begin{align}\label{eq:G_bound_1}
&\quad \E \left[ \left\| \bnabla_1 L_\lambda(\hat{\mu}_{\scrX,0}, \tilde{\mu}_z, \mu_z)(\hat{x}_0^{(i)}) - \bnabla_1 L_\lambda(\hat{\mu}_{\scrX,\tau}, \tilde{\mu}_z, \mu_z)(\hat{x}_\tau^{(i)}) \right\|^2 \right] \nonumber \\
&\lesssim (2\lambda^2 R^4 + \zeta_2^2) \left( \tau^2 \left(\zeta_2^2 \frac{1}{N_x} \sum_{i=1}^{N_x} \E[\|\hat{x}_0^{(i)}\|^2] + \lambda^2 R^2 \right) + 2\tau \sigma_2 d_x \right) . 
\end{align}
Here, the expectation is taken over the joint distribution of the particles $\{\hat{x}_\tau^{(i)}\}_{i=1}^{N_x} \sim\hat{\mu}_{x,\tau}^{(N_x)}$ and $\{\hat{x}_0^{(i)}\}_{i=1}^{N_x} \sim\hat{\mu}_{x,0}^{(N_x)}$. 
\end{lem}
\begin{proof}
From the definition of $L_\lambda$ in Eq.~\eqref{eq:defi_L_lambda_new} and its Wasserstein gradient in Eq.~\eqref{eq:gradient_L_lambda}, we have
\begin{align*}
    &\quad \bnabla_1 L_\lambda(\mu_{\scrX,0}, \tilde{\mu}_z, \mu_z)(\hat{x}_0^{(i)}) \\
    &= \bnabla_1 F_2(\mu_{\scrX,0}, \tilde{\mu}_z)(\hat{x}_0^{(i)}) + \lambda \cdot \bnabla_1 F_1(\mu_{\scrX,0}, \tilde{\mu}_z)(\hat{x}_0^{(i)}) - \lambda \cdot \bnabla_1 F_1(\mu_{\scrX,0}, \mu_z)(\hat{x}_0^{(i)}) \nonumber \\
    &= \zeta_2 \hat{x}_0^{(i)} + \lambda \cdot \bnabla_1 U_1(\mu_{\scrX,0}, \tilde{\mu}_z)(\hat{x}_0^{(i)}) - \lambda \cdot \bnabla_1 U_1(\mu_{\scrX,0}, \mu_z)(\hat{x}_0^{(i)}) \nonumber .
\end{align*}
Hence, we have
\begin{align*}
    &\quad \left\| \bnabla_1 L_\lambda(\mu_{\scrX,0}, \tilde{\mu}_z, \mu_z)(\hat{x}_0^{(i)})  - \bnabla_1 L_\lambda(\mu_{\scrX,\tau}, \tilde{\mu}_z, \mu_z)(\hat{x}_\tau^{(i)}) \right\|^2 \\
    &\leq \zeta_2^2 \|\hat{x}_0^{(i)}-\hat{x}_\tau^{(i)} \|^2 + \lambda^2 \cdot \| \bnabla_1 U_1(\mu_{\scrX,0}, \tilde{\mu}_z)(\hat{x}_0^{(i)}) - \bnabla_1 U_1(\mu_{\scrX,\tau}, \tilde{\mu}_z)(\hat{x}_\tau^{(i)})\|^2 \\
    &\qquad + \lambda^2 \cdot \| \bnabla_1 U_1(\mu_{\scrX,0}, \tilde{\mu}_z)(\hat{x}_0^{(i)}) - \bnabla_1 U_1(\mu_{\scrX,\tau}, \tilde{\mu}_z)(\hat{x}_\tau^{(i)})\|^2  \\
    &\leq 2\lambda^2 R^4 \left( W_2^2(\mu_{\scrX,0}, \mu_{\scrX,\tau}) + \|\hat{x}_0^{(i)} - \hat{x}_\tau^{(i)}\|^2 \right) + \zeta_2^2 \|\hat{x}_0^{(i)} - \hat{x}_\tau^{(i)}\|^2 \\
    &\leq 2\lambda^2 R^4 \left( \frac{1}{N_x}\sum_{i=1}^{N_x} \|\hat{x}_0^{(i)} - \hat{x}_\tau^{(i)}\|^2  + \|\hat{x}_0^{(i)} - \hat{x}_\tau^{(i)}\|^2 \right) + \zeta_2^2 \|\hat{x}_0^{(i)} - \hat{x}_\tau^{(i)}\|^2 .
\end{align*}
Here, the second last inequality holds from the Lipschitz continuity of Wasserstein gradients $\bnabla_1 U_1$ proved in \Cref{lem:lip_wass_grad}, and the last inequality holds by the definition of Wasserstein distance.  
From the update scheme and the same derivations as done in Eq.~\eqref{eq:wass_distance_iterates}, we have
\begin{align*}
    &\quad \E_{\{\hat{x}_0^{(i)}\}_{i=1}^{N_x}\sim \hat{\mu}_{x,0}^{(N_x)} \; \{\hat{x}_\tau^{(i)}\}_{i=1}^{N_x}\sim \hat{\mu}_{x,\tau}^{(N_x)}} \left[ \|\hat{x}_0^{(i)} - \hat{x}_\tau^{(i)}\|^2 \right] \leq \tau^2 \left(\zeta_2^2 \frac{1}{N_x} \sum_{i=1}^{N_x} \E[\|\hat{x}_0^{(i)}\|^2] + \lambda^2 R^2 \right) + 2\tau \sigma_2 d_x .
\end{align*}
The left hand side of Eq.~\eqref{eq:G_bound_1} can be upper bounded by
\begin{align*}
    &\quad \E\left[ \left\| \bnabla_1 L_\lambda(\mu_{\scrX,0}, \tilde{\mu}_z, \mu_z)(\hat{x}_0^{(i)})  - \bnabla_1 L_\lambda(\mu_{\scrX,\tau}, \tilde{\mu}_z, \mu_z)(\hat{x}_\tau^{(i)}) \right\|^2\right] \\
    &\lesssim (2\lambda^2 R^4 + \zeta_2^2) \left( \tau^2 \left(\zeta_2^2 \frac{1}{N_x} \sum_{i=1}^{N_x} \E[\|\hat{x}_0^{(i)}\|^2] + \lambda^2 R^2 \right) + 2\tau \sigma_2 d_x \right) .
\end{align*}
The proof is thus concluded. 
\end{proof}

\begin{lem}\label{lem:bnabla_G_lip_2}
Suppose \Cref{ass:network} holds. Let $L_\lambda$ be as defined in Eq.~\eqref{eq:defi_L_lambda_new} where $L_\lambda(\mu_x, \tilde{\mu}_z, \mu_z) = F_2(\mu_x,\tilde{\mu}_z) + \lambda \cdot F_1(\mu_x, \tilde{\mu}_z) + \lambda \sigma_1 \cdot \mathrm{Ent}(\tilde{\mu}_z) - \lambda \cdot F_1(\mu_x, \mu_z) - \lambda \sigma_1 \cdot \mathrm{Ent}(\mu_z) $. 
Consider the following particle system: 
for $i=1,\ldots,N_x$ and $0 < \tau \leq \gamma$,  
\begin{align}\label{eq:updatecheme_in_lemma}
    \hat{x}_\tau^{(i)} = \hat{x}_0^{(i)} - \tau \cdot \bnabla_1 L_\lambda(\hat{\mu}_{\scrx, 0}, \tilde{\mu}_{\scrz} , \mu_{\scrz})(\hat{x}_0^{(i)}) + \sqrt{2\sigma_2 \tau} \xi_{x}^{(i)},
\end{align}
where $\{\xi_{x}^{(i)}\}_{i=1}^{N_x}$ are $N_x$ i.i.d unit normal random variables, and  
$\tilde{\mu}_{\scrz} = \frac{1}{N_z} \sum_{i=1}^{N_z}\delta_{z}, \mu_{\scrz}= \frac{1}{N_z} \sum_{i=1}^{N_z}\delta_{\tilde{z}}$ are the output of the inner-loop algorithm \textsc{InnerLoop}($\hat{\mu}_{\scrx,0}$, $T$, $\alpha$, $\beta$, $\lambda$, $\sigma_1$) detailed in \Cref{alg:inner_loop}. Denote $\tilde{\mu}_{z}^{(N_z)}, \mu_{z}^{(N_z)}$ as the joint distribution of the corresponding $N_z$ particles. 
Then, we have
\begin{align}\label{eq:G_bound_2}
&\quad \E \left[ \left\| \bnabla_1 L_\lambda(\hat{\mu}_{\scrX,0}, \tilde{\mu}_{z}^\ast(\hat{\mu}_{\scrx,0}), \mu_{z}^\ast(\hat{\mu}_{\scrx,0}))(\hat{x}_0^{(i)}) - \bnabla_1 L_\lambda(\hat{\mu}_{\scrX,0}, \tilde{\mu}_{\scrz} , \mu_{\scrz})(\hat{x}_0^{(i)}) \right\|^2 \right] \nonumber \\
&\lesssim \lambda^2 R^4 \left(\sqrt{ \frac{\kl\left(\mu_{z}^{(N_z)}, {(\mu_{z}^\ast(\hat{\mu}_{\scrx,0}))}^{\otimes N_z} \right)}{N_z} } + \sqrt{ \frac{\kl\left(\tilde{\mu}_{z}^{(N_z)}, {(\tilde{\mu}_{z}^\ast(\hat{\mu}_{\scrx,0}))}^{\otimes N_z} \right)}{N_z} } + 1 \right) . 
\end{align} 
Here, the expectation is taken over the joint distribution of the particles  $\{\tilde{z}^{(i)}\}_{i=1}^{N_z}\sim\tilde{\mu}_{z}^{(N_z)}$ and $\{z^{(i)}\}_{i=1}^{N_z}\sim\mu_{z}^{(N_z)}$. 
\end{lem}
\begin{proof}
From the definition of $L_\lambda$ in Eq.~\eqref{eq:defi_L_lambda_new} and its Wasserstein gradient in Eq.~\eqref{eq:gradient_L_lambda}, we have
\begin{align}
    &\quad \bnabla_1 L_\lambda(\mu_{\scrX,0}, \tilde{\mu}_{z}^\ast(\hat{\mu}_{\scrx,0}), \mu_{z}^\ast(\hat{\mu}_{\scrx,0}))(\hat{x}_0^{(i)}) - \bnabla_1 L_\lambda(\mu_{\scrX,0}, \tilde{\mu}_{\scrz} , \mu_{\scrz})(\hat{x}_0^{(i)}) \nonumber \\
    &= \bnabla_1 F_2(\mu_{\scrX,0}, \tilde{\mu}_{z}^\ast(\hat{\mu}_{\scrx,0}))(\hat{x}_0^{(i)}) + \lambda \cdot \bnabla_1 F_1(\mu_{\scrX,0}, \mu_{z}^\ast(\hat{\mu}_{\scrx,0}))(\hat{x}_0^{(i)}) - \lambda \cdot \bnabla_1 F_1(\mu_{\scrX,0}, \tilde{\mu}_{z}^\ast(\hat{\mu}_{\scrx,0}))(\hat{x}_0^{(i)}) \nonumber \\
    &\qquad - \bnabla_1 F_2(\mu_{\scrX,0}, \tilde{\mu}_{\scrz})(\hat{x}_0^{(i)}) - \lambda \cdot \bnabla_1 F_1(\mu_{\scrX,0}, \mu_{\scrz} )(\hat{x}_0^{(i)}) + \lambda \cdot \bnabla_1 F_1(\mu_{\scrX,0}, \tilde{\mu}_{\scrz} )(\hat{x}_0^{(i)}) \nonumber \\
    &= \lambda \cdot \left( \bnabla_1 U_1(\mu_{\scrX,0}, \mu_{z}^\ast(\hat{\mu}_{\scrx,0}))(\hat{x}_0^{(i)}) - \bnabla_1 U_1(\mu_{\scrX,0}, \mu_{\scrz} )(\hat{x}_0^{(i)}) \right) \nonumber \\
    &\qquad\qquad -\lambda\cdot \left( \bnabla_1 U_1(\mu_{\scrX,0}, \tilde{\mu}_{z}^\ast(\hat{\mu}_{\scrx,0}))(\hat{x}_0^{(i)}) - \bnabla_1 U_1(\mu_{\scrX,0}, \tilde{\mu}_{\scrz} )(\hat{x}_0^{(i)})\right) \label{eq:nabla_1_U_1_two_terms} .
\end{align}
From the Lipschitz continuity of the Wasserstein gradients $\bnabla_1 U_1$ proved in \Cref{lem:lip_wass_grad}, we have
\begin{align*}
    \left\| \bnabla_1 U_1(\mu_{\scrX,0}, \mu_{z}^\ast(\hat{\mu}_{\scrx,0}))(\hat{x}_0^{(i)}) - \bnabla_1 U_1(\mu_{\scrX,0}, \mu_{\scrz})(\hat{x}_0^{(i)}) \right\| \leq R \left| \E_\rho \left[ \int \Psi_\bw \; (\dd \mu_{\scrz} - \dd \mu_{z}^\ast(\hat{\mu}_{\scrx,0})) \right] \right| .
\end{align*} 
Next, from \citet[Proposition 1]{nitanda2025propagation}, we obtain
\begin{align*}
    &\quad \E_{\{z^{(i)}\}_{i=1}^{N_z} \sim\mu_{z}^{(N_z)}}\left[ \left\| \bnabla_1 U_1(\mu_{\scrX,0}, \mu_{z}^\ast(\hat{\mu}_{\scrx,0}))(\hat{x}_0^{(i)}) - \bnabla_1 U_1(\mu_{\scrX,0}, \mu_{\scrz})(\hat{x}_0^{(i)}) \right\|^2 \right] \\
    &\leq R^2 \E_{\{z^{(i)}\}_{i=1}^{N_z} \sim\mu_{z}^{(N_z)}} \left( \E_\rho \left[ \smallint \Psi_\bw \; (\dd \mu_{\scrz} - \dd \mu_{z}^\ast(\hat{\mu}_{\scrx,0})) \right] \right)^2 \leq 8 R^4 \sqrt{ \frac{\kl\left(\mu_{z}^{(N_z)}, {(\mu_{z}^\ast(\hat{\mu}_{\scrx,0}))}^{\otimes N_z} \right)}{N_z} } + \frac{4R^4}{N_z}.
\end{align*}
Similarly, we can do the same for the second term in Eq.~\eqref{eq:nabla_1_U_1_two_terms}, which concludes the proof. 
\end{proof}

\begin{prop}[Partial-convexity]\label{prop:partial_convex}
    Let $U_\lambda: \calP_2(\R^{d_x})\times \calP_2(\R^{d_z})\times\calP_2(\R^{d_z})$ to $\R$ be a mapping as defined in Eq.~\eqref{eq:defi_U_lambda}. For fixed $\tilde{\mu}_z, \mu_z$, the mapping $\mu_x \mapsto U_\lambda(\mu_x, \tilde{\mu}_z, \mu_z)$ is convex. 
\end{prop}
\begin{proof}
Notice that
\begin{align*}
    &\quad U_\lambda(\mu_x, \tilde{\mu}_z, \mu_z) = U_2(\mu_z) + \lambda \cdot F_1(\mu_x, \mu_z) + \sigma_1 \mathrm{Ent}(\mu_z) - \lambda \cdot F_1(\mu_x, \tilde{\mu}_z) - \lambda \sigma_1 \mathrm{Ent} (\tilde{\mu}_z) \\
    &= \lambda \cdot \left(U_1(\mu_x, \mu_z) -U_1(\mu_x, \tilde{\mu}_z) \right) + \text{const} \\
    &= \frac{\lambda}{2} \cdot \left( \E_\rho \left[ \left( \smallint \Psi_\ba \; \dd \mu_x - \smallint \Psi_\bw \; \dd \mu_z \right)^2 \right] - \E_\rho \left[ \left( \smallint \Psi_\ba \; \dd \mu_x - \smallint \Psi_\bw \; \dd \tilde{\mu}_z \right)^2 \right] \right) + \text{const} \\
    &= \frac{\lambda}{2} \cdot \E_\rho \left[ \left( 2 \smallint \Psi_\ba \; \dd \mu_x - \smallint \Psi_\bw \; \dd \mu_z - \smallint \Psi_\bw \; \dd \tilde{\mu}_z \right) \cdot \text{const}\right] + \text{const}.
\end{align*}
Here, the const are constants that are independent of $\mu_x$. Hence, we can see that the partial mapping $\mu_x \mapsto U_\lambda(\mu_x, \tilde{\mu}_z, \mu_z)$ is a linear mapping and hence convex. 
\end{proof}

\subsection{Proofs in \Cref{sec:generalization}}\label{sec:proof_generalization}
\begin{prop}[Stage I generalization bound]\label{prop:stage_1_generalization}
Suppose \Cref{ass:npiv}, \ref{ass:network} and \ref{ass:s1_well_posed} hold. 
For any $\mu_x$ that satisfies $\kl(\mu_x,\nu_x)\leq \kl(\mu_x^\circ,\nu_x) + 2\sigma_2^{-1}R^2$, let $\mu_z^\ast(\mu_x) = \arg\min_{\mu_z\in\calP_2(\R^{d_z})} \scrF_1(\mu_x, \mu_z)$ be the optimal solution to Stage I. Then, with $P^{\otimes m}$ probability at least $1-2\delta$, 
\begin{align*}
    \E_{P_W} \left[ \Big( \smallint \Psi(W, z) \dd \mu_z^\ast(\mu_x) - \smallint \Psi(W, z) \dd \mu_z^\circ(\mu_x) \Big)^2 \right] \lesssim \sigma_1 M_z + R^2 \sqrt{\frac{\log(\delta^{-1})}{m}} + R^2 \sqrt{\frac{M_z + \sigma_1^{-1} R^2 }{m}} . 
\end{align*}
\end{prop}
\begin{proof}[Proof of \Cref{prop:stage_1_generalization}]
From the optimality of $\mu_z^\ast(\mu_x)$, we have
\begin{align}
    \scrF_1(\mu_x, \mu_z^\ast(\mu_x)) &= \frac{1}{m}\sum_{i=1}^m \Big( \smallint \Psi(\ba_i, x) \dd \mu_x - \smallint \Psi(\bw_i, z) \dd \mu_z^\ast(\mu_x) \Big)^2 + \sigma_1 \kl(\mu_z^\ast(\mu_x), \nu_z) \nonumber \\
    &\leq \scrF_1(\mu_x, \mu_z^\circ(\mu_x)) \nonumber \\
    &= \frac{1}{m}\sum_{i=1}^m \Big( \smallint \Psi(\ba_i, x) \dd \mu_x - \smallint \Psi(\bw_i, z) \dd \mu_z^\circ(\mu_x) \Big)^2 + \sigma_1 \kl(\mu_z^\circ(\mu_x), \nu_z) \label{eq:optimization_gap}.
\end{align}
For each $i$, denote $\Upsilon(\ba_i,\bw_i) = \smallint \Psi(\ba_i, x) \dd \mu_x - \smallint \Psi(\bw_i, z) \dd \mu_z^\circ(\mu_x)$ which is zero mean by definition of $\mu_z^\circ(\mu_x)$ and it is also subgaussian as it is bounded by \Cref{ass:network}.  
From Corollary 2.8.3 in \citet{vershynin2018high}, with probability at least $1-\delta$, we have 
\begin{align*}
    \frac{1}{m}\sum_{i=1}^m \Upsilon(\ba_i,\bw_i)^2 \leq \E_{P_{AW}}[\Upsilon(A, W)^2] + R^2 \left( \sqrt{\frac{\log(\delta^{-1})}{m}} + \frac{\log(\delta^{-1})}{m} \right) . 
\end{align*}
Here, $\E_{P_{AW}}[\Upsilon(A, W)^2]$ represents the Bayes optimal risk in stage I regression. 
For sufficiently large $m$, we have 
\begin{align*}
    \eqref{eq:optimization_gap} &\leq  \E_{P_{AW}}[\Upsilon(A, W)^2] + R^2 \left( \sqrt{\frac{\log(\delta^{-1})}{m}} + \frac{\log(\delta^{-1})}{m} \right) + \sigma_1 \kl(\mu_z^\circ(\mu_x), \nu_z) \\
    &\leq 2\E_{P_{AW}}[\Upsilon(A, W)^2] + \sigma_1 M_z . 
\end{align*}
Hence, we have $\kl(\mu_z^\ast(\mu_x), \nu_z)\leq 2\sigma_1^{-1} \E_{P_{AW}}[\Upsilon(A,W)^2] +  M_z$. 
Denote 
\begin{align*}
    \overline{\calB}(M_z) := \left\{ (\ba, \bw) \mapsto \Big( f(\bw) - \smallint \Psi(\ba, x) \dd \mu_x \Big)^2 \Big| \quad f: \calW\to\R \in \calB_{M_z} \right\} .
\end{align*}
\begin{defi}[Empirical Rademacher complexity]
The empirical Rademacher complexity of a function class $\calF$ of functions $f:\calX \to\R$ is defined as $\mathfrak{R}(\mathcal{F}) := \mathbb{E}_\sigma [\sup _{f \in \mathcal{F}} \frac{1}{n} \sum_{i=1}^n \sigma_{i} f_t(x_i)]$, where $\{\sigma_i\}_{i=1}^n$ are $n$ i.i.d Rademacher random variables. 
\end{defi}
Since $\Psi$ is bounded from \Cref{ass:network}, by Talagrand's contraction
lemma~\citep[Lemma 5.7]{mohri2018foundations}, we obtain that $\mathfrak{R}(\overline{\calB}(M_z+2\sigma_1^{-1} \E_{P_{AW}}[\Upsilon(A,W)^2])) \leq 2R \cdot \mathfrak{R}(\calB(M_z + 2\sigma_1^{-1} \E_{P_{AW}}[\Upsilon(A,W)^2]))$. 
Utilizing the standard uniform bound~\citep{wainwright2019high}, with probability at least $1-\delta$,
\begin{align*}
    &\quad \sup_{g\in \overline{\calB}(M_z + 4\sigma_1^{-1})} \left\{\E_{P_{AW}}[g(A, W)] - \frac{1}{m} \sum_{i=1}^m g(\ba_i, \bw_i) \right\} \\
    &\leq 2 \mathfrak{R}(\overline{\calB}(M_z + 2\sigma_1^{-1} \E_{P_{AW}}[\Upsilon(A,W)^2])) + 12 R^2 \sqrt{\frac{\log(\delta^{-1})}{2 m}} \\
    &\leq 4R \cdot \mathfrak{R}(\calB(M_z + 2\sigma_1^{-1} \E_{P_{AW}}[\Upsilon(A,W)^2])) + 12 R^2 \sqrt{\frac{\log(\delta^{-1})}{2 m}} \\
    &\leq 4R^2 \sqrt{\frac{M_z + 2\sigma_1^{-1} \E_{P_{AW}}[\Upsilon(A,W)^2] }{m}} + 12 R^2 \sqrt{\frac{\log(\delta^{-1})}{2 m}} . 
\end{align*}
The last inequality holds by using the upper bound on the empirical Rademacher complexity proved in \Cref{lem:rad}. 
Therefore, we obtain
\begin{align}
    \E_{P_{AW}} \left[ \Big( \smallint \Psi(W, z) \dd \mu_z^\ast(\mu_x) - \smallint \Psi(A, x) \dd \mu_x \Big)^2 \right] &\leq \frac{1}{m} \sum_{i=1}^m \Big( \smallint \Psi(\bw_i, z) \dd \mu_z^\ast(\mu_x) - \smallint \Psi(\ba_i, x) \dd \mu_x \Big)^2 \nonumber \\
    &\hspace{-1cm} + 4R^2 \sqrt{\frac{M_z + 2\sigma_1^{-1} \E_{P_{AW}}[\Upsilon(A,W)^2]}{m}} + 12 R^2 \sqrt{\frac{\log(\delta^{-1})}{2 m}} \nonumber \\
    &\hspace{-7cm} \leq \eqref{eq:optimization_gap} + 4R^2 \sqrt{\frac{M_z + 2\sigma_1^{-1} \E_{P_{AW}}[\Upsilon(A,W)^2]}{m}} + 12 R^2 \sqrt{\frac{\log(\delta^{-1})}{2 m}} \nonumber \\
    &\hspace{-7cm} \leq \E_{P_{AW}}[\Upsilon(A,W)^2] + R^2 \left( \sqrt{\frac{\log(\delta^{-1})}{m}} + \frac{\log(\delta^{-1})}{m} \right) + \sigma_1 M_z \nonumber \\
    &\hspace{-1cm}+ 4R^2 \sqrt{\frac{M_z + 2\sigma_1^{-1} \E_{P_{AW}}[\Upsilon(A,W)^2]}{m}} + 12 R^2 \sqrt{\frac{\log(\delta^{-1})}{2 m}} \nonumber \\
    &\hspace{-7cm} \lesssim \E_{P_{AW}}[\Upsilon(A,W)^2] + \sigma_1 M_z + R^2 \sqrt{\frac{\log(\delta^{-1})}{m}} + R^2 \sqrt{\frac{M_z + \sigma_1^{-1} \E_{P_{AW}}[\Upsilon(A,W)^2]}{m}} \label{eq:dummy_generalization_proof}.
\end{align}
Finally, notice that the left hand side of the above inequality equals
\begin{align*}
    \text{LHS of } \eqref{eq:dummy_generalization_proof} = \E_{P_W} \left[ \Big( \smallint \Psi(W, z) \dd \mu_z^\ast(\mu_x) - \smallint \Psi(W, z) \dd \mu_z^\circ(\mu_x) \Big)^2 \right] + \E_{P_{AW}}[\Upsilon(A,W)^2]. 
\end{align*}
The proof is thus concluded. 
\end{proof}

\subsubsection{Proof of \Cref{thm:stage_2_generalization}}
\begin{proof}[Proof of \Cref{thm:stage_2_generalization}]
From \Cref{ass:s2_well_posed}, there exists $\mu_x^\circ$ such that $h_\circ(\ba) = \int \Psi(\ba, x) \; \dd \mu_x^\circ$ and $\kl(\mu_x^\circ, \nu_x) \leq M_x$.  
From the optimality of $\mu_x^\ast$, we have
\begin{align*}
    \scrF_2(\mu_x^\ast, \mu_z^\ast(\mu_x^\ast)) &= \frac{1}{n}\sum_{i=1}^n \Big( \smallint \Psi(\bw_i, z) \dd \mu_z^\ast(\mu_x^\ast) - \by_i \Big)^2 + \sigma_2 \kl(\mu_x^\ast, \nu_x) \leq \scrF_2(\mu_x^\circ, \mu_z^\ast(\mu_x^\circ)). 
\end{align*}
Next, from the second point of \Cref{thm:relation}, we know $(\mu_{x,\lambda}^{\ast}, \mu_{z,\lambda}^{\ast})$, the global optimum of \eqref{eq:lagrangian} with $\lambda > \lambda_0$, is also the global-minimum of \eqref{eq:penalty_formulation} with $\varepsilon = \epsilon_1 /(\lambda-\lambda_0)$ where $\epsilon_1 = \frac{R^2(R+M)^2}{8 \sigma_1 \lambda_0}$. 
From the third point of \Cref{thm:relation}, we know that
\begin{align*}
\scrF_2(\mu_{x,\lambda}^{\ast}, \mu_z^\ast(\mu_{x,\lambda}^{\ast})) - R(R+M) \sqrt{(2\sigma_1)^{-1} \varepsilon} \leq \scrF_2(\mu_{x,\lambda}^{\ast}, \mu_{z,\lambda}^{\ast}) \leq \scrF_2(\mu_x^\ast,\mu_z^\ast(\mu_x^\ast)) 
\end{align*}
If we take $\lambda_0 = \frac{\lambda}{2}$, then we have $\varepsilon = \frac{R^2(R+M)^2}{2\sigma_1 \lambda^2}$. 
So we have obtained that
\begin{align*}
    \scrF_2(\mu_{x,\lambda}^{\ast}, \mu_z^\ast(\mu_{x,\lambda}^{\ast})) \leq \scrF_2(\mu_x^\ast, \mu_z^\ast(\mu_x^\ast)) + \frac{R^2(R+M)^2}{2\sigma_1 \lambda} \leq \scrF_2(\mu_x^\circ, \mu_z^\ast(\mu_x^\circ)) + \frac{R^2(R+M)^2}{2\sigma_1 \lambda} .
\end{align*}
By adding and subtracting the same term on both sides, it gives 
\begin{align}
    &\quad \scrF_2(\mu_{x,\lambda}^{\ast}, \mu_z^\circ(\mu_{x,\lambda}^{\ast})) + \Big( \scrF_2(\mu_{x,\lambda}^{\ast}, \mu_z^\ast(\mu_{x,\lambda}^{\ast})) - \scrF_2(\mu_{x,\lambda}^{\ast}, \mu_z^\circ(\mu_{x,\lambda}^{\ast})) \Big) \nonumber \\
    &\leq \scrF_2(\mu_x^\circ, \mu_z^\circ(\mu_x^\circ)) + \frac{R^2(R+M)^2}{2\sigma_1 \lambda} + \Big( \scrF_2(\mu_x^\circ, \mu_z^\ast(\mu_x^\circ)) - \scrF_2(\mu_x^\circ, \mu_z^\circ(\mu_x^\circ)) \Big) . \label{eq:F_2_difference_later}
\end{align}
Notice that 
\begin{align}
    &\quad \scrF_2(\mu_x^\circ, \mu_z^\ast(\mu_x^\circ)) - \scrF_2(\mu_x^\circ, \mu_z^\circ(\mu_x^\circ)) \nonumber \\
    &= \frac{1}{n}\sum_{i=1}^n \left[ \Big( \smallint \Psi(\bw_i, z) \dd \mu_z^\ast(\mu_x^\circ) - \by_i \Big)^2 - \Big( \smallint \Psi(\bw_i, z) \dd \mu_z^\circ(\mu_x^\circ) - \by_i \Big)^2 \right] \nonumber \\
    &\leq \frac{1}{n}\sum_{i=1}^n \left[ \Big( \smallint \Psi(\bw_i, z) \dd \mu_z^\ast(\mu_x^\circ) + \smallint \Psi(\bw_i, z) \dd \mu_z^\circ(\mu_x^\circ) - 2\by_i \Big) \cdot \Big( \smallint \Psi(\bw_i, z) \dd \mu_z^\ast(\mu_x^\circ) - \smallint \Psi(\bw_i, z) \dd \mu_z^\circ(\mu_x^\circ) \Big) \right] \nonumber \\
    &\leq \frac{1}{n}\sum_{i=1}^n \Big( \smallint \Psi(\bw_i, z) \dd \mu_z^\ast(\mu_x^\circ) - \smallint \Psi(\bw_i, z) \dd \mu_z^\circ(\mu_x^\circ) \Big)^2 + 4R M \sqrt{\frac{\log(\delta^{-1})}{n}} 
    \label{eq:concentrate_stage_1_gen} .
\end{align}
The last inequality holds by that, for each $i\in\{1,\ldots,n\}$, $\by_i- \smallint \Psi(\bw_i, z) \dd \mu_z^{\circ}(\mu_x^{\circ}) = \by_i-(T h_{\circ})(\bw_i)$ is bounded and hence subgaussian as per \Cref{ass:npiv}. 
From the Bernstein's concentration inequality along with the generalization bounded in \Cref{prop:stage_1_generalization}, with probability at least $1-2\delta$, 
\begin{align*}
    \eqref{eq:concentrate_stage_1_gen} &\leq \E_{P_W} \left[ \Big( \smallint \Psi(W, z) \dd \mu_z^\ast(\mu_x^\circ) - \smallint \Psi(W, z) \dd \mu_z^\circ(\mu_x^\circ) \Big)^2 \right] + R^2\sqrt{\frac{\log(\delta^{-1})}{n}} + 4R M \sqrt{\frac{\log(\delta^{-1})}{n}} \\
    &:= (\dagger) + (R^2 + 4RM) \sqrt{\frac{\log(\delta^{-1})}{n}} .
\end{align*}
Here, $(\dagger)$ represents the generalization bounded in \Cref{prop:stage_1_generalization}. 

Similarly, we can obtain
\begin{align}
    &\quad \scrF_2(\mu_{x,\lambda}^{\ast}, \mu_z^\circ(\mu_{x,\lambda}^{\ast})) - \scrF_2(\mu_{x,\lambda}^{\ast}, \mu_z^\ast(\mu_{x,\lambda}^{\ast})) \nonumber \\
    &\leq \frac{1}{n}\sum_{i=1}^n \Big( \smallint \Psi(\bw_i, z) \dd \mu_z^\ast(\mu_{x,\lambda}^{\ast}) + \smallint \Psi(\bw_i, z) \dd \mu_z^\circ(\mu_{x,\lambda}^{\ast}) - 2 \smallint \Psi(\bw_i, z) \dd \mu_z^\circ(\mu_x^\circ) \Big) \nonumber \\
    &\qquad \cdot \Big( \smallint \Psi(\bw_i, z) \dd \mu_z^\ast(\mu_{x,\lambda}^{\ast}) - \smallint \Psi(\bw_i, z) \dd \mu_z^\circ(\mu_{x,\lambda}^{\ast})\Big) + 4R M \sqrt{\frac{\log(\delta^{-1})}{n}}  \nonumber \\
    &= \frac{1}{n}\sum_{i=1}^n \Big( \smallint \Psi(\bw_i, z) \dd \mu_z^\ast(\mu_{x,\lambda}^{\ast}) - \smallint \Psi(\bw_i, z) \dd \mu_z^\circ(\mu_{x,\lambda}^{\ast})\Big)^2 + 4R M \sqrt{\frac{\log(\delta^{-1})}{n}} \label{eq:scrF_2_bound_term_one} \\
    &\quad + \frac{2}{n}\sum_{i=1}^n \Big( \smallint \Psi(\bw_i, z) \; \dd \left(\mu_z^\circ(\mu_{x,\lambda}^{\ast}) - \mu_z^\circ(\mu_x^\circ) \right) \Big) \cdot \Big( \smallint \Psi(\bw_i, z) \dd \mu_z^\ast(\mu_{x,\lambda}^{\ast}) - \smallint \Psi(\bw_i, z) \dd \mu_z^\circ(\mu_{x,\lambda}^{\ast})\Big) \label{eq:scrF_2_bound_term_two} . 
\end{align}
Notice that 
\begin{align*}
    &\eqref{eq:scrF_2_bound_term_two} = \frac{2}{n}\sum_{i=1}^n \Big( T\left[\smallint \Psi(\cdot, x) \dd \mu_{x,\lambda}^{\ast} \right](\bw_i) - T\left[\smallint \Psi(\cdot, x) \dd \mu_x^\circ \right](\bw_i) \Big) \cdot \Big( \smallint \Psi(\bw_i, z) \; \dd \left( \mu_z^\ast(\mu_{x,\lambda}^{\ast})- \mu_z^\circ(\mu_{x,\lambda}^{\ast})\right) \Big) \\
    &\leq \frac{1}{2n} \sum_{i=1}^n \Big( T\left[\smallint \Psi(\cdot, x) \dd \mu_{x,\lambda}^{\ast} \right](\bw_i) - (Th_\circ)(\bw_i) \Big)^2 + \frac{2}{n}\sum_{i=1}^n \Big( \smallint \Psi(\bw_i, z) \dd \mu_z^\ast(\mu_{x,\lambda}^{\ast}) - \smallint \Psi(\bw_i, z) \dd \mu_z^\circ(\mu_{x,\lambda}^{\ast})\Big)^2
\end{align*}
From the Bernstein's concentration inequality along with the generalization bound proved in \Cref{prop:stage_1_generalization}, with probability at least $1-2\delta$, 
\begin{align*}
    \eqref{eq:scrF_2_bound_term_one} + \eqref{eq:scrF_2_bound_term_two} &\leq 3\E_{P_W} \left[ \Big( \smallint \Psi(W, z) \dd \mu_z^\ast(\mu_{x,\lambda}^{\ast}) - \smallint \Psi(W, z) \dd \mu_z^\circ(\mu_{x,\lambda}^{\ast}) \Big)^2 \right] + (R^2 + 4R M) \sqrt{\frac{\log(\delta^{-1})}{n}} \\
    &\qquad + \frac{1}{2n} \sum_{i=1}^n \Big( T\left[\smallint \Psi(\cdot, x) \dd \mu_{x,\lambda}^{\ast} \right](\bw_i) - (Th_\circ)(\bw_i) \Big)^2  \\
    &\leq 3(\dagger) + (R^2 + 4 RM) \sqrt{\frac{\log(\delta^{-1})}{n}} + \frac{1}{2n} \sum_{i=1}^n \Big( T\left[\smallint \Psi(\cdot, x) \dd \mu_{x,\lambda}^{\ast} \right](\bw_i) - (Th_\circ)(\bw_i) \Big)^2 . 
\end{align*}
Here, $(\dagger)$ represents the generalization bounded in \Cref{prop:stage_1_generalization}. 

Therefore, we are about to plug the above upper bound on $\scrF_2(\mu_{x,\lambda}^{\ast}, \mu_z^\circ(\mu_{x,\lambda}^{\ast})) - \scrF_2(\mu_{x,\lambda}^{\ast}, \mu_z^\ast(\mu_{x,\lambda}^{\ast}))$ and the upper bound on $ \scrF_2(\mu_x^\circ, \mu_z^\ast(\mu_x^\circ)) - \scrF_2(\mu_x^\circ, \mu_z^\circ(\mu_x^\circ))$ back to Eq.~\eqref{eq:F_2_difference_later}. To simplify the expression, we use $\lesssim$ to suppress the constants. We obtain, with probability at least $1-4\delta$, 
\begin{align}
    &\quad \scrF_2(\mu_{x,\lambda}^{\ast}, \mu_z^\circ(\mu_{x,\lambda}^{\ast})) \lesssim \scrF_2(\mu_x^\circ, \mu_z^\circ(\mu_x^\circ)) + (\dagger) + \frac{R^2(R+M)^2}{\sigma_1 \lambda} + (R^2 + RM) \sqrt{\frac{\log(\delta^{-1})}{n}} \nonumber  \\
    &\qquad + \frac{1}{2n} \sum_{i=1}^n \Big( T\left[\smallint \Psi(\cdot, x) \dd \mu_{x,\lambda}^{\ast} \right](\bw_i) - (Th_\circ)(\bw_i) \Big)^2 \nonumber \\
    &= \frac{1}{n}\sum_{i=1}^n \Big( \smallint \Psi(\bw_i, z) \dd \mu_z^\circ(\mu_x^\circ) - \by_i \Big)^2 + \sigma_2 \kl(\mu_x^\circ, \nu_x) + (\dagger) + \frac{R^2(R+M)^2}{\sigma_1 \lambda} + (R^2 + RM) \sqrt{\frac{\log(\delta^{-1})}{n}} \nonumber \\ 
    &\qquad + \frac{1}{2n} \sum_{i=1}^n \Big( T\left[\smallint \Psi(\cdot, x) \dd \mu_{x,\lambda}^{\ast} \right](\bw_i) - (Th_\circ)(\bw_i) \Big)^2 \nonumber \\
    &\leq \mathrm{Var} + M^2 \sqrt{\frac{\log(\delta^{-1})}{n}} + \sigma_2 \kl(\mu_x^\circ, \nu_x) + (\dagger) + \frac{R^2(R+M)^2}{\sigma_1 \lambda} + (R^2 + RM) \sqrt{\frac{\log(\delta^{-1})}{n}} \nonumber \\ 
    &\qquad + \frac{1}{2} \E_{P_W} \left[ \Big( T\left[\smallint \Psi(\cdot, x) \dd \mu_{x,\lambda}^{\ast} \right](W) - (Th_\circ)(W) \Big)^2 \right] + (R+M)^2\sqrt{\frac{\log(\delta^{-1})}{n}} \label{eq:optimization_gap_2}.
\end{align}
The last inequality holds by applying concentration inequalities. 
In the last inequality above, notice that for each $i$, the error term
$\by_i - \smallint \Psi(\bw_i, z) \dd \mu_z^\circ(\mu_x^\circ) = \by_i - (Th_\circ)(\bw_i)$ is zero-mean and bounded and hence $M$-subgaussian by \Cref{ass:npiv}. If we denote $\mathrm{Var} \leq M^2$ being the variance of $Y - (Th_\circ)(W)$ which corresponds to the Bayes optimal risk, then
from Corollary 2.8.3 in \citet{vershynin2018high}, with probability at least $1-\delta$, 
\begin{align*}
    \frac{1}{n}\sum_{i=1}^n \Big( \smallint \Psi(\bw_i, z) \dd \mu_z^\circ(\mu_x^\circ) - \by_i \Big)^2 \leq \mathrm{Var} + M^2 \sqrt{\frac{\log(\delta^{-1})}{n}} . 
\end{align*} 
As proved in \Cref{lem:kl_bound_assumption}, we have 
\begin{align}\label{eq:defi_M_x}
    \kl(\mu_{x,\lambda}^{\ast}, \nu_x) \leq 2\sigma_2^{-1} R^2 + M_x =: \mathfrak{M}_x . 
\end{align}
Denote 
\begin{align*}
    \overline{\calB}(\mathfrak{M}_x) := \left\{ (\bw,\by) \mapsto \big( (Tf)(\bw) - \by \big)^2 \big| \quad f: \calW\to\R \in \calB_{\mathfrak{M}_x} \right\} .
\end{align*}
By the contraction lemma, we obtain that $\mathfrak{R}(\overline{\calB}(\mathfrak{M}_x) \leq (R+M) \cdot \mathfrak{R}(\calB(\mathfrak{M}_x))$. 
Utilizing the standard uniform bound~\citep{wainwright2019high}, with probability at least $1-\delta$, 
\begin{align*}
    \sup_{g\in \overline{\calB}(\mathfrak{M}_x)} \left\{\E_{P_{WY}}[g(W, Y)] - \frac{1}{n} \sum_{i=1}^n g(\bw_i, \by_i) \right\} &\leq 2 \mathfrak{R}(\mathfrak{M}_x) + 12 R^2 \sqrt{\frac{\log(\delta^{-1})}{2 n}} \\
    &\leq 2R(R+M) \sqrt{\frac{\mathfrak{M}_x}{n}} + 12 R^2 \sqrt{\frac{\log(\delta^{-1})}{2 n}} . 
\end{align*}
The last inequality holds by using the upper bound on the empirical Rademacher complexity proved in \Cref{lem:rad}. 
Therefore, we obtain
\begin{align*}
    &\quad \E_{P_{WY}} \left[ \Big( T[\smallint \Psi(\cdot, x) \dd \mu_{x,\lambda}^{\ast}](W) - Y \Big)^2 \right] \\
    &\leq \frac{1}{n} \sum_{i=1}^n \Big( T[\smallint \Psi(\cdot, x) \dd \mu_{x,\lambda}^{\ast}](\bw_i) - \by_i \Big)^2 + 2R(R+M) \sqrt{\frac{\mathfrak{M}_x}{n}} + 12 R^2 \sqrt{\frac{\log(\delta^{-1})}{2 n}} .
\end{align*}
Note that the first empirical mean squared error term can be upper bounded by $\scrF_2(\mu_{x,\lambda}^{\ast}, \mu_z^\circ(\mu_{x,\lambda}^{\ast}))$ which is upper bounded in Eq.~\eqref{eq:optimization_gap_2} with with $P^{\otimes n}$ probability at least $1-6\delta$. 
We proceed from above to obtain, 
\begin{align*}
    &\leq \eqref{eq:optimization_gap_2} + 2R(R+M) \sqrt{\frac{\mathfrak{M}_x}{n}} + 12 R^2 \sqrt{\frac{\log(\delta^{-1})}{2 n}} \nonumber \\
    &\leq \mathrm{Var} + M^2 \sqrt{\frac{\log(\delta^{-1})}{n}} + \sigma_2 \kl(\mu_x^\circ, \nu_x) + (\dagger) + \frac{R^2(R+M)^2}{\sigma_1 \lambda} + (R^2 + RM) \sqrt{\frac{\log(\delta^{-1})}{n}} \\
    &+ \frac{1}{2} \E\left[\Big( T\left[\smallint \Psi(\cdot, x) \dd \mu_{x,\lambda}^{\ast} \right](W) - (Th_\circ)(W) \Big)^2 \right] + (R+M)^2\sqrt{\frac{\log(\delta^{-1})}{n}} + R(R+M) \sqrt{\frac{\sigma_2^{-1} R^2 + M_x}{n}}.
\end{align*}
Note that we have removed the unnecessary scalar factors to simplify the formula. 
Next, we inspect the LHS of the above inequality, 
\begin{align*}
    \E_{P_{WY}} \left[ \Big( T[\smallint \Psi(\cdot, x) \dd \mu_{x,\lambda}^{\ast}](W) - Y \Big)^2 \right] = \E\left[\Big( T\left[\smallint \Psi(\cdot, x) \dd \mu_{x,\lambda}^{\ast} \right](W) - (Th_\circ)(W) \Big)^2 \right] + \mathrm{Var} . 
\end{align*}
Therefore, we reach, with $P^{\otimes n}$ probability at least $1-6\delta$, 
\begin{align*}
    &\quad \E_{P_W} \left[ \Big( T[\smallint \Psi(\cdot, x) \dd \mu_{x,\lambda}^{\ast}](W) - (Th_\circ)(W) \Big)^2 \right] \\
    &\lesssim \sigma_2 M_x + (\dagger) + \frac{R^2(R+M)^2}{\sigma_1 \lambda} + (R + M)^2 \sqrt{\frac{\log(\delta^{-1})}{n}} + R(R+M) \sqrt{\frac{\frac{R^2}{\sigma_2} + M_x}{n}} . 
\end{align*}
Finally, $(\dagger)$ denotes the generalization bound established in \Cref{prop:stage_1_generalization}. With probability at least $1-2\delta$ under $P^{\otimes m}$, we have 
$(\dagger) \leq \sigma_1 M_z+R^2 \sqrt{\frac{\log \left(\delta^{-1}\right)}{m}}+R^2 \sqrt{\frac{M_z+R^2 / \sigma_1}{m}}$. 
The proof is thus concluded. 
\end{proof}

\begin{lem}\label{lem:kl_bound_assumption}
    Suppose \Cref{ass:npiv}, \ref{ass:network} and \ref{ass:s1_well_posed} hold. Then, for any fixed $\lambda>0$, for $\mu_{x,\lambda}^\ast$ defined in \eqref{eq:lagrangian}, we have $\kl(\mu_{x,\lambda}^\ast, \nu_x) \leq \kl(\mu_{x}^\circ, \nu_x) + 2 \sigma_2^{-1} R^2$. 
\end{lem}
\begin{proof}
By definition in \eqref{eq:lagrangian}, we have $(\mu_{x,\lambda}^\ast, \tilde{\mu}_z^\ast) = \arg\min_{\mu_x,\mu_z} \scrF_2(\mu_x,\mu_z) + \lambda (\scrF_1(\mu_x,\mu_z) - \scrF_1(\mu_x,\mu_z^\ast(\mu_x)))$. 
The proof of this lemma holds for both population and empirical distributions $\rho$ in the definition of $\scrF_1$ and $\scrF_2$, so we do not make this distinction in the following derivations. 
\begin{align*}
    &\quad \sigma_2 \kl(\mu_{x,\lambda}^\ast, \nu_x) \\
    &\leq \frac{1}{2} \E_{\rho} [ ( \smallint \Psi_\bw \dd \tilde{\mu}_z^\ast - \smallint \Psi_\ba \dd \mu_{x,\lambda}^\ast )^2] + \sigma_2 \kl(\mu_{x,\lambda}^\ast, \nu_x) + \lambda (\scrF_1(\mu_{x,\lambda}^\ast, \tilde{\mu}_z^\ast) - \scrF_1(\mu_{x,\lambda}^\ast, \mu_z^\ast(\mu_{x,\lambda}^\ast))) \\
    &\leq \frac{1}{2} \E_{\rho} [ ( \smallint \Psi_\bw \dd \mu_z^\ast(\mu_{x}^\circ) - \smallint \Psi_\ba \dd \mu_{x}^\circ)^2] + \sigma_2 \kl(\mu_{x}^\circ, \nu_x) . 
\end{align*}
The last inequality holds by the optimality of $(\mu_{x,\lambda}^\ast, \tilde{\mu}_z^\ast)$. The proof is thus concluded by the boundedness of $\Psi_\ba$ and $\Psi_\bw$ as per \Cref{ass:network}. 
\end{proof}
\section{Proof of Lemmas and Auxiliary Results}
\subsection{Proof of \Cref{lem:first_variation_epsilon}}\label{sec:proof_first_variation_epsilon}
\begin{proof}
From its optimality, $\mu_z^\ast(\mu_x)$ satisfies, for any $z\in\R^{d_z}$,
\begin{align*}
    \E_{\rho} \left[ \left( \smallint \Psi_\bw \; \dd \mu_z^\ast(\mu_x) - \smallint \Psi_\ba \; \dd \mu_x) \right)\cdot \Psi_\bw(z) \right] + \frac{\zeta_1}{2} \|z\|^2 = 0.
\end{align*}
And similarly, 
\begin{align*}
    \E_{\rho} \left[ \left( \smallint \Psi_\bw \; \dd \mu_z^\ast(\mu_x+\epsilon\nu_x) - \smallint \Psi_\ba \; \dd (\mu_x+\epsilon\nu_x) ) \right)\cdot  \Psi_\bw(z) \right] + \frac{\zeta_1}{2} \|z\|^2 = 0.
\end{align*}
Subtract the above two equations, 
\begin{align*}
    0&= -\epsilon \cdot \E_{\rho}\left[\left(\smallint \Psi_\ba \; \dd \nu_x \right)\cdot \Psi_\bw(z) \right] + \E_{\rho} \left[ \left( \smallint \Psi_\bw \;\dd \mu_z^\ast(\mu_x+\epsilon \nu_x) \right)\cdot \Psi_\bw(z) \right] - \E_{\rho} \left[ \left( \smallint \Psi_\bw \dd \mu_z^\ast(\mu_x) \right)\cdot \Psi_\bw(z) \right] \\
    &= \E_{\rho} \left[ \left( -\epsilon \cdot \smallint \Psi_\ba \; \dd \nu_x + \smallint \Psi_\bw \;\dd \mu_z^\ast(\mu_x+\epsilon \nu_x) - \smallint \Psi_\bw \;\dd \mu_z^\ast(\mu_x) \right)\cdot \Psi_\bw(z) \right] .
\end{align*}
Since the above equality holds for any $z$, and by the completeness assumption, we have $\epsilon\cdot \smallint \Psi_\ba \dd \nu_x = \smallint \Psi_\bw \dd \mu_z^\ast(\mu_x+\epsilon \nu_x) - \smallint \Psi_\bw \dd \mu_z^\ast(\mu_x)$ for $(\ba,\bw)$-$\rho$ almost everywhere. Therefore, we have concluded the proof for the first claim. 

From its optimality, $\tilde{\mu}_z^\ast(\mu_x)$ and $\tilde{\mu}_z^\ast(\mu_x+\epsilon\nu_x)$ satisfy, for any $z\in\R^{d_z}$,
\begin{align*}
    &\E_{\rho} \left[ \left( \smallint \Psi_\bw \; \dd \tilde{\mu}_z^\ast(\mu_x) - \by \right)\cdot \Psi_\bw(z) \right] + \lambda \E_{\rho} \left[ \left( \smallint \Psi_\bw \; \dd \tilde{\mu}_z^\ast(\mu_x) - \smallint \Psi_\ba \; \dd \mu_x) \right)\cdot \Psi_\bw(z) \right] + \frac{\lambda\zeta_1}{2} \|z\|^2 = 0 \\
    &\E_{\rho} \left[ \left( \smallint \Psi_\bw \; \dd \tilde{\mu}_z^\ast(\mu_x+\epsilon\nu_x) - \by \right)\cdot \Psi_\bw(z) \right] + \lambda \E_{\rho} \left[ \left( \smallint \Psi_\bw \; \dd \tilde{\mu}_z^\ast(\mu_x+\epsilon\nu_x) - \smallint \Psi_\ba \; \dd (\mu_x+\epsilon\nu_x) ) \right)\cdot \Psi_\bw(z) \right] \\
    &\hspace{10em} + \frac{\lambda\zeta_1}{2} \|z\|^2 = 0 .
\end{align*}
Subtract the above two equations
\begin{align*}
    0 = -\lambda \epsilon \cdot \E_{\rho}\left[\left(\smallint \Psi_\ba \; \dd \nu_x \right)\cdot \Psi_\bw(z) \right] + (\lambda+1) \E_{\rho} \left[ \left( \smallint \Psi_\bw \;\dd \tilde{\mu}_z^\ast(\mu_x+\epsilon \nu_x) - \smallint \Psi_\bw \;\dd \tilde{\mu}_z^\ast(\mu_x) \right)\cdot \Psi_\bw(z) \right]
\end{align*}
Since the above equality holds for any $z$, and by the  completeness assumption, we have  $\frac{\lambda}{\lambda+1} \epsilon\cdot \smallint \Psi_\ba \dd \nu_x = \smallint \Psi_\bw \dd \tilde{\mu}_z^\ast(\mu_x+\epsilon \nu_x) - \smallint \Psi_\bw \dd \tilde{\mu}_z^\ast(\mu_x)$. 
\end{proof}

\begin{lem}[Uniform boundedness of the second moment]\label{lem:second_mom_bound}
Let $U:\calP_2(\R^{d})\to\R$ be a functional that admits a well-defined Wasserstein gradient $\bnabla U(\mu): \R^d\to\R^d$ for any $\mu \in\calP_2(\R^{d})$. 
Suppose $\bnabla U$ satisfies $\|\bnabla U(\mu)(x)\|\leq R$ for any $x\in\R^{d}$. 
Consider the following particle update scheme with step size $\gamma \leq \frac{1}{\zeta_2}$ and $\zeta_2>0$: for $i = 1, \ldots, N_x$, 
\begin{align*}
    x_{t+1}^{(i)} = x_t^{(i)} - \gamma \cdot \left( \bnabla U(\mu_{\scrx}) (x_t^{(i)}) + \zeta_2 x_t^{(i)} \right) + \sqrt{2 \gamma \sigma} \xi_{z,t}^{(i)}
\end{align*}
Then, for any $t\in\N^+$, we have a uniform upper bound on the second moment
\begin{align*}
    \E[\| x_t^{(i)}\|^2] \leq \E[\|x_0^{(i)}\|^2] + \frac{2}{\zeta_2} \left( \frac{R^2}{2\zeta_2} + \sigma d\right) .
\end{align*}
\end{lem}
\begin{proof}
    The lemma is Lemma 1 in \citet{nitanda2024improved} and Lemma 1 in \citet{suzuki2023convergence}. 
\end{proof}

\begin{prop}[Lipschitzness of $U_\lambda$] \label{prop:continuity_U_lambda}
Suppose \Cref{ass:network} holds. 
Recall the definition of $U_\lambda:\calP_2(\R^{d_x})\times\calP_2(\R^{d_z})\times\calP_2(\R^{d_z}) \to \R$ in Eq.~\eqref{eq:defi_U_lambda} that $U_\lambda(\mu_x, \tilde{\mu}_z, \mu_z) = U_2(\tilde{\mu}_z) + \lambda \cdot F_1(\mu_x, \tilde{\mu}_z) + \lambda \sigma_1 \cdot \mathrm{Ent}(\tilde{\mu}_z) - \lambda \cdot F_1(\mu_x, \mu_z) - \lambda \sigma_1 \cdot \mathrm{Ent}(\mu_z)$. 
For any $\mu_x,\mu_x^\prime \in\calP_2(\R^{d_x})$ and any $\mu_z \in \calP_2(\R^{d_z})$,  
\begin{align*}
    U_\lambda(\mu_x, \tilde{\mu}_z^\ast(\mu_x), \mu_z) - U_\lambda(\mu_x, \tilde{\mu}_z^\ast(\mu_x^\prime), \mu_z) \geq -\frac{R \lambda}{4\sigma_1} \E_{\rho} \left[ \left( \int \Psi_\ba \; \dd (\mu_x - \mu_x^\prime) \right)^2 \right] \geq -\frac{R^3 \lambda}{4\sigma_1} \tv^2(\mu_x,\mu_x^\prime) .
\end{align*}
\end{prop}
\begin{proof}
We start from the definition of $U_\lambda$: 
\begin{align}
    &\quad U_\lambda(\mu_x, \tilde{\mu}_z^\ast(\mu_x), \mu_z) - U_\lambda(\mu_x, \tilde{\mu}_z^\ast(\mu_x^\prime), \mu_z) \nonumber \\
    &= U_2(\tilde{\mu}_z^\ast(\mu_{x})) - U_2(\tilde{\mu}_z^\ast(\mu_{x}^\prime)) + \lambda \cdot U_1(\mu_x, \tilde{\mu}_z^\ast(\mu_{x}))  - \lambda \cdot U_1(\mu_x, \tilde{\mu}_z^\ast(\mu_{x}^\prime)) \label{eq:convexity_tern_one} \\
    &\qquad + \lambda \frac{\zeta_2}{2} \E_{\tilde{\mu}_z^\ast(\mu_{x})}\left[\|z\|^2\right]- \lambda \frac{\zeta_2}{2} \E_{\tilde{\mu}_z^\ast(\mu_{x}^\prime)}\left[\|z\|^2\right] + \lambda\sigma_1 \cdot \mathrm{Ent}\left(\tilde{\mu}_z^\ast(\mu_{x})\right) - \lambda\sigma_1 \cdot \mathrm{Ent}\left(\tilde{\mu}_z^\ast(\mu_{x}^\prime) \right) \label{eq:convexity_tern_two} .
\end{align}
Notice that
\begin{align*}
    &\quad \mathrm{Ent} \left(\tilde{\mu}_z^\ast(\mu_{x})\right) - \mathrm{Ent}\left(\tilde{\mu}_z^\ast(\mu_{x}^\prime)\right) \\
    &= \int \log(\tilde{\mu}_z^\ast(\mu_{x})(z)) \tilde{\mu}_z^\ast(\mu_{x})(z) \; \dd z - \int \log(\tilde{\mu}_z^\ast(\mu_{x}^\prime)(z)) \tilde{\mu}_z^\ast(\mu_{x}^\prime)(z) \; \dd z \\
    & = \kl(\tilde{\mu}_z^\ast(\mu_{x}), \tilde{\mu}_z^\ast(\mu_{x}^\prime)) + \int \log(\tilde{\mu}_z^\ast(\mu_{x}^\prime)(z)) \Big( \tilde{\mu}_z^\ast(\mu_{x})(z) - \tilde{\mu}_z^\ast(\mu_{x}^\prime)(z) \Big) \; \dd z .
\end{align*}
Given the optimality of $\tilde{\mu}_z^\ast(\mu_{x}^\prime)$, we know that the first variation of the mapping $\mu_z\mapsto \scrF_2(\mu_x,\mu_z)+ \lambda \scrF_1(\mu_x,\mu_z)$ at $\tilde{\mu}_z^\ast(\mu_{x}^\prime)$ equals a constant~\citep[Proposition 2.5]{hu2021mean}. 
\begin{align*}
    &\quad \E_{\rho} \left[  \left( \int \Psi_\bw \; \dd \tilde{\mu}_z^\ast(\mu_x^\prime) - \by \right) \Psi_\bw(z) \right] + \lambda\cdot \E_{\rho} \left[  \left( \int \Psi_\bw \; \dd \tilde{\mu}_z^\ast(\mu_x^\prime) - \int \Psi_\ba \; \dd \mu_x^\prime \right) \Psi_\bw(z) \right] \\
    &\qquad + \lambda \frac{\zeta_2}{2} \|z\|^2 + \lambda \sigma_1 \log \tilde{\mu}_{z}^\ast(\mu_x^\prime)(z) = C' .
\end{align*}
Here, $C'$ is a constant that does not vary with $z$. 
So we obtain
\begin{align*}
    &\quad \lambda \sigma_1 \int \log(\tilde{\mu}_z^\ast(\mu_{x}^\prime)(z)) \Big( \tilde{\mu}_z^\ast(\mu_{x})(z) - \tilde{\mu}_z^\ast(\mu_{x}^\prime)(z) \Big) \; \dd z \\
    &= \int \lambda \frac{\zeta_2}{2}\|z\|^2  \; \left( \dd \tilde{\mu}_z^\ast(\mu_x^\prime) - \dd \tilde{\mu}_z^\ast(\mu_x) \right) + \E_{\rho} \left[  \left( \int \Psi_\bw \; \dd \tilde{\mu}_z^\ast(\mu_x^\prime) - \by \right) \left( \int \Psi_\bw \; \left( \dd \tilde{\mu}_z^\ast(\mu_x^\prime) - \dd \tilde{\mu}_z^\ast(\mu_x) \right) \right) \right] \\
    &\qquad + \lambda\cdot \E_{\rho} \left[  \left( \int \Psi_\bw \; \dd \tilde{\mu}_z^\ast(\mu_x^\prime) - \int \Psi_\ba \; \dd \mu_x^\prime \right) \left( \int \Psi_\bw \; \left( \dd \tilde{\mu}_z^\ast(\mu_x^\prime) - \dd \tilde{\mu}_z^\ast(\mu_x) \right) \right) \right] . 
\end{align*}
So the term in Eq.~\eqref{eq:convexity_tern_two} equals
\begin{align*}
    \eqref{eq:convexity_tern_two} &= \lambda \sigma_1 \kl(\tilde{\mu}_z^\ast(\mu_{x}), \tilde{\mu}_z^\ast(\mu_{x}^\prime)) + \E_{\rho} \left[  \left( \int \Psi_\bw \; \dd \tilde{\mu}_z^\ast(\mu_x^\prime) - \by \right) \left( \int \Psi_\bw \; \left( \dd \tilde{\mu}_z^\ast(\mu_x^\prime) - \dd \tilde{\mu}_z^\ast(\mu_x) \right) \right) \right] \\
    &\qquad + \lambda\cdot \E_{\rho} \left[  \left( \int \Psi_\bw \; \dd \tilde{\mu}_z^\ast(\mu_x^\prime) - \int \Psi_\ba \; \dd \mu_x^\prime \right) \left( \int \Psi_\bw \; \left( \dd \tilde{\mu}_z^\ast(\mu_x^\prime) - \dd \tilde{\mu}_z^\ast(\mu_x) \right) \right) \right] . 
\end{align*}
We also have, by definition of $U_2$, 
\begin{align*}
    U_2(\tilde{\mu}_z^\ast(\mu_{x})) - U_2(\tilde{\mu}_z^\ast(\mu_{x}^\prime)) &= \frac{1}{2} \E_{\rho} \left[  \left( \int \Psi_\bw \; \dd \tilde{\mu}_z^\ast(\mu_x^\prime) + \Psi_\bw \; \dd \tilde{\mu}_z^\ast(\mu_x) - 2 \by \right) \right. \\
    &\qquad\qquad \cdot \left. \left( \int \Psi_\bw \; \left( \dd \tilde{\mu}_z^\ast(\mu_x) - \dd \tilde{\mu}_z^\ast(\mu_x^\prime) \right) \right) \right] . 
\end{align*}
And, by definition of $U_1$, 
\begin{align*}
    U_1(\mu_x, \tilde{\mu}_z^\ast(\mu_{x})) - U_1(\mu_x, \tilde{\mu}_z^\ast(\mu_{x}')) &= \frac{1}{2} \E_{\rho} \left[  \left( -2 \int \Psi_\ba \; \dd \mu_x + \int \Psi_\bw \; \dd \tilde{\mu}_z^\ast(\mu_x^\prime) + \int \Psi_\bw \; \dd \tilde{\mu}_z^\ast(\mu_x) \right) \right. \\
    &\qquad\qquad \cdot \left. \left( \int \Psi_\bw \; \left( \dd \tilde{\mu}_z^\ast(\mu_x) - \dd \tilde{\mu}_z^\ast(\mu_x^\prime) \right) \right) \right] .
\end{align*}
We plug the above two equations back to Eq.~\eqref{eq:convexity_tern_one} which gives
\begin{align*}
    \eqref{eq:convexity_tern_one} &= \frac{1}{2} \E_{\rho} \left[  \left( -2 \lambda \int \Psi_\ba \; \dd \mu_x + (1+\lambda) \int \Psi_\bw \; \dd \tilde{\mu}_z^\ast(\mu_x^\prime) + (1+\lambda) \int \Psi_\bw \; \dd \tilde{\mu}_z^\ast(\mu_x) -2\by \right) \right. \\
    &\qquad \left. \left( \int \Psi_\bw \; \left( \dd \tilde{\mu}_z^\ast(\mu_x) - \dd \tilde{\mu}_z^\ast(\mu_x^\prime) \right) \right) \right].
\end{align*}
Combine the above derivations about Eq.~\eqref{eq:convexity_tern_one} and Eq.~\eqref{eq:convexity_tern_two}, we obtain
\begin{align*}
    &\quad U_\lambda(\mu_x, \tilde{\mu}_z^\ast(\mu_x), {\mu}_z) - U_\lambda(\mu_x, \tilde{\mu}_z^\ast(\mu_x^\prime), {\mu}_z) = \lambda \sigma_1 \kl(\tilde{\mu}_z^\ast(\mu_{x}), \tilde{\mu}_z^\ast(\mu_{x}^\prime)) \\
    &+ \E_{\rho} \left[  \left( -\lambda \int \Psi_\ba \; \dd \mu_x + \lambda \int \Psi_\ba \; \dd \mu_x' - \frac{1+\lambda}{2} \int \Psi_\bw \; \dd \tilde{\mu}_z^\ast(\mu_x^\prime) + \frac{1+\lambda}{2} \int \Psi_\bw \; \dd \tilde{\mu}_z^\ast(\mu_x) \right) \right. \\
    &\qquad \left. \left( \int \Psi_\bw \; \left( \dd \tilde{\mu}_z^\ast(\mu_x) - \dd \tilde{\mu}_z^\ast(\mu_x^\prime) \right) \right) \right] \\
    &\geq -\frac{1+\lambda}{2} \E_{\rho} \left[  \left( \int \Psi_\bw \; \dd \left(\tilde{\mu}_z^\ast(\mu_x) - \tilde{\mu}_z^\ast(\mu_x^\prime) \right) \right)^2 \right] \\
    &\geq -\frac{R(1+\lambda)}{2} \tv^2(\tilde{\mu}_z^\ast(\mu_x), \tilde{\mu}_z^\ast(\mu_x^\prime)) \\
    & \geq -R(1+\lambda) \kl(\tilde{\mu}_z^\ast(\mu_x), \tilde{\mu}_z^\ast(\mu_x^\prime)) . 
\end{align*}
The first inequality holds by Eq.~\eqref{eq:continuity_one} in \Cref{prop:continuity} and the last inequality holds by Pinsker's inequality.
Finally, we apply Eq.~\eqref{eq:kl_upper_bound} in \Cref{prop:continuity} to conclude the proof. \begin{align*}
    U_\lambda(\mu_x, \tilde{\mu}_z^\ast(\mu_x), {\mu}_z) - U_\lambda(\mu_x, \tilde{\mu}_z^\ast(\mu_x^\prime), {\mu}_z) \geq -\frac{R \lambda}{4\sigma_1} \E_{\rho} \left[ \left( \int \Psi_\ba \; \dd (\mu_x - \mu_x^\prime) \right)^2 \right] \geq - \frac{R^3 \lambda}{4\sigma_1} \tv^2(\mu_x,\mu_x^\prime) . 
\end{align*} 
\end{proof}

\begin{prop}[Leave-one-out Lipschitzness of $\delta_1 U_\lambda$ and $\bnabla_1 U_1$]\label{prop:leave_one_out_lipschtiz_delta_Ux}
Let $\mu_{\scrX}=\frac{1}{N_x}\sum_{i=1}^{N_x} \delta_{x^{(i)}}$ be the empirical distribution of $N_x$ particles; and let $\mu_{x\cup\scrX_{-i}}=\frac{1}{N_x}\sum_{j\neq i} \delta_{x^{(j)}}+\frac{1}{N_x}\delta_x$ be another empirical distribution of $N_x$ particles with the $i$-th particle replaced with $x$. Then we have, for any $x\in\R^{d_x}$,
\begin{align*}
    \left| \delta_1 U_1(\mu_{\scrX}, \tilde{\mu}_z^\ast(\mu_{\scrX}))(x) - \delta_1 U_1(\mu_{x\cup\scrX_{-i}}, \tilde{\mu}_z^\ast(\mu_{x\cup\scrX_{-i}}))(x) \right| &\leq \frac{1}{N_x}\left( R^2 + \sqrt{\frac{2}{\sigma_1}}R^3 \right) , \\
    \left| \delta U_2(\mu_z^\ast(\mu_{\scrX}))(x) - \delta U_2(\mu_z^\ast(\mu_{x\cup\scrX_{-i}}))(x) \right| &\leq \frac{1}{N_x}\left( R^2 + \sqrt{\frac{2}{\sigma_1}}R^3 \right) .
\end{align*}
And 
\begin{align*}
    \left| \bnabla_1 U_1(\mu_{\scrX}, \tilde{\mu}_z^\ast(\mu_{\scrX}))(x) - \bnabla_1 U_1(\mu_{x\cup\scrX_{-i}}, \tilde{\mu}_z^\ast(\mu_{x\cup\scrX_{-i}}))(x) \right| &\leq \frac{1}{N_x}\left( R^2 + \sqrt{\frac{2}{\sigma_1}}  R^3 \right) , \\
    \left| \bnabla_1 U_2(\mu_z^\ast(\mu_{\scrX}))(x) - \bnabla_1 U_2(\mu_z^\ast(\mu_{x\cup\scrX_{-i}}))(x) \right| &\leq \frac{1}{N_x}\left( R^2 + \sqrt{\frac{2}{\sigma_1}}R^3 \right)  .
\end{align*}
\end{prop}
\begin{proof}
We have from \Cref{lem:lip_wass_grad}, 
\begin{align*}
    &\quad \left| \delta_1 U_1(\mu_{\scrX}, \tilde{\mu}_z^\ast(\mu_{x\cup\scrX_{-i}}))(x) - \delta_1 U_1(\mu_{x\cup\scrX_{-i}}, \tilde{\mu}_z^\ast(\mu_{\scrX}))(x) \right| \\
    &\leq R \left| \E_\rho \left[ \frac{1}{N_x}\sum_{i=1}^{N_x} \Psi_\ba(x^{(i)}) - \left(\frac{1}{N_x}\sum_{j \neq i} \Psi_\ba(x^{(j)}) + \frac{1}{N_x} \Psi_\ba(x) \right) \right] \right| \\
    &\qquad\qquad + R \left| \E_\rho \left[ \int \Psi_\bw \; \dd \tilde{\mu}_z^\ast(\mu_{\scrx}) - \int \Psi_\bw \; \dd \tilde{\mu}_z^\ast(\mu_{x\cup\scrX_{-i}}) \right] \right| \\
    &\leq \frac{2}{N_x} R^2 + R^2 \cdot \tv(\tilde{\mu}_z^\ast(\mu_{\scrx}), \tilde{\mu}_z^\ast(\mu_{x\cup\scrX_{-i}})).
\end{align*}
The last inequality holds by the boundedness of $\Psi_\ba$ and the boundness of $\Psi_\bw$. 
% Since we know from \Cref{lem:LSI} that both $\tilde{\mu}_z^\ast(\mu_{\scrx})$ and $\tilde{\mu}_z^\ast(\mu_{x\cup\scrX_{-i}}))$ satisfy a Log-Sobolev inequality with a constant $C_{\LSI,z}$, 
Next, we have
\begin{align*}
    \tv(\tilde{\mu}_z^\ast(\mu_{\scrx}), \tilde{\mu}_z^\ast(\mu_{x\cup\scrX_{-i}})) &\leq \sqrt{2} \sqrt{ \kl(\tilde{\mu}_z^\ast(\mu_{\scrx}), \tilde{\mu}_z^\ast(\mu_{x\cup\scrX_{-i}}))} \\
    &\leq \sqrt{2} \sqrt{\frac{1}{4\sigma_1} \E_{\rho} \left[ \left( \int \Psi_\ba \; \dd (\mu_\scrx - \mu_{x\cup\scrX_{-i}}) \right)^2 \right] } \\
    &\leq \sqrt{\frac{2}{\sigma_1}} \frac{R}{N_x} .
\end{align*}
The first inequality holds by Pinsker’s inequality, the second inequality holds by using \Cref{prop:continuity} and the last inequality holds by the boundedness of $\Psi_\ba$. 
Therefore, we have shown that
\begin{align*}
    &\quad \left| \delta_1 U_1(\mu_{\scrX}, \tilde{\mu}_z^\ast(\mu_{\scrX}))(x) - \delta_1 U_1(\mu_{x\cup\scrX_{-i}}, \tilde{\mu}_z^\ast(\mu_{x\cup\scrX_{-i}}))(x) \right| \leq \frac{1}{N_x}\left( R^2 + \sqrt{\frac{2}{\sigma_1}}  R^3 \right) .
\end{align*}
The similar derivations hold for the other three inequalities as well. So the proof is concluded. 
\end{proof}

\begin{prop}[Continuity of the mappings $\mu_x\mapsto \tilde{\mu}_z^\ast(\mu_x)$ and $\mu_x\mapsto \mu_z^\ast(\mu_x)$]
\label{prop:continuity}
Suppose \Cref{ass:network} holds. 
Let $\mu_x, \mu_x^\prime\in\calP_2(\R^{d_x})$. Let $\tilde{\mu}_z^\ast(\mu_x) = \arg\min_{\mu_z} \scrL_\lambda(\mu_x, \mu_z) = \arg\min_{\mu_z} U_2(\mu_z) + \lambda F_1(\mu_x, \mu_z) + \lambda \sigma_1 \mathrm{Ent}(\mu_z)$ in Eq.~\eqref{eq:inner_loop_problem_2}. Then, we have
\begin{align}
    \lambda \sigma_1 \Big( \kl(\tilde{\mu}_z^\ast(\mu_x), \; \tilde{\mu}_z^\ast(\mu_x^\prime)) + \kl(\tilde{\mu}_z^\ast(\mu_x^\prime), \; \tilde{\mu}_z^\ast(\mu_x))\Big) &= -(1+\lambda)\E_{\rho} \left[  \left( \int \Psi_\bw \; \dd \left(\tilde{\mu}_z^\ast(\mu_x) - \tilde{\mu}_z^\ast(\mu_x^\prime) \right) \right)^2 \right] \nonumber \\
    &\hspace{-20em} + \lambda \E_{\rho} \left[  \left( \int \Psi_\bw \; \dd \left(\tilde{\mu}_z^\ast(\mu_x) - \tilde{\mu}_z^\ast(\mu_x^\prime) \right) \right) \cdot \left( \int \Psi_\ba \; \dd (\mu_x - \mu_x^\prime) \right) \right]  \label{eq:continuity_one} \\
    \lambda \sigma_1 \Big( \kl(\tilde{\mu}_z^\ast(\mu_x), \; \tilde{\mu}_z^\ast(\mu_x^\prime)) + \kl(\tilde{\mu}_z^\ast(\mu_x^\prime), \; \tilde{\mu}_z^\ast(\mu_x))\Big) &\leq \frac{\lambda^2}{4(1+\lambda)} \E_{\rho} \left[ \left( \int \Psi_\ba \; \dd (\mu_x - \mu_x^\prime) \right)^2 \right] \nonumber \\
    &\leq \frac{\lambda^2 R^2}{4(1+\lambda)} \tv^2(\mu_x, \mu_x^\prime) \label{eq:kl_upper_bound}.
\end{align}
Similarly, let $\mu_z^\ast(\mu_x) = \arg\min_{\mu_z} \scrF_1(\mu_x,\mu_z) = \arg\min_{\mu_z} F_1(\mu_x,\mu_z) + \sigma_1 \mathrm{Ent}(\mu_z)$ in Eq.~\eqref{eq:inner_loop_problem_1}. Then, 
\begin{align}
    \sigma_1 \Big( \kl(\mu_z^\ast(\mu_x), \; \mu_z^\ast(\mu_x^\prime)) + \kl(\mu_z^\ast(\mu_x^\prime), \; \mu_z^\ast(\mu_x)) \Big) &\leq \frac{1}{4} \E_{\rho} \left[  \left( \int \Psi_\ba \; \dd (\mu_x - \mu_x^\prime) \right)^2 \right] \nonumber \\
    &\leq \frac{R^2}{4} \tv^2(\mu_x, \mu_x^\prime) \label{eq:kl_upper_bound_2} .
\end{align}
\end{prop}
\begin{proof}
From the optimality of $\tilde{\mu}_z^\ast(\mu_x)$, the first variation of the mapping $\mu_z\mapsto U_2(\mu_z) + \lambda F_1(\mu_x, \mu_z) + \lambda \sigma_1 \mathrm{Ent}(\mu_z)$~\citep[Proposition 2.5]{hu2021mean} equals a constant, for any $z\in\R^{d_z}$. 
\begin{align*}
    &\E_{\rho} \left[  \left( \int \Psi_\bw \; \dd \tilde{\mu}_z^\ast(\mu_x) - \by \right) \Psi_\bw(z) \right] + \lambda\cdot \E_{\rho} \left[  \left( \int \Psi_\bw \; \dd \tilde{\mu}_z^\ast(\mu_x) - \int \Psi_\ba \; \dd \mu_x \right) \Psi_\bw(z) \right] \\
    &\qquad + \lambda \frac{\zeta_2}{2} \|z\|^2 + \lambda \sigma_1 \log \tilde{\mu}_{z}^\ast(\mu_x)(z) = C, \\
    &\E_{\rho} \left[  \left( \int \Psi_\bw \; \dd \tilde{\mu}_z^\ast(\mu_x^\prime) - \by \right) \Psi_\bw(z) \right] + \lambda\cdot \E_{\rho} \left[  \left( \int \Psi_\bw \; \dd \tilde{\mu}_z^\ast(\mu_x^\prime) - \int \Psi_\ba \; \dd \mu_x^\prime \right) \Psi_\bw(z) \right] \\
    &\qquad + \lambda \frac{\zeta_2}{2} \|z\|^2 + \lambda\sigma_1 \log \tilde{\mu}_{z}^\ast(\mu_x^\prime)(z) = C' ,
\end{align*}
where $C, C^\prime$ are two constants that do not vary with $z$. 
If we subtract the above two equations, we obtain
\begin{align}
    & \E_{\rho} \left[  \left( \int (1+\lambda) \Psi_\bw \; \dd \left(\tilde{\mu}_z^\ast(\mu_x) - \tilde{\mu}_z^\ast(\mu_x^\prime) \right) - \lambda \int \Psi_\ba \; \dd (\mu_x - \mu_x^\prime) \right) \Psi_\bw(z) \right] \nonumber \\
    &\qquad \qquad + \lambda \sigma_1 \left( \log \tilde{\mu}_{z}^\ast(\mu_x)(z) - \log \tilde{\mu}_{z}^\ast(\mu_x^\prime)(z) \right) = C - C^\prime \label{eq:diff_stationary}.
\end{align}
Take expectation with respect to $\tilde{\mu}_z^\ast(\mu_x^\prime)$ for both sides of Eq.~\eqref{eq:diff_stationary}, then we obtain
\begin{align*}
    &\E_{\rho} \left[ \left( \int (1+\lambda) \Psi_\bw \; \dd \left(\tilde{\mu}_z^\ast(\mu_x) - \tilde{\mu}_z^\ast(\mu_x^\prime) \right) - \lambda \int \Psi_\ba \; \dd (\mu_x - \mu_x^\prime) \right) \cdot \int \Psi_\bw \; \dd \tilde{\mu}_z^\ast(\mu_x^\prime) \right] \\
    &\qquad \qquad - \lambda \sigma_1 \kl(\tilde{\mu}_{z}^\ast(\mu_x^\prime), \tilde{\mu}_{z}^\ast(\mu_x)) = C - C^\prime .
\end{align*}
Take expectation of Eq.~\eqref{eq:diff_stationary} with respect to $\tilde{\mu}_z^\ast(\mu_x)$ on both sides, then we obtain
\begin{align*}
    &\E_{\rho} \left[ \left( \int (1+\lambda) \Psi_\bw \; \dd \left(\tilde{\mu}_z^\ast(\mu_x) - \tilde{\mu}_z^\ast(\mu_x^\prime) \right) - \lambda \int \Psi_\ba \; \dd (\mu_x - \mu_x^\prime) \right) \cdot \int \Psi_\bw \; \dd \tilde{\mu}_z^\ast(\mu_x) \right] \\
    &\qquad \qquad + \lambda \sigma_1 \kl(\tilde{\mu}_{z}^\ast(\mu_x),  \tilde{\mu}_{z}^\ast(\mu_x^\prime)) = C - C^\prime .
\end{align*}
We subtract the above two equations, and obtain
\begin{small}
\begin{align*}
    &\quad \E_{\rho} \left[  \left( \int (1+\lambda) \Psi_\bw \; \dd \left(\tilde{\mu}_z^\ast(\mu_x) - \tilde{\mu}_z^\ast(\mu_x^\prime) \right) - \lambda \int \Psi_\ba \; \dd (\mu_x - \mu_x^\prime) \right) \cdot \left( \int \Psi_\bw \; \dd \tilde{\mu}_z^\ast(\mu_x) - \int \Psi_\bw \; \dd \tilde{\mu}_z^\ast(\mu_x') \right) \right] \\
    &+\lambda \sigma_1 \kl(\tilde{\mu}_{z}^\ast(\mu_x^\prime), \tilde{\mu}_{z}^\ast(\mu_x)) + \lambda \sigma_1 \kl(\tilde{\mu}_{z}^\ast(\mu_x) , \tilde{\mu}_{z}^\ast(\mu_x^\prime)) = 0. 
\end{align*}
\end{small}
After a reordering of the above equality, we have
\begin{align*}
    &\quad \lambda \sigma_1 \kl(\tilde{\mu}_{z}^\ast(\mu_x^\prime), \tilde{\mu}_{z}^\ast(\mu_x)) + \lambda \sigma_1 \kl(\tilde{\mu}_{z}^\ast(\mu_x) , \tilde{\mu}_{z}^\ast(\mu_x^\prime)) \\
    &= -(1+\lambda)\E_{\rho} \left[  \left( \int \Psi_\bw \; \dd \left(\tilde{\mu}_z^\ast(\mu_x) - \tilde{\mu}_z^\ast(\mu_x^\prime) \right) \right)^2 \right] \\
    &\qquad\qquad + \lambda \E_{\rho} \left[  \left( \int \Psi_\bw \; \dd \left(\tilde{\mu}_z^\ast(\mu_x) - \tilde{\mu}_z^\ast(\mu_x^\prime) \right) \right) \cdot \left( \int \Psi_\ba \; \dd (\mu_x - \mu_x^\prime) \right) \right] .
\end{align*}
So we have finished the proof for Eq.~\eqref{eq:continuity_one}. 
Notice that the right hand side of Eq.~\eqref{eq:continuity_one} can be further upper bounded by,
\begin{align*}
    \text{RHS of } \eqref{eq:continuity_one} \leq \frac{\lambda^2}{4(1+\lambda)} \E_{\rho} \left[ \left( \int \Psi_\ba \; \dd (\mu_x - \mu_x^\prime) \right)^2 \right] \leq \frac{\lambda^2 R^2}{4(1+\lambda)} \tv^2(\mu_x, \mu_x^\prime) ,
\end{align*}
where we use the fact $ab \leq \frac{1+\lambda}{\lambda}a^2 + \frac{\lambda}{4(1+\lambda)} b^2$ for the first inequality and the boundness of $\Psi_\ba$ for the second inequality. 
So we have finished the proof for Eq.~\eqref{eq:kl_upper_bound}. 
The same derivations hold for the proof of Eq.~\eqref{eq:kl_upper_bound_2} as well. 
\end{proof}

\begin{lem}[Lipschitz continuity of Wasserstein gradient and first variation of $U_1, U_2$]\label{lem:lip_wass_grad}
Suppose \Cref{ass:network} and \ref{ass:npiv} hold. 
Let $U_1, U_2$ be defined in \eqref{eq:distribution_space_optimization}. 
For any $i,j\in\{1,2\}$, we have the following Lipschitz continuity of the Wasserstein gradients: for any $x_1,x_2\in\R^{d_x}$, any $\mu_{x,1},\mu_{x,2}\in\calP_2(\R^{d_x})$ and any $\mu_{z,1},\mu_{z,2}\in\calP_2(\R^{d_z})$ 
\begin{align}\label{eq:lipschitz_wass_grad}
\begin{aligned}
    &\quad \left\| \bnabla_i U_j(\mu_{x,1}, \mu_{z,1})(x_1) - \bnabla_i U_j(\mu_{x,2}, \mu_{z,2})(x_2) \right\| \\
    &\leq R \left| \E_\rho \left[ \smallint \Psi_\ba \; (\dd \mu_{x,1} - \dd \mu_{x,2}) \right] \right| + R \left| \E_\rho \left[ \smallint \Psi_\bw \; (\dd \mu_{z,1} - \dd \mu_{z,2}) \right] \right| + R^2 \|x_1 - x_2\| \\
    &\leq R^2 \left( W_2(\mu_{x,1},\mu_{x,2}) + W_2(\mu_{z,1},\mu_{z,2}) + \|x_1 - x_2\| \right). 
\end{aligned}
\end{align}
For any $i,j\in\{1,2\}$, we have the following Lipschitz continuity of the first variations: for any $x_1,x_2\in\R^{d_x}$, any $\mu_{x,1},\mu_{x,2}\in\calP_2(\R^{d_x})$ and any $\mu_{z,1},\mu_{z,2}\in\calP_2(\R^{d_z})$ 
\begin{align}\label{eq:lipschitz_first_var}
\begin{aligned}
    &\quad \left| \delta_i U_j(\mu_{x,1}, \mu_{z,1})(x_1) - \delta_i U_j(\mu_{x,2}, \mu_{z,2})(x_2) \right| \\
    &\leq R \left| \E_\rho \left[ \smallint \Psi_\ba \; (\dd \mu_{x,1} - \dd \mu_{x,2}) \right] \right| + R \left| \E_\rho \left[ \smallint \Psi_\bw \; (\dd \mu_{z,1} - \dd \mu_{z,2}) \right] \right| + R^2 \|x_1 - x_2\| \\
    &\leq R^2 \left( W_2(\mu_{x,1},\mu_{x,2}) + W_2(\mu_{z,1},\mu_{z,2}) + \|x_1 - x_2\| \right). 
\end{aligned}
\end{align}
We have the following Lipschitz continuity of the objectives themselves: for any $x_1,x_2\in\R^{d_x}$, any $\mu_{x,1},\mu_{x,2}\in\calP_2(\R^{d_x})$ and any $\mu_{z,1},\mu_{z,2}\in\calP_2(\R^{d_z})$ 
\begin{align}\label{eq:lipschitz_U}
\begin{aligned}
    &\quad \left| U_1(\mu_{x,1}, \mu_{z,1}) - U_1(\mu_{x,2}, \mu_{z,2}) \right| \\
    &\leq 2R \left| \E_\rho \left[ \smallint \Psi_\ba \; (\dd \mu_{x,1} - \dd \mu_{x,2}) \right] \right| + 2R \left| \E_\rho \left[ \smallint \Psi_\bw \; (\dd \mu_{z,1} - \dd \mu_{z,2}) \right] \right| \\
    &\leq 2R^2 \left( W_2(\mu_{x,1},\mu_{x,2}) + W_2(\mu_{z,1},\mu_{z,2}) \right). 
\end{aligned}
\end{align}
And 
\begin{align}\label{eq:lipschitz_U_2}
\begin{aligned}
    \left| U_2(\mu_{z,1}) - U_2(\mu_{z,2}) \right| \leq (R+M) \E_\rho \left[\left| \smallint \Psi_\bw \; \dd \mu_{z,1} - \smallint \Psi_\bw \; \dd \mu_{z,2}\right| \right] \leq R(R+M) W_2(\mu_{z,1}, \mu_{z,2}) . 
\end{aligned}
\end{align}
\end{lem}
\begin{proof}
From the definition of $U_1, U_2$, we have
\begin{align*}
    &\quad \left| \bnabla_1 U_1(\mu_{x,1}, \mu_{z,1})(x_1) - \bnabla_1 U_1(\mu_{x,2}, \mu_{z,2})(x_2) \right| \\
    &= \left| \E_{\rho} \left[ \left(  \smallint \Psi_\ba \; \dd \mu_{x,1} - \smallint \Psi_\bw \; \dd \mu_{z,1} \right) \cdot \nabla \Psi_\ba(x_1) \right] \right.  - \left. \E_{\rho} \left[ \left(  \smallint \Psi_\ba \; \dd \mu_{x,2} - \smallint \Psi_\bw \; \dd \mu_{z,2} \right) \cdot \nabla \Psi_\ba(x_2) \right] \right| \\
    &\leq R \left| \E_\rho \left[ \smallint \Psi_\ba \; (\dd \mu_{x,1} - \dd \mu_{x,2}) \right] \right| + R \left| \E_\rho \left[ \smallint \Psi_\bw \; (\dd \mu_{z,1} - \dd \mu_{z,2}) \right] \right| + R^2 \|x_1 - x_2\| \\
    &\leq R^2 \left( W_2(\mu_{x,1},\mu_{x,2}) + W_2(\mu_{z,1},\mu_{z,2}) + \|x_1 - x_2\| \right).
\end{align*}
The same proof applies to the Lipschitzness of $\bnabla_2 U_1$, $\bnabla_1 U_2$ and $\bnabla_2 U_2$ as well. 
So we have finished the proof of Eq.~\eqref{eq:lipschitz_wass_grad}. 
Next, following the same derivations, we have 
\begin{align*}
    &\quad \left| \delta_1 U_1(\mu_{x,1}, \mu_{z,1})(x_1) - \delta_1 U_1(\mu_{x,2}, \mu_{z,2})(x_2) \right| \\
    &\leq R \left| \E_\rho \left[ \smallint \Psi_\ba \; (\dd \mu_{x,1} - \dd \mu_{x,2}) \right] \right| + R \left| \E_\rho \left[ \smallint \Psi_\bw \; (\dd \mu_{z,1} - \dd \mu_{z,2}) \right] \right| + R^2 \|x_1 - x_2\| \\
    &\leq R^2 \left( W_2(\mu_{x,1},\mu_{x,2}) + W_2(\mu_{z,1},\mu_{z,2}) + \|x_1 - x_2\| \right).
\end{align*}
So we have finished the proof of Eq.~\eqref{eq:lipschitz_first_var}. 
Finally, we have
\begin{align*}
    &\quad \left| U_1(\mu_{x,1}, \mu_{z,1}) - U_1(\mu_{x,2},  \mu_{z,2}) \right| = \frac{1}{2} \left| \E_\rho \left[ \left( \smallint \Psi_\ba \; \dd \mu_{x,1} - \smallint \Psi_\bw \; \dd \mu_{z,1} \right)^2 \right]\right. \\
    &\qquad - \left. \frac{1}{2} \E_\rho \left[ \left( \smallint \Psi_\ba \; \dd \mu_{x,2} - \smallint \Psi_\bw \; \dd \mu_{z,2} \right)^2 \right]\right| \\
    &\leq 2R \cdot \left| \smallint \Psi_\ba \; \dd (\mu_{x,1} - \mu_{x,2}) \right| + 2R \left| \smallint \Psi_\bw \; \dd (\mu_{z,1}- \mu_{z,2} )\right| \\
    &\leq 2R^2 W_2(\mu_{x,1}, \mu_{x,2}) + 2R^2 W_2(\mu_{z,1}, \mu_{z,2}) .
\end{align*}
And
\begin{align*}
    &\left| U_2(\mu_{z,1}) - U_2(\mu_{z,2}) \right| = \frac{1}{2} \left| \E_\rho \left[ \left( \smallint \Psi_\bw \; \dd \mu_{z,1} - \by \right)^2 \right] - \E_\rho \left[ \left( \smallint \Psi_\bw \; \dd \mu_{z,2} - \by \right)^2 \right] \right| \\
    &= \frac{1}{2} \left| \E_\rho \left[ \left(\smallint \Psi_\bw \; \dd \mu_{z,1} - \smallint \Psi_\bw \; \dd \mu_{z,2} \right) \cdot \left(\smallint \Psi_\bw \; \dd \mu_{z,1} + \smallint \Psi_\bw \; \dd \mu_{z,2} - 2 \by \right) \right] \right| \\
    &\leq (R+M) \E_\rho \left[\left| \smallint \Psi_\bw \; \dd \mu_{z,1} - \smallint \Psi_\bw \; \dd \mu_{z,2}\right| \right] \leq R(R+M) W_2(\mu_{z,1}, \mu_{z,2}) .
\end{align*}
So we have finished the proof. 
\end{proof}

\begin{lem}\label{lem:tedious_calculation}
Let $V:\R^d\to\R^d$ be an admissible transport map. Let $\mu\in\calP_2^r(\R^d)$ with a $C^1$ density function $p:\R^d\to\R$. Define $\phi_t := \Id + tV$ for $0 < t < 1$. Denote $q$ as the density of another probability measure $(\Id + t V)_\# \mu = {\phi_t}_\# \mu$. Then, we have $\nabla \log q(\phi_t(x)) = \nabla \log p(x) - t\cdot \bJ V(x) \nabla \log p(x) - t \cdot \nabla \mathrm{Tr}(\bJ V(x)) + o(t)$ and $\nabla \log p(\phi_t(x)) = \nabla \log p(x) + t\cdot \bH \log p(x) V(x) + o(t)$. 
\end{lem}
\begin{proof}
From the change of variables formula, we have $q(\phi_t(x)) = \frac{p(x)}{|\Id + t \bJ V(x)|}$. Hence, we have $q(x) = \frac{p(\phi_t^{-1}(x))}{|\Id + t \bJ V(\phi_t^{-1}(x))|}$.
\begin{align*}
    \nabla \log q(x) = \frac{\nabla p(\phi_t^{-1}(x))^\top \left[\bJ \phi_t(\phi_t^{-1}(x)) \right]^{-1} }{p(\phi_t^{-1}(x))} -  t \cdot\left[\bJ \phi_t(\phi_t^{-1}(x)) \right]^{-1} F,
\end{align*}
where $F=[F_1,\ldots,F_d]^\top \in\R^d$ with $F_i = \mathrm{Tr}([\Id + t \bJ V(\phi_t^{-1}(x))]^{-1} \nabla_{x_i} \bJ V(\phi_t^{-1}(x)))$. 
Hence, 
\begin{align*}
    \nabla \log q(\phi_t(x)) = \frac{\nabla p(x)^\top \left[\bJ \phi_t(x) \right]^{-1} }{p(x)} - t \cdot \left[\bJ \phi_t(x) \right]^{-1} \tilde{F}, 
\end{align*}
where $\tilde{F}=[\tilde{F}_1,\ldots, \tilde{F}_d]^\top \in\R^d$ with $\tilde{F}_i = \mathrm{Tr}([\Id + t \bJ V(x)]^{-1} \nabla_{x_i} \bJ V(x))$. 
Since $V:\R^d\to\R^d$ is an admissible transport map, its derivatives are infinitely times differentiable over a compact domain, so we have $ \left[\bJ \phi_t(x) \right]^{-1} = [\Id + t \bJ V(x)]^{-1} = \Id - t \bJ V(x) + o(t)$. 
So we obtain 
\begin{align*}
    \nabla \log q(\phi_t(x)) = \nabla \log p(x) - t\cdot \bJ V(x) \nabla \log p(x) - t \cdot \nabla \mathrm{Tr}(\bJ V(x)) + o(t) . 
\end{align*} 
Repeating the same steps as above, we obtain
\begin{align*}
    \nabla \log p(\phi_t(x)) = \nabla \log p(x) + t\cdot \bH \log p(x) V(x) + o(t) . 
\end{align*}
\end{proof}

\begin{lem}[Rademacher complexity]
\label{lem:rad}
Suppose \Cref{ass:network} holds. 
Consider a KL restricted Barron space:
$\calB_{M}:= \{ \smallint \Psi(\cdot, x) \dd \mu(x) \mid \kl(\mu, \nu)\leq M \}$, where $\nu=\calN(0, \zeta_1\sigma_1^{-1} \Id_d)$. Then, the Rademacher complexity of $\calB_{M}$ satisfies $\mathfrak{R}(\calB_{M}) = \calO(\sqrt{n^{-1}R^2M})$. 
\end{lem}
\begin{proof}
Since Rademacher complexity is smaller than Gaussian complexity, it suffices to bound the Gaussian complexity $\mathfrak{G}(\calB_{M}) := \mathbb{E}_{\varepsilon_i \sim \calN(0,1)}[\sup _{f \in \calB_{M}} \frac{1}{n} \sum_{i=1}^n  \varepsilon_i f(\ba_i)]$. Let $Z(x):=\frac{1}{\sqrt{n}} \sum_{i=1}^n \varepsilon_i \Psi(\ba_i, x)$. Note that $Z(x)$ follows a Gaussian distribution with mean 0 and variance $\sigma(x)^2:=\frac{1}{n} \sum_{i=1}^n \Psi(\ba_i, x)^2 \leq R^2$ as per \Cref{ass:network}. Then, we have
\begin{align*}
\mathfrak{G}(\calB_M) & :=\mathbb{E}_{\varepsilon}\left[\sup _{f \in \calB_M} \frac{1}{n}  \sum_{i=1}^n \varepsilon_i f(\ba_i)\right] =\mathbb{E}_{\varepsilon}\left[\sup _{\kl(\mu, \nu) \leq M}  \frac{1}{n} \sum_{i=1}^n \varepsilon_i \int  \Psi(\ba_i, x) \mathrm{d} \mu(x)\right] \\
& =\frac{1}{\sqrt{n}} \mathbb{E}_{\varepsilon}\left[\sup _{\kl(\mu, \nu) \leq M}   \int Z(x) \mathrm{d} \mu(x)\right] \leq \sqrt{\frac{1}{n}} \mathbb{E}_{\varepsilon}\left[\sup _{\kl(\mu, \nu) \leq M} \sqrt{\int  Z(x)^2 \dd \mu(x)} \right] \\
& \leq \sqrt{\frac{1}{n}} \sqrt{\mathbb{E}_{\varepsilon}\left[\sup _{\kl(\mu, \nu) \leq M} \int  Z(x)^2 \dd \mu(x)\right]} .
\end{align*}
From the Donsker-Varadhan duality formula of the KL-divergence, we have, for any $\gamma > 0$, 
\begin{align*}
    \frac{1}{\gamma} \mathbb{E}_{\varepsilon}\left[\sup _{\kl(\mu, \nu) \leq M} \gamma \int Z(x)^2 \dd \mu(x)\right] & \leq \frac{1}{\gamma}\left\{M+\mathbb{E}_{\varepsilon}\left[\log \int \exp \left(\gamma Z(x)^2\right) \mathrm{d} \nu(x)\right]\right\} \\ & \leq \frac{1}{\gamma}\left\{M + \log \int \mathbb{E}_{\varepsilon}\left[\exp \left( \gamma Z(x)^2\right)\right] \mathrm{d} \nu(x)\right\} .
\end{align*}
Since $Z(x)$ is a zero mean Gaussian random variable with variance $\sigma(x) \leq R^2$. We have, taking $\gamma = 1/(4R^2)$, 
\begin{align*}
    \mathbb{E}_{\varepsilon}  \left[\exp \left(\gamma Z(x)^2\right)\right] = \sqrt{\frac{1}{1-2 \gamma \sigma(x)^2}} \leq \sqrt{2} . 
\end{align*}
We have
\begin{align*}
    \mathfrak{G}\left(\calB_M\right) \leq \frac{1}{\sqrt{n}} \sqrt{4R^2 (M + \log \sqrt{2})} \lesssim \sqrt{\frac{R^2M}{n}}. 
\end{align*}
The proof is concluded. 
\end{proof}

% References section
%Bibliography
% \bibliographystyle{unsrt}  
% \bibliographystyle{plainnat}
\bibliographystyle{abbrvnat}
\bibliography{main}  

\renewcommand{\clearpage}{}

\begin{appendices}

\crefalias{section}{appendix}
\crefalias{subsection}{appendix}
\crefalias{subsubsection}{appendix}

\setcounter{equation}{0}
\renewcommand{\theequation}{\thesection.\arabic{equation}}
% Define a new command for appendix sections
\newcommand{\appsection}[1]{
  \refstepcounter{section}
  \section*{Appendix \thesection: #1}
  \addcontentsline{toc}{section}{Appendix \thesection: #1}
}

\onecolumn

\section{Additional theoretical results}\label{sec:additional_theory}
\subsection{Convexity of $\mu_x \mapsto L_\lambda(\mu_x, \tilde{\mu}_z^\ast(\mu_x))$.}

\begin{prop}[Convexity of $\mu_x \mapsto L_\lambda(\mu_x, \tilde{\mu}_z^\ast(\mu_x))$] \label{prop:partial_convex_L_2}
The mapping $\mu_x \mapsto L_\lambda(\mu_x, \tilde{\mu}_z^\ast(\mu_x))$ is \emph{linear} convex under the following additional conditions. \\
1) $\tilde{\mu}_z^\ast(\mu_x)$ and $\mu_z^\ast(\mu_x)$ coincide with the global optima of their respective objectives excluding entropic regularization in the space of signed measures. In other words, $\tilde{\mu}_z^\ast(\mu_x)=\arg\min_{\mu_z\in\calM(\R^{d_z})} F_2(\mu_x,\mu_z) + \lambda F_1(\mu_x,\mu_z)$ and $\mu_z^\ast(\mu_x)=\arg\min_{\mu_z\in\calM(\R^{d_z})} F_1(\mu_x,\mu_z)$. \\
2) For any bounded measurable function $f \in L^{\infty}(\calA\times\calW)$, if there is  $\E_{(\ba,\bw)\sim\rho}[f(\ba,\bw)\Psi_\bw(z)]=0$ for any $z\in\R^{d_z}$, then $f(\ba,\bw)=0$ holds $\rho$-almost surely.  
\end{prop}
The proof can be found later in this section. 
The convexity of $\mu_z\mapsto L_\lambda(\mu_x, \mu_z)$ for fixed $\mu_x$ is straightforward given the partial convexity proved in \Cref{prop:partial_convex_simple}.
In contrast, analyzing the convexity of the mapping $\mu_x \mapsto L_\lambda(\mu_x, \tilde{\mu}_z^\ast(\mu_x))$ is more subtle due to the nested structure and the subtraction in the definition of $L_\lambda$ in Eq.~\eqref{eq:defi_L_lambda}. Fortunately, $L_\lambda(\mu_x, \tilde{\mu}_z^\ast(\mu_x))$ can be written--- ignoring the entropic regularizations as well as the terms in $\scrF_1$ that are constant in $\mu_x$--- as 
\begin{align*}
    L_\lambda(\mu_x, \tilde{\mu}_z^\ast(\mu_x)) = \min_{\mu_z} \Big( F_2(\mu_x, \mu_z) + \lambda F_1(\mu_x, \mu_z) \Big)- \min_{\mu_z} \Big(\lambda F_1(\mu_x, \mu_z) \Big) + C,
\end{align*}
which expresses $L_\lambda$ as the difference between the minima of two quadratic functionals. 
Since the first functional is ``more quadratic'' than the second, it is still possible to establish convexity of the overall expression by carefully examining the second-order derivative---despite the subtraction.
However, this convexity result holds only when we ignore the manifold structure of the space of probability measures, i.e., when working over the space of signed measures $\calM(\R^d)$ instead of the probability simplex $\calP(\R^d)$. 
As a result, the entropic regularization are not well-defined over signed measures and also need to be ignored. 
See \Cref{lem:first_variation_epsilon} for details. 

\begin{rem}[Restrictiveness of the conditions of \Cref{prop:partial_convex_L_2}]\label{rem:restrictiveness}
    The first condition in \Cref{prop:partial_convex_L_2} requires access to the global optimum in the space of signed measures, which is a challenging task even in the mean-field limit of neural networks. To the best of our knowledge, the only works addressing this problem are \citet{takakura2024mean} and \citet{wang2024mean}, both of which require the second-layer weights to be trained at a much faster rate than those of the first layer. Moreover, these works consider unbounded neural networks, which are not compatible with our \Cref{ass:network}. 
    
    The second completeness condition in \Cref{prop:partial_convex_L_2} is equivalent to the condition that the $\mathrm{span}\{(\ba,\bw)\mapsto \rho(\ba, \bw) \Psi_\bw(z) \mid z\in\R^{d_z}\}$ is dense in $L^1(\calA\times\calW)$, a consequence of the duality between  $L^1(\calA\times\calW)$ and $L^{\infty}(\calA\times\calW)$. 
    This condition excludes, in particular, the case where $\rho$ is an empirical distribution.
    Owing to the restrictiveness of both conditions in \Cref{prop:partial_convex_L_2}, we do not assume convexity of the mapping $\mu_x \mapsto L_\lambda(\mu_x, \tilde{\mu}_z^\ast(\mu_x))$ in our paper. The absence of convexity introduces substantial difficulties in the convergence analysis in \Cref{sec:convergence}. 
\end{rem}

% For any fixed $\mu_x\in\calP_2(\R^{d_x})$,
% \begin{align*}
%     \scrL_\lambda(\mu_x, \mu_z) &= \scrF_2(\mu_x, \mu_z) + \lambda \left(\scrF_1(\mu_x, \mu_z) - \scrF_1(\mu_x, \mu_z^\ast(\mu_x)) \right) \\
%     &= U_2(\mu_z) + \lambda U_1(\mu_x, \mu_z) + \lambda \frac{\zeta}{2} \E_{\mu_z}[\|z\|^2] + \lambda \sigma \mathrm{Ent}(\mu_z) + \text{const}.
% \end{align*}
% Here, the constant is independent of $\mu_z$. 
% From the partial convexity of $\mu_z\mapsto U_1(\mu_x,\mu_z)$ and $\mu_z \mapsto U_2(\mu_z)$ proved in \Cref{prop:partial_convex}, the proof of the first claim is concluded. 
\begin{proof}[Proof of \Cref{prop:partial_convex_L_2}]\label{sec:proof_partial_convex_L_2}

The proof relies on the following \Cref{lem:first_variation_epsilon} proved in \Cref{sec:proof_first_variation_epsilon}. It uses the formal definition of first-order and second-order Fréchet derivatives on the space $\mathcal{M}(\R^d)$, treated as a Banach space equipped with the total variation norm~\citep[Chapter 3]{pathak2018introduction}. Specifically, these derivatives are expressed via dual representations, leveraging the fact that the dual space of $\mathcal{M}(\R^d)$ is $C_b(\R^d)$, the space of bounded continuous functions~\citep[Chapter 8]{bogachev2007measure}. 
 
\begin{lem}\label{lem:first_variation_epsilon}
Fix any $\mu_x\in\calP_2(\R^{d_x})$. 
Let $\tilde{\mu}_z^\ast(\mu_x) = \arg\min_{\mu_z\in\calM(\R^{d_z})} F_2(\mu_x,\mu_z) + \lambda F_1(\mu_x,\mu_z)$ and $\mu_z^\ast(\mu_x) = \arg\min_{\mu_z\in\calM(\R^{d_z})} F_1(\mu_x,\mu_z)$. 
Suppose for any bounded measurable function $f:\calA\times\calW\to\R$, if there is  $\E_{(\ba,\bw)\sim\rho}[f(\ba,\bw)\Psi_\bw(z)]=0$ for any $z\in\R^{d_z}$, then $f(\ba, \bw)=0$ holds $\rho$ almost surely. 
Then, for any $\nu_x\in\calP_2(\R^{d_x})$ and $\epsilon>0$, the following two equations hold $\rho$ almost everywhere: 
\begin{align*}
    \int \Psi_\bw \;\dd \mu_z^\ast(\mu_x+\epsilon \nu_x) & = \int\Psi_\bw \;\dd \mu_z^\ast(\mu_x) + \epsilon \int \Psi_\ba \; \dd \nu_x \\
    \int \Psi_\bw \;\dd \tilde{\mu}_z^\ast(\mu_x+\epsilon \nu_x) &= \int\Psi_\bw \;\dd \tilde{\mu}_z^\ast(\mu_x) + \frac{\lambda}{\lambda+1} \epsilon \int \Psi_\ba \; \dd \nu_x .
\end{align*}
\end{lem}
\begin{defi}[Dual representations of first-order and second-order Fr\'{e}chet derivatives]
For a functional $\calF:\calM(\R^d)\to\R$, its \emph{first-order Fr\'{e}chet derivative} at $\mu$ denoted as $\delta_\mu\calF : \R^d \to\R$ satisfies $\frac{\dd}{\dd \epsilon}|_{\epsilon=0} \mathcal{F}(\mu+\epsilon \nu) = \int \delta_\mu \calF \dd \nu$ for any direction $\nu$. 
Its \emph{second-order Fréchet derivative} (also referred to as the \emph{Fréchet Hessian}) at $\mu$, denoted by $\bH_\mu \mathcal{F}: \mathbb{R}^d \times \mathbb{R}^d \to \mathbb{R}$  satisfies $\frac{\dd^2}{\dd \epsilon^2}|_{\epsilon=0} \mathcal{F}(\mu+\epsilon \nu) = \int \bH_\mu\calF(x, x^\prime) \dd \nu(x) \dd \nu(x')$ for any direction $\nu$. 
\end{defi}

Next, we consider the Fréchet Hessian of the following two maps: 
\begin{align}\label{eq:defi_L1_L2}
    \mu_x \mapsto \lambda F_1\left(\mu_x, \mu_z^*(\mu_x)\right) := L_1(\mu_x) , \quad \mu_x \mapsto F_2\left(\mu_x, \tilde{\mu}_z^*(\mu_x)\right) + \lambda F_1(\mu_x, \tilde{\mu}_z^*(\mu_x) ) := L_2(\mu_x) .
\end{align}
The proof is divided into two parts, the first part proves that $\bH_{\mu_x} L_1(\mu_x)=0$ and the second part proves that $\bH_{\mu_x} L_2(\mu_x) = \frac{\lambda}{2(\lambda+1)} \E_\rho[\Psi_\ba \otimes \Psi_\ba]$. 
From \Cref{lem:first_variation_epsilon}, for a fixed $\nu_x$, consider
\begin{align*}
    &\quad L_1(\mu_x+\epsilon\nu_x, \mu_z^\ast(\mu_x+\epsilon\nu_x)) \\
    &=  \frac{\lambda}{2} \E_{\rho} \left[ \left( \smallint \Psi_\bw \; \dd \mu_z^\ast(\mu_x+\epsilon\nu_x) - \smallint \Psi_\ba \dd (\mu_x+\epsilon\nu_x) \right)^2\right] + \frac{\lambda\zeta_1}{2} \E_{\mu_z^\ast(\mu_x+\epsilon\nu_x)}[\|z\|^2] \\
    &= \frac{\lambda}{2} \E_{\rho} \left[ \left( \smallint\Psi_\bw \;\dd \mu_z^\ast(\mu_x) + \epsilon \smallint \Psi_\ba \; \dd \nu_x - \smallint \Psi_\ba \dd (\mu_x+\epsilon\nu_x) \right)^2\right] + \frac{\lambda\zeta_1}{2} \E_{\mu_z^\ast(\mu_x+\epsilon\nu_x)}[\|z\|^2] \\
    &= \frac{\lambda}{2} \E_{\rho} \left[ \left( \smallint\Psi_\bw \;\dd \mu_z^\ast(\mu_x) - \smallint \Psi_\ba \dd \mu_x\right)^2\right] + \frac{\lambda\zeta_1}{2} \E_{\mu_z^\ast(\mu_x+\epsilon\nu_x)}[\|z\|^2] \\
    &= L_1(\mu_x, \mu_z^\ast(\mu_x)) - \frac{\lambda\zeta_1}{2} \E_{\mu_z^\ast(\mu_x)}[\|z\|^2] + \frac{\lambda\zeta_1}{2} \E_{\mu_z^\ast(\mu_x+\epsilon\nu_x)}[\|z\|^2] .
\end{align*}
So, 
\begin{align*}
    &\quad L_1(\mu_x+\epsilon\nu_x, \mu_z^\ast(\mu_x+\epsilon\nu_x)) - L_1(\mu_x, \mu_z^\ast(\mu_x)) = - \frac{\lambda\zeta_1}{2} \E_{\mu_z^\ast(\mu_x)}[\|z\|^2] + \frac{\lambda\zeta_1}{2} \E_{\mu_z^\ast(\mu_x+\epsilon\nu_x)}[\|z\|^2] \\
    &= \lambda\Big( \E_{\rho} \left[ \left( \smallint \Psi_\bw \; \dd \mu_z^\ast(\mu_x) - \smallint \Psi_\ba \; \dd \mu_x) \right)\cdot \left( \smallint \Psi_\bw \; \dd \mu_z^\ast(\mu_x) - \smallint \Psi_\bw \; \dd \mu_z^\ast(\mu_x+\epsilon\nu_x) \right) \right] \Big) \\
    &= \epsilon \lambda\Big( \E_{\rho} \left[ \left( \smallint \Psi_\bw \; \dd \mu_z^\ast(\mu_x) - \smallint \Psi_\ba \; \dd \mu_x) \right)\cdot \smallint \Psi_\ba \; \dd \nu_x \right] \Big) . 
\end{align*}
The second last equality holds by the optimality of $\mu_z^\ast(\mu_x)$ and  $\mu_z^\ast(\mu_x+\epsilon\nu_x)$; the last equality holds by using \Cref{lem:first_variation_epsilon}. 
A quick sanity check of the above derivations is to notice that the first-order Fréchet derivative $\delta_{\mu_x} L_1(\mu_x,\mu_z^\ast(\mu_x)) = -\lambda \E_{\rho} \left[ \left( \smallint \Psi_\bw \; \dd \mu_z^\ast(\mu_x)- \smallint \Psi_\ba \; \dd \mu_x \right) \cdot \Psi_\ba(\cdot) \right]$, which agrees with the derivation from the envelope theorem using the optimality of $\mu_z^\ast(\mu_x)$. 
Therefore, we have
\begin{align*}
   \frac{\dd^2}{\dd \epsilon^2} \mid_{\epsilon=0} L_1(\mu_x+\epsilon\nu_x,\mu_z^\ast(\mu_x+\epsilon\nu_x)) = 0, \quad \Rightarrow \quad \bH_{\mu_x} L_1(\mu_x,\mu_z^\ast(\mu_x)) = 0 .
\end{align*}
Next, we are about to compute and show that $\bH_{\mu_x} L_2(\mu_x) = \frac{\lambda}{2(\lambda+1)} \E_\rho[\Psi_\ba \otimes \Psi_\ba]$. 
From \Cref{lem:first_variation_epsilon}, consider
\begin{align*}
    &\quad L_2(\mu_x+\epsilon\nu_x, \tilde{\mu}_z^\ast(\mu_x+\epsilon\nu_x)) \\
    &=  \frac{1}{2} \E_{\rho} \left[ \left( \smallint \Psi_\bw \; \dd \tilde{\mu}_z^\ast(\mu_x+\epsilon\nu_x) - \by \right)^2\right] + \frac{\zeta_2}{2} \E_{\mu_x+\epsilon \nu_x}[\|x\|^2] \\
    &\hspace{10em} + \frac{\lambda}{2} \E_{\rho} \left[ \left( \smallint \Psi_\bw \; \dd \tilde{\mu}_z^\ast(\mu_x+\epsilon\nu_x) - \smallint \Psi_\ba \dd (\mu_x+\epsilon\nu_x) \right)^2\right] + \frac{\lambda\zeta_1}{2} \E_{\tilde{\mu}_z^\ast(\mu_x+\epsilon\nu_x)}[\|z\|^2] \\
    &= \frac{1}{2} \E_{\rho} \left[ \left( \smallint \Psi_\bw \; \dd \tilde{\mu}_z^\ast(\mu_x) + \frac{\lambda}{\lambda+1}\epsilon \smallint \Psi_\ba \; \dd \nu_x - \by \right)^2\right] + \frac{\zeta_2}{2} \E_{\mu_x+\epsilon \nu_x}[\|x\|^2] \\
    &\qquad\qquad + \frac{\lambda}{2} \E_{\rho} \left[ \left( \smallint \Psi_\bw \; \dd \tilde{\mu}_z^\ast(\mu_x) + \frac{\lambda}{\lambda+1}\epsilon \smallint \Psi_\ba \; \dd \nu_x - \smallint \Psi_\ba \dd (\mu_x+\epsilon\nu_x) \right)^2\right] + \frac{\lambda\zeta_1}{2} \E_{\mu_z^\ast(\mu_x+\epsilon\nu_x)}[\|z\|^2] \\
    &= \frac{1}{2} \E_{\rho} \left[ \left( \smallint \Psi_\bw \; \dd \tilde{\mu}_z^\ast(\mu_x) - \by \right)^2\right] + \E_{\rho} \left[  \frac{\lambda}{\lambda+1}\epsilon \left(\smallint \Psi_\ba \; \dd \nu_x\right) \cdot \left( \smallint \Psi_\bw \; \dd \tilde{\mu}_z^\ast(\mu_x) - \by \right) \right] \\
    &\quad + \frac{1}{2} \left(\frac{\lambda}{\lambda+1}\epsilon\right)^2 \E_\rho[\left(\smallint \Psi_\ba \; \dd \nu_x\right)^2] + \frac{\zeta_2}{2} \E_{\mu_x+\epsilon \nu_x}[\|x\|^2] \\
    &\qquad + \frac{\lambda}{2} \E_{\rho} \left[ \left( \smallint \Psi_\bw \; \dd \tilde{\mu}_z^\ast(\mu_x) - \smallint \Psi_\ba \dd \mu_x\right)^2\right] - \lambda \E_{\rho} \left[ \frac{1}{\lambda+1}\epsilon \left(\smallint \Psi_\ba \; \dd \nu_x\right) \cdot \left( \smallint \Psi_\bw \; \dd \tilde{\mu}_z^\ast(\mu_x) - \smallint \Psi_\ba \dd \mu_x \right) \right] \\
    &\qquad\quad + \frac{\lambda}{2} \left(\frac{1}{\lambda+1} \epsilon\right)^2 \E_\rho[\left(\smallint \Psi_\ba \; \dd \nu_x\right)^2] + \frac{\lambda\zeta_1}{2} \E_{\mu_z^\ast(\mu_x+\epsilon\nu_x)}[\|z\|^2] \\
    &= L_2(\mu_x, \tilde{\mu}_z^\ast(\mu_x)) - \frac{\lambda\zeta_1}{2} \E_{\mu_z^\ast(\mu_x)}[\|z\|^2] + \frac{\lambda\zeta_1}{2} \E_{\mu_z^\ast(\mu_x+\epsilon\nu_x)}[\|z\|^2] + \frac{\zeta_2}{2} \epsilon^2 \E_{\nu_x}[\|x\|^2] \\
    &\qquad\qquad + \frac{\lambda}{2(\lambda+1)}\epsilon^2 \E_\rho[\left(\smallint \Psi_\ba \; \dd \nu_x\right)^2] + \frac{\lambda}{\lambda+1} \epsilon \E_{\rho} \left[  \left(\smallint \Psi_\ba \; \dd \nu_x\right) \cdot \left( \smallint \Psi_\ba \; \dd \mu_x - \by \right) \right] .
\end{align*}
Also, from the optimality condition of $\tilde{\mu}_z^\ast(\mu_x+\epsilon\nu_x)$ and $\tilde{\mu}_z^\ast(\mu_x)$, we have
\begin{align*}
    &-\frac{\lambda\zeta_1}{2} \E_{\tilde{\mu}_z^\ast(\mu_x+\epsilon\nu_x)}[\|z\|^2] + \frac{\lambda\zeta_1}{2} \E_{\tilde{\mu}_z^\ast(\mu_x)}[\|z\|^2] \\
    &= \E_{\rho} \left[ \left( \smallint \Psi_\bw \; \dd \tilde{\mu}_z^\ast(\mu_x+\epsilon\nu_x) - \by \right)\cdot (\smallint \Psi_\bw \; \dd \tilde{\mu}_z^\ast(\mu_x+\epsilon\nu_x)) \right] \\
    &\quad + \lambda \E_{\rho} \left[ \left( \smallint \Psi_\bw \; \dd \tilde{\mu}_z^\ast(\mu_x+\epsilon\nu_x) - \smallint \Psi_\ba \; \dd (\mu_x+\epsilon\nu_x) ) \right)\cdot (\smallint \Psi_\bw \; \dd \tilde{\mu}_z^\ast(\mu_x+\epsilon\nu_x)) \right] \\
    &\qquad - \E_{\rho} \left[ \left( \smallint \Psi_\bw \; \dd \tilde{\mu}_z^\ast(\mu_x) - \by \right)\cdot (\smallint \Psi_\bw \; \dd \tilde{\mu}_z^\ast(\mu_x)) \right] - \lambda \E_{\rho} \left[ \left( \smallint \Psi_\bw \; \dd \tilde{\mu}_z^\ast(\mu_x) - \smallint \Psi_\ba \; \dd \mu_x) \right)\cdot (\smallint \Psi_\bw \; \dd \tilde{\mu}_z^\ast(\mu_x)) \right]\\
    &= \frac{\lambda}{\lambda+1} \epsilon \E_\rho\left[ (\smallint \Psi_\ba \; \dd \mu_x) \cdot (2 \smallint \Psi_\bw \; \dd \tilde{\mu}_z^\ast(\mu_x) - \by )\right] + \left( \frac{\lambda}{\lambda+1} \epsilon\right)^2 \E_\rho\left[(\smallint \Psi_\ba \; \dd \mu_x)^2 \right] \\
    &\quad + \lambda \cdot \left( -\frac{1}{1+\lambda}\epsilon (\smallint \Psi_\ba \; \dd \mu_x)\cdot (\smallint \Psi_\bw \; \dd \tilde{\mu}_z^\ast(\mu_x))+ \frac{\lambda}{1+\lambda} \epsilon (\smallint \Psi_\ba \; \dd \mu_x)\cdot (\smallint \Psi_\bw \; \dd \tilde{\mu}_z^\ast(\mu_x) - \smallint \Psi_\ba \; \dd \mu_x) \right) \\
    &\qquad - \lambda \cdot \frac{\lambda}{(1+\lambda)^2}\epsilon^2 (\smallint \Psi_\ba \; \dd \mu_x)^2 \\
    &= \frac{\lambda}{\lambda+1}\epsilon \E_\rho\left[ (\smallint \Psi_\ba \; \dd \nu_x) \left(\smallint \Psi_\bw \; \dd \tilde{\mu}_z^\ast(\mu_x) - \by 
    + \lambda \smallint \Psi_\bw \; \dd \tilde{\mu}_z^\ast(\mu_x) - \lambda \smallint \Psi_\ba \; \dd \mu_x) \right) \right].
\end{align*}
Therefore, we have
\begin{align*}
    &\quad L_2(\mu_x+\epsilon\nu_x, \tilde{\mu}_z^\ast(\mu_x+\epsilon\nu_x))  - L_2(\mu_x, \tilde{\mu}_z^\ast(\mu_x)) \\
    &= \frac{\lambda}{2(\lambda+1)}\epsilon^2 \E_\rho[\left(\smallint \Psi_\ba \; \dd \nu_x\right)^2] + \lambda \epsilon \cdot \E_\rho\left[ (\smallint \Psi_\ba \; \dd \nu_x) \left(\smallint \Psi_\bw \; \dd \tilde{\mu}_z^\ast(\mu_x) - \smallint \Psi_\ba \; \dd \mu_x) \right) \right] + \frac{\zeta_2}{2} \epsilon^2 \E_{\nu_x}[\|x\|^2]. 
\end{align*}
A quick sanity check of the above derivations is to know that the first variation $\delta_{\mu_x} L_2(\mu_x, \tilde{\mu}_z^\ast(\mu_x))(x) = \lambda \E_\rho\left[ \Psi_\ba(x) \cdot \left(\smallint \Psi_\bw \; \dd \tilde{\mu}_z^\ast(\mu_x) - \smallint \Psi_\ba \; \dd \mu_x) \right) \right]$, which agrees with the derivation from the envelope theorem using the optimality of $\tilde{\mu}_z^\ast(\mu_x)$. 
Therefore, we have
\begin{align*}
   &\quad \frac{\dd^2}{\dd \epsilon^2} \mid_{\epsilon=0} L_2(\mu_x+\epsilon\nu_x, \tilde{\mu}_z^\ast(\mu_x+\epsilon\nu_x)) = \frac{\lambda}{\lambda+1} \E_\rho[\left(\smallint \Psi_\ba \; \dd \nu_x\right)^2] + \zeta_2 \E_{\nu_x}[\|x\|^2]\\
   &\Rightarrow \quad \bH_{\mu_x} L_2(\mu_x, \tilde{\mu}_z^\ast(\mu_x)) = \frac{\lambda}{\lambda+1} \E_\rho[\Psi_\ba \otimes \Psi_\ba] + \zeta_2 \Id .
\end{align*}
which is a positive definite operator. Finally, from the definition of $L_1$ and $L_2$ in Eq.~\eqref{eq:defi_L1_L2}, and that $L_\lambda(\mu_x, \tilde{\mu}_z^\ast(\mu_x)) = L_1(\mu_x) + L_2(\mu_x)$,  we can conclude that the Frechet Hessian of the mapping $\mu_x\mapsto L_\lambda(\mu_x, \tilde{\mu}_z^\ast(\mu_x))$ is positive definite. Hence, the mapping $\mu_x\mapsto L_\lambda(\mu_x, \tilde{\mu}_z^\ast(\mu_x))$ is convex. 
\end{proof}

\subsection{Implicit gradient method}\label{sec:nested_gradient}
In this section, we are about to show that solving \eqref{eq:distribution_space_optimization} without resorting to its Lagrangian formulation would lead to a Wasserstein gradient that is hard to approximate with finite particles. 
This approach is known as implicit gradient method in the literature of bilevel optimization~\citep{domke2012generic}. 
In particular, unlike variational integrals defined in Equation 10.4.1 of \citet{ambrosio2008gradient}, the mapping of the outer level objective $\mu_x \mapsto F_2(\mu_x, \mu_z^\ast(\mu_x))$ does not preserve a structure where the Wasserstein gradient can be expressed simply as the gradient of its first variation. Consequently, we must revert to the original definition of the Fréchet subdifferential, as introduced in Section 10 of \citet{ambrosio2008gradient}.

In the following proposition, $\bJ f: \R^d \to \R^d \times \R^d$ denotes the Jacobian of a mapping $f:\R^d\to\R^d$, $\bH f: \R^d\to\R^d\times\R^d$ denotes the Hessian of a mapping $f:\R^d\to\R$, $\otimes$ denotes the tensor product. 

\begin{prop}\label{prop:xi_frechet_subdiff}
Suppose \Cref{ass:npiv} and \ref{ass:network} hold. 
Suppose that $\sup_{\bw\in\calW, z\in\R^{d_z}} \|\bH \Psi_\bw(z) \|_{\mathrm{op}} \leq (4R)^{-1} \zeta_1$ and that $\E_\rho[\nabla \Psi_\bw(\cdot) \otimes \nabla \Psi_\bw(\cdot)] \succeq \Lambda \Id \succ 0$. 
% For any $\mu_x\in\calP_2(\R^{d_x})$, let $\mu_z^\ast(\mu_x)$ be the solution to the stage I optimization problem in \eqref{eq:distribution_space_optimization}. 
For any $\mu_z\in\calP_2(\R^{d_z})$ that admits density, define an operator $\mathfrak{G}_{\mu_z}: L^2(\mu_z) \to L^2(\mu_z)$ as 
\begin{align*}
    \mathfrak{G}_{\mu_z} [V] = \left( z \mapsto -\bJ V(z) \nabla \log \mu_z(z) - \nabla \mathrm{Tr}(\bJ V(z)) - \bH \log \mu_z(z) V(z) \right) . 
\end{align*} 
Then, the mapping $\mu_x \mapsto F_2(\mu_x, \mu_z^\ast(\mu_x))$ admits the following Fréchet subdifferential 
\begin{align}\label{eq:defi_xi_main}
    \xi: x \mapsto \zeta_2 x + \mathfrak{F}_{\mu_x}^\ast g_{\mu_z^\ast(\mu_x)}(x),
\end{align}
where the above $g_{\mu_z}(\cdot) := \E_{\rho} [ (\int \Psi_\bw \; \dd \mu_z - \by) \nabla \Psi_\bw(\cdot)] \in L^2(\mu_z)$ and 
$\mathfrak{F}_{\mu_x} := \big( \E_{\rho}[\nabla \Psi_\bw(\cdot) \otimes \nabla \Psi_\bw(\cdot)] + \sigma_1 \mathfrak{G}_{\mu_z^\ast(\mu_x)} \big)^{-1} ( \E_{\rho}[\nabla \Psi_\bw(\cdot) \otimes \nabla \Psi_\ba(\cdot)])$ is a bounded linear operator from $L^2(\mu_x)$ to $L^2(\mu_z)$. 
\end{prop}
The proof can be found in \Cref{sec:proof_xi_frechet}. 
The derivative conditions imposed on the network $\Psi_\bw$ are technical assumptions ensuring the existence of a Fréchet subdifferential, which intuitively, guarantee the invertibility condition underlying the implicit function theorem~\citep{krantz2002implicit}.
However, since we are working in $\calP_2$, a direct application of the implicit function theorem is not warranted (see \citet{lessel2020differentiable}). 
Instead, we adopt a constructive approach: we explicitly build the candidate differential maps and then verify that they indeed coincide with the Fréchet subdifferential (see \Cref{prop:mathfrak_F_wasserstein}).

\subsubsection{Proof of \Cref{prop:xi_frechet_subdiff}}\label{sec:proof_xi_frechet}
\begin{proof}[Proof of \Cref{prop:xi_frechet_subdiff}]
At the start of the proof, we are about to verify that $\E_{\rho}[\nabla \Psi_\bw(\cdot) \otimes \nabla \Psi_\bw(\cdot)] + \sigma_1 \mathfrak{G}_{\mu_z^\ast(\mu_x)}$ is a strictly positive definite operator such that its inverse is well-defined. 
For any $V\in L^2(\mu_z)$, we have $\langle V, \E_{\rho}[\nabla \Psi_\bw(\cdot) \otimes \nabla \Psi_\bw(\cdot)] V \rangle_{L^2(\mu_z)} = \E_\rho[\langle V, \; \nabla \Psi_\bw\rangle_{L^2(\mu_z)}^2] > 0$. 
Next, we consider $\langle V, \;  \mathfrak{G}_{\mu_z^\ast(\mu_x)} [V] \rangle_{L^2(\mu_z^\ast(\mu_x))}$, which consists of two terms. 
The first term is, 
\begin{align*}
    &\quad  \int \left( - V(z)^\top \bJ V(z) \nabla \log \mu_z^\ast(\mu_x)(z) - V(z)^\top \nabla \mathrm{Tr}(\bJ V(z)) \right) \mu_z^\ast(\mu_x)(z) \dd z \\
    &= -\int V(z)^\top \bJ V(z) \nabla \mu_z^\ast(\mu_x)(z) \; \dd z - \int V(z)^\top \nabla \mathrm{Tr}(\bJ V(z)) \mu_z^\ast(\mu_x)(z) \; \dd z \\
    &= \int \nabla \cdot \left( V(z)^\top \bJ V(z) \right) \mu_z^\ast(\mu_x)(z) \; \dd z - \int V(z)^\top \nabla \mathrm{Tr}(\bJ V(z)) \mu_z^\ast(\mu_x)(z) \; \dd z \\
    &= \int \|\bJ V(z)\|_{\mathrm{HS}}^2 \; \dd \mu_z^\ast(\mu_x)(z) \geq 0,
\end{align*}
where the last two equalities hold by integration by parts. 
The second term is, 
\begin{align*}
    &\quad -\int V(z)^\top \sigma_1 \bH \log \mu_z^\ast(\mu_x)(z) V(z) \; \dd \mu_z^\ast(\mu_x) (z) \\
    &= \int V(z)^\top \Big( \zeta_1 + \E_{\rho} \left[ \left( \smallint \Psi_\bw \; \dd \mu_z^\ast(\mu_x) - \smallint \Psi_\ba \; \dd \mu_x \right) \bH \Psi_\bw(z) \right] \Big) V(z) \; \dd \mu_z^\ast(\mu_x) (z) \\
    &\geq \| V\|_{L^2(\mu_z^\ast(\mu_x))}^2 \Big( \zeta_1 - \| \E_{\rho} \left[ \left( \smallint \Psi_\bw \; \dd \mu_z^\ast(\mu_x) - \smallint \Psi_\ba \; \dd \mu_x \right) \bH \Psi_\bw(z)  \right]\|_{\mathrm{op}} \Big) \\
    &\geq \| V\|_{L^2(\mu_z^\ast(\mu_x))}^2 \Big( \zeta_1 - | \E_{\rho} \left[ \smallint \Psi_\bw \; \dd \mu_z^\ast(\mu_x) - \smallint \Psi_\ba \; \dd \mu_x\right] | \cdot \sup_{\bw, z} \|\bH \Psi_\bw(z) \|_{\mathrm{op}} \Big) \\ 
    &\geq \|V\|_{L^2(\mu_z^\ast(\mu_x))}^2 (\zeta_1 - 2 R \cdot \sup_{\bw, z} \|\bH \Psi_\bw(z) \|_{\mathrm{op}}) > \frac{\zeta_1}{2} \|V\|_{L^2(\mu_z^\ast(\mu_x))}^2 ,
\end{align*}
provided that $\zeta_1 \geq 4R \sup_{\bw, z} \|\bH \Psi_\bw(z) \|_{\mathrm{op}}$ as in the statement of the proposition. 
Therefore, we have verified that 
\begin{align*}
     \sigma_1 \left\langle V, \mathfrak{G}_{\mu_z^\ast(\mu_x)} [V] \right\rangle_{L^2(\mu_z^\ast(\mu_x))}  \geq \frac{\zeta_1}{2} \|V\|_{L^2(\mu_z^\ast(\mu_x))}^2 . 
\end{align*}
So $\E_{\rho}[\nabla \Psi_\bw(\cdot) \otimes \nabla \Psi_\bw(\cdot)] + \sigma_1 \mathfrak{G}_{\mu_z^\ast(\mu_x)}$ is indeed a positive definite operator $\succeq \frac{\zeta_1}{2} \Id$ so that its inverse is well-defined. 
Furthermore, we have proved that its inverse is bounded hence $\mathfrak{F}_{\mu_x}:L^2(\mu_x)\to L^2(\mu_z)$ is a bounded linear operator whose adjoint $\mathfrak{F}_{\mu_x}^\ast: L^2(\mu_z)\to L^2(\mu_x)$ is well-defined. 

To prove that $\xi$ defined in Eq.~\eqref{eq:defi_xi_main} is the Fréchet subdifferential of $\mu_x \mapsto F_2(\mu_x, \mu_z^\ast(\mu_x))$, we use \citet[Definition 10.1.1]{ambrosio2008gradient}, which we recall below. 
\begin{defi}[Fréchet subdifferential~\citep{ambrosio2008gradient}]\label{defi:frechet}
Let $\phi: \calP_2(\R^{d}) \to\R$ be a proper and lower semicontinuous functional. 
We say that $\xi:\R^d\to\R^d \in L^2(\mu)$ belongs to the Fréchet subdifferential of $\phi$ if for any $\nu\in\calP_2(\R^{d})$, with $T_\mu^\nu$ being the optimal transport map from $\mu$ to $\nu$, 
\begin{align*}
    \phi(\nu)-\phi(\mu) \geq \int_{\R^{d}} \xi(x)^\top (T_\mu^\nu(x) - x) \; \dd \mu(x) + o(W_2(\mu, \nu)) .
\end{align*} 
\end{defi}
\vspace{1em}
\noindent
Let $\nu_x\in\calP_2(\R^{d_x})$ and $V:\R^{d_x}\to\R^{d_x}$ be the optimal transport map from $\mu_x$ to $\nu_x$. 
Consider
\begin{align}\label{eq:F_2_difference}
    &\quad F_2((\Id+tV)_\# \mu_x, \mu_z^\ast((\Id+tV)_\# \mu_x)) - F_2(\mu_x, \mu_z^\ast(\mu_x)) \nonumber \\
    &= \frac{1}{2} \E_{\rho} \left[ \left(\by - \smallint \Psi_\bw \; \dd \mu_z^\ast((\Id+tV)_\# \mu_x) \right)^2 \right] + \frac{\zeta_2}{2} \E_{(\Id+tV)_\# \mu_x}[\|x\|^2] \nonumber \\
    &\qquad\qquad - 
    \frac{1}{2} \E_{\rho} \left[ \left(\by - \smallint \Psi_\bw \; \dd \mu_z^\ast(\mu_x) \right)^2 \right] - \frac{\zeta_2}{2} \E_{\mu_x}[\|x\|^2] \nonumber \\
    &= \underbrace{\frac{1}{2} \E_{\rho} \left[ \left(\by - \smallint \Psi_\bw \; \dd \mu_z^\ast((\Id+tV)_\# \mu_x) \right)^2 - \left( \by - \smallint \Psi_\bw \; \dd (\Id+t \mathfrak{F}_{\mu_x}[V] )_\# \mu_z^\ast(\mu_x) \right)^2 \right] }_{\calE_1} \\
    &\quad+ \underbrace{\frac{1}{2} \E_{\rho} \left[ \left(\by - \smallint \Psi_\bw \; \dd (\Id+t \mathfrak{F}_{\mu_x}[V] )_\# \mu_z^\ast(\mu_x) \right)^2 \right]  - \frac{1}{2} \E_{\rho} \left[ \left(\by - \smallint \Psi_\bw \; \dd \mu_z^\ast(\mu_x) \right)^2 \right]}_{\calE_2} \nonumber \\
    &\qquad + \zeta_2 t \int x^\top V(x) \; \dd \mu_x + o(t) \nonumber . 
\end{align}
For the first term $\calE_1$, from \Cref{prop:mathfrak_F_wasserstein}, we have proved that $W_2(\mu_z^\ast((\Id + tV)_\# \mu_x), (\Id+t \mathfrak{F}_{\mu_x}[V])_\# \mu_z^\ast(\mu_x)) = o(t)$. Since $\Psi_\bw(\cdot)$ is Lipschitz as per \Cref{ass:network}, we have 
\begin{align*}
    \Delta := \left| \smallint \Psi_\bw \; \dd \mu_z^\ast((\Id+tV)_\# \mu_x) - \smallint \Psi_\bw \; \dd (\Id+t \mathfrak{F}_{\mu_x}[V] )_\# \mu_z^\ast(\mu_x) \right| = o(t) .
\end{align*} 
So we have
\begin{align*}
    \calE_1 &= \frac{1}{2} \E_{\rho} \left[ \big( 2 \by - \smallint \Psi_\bw \; \dd \mu_z^\ast((\Id+tV)_\# \mu_x) - \smallint \Psi_\bw \; \dd (\Id+t \mathfrak{F}_{\mu_x}[V] )_\# \mu_z^\ast(\mu_x) \big) \cdot \Delta \right] 
\end{align*}
which is still $o(t)$ because $|\Psi_\bw(z)|$ is uniformly bounded for any $z\in\R^{d_z}$ and $\by$ is also uniformly bounded by \Cref{ass:npiv}. 

Now, for the other term $\calE_2$, we have
\begin{align*}
    \calE_2 &= \frac{1}{2} \E_{\rho} \left[ \left( \by - \smallint \Psi_\bw(z + t \mathfrak{F}_{\mu_x}[V](z)) \; \dd \mu_z^\ast(\mu_x) \right)^2 \right]  - \frac{1}{2} \E_{\rho} \left[ \left( \by -\smallint \Psi_\bw \; \dd \mu_z^\ast(\mu_x) \right)^2 \right] \\
    &= \frac{1}{2} \E_{\rho} \left[ \left(\by -\smallint \Psi_\bw \; \dd  \mu_z^\ast(\mu_x) - t \cdot \smallint \nabla \Psi_\bw(z)^\top \mathfrak{F}_{\mu_x}[V](z) \; \dd  \mu_z^\ast(\mu_x) + o(t) \right)^2 \right] \\
    &\qquad\qquad - \frac{1}{2} \E_{\rho} \left[ \left(\by - \smallint \Psi_\bw \; \dd \mu_z^\ast(\mu_x) \right)^2 \right] \\
    &= t \cdot \Big\langle \E_{\rho} \left[ \left( \smallint \Psi_\bw \; \dd \mu_z^\ast(\mu_x) - \by \right) \cdot \nabla \Psi_\bw \right] ,  \mathfrak{F}_{\mu_x}[V] \Big \rangle_{L^2(\mu_z^\ast(\mu_x))} + o(t) \\
    &= t \cdot \left\langle g_{\mu_z^\ast(\mu_x)}, \mathfrak{F}_{\mu_x}[V] \right\rangle_{L^2(\mu_z^\ast(\mu_x))} + o(t) . 
\end{align*}
The last equality holds by the definition of $g_{\mu_z^\ast(\mu_x)}$ in the statement of the proposition. 
We now combine the above two terms and plug them back to Eq.~\eqref{eq:F_2_difference} and we can thus obtain
\begin{align*}
    &\quad F_2((\Id+tV)_\# \mu_x, \mu_z^\ast((\Id+tV)_\# \mu_x)) - F_2(\mu_x, \mu_z^\ast(\mu_x)) \\
    &= \zeta_2 t \int x^\top V(x) \; \dd \mu_x + t \cdot \left\langle g_{\mu_z^\ast(\mu_x)}, \mathfrak{F}_{\mu_x}[V] \right\rangle_{L^2(\mu_z^\ast(\mu_x))} + o(t) \\
    &= \left\langle \zeta_2 \Id + \mathfrak{F}_{\mu_x}^\ast g_{\mu_z^\ast(\mu_x)}, \; t \cdot V \right\rangle_{L^2(\mu_x)} + o(t).
\end{align*}
Finally, we apply the definition of Fréchet subdifferential in \Cref{defi:frechet} and prove that $x \mapsto \zeta_2 x + \mathfrak{F}_{\mu_x}^\ast [g_{\mu_z^\ast(\mu_x)}](x)$ is indeed the Fréchet subdifferential of the mapping $\mu_x \mapsto F_2(\mu_x, \mu_z^\ast(\mu_x))$. 
\end{proof}

\begin{prop}\label{prop:mathfrak_F_wasserstein}
Suppose \Cref{ass:npiv} and \ref{ass:network} hold. 
Recall $\mathfrak{G}_{\mu_z}: L^2(\mu_z) \to L^2(\mu_z)$ defined in \Cref{prop:xi_frechet_subdiff}. 
Let $\mathfrak{F}_{\mu_x}$ be a mapping from $L^2(\mu_x)$ to $L^2(\mu_z)$,  
\begin{align}\label{eq:defi_mathfrak_F}
    \mathfrak{F}_{\mu_x} := \big( \E_{\rho}[\nabla \Psi_\bw(\cdot) \otimes \nabla \Psi_\bw(\cdot)] + \sigma_1 \mathfrak{G}_{\mu_z^\ast(\mu_x)} \big)^{-1} ( \E_{\rho}[\nabla \Psi_\bw(\cdot) \otimes \nabla \Psi_\ba(\cdot)]) .
\end{align}   
For any $0<t<1$ and any $\mu_x^\prime\in\calP_2(\R^{d_x})$ with an optimal transport map $V$ from $\mu_x$ to $\mu_x'$, there is $W_2(\mu_z^\ast((\Id + tV)_\# \mu_x), (\Id+t \mathfrak{F}_{\mu_x}[V])_\# \mu_z^\ast(\mu_x)) = o(t)$. 
\end{prop}

\begin{proof}[Proof of \Cref{prop:mathfrak_F_wasserstein}]
Recall that $\scrF_1(\mu_x,\mu_z) = \frac{1}{2} \E_{\rho} [(\smallint \Psi_\bw \dd \mu_z - \smallint \Psi_\ba \dd \mu_x )^2] + \frac{\zeta_1}{2} \E_{\mu_z}[\|z\|^2] + \sigma_1 \mathrm{Ent}(\mu_z)$. 
For a fixed $\mu_x$, since the entropy term is geodesically convex~\citep{villani2008optimal} and the mean squared loss term is strongly geodesically convex as $\E_{\rho}[\nabla \Psi_\bw(\cdot) \otimes \nabla \Psi_\bw(\cdot)] \succeq \Lambda^2 \Id > 0$. 
Therefore, for a fixed $\mu_x$, $\scrF_1(\mu_x,\mu_z)$ is $\Lambda$-geodesically convex in $\mu_z$.

Recall that the Wasserstein gradient of $\calF_1(\mu_x, \mu_z)$ with respect to $\mu_z$, denoted as $\calG(\mu_x, \mu_z)$, is 
\begin{align}\label{eq:L_1_stationary_calG}
    \calG(\mu_x, \mu_z) := \left\{ z \mapsto \E_{\rho} \left[  \left( \smallint \Psi_\bw \; \dd \mu_z - \smallint \Psi_\ba \; \dd \mu_x \right) \nabla \Psi_\bw(z) \right] + \zeta_1 z + \sigma_1 \nabla \log \mu_{z}(z) \right\}. 
\end{align}
From \citet[Section 10.1.1]{ambrosio2008gradient}, for any $\mu_z, \mu_z^\prime\in \calP_2(\R^{d_z})$ and any $\mu_x\in \calP_2(\R^{d_x})$, let $\mathscr{V}$ be the optimal transport map from $\mu_z$ to $\mu_z^\prime$, then
\begin{align*}
    \Lambda W_2^2(\mu_z^\prime, \mu_z) \leq \left\langle \calG(\mu_x, \mu_z^\prime) -\calG(\mu_x, \mu_z), \mathscr{V} - \Id \right\rangle_{L^2(\mu_z)} .
\end{align*}
Consider two probability measures $\mu_z^\ast((\Id+tV)_\# \mu_x)$ and $ (\Id+t\mathfrak{F}_{\mu_x}[V])_\# \mu_z^\ast(\mu_x)$. 
Denote $V_0$ as the optimal transport map from $ (\Id+t\mathfrak{F}_{\mu_x}[V])_\# \mu_z^\ast(\mu_x)$ to $\mu_z^\ast((\Id+tV)_\# \mu_x)$. 
We have 
\begin{align*}
    &\quad \Lambda \cdot \| V_0 - \Id\|_{L^2((\Id+t\mathfrak{F}_{\mu_x}[V])_\# \mu_z^\ast(\mu_x))} = 
    \Lambda \cdot W_2^2\left( \mu_z^\ast((\Id+tV)_\# \mu_x), (\Id+t\mathfrak{F}_{\mu_x}[V])_\# \mu_z^\ast(\mu_x) \right) \\
    &\leq \left\langle \calG\left( (\Id+tV)_\# \mu_x, \mu_z^\ast((\Id+tV)_\# \mu_x) \right) -\calG \left((\Id+tV)_\# \mu_x), (\Id+t\mathfrak{F}_{\mu_x}[V])_\# \mu_z^\ast(\mu_x) \right) \right., \\
    &\quad \left. V_0 - \Id \right\rangle_{L^2((\Id+t\mathfrak{F}_{\mu_x}[V])_\# \mu_z^\ast(\mu_x))} ,
\end{align*}
By the optimality condition, we have $\calG\left( (\Id+tV)_\# \mu_x, \mu_z^\ast((\Id+tV)_\# \mu_x) \right) = 0$. Therefore,
\begin{align*}
    &\quad \Lambda W_2^2 \left(\mu_z^\ast((\Id+tV)_\# \mu_x), (\Id+t\mathfrak{F}_{\mu_x}[V])_\# \mu_z^\ast(\mu_x) \right) \\
    &\leq \|\calG \left((\Id+tV)_\# \mu_x), (\Id+t\mathfrak{F}_{\mu_x}[V])_\# \mu_z^\ast(\mu_x) \right)\|_{L^2((\Id+t\mathfrak{F}_{\mu_x}[V])_\# \mu_z^\ast(\mu_x))} \cdot \|V_0 - \Id \|_{L^2((\Id+t\mathfrak{F}_{\mu_x}[V])_\# \mu_z^\ast(\mu_x))} \\
    &= \|\calG \left((\Id+tV)_\# \mu_x), (\Id+t\mathfrak{F}_{\mu_x}[V])_\# \mu_z^\ast(\mu_x) \right)\|_{L^2((\Id+t\mathfrak{F}_{\mu_x}[V])_\# \mu_z^\ast(\mu_x))} \\
    &\quad \cdot W_2 \left(\mu_z^\ast((\Id+tV)_\# \mu_x), (\Id+t\mathfrak{F}_{\mu_x}[V])_\# \mu_z^\ast(\mu_x) \right).
\end{align*}
So we have
\begin{align}\label{eq:wasserstein_2_bound_calG}
\begin{aligned}
    &\quad \Lambda W_2\left(\mu_z^\ast((\Id+tV)_\# \mu_x), (\Id+t\mathfrak{F}_{\mu_x}[V])_\# \mu_z^\ast(\mu_x) \right) \\
    &\leq \|\calG \left((\Id+tV)_\# \mu_x), (\Id+t\mathfrak{F}_{\mu_x}[V])_\# \mu_z^\ast(\mu_x) \right)\|_{L^2((\Id+t\mathfrak{F}_{\mu_x}[V])_\# \mu_z^\ast(\mu_x))}.
\end{aligned}
\end{align}
Recall the definition of $\calG$ in Eq.~\eqref{eq:L_1_stationary_calG}, there is 
\begin{align}\label{eq:calG_proposition}
    &\quad \calG \left((\Id+tV)_\# \mu_x), (\Id+t\mathfrak{F}_{\mu_x}[V])_\# \mu_z^\ast(\mu_x) \right) \left( z + t\mathfrak{F}_{\mu_x}[V](z) \right) \\
    &= \E_{\rho} \left[  \left( \smallint \Psi_\bw(z + t\mathfrak{F}_{\mu_x}[V](z)) \; \dd \mu_z^\ast(\mu_x) - \smallint \Psi_\ba(x + tV(x)) \; \dd  \mu_x \right) \nabla \Psi_\bw(z + t\mathfrak{F}_{\mu_x}[V](z)) \right] \label{eq:calG_1} \\
    &\qquad + \zeta_1 \left( z + t\mathfrak{F}_{\mu_x}[V](z) \right) + \sigma_1 \nabla \log \left( (\Id+t\mathfrak{F}_{\mu_x}[V])_\# \mu_z^\ast(\mu_x) \right) (z + t\mathfrak{F}_{\mu_x}[V](z) ) . \label{eq:calG_2} 
\end{align}
Then we have, 
\begin{align*}
    \eqref{eq:calG_1} & = \E_{\rho} \left[  \smallint \Psi_\bw(z + t\mathfrak{F}_{\mu_x}[V](z)) \; \dd \mu_z^\ast(\mu_x) \cdot \nabla \Psi_\bw(z + t \mathfrak{F}_{\mu_x}[V](z)) \right] \\
    &\qquad\qquad - \E_{\rho} \left[  \smallint \Psi_\ba(x + t V(x)) \; \dd \mu_x \cdot \nabla \Psi_\bw(z + t \mathfrak{F}_{\mu_x}[V](z)) \right] \\
    &= \underbrace{\E_{\rho} \left[  \left(  \smallint \Psi_\bw(z) \; \dd \mu_z^\ast(\mu_x) - \smallint \Psi_\ba(x) \; \dd \mu_x \right) \nabla \Psi_\bw(z+ t \mathfrak{F}_{\mu_x}[V](z)) \right] }_{(\ast)} \\
    &\qquad\qquad+ t \cdot \E_{\rho} \left[ \left(\smallint \mathfrak{F}_{\mu_x}[V](z)^\top \nabla \Psi_\bw(z) \; \dd \mu_z - \smallint V(x)^\top \nabla \Psi_\ba(x) \; \dd \mu_x \right) \nabla \Psi_\bw(z) \right] + o(t) . 
\end{align*}
The last equality holds by a Taylor expansion that $\Psi_\bw(z + t\mathfrak{F}_{\mu_x}[V](z)) = \Psi_\bw(z) + t \mathfrak{F}_{\mu_x}[V](z)^\top \nabla \Psi_\bw(z) + o(t)$ and  $\Psi_\ba(x + t V(x)) = \Psi_\ba(x) + t V(x)^\top \nabla \Psi_\ba(x) + o(t)$. 
And we also have, 
\begin{align*}
    \eqref{eq:calG_2} & = \zeta_1 (z   + t \mathfrak{F}_{\mu_x}[V](z)) + \sigma_1 \nabla \log (\Id +t \mathfrak{F}_{\mu_x}[V])_\# \mu_z^\ast(\mu_x) (z + t \mathfrak{F}_{\mu_x}[V](z)) \\
    &= \zeta_1 (z + t \mathfrak{F}_{\mu_x}[V](z)) + \sigma_1 \nabla \log \mu_z^\ast(\mu_x) (z + t \mathfrak{F}_{\mu_x}[V](z)) - t \cdot \sigma_1\bH \log \mu_z^\ast(\mu_x)(z) \mathfrak{F}_{\mu_x}[V](z) \\
    &\qquad - t \cdot \sigma_1\bJ \mathfrak{F}_{\mu_x}[V](z) \nabla \log \mu_z^\ast(\mu_x)(z) - t \cdot \sigma_1 \nabla \mathrm{Tr}(\bJ \mathfrak{F}_{\mu_x}[V](z)) + o(t) \\
    &= \underbrace{\zeta_1 (z + t \mathfrak{F}_{\mu_x}[V](z)) + \sigma_1 \nabla \log \mu_z^\ast(\mu_x) (z + t \mathfrak{F}_{\mu_x}[V](z))}_{(\ast\ast)}  + t\cdot \sigma_1 \mathfrak{G}_{\mu_z^\ast(\mu_x)} \mathfrak{F}_{\mu_x}[V] (z) + o(t).
\end{align*}
The second last step holds by \Cref{lem:tedious_calculation} and the last step holds by the definition of $\mathfrak{G}$. 
Since $(\ast)+(\ast\ast) = \calG(\mu_x, \mu_z^\ast(\mu_x))(z + t \mathfrak{F}_{\mu_x}[V](z))=0$ from the optimality of $\mu_z^\ast(\mu_x)$, we achieve
\begin{align*}
    &\quad \eqref{eq:calG_1} + \eqref{eq:calG_2}  \\
    &= t\cdot \Big( - ( \E_{\rho}[\nabla \Psi_\bw \otimes \nabla \Psi_\ba]) V + ( \E_{\rho}[\nabla \Psi_\bw \otimes \nabla \Psi_\bw] + \sigma_1 \mathfrak{G}_{\mu_z^\ast(\mu_x)}) \mathfrak{F}_{\mu_x}[V] \Big) + o(t) \\
    &= o(t). 
\end{align*}
The last inequality holds by the definition of $\mathfrak{F}_{\mu_x}$ in Eq.~\eqref{eq:defi_mathfrak_F} which kills the linear term in $t$. 
Therefore, we have proved that $\calG \left((\Id+tV)_\# \mu_x), (\Id+t\mathfrak{F}_{\mu_x}[V])_\# \mu_z^\ast(\mu_x) \right) \left( z + t\mathfrak{F}_{\mu_x}[V](z) \right) = o(t)$. Hence, putting it back to Eq.~\eqref{eq:wasserstein_2_bound_calG}, and we have proved that $W_2(\mu_z^\ast((\Id+tV)_\# \mu_x), (\Id+t\mathfrak{F}_{\mu_x}[V])_\# \mu_z^\ast(\mu_x)) = 0(t)$ which concludes the proof. 
\end{proof}

\newpage
\end{appendices}
\end{document}